\documentclass{article}

\usepackage[top=1in, bottom=1in, left=1in, right=1in]{geometry}

\usepackage{comment,url,algorithm,algorithmic,graphicx,subcaption,relsize}
\usepackage{amssymb,amsfonts,amsmath,amsthm,amscd,dsfont,mathrsfs,mathtools,nicefrac}
\usepackage{float,psfrag,epsfig,color,xcolor,url,hyperref}
\usepackage{epstopdf,bbm,enumitem}
\usepackage[toc,page]{appendix}
\usepackage[mathscr]{euscript}

\newcommand{\dee}{\mathop{\mathrm{d}\!}}
\newcommand{\dt}{\,\dee t}

\def\ddefloop#1{\ifx\ddefloop#1\else\ddef{#1}\expandafter\ddefloop\fi}

\def\ddef#1{\expandafter\def\csname bb#1\endcsname{\ensuremath{\mathbb{#1}}}}
\ddefloop ABCDEFGHIJKLMNOPQRSTUVWXYZ\ddefloop

\def\ddef#1{\expandafter\def\csname c#1\endcsname{\ensuremath{\mathcal{#1}}}}
\ddefloop ABCDEFGHIJKLMNOPQRSTUVWXYZ\ddefloop

\def\ddef#1{\expandafter\def\csname v#1\endcsname{\ensuremath{\boldsymbol{#1}}}}
\ddefloop ABCDEFGHIJKLMNOPQRSTUVWXYZabcdefghijklmnopqrstuvwxyz\ddefloop

\def\ddef#1{\expandafter\def\csname v#1\endcsname{\ensuremath{\boldsymbol{\csname #1\endcsname}}}}
\ddefloop {alpha}{beta}{gamma}{delta}{epsilon}{varepsilon}{zeta}{eta}{theta}{vartheta}{iota}{kappa}{lambda}{mu}{nu}{xi}{pi}{varpi}{rho}{varrho}{sigma}{varsigma}{tau}{upsilon}{phi}{varphi}{chi}{psi}{omega}{Gamma}{Delta}{Theta}{Lambda}{Xi}{Pi}{Sigma}{varSigma}{Upsilon}{Phi}{Psi}{Omega}{ell}\ddefloop

\def\balign#1\ealign{\begin{align}#1\end{align}}
\def\baligns#1\ealigns{\begin{align*}#1\end{align*}}
\def\balignat#1\ealign{\begin{alignat}#1\end{alignat}}
\def\balignats#1\ealigns{\begin{alignat*}#1\end{alignat*}}
\def\bitemize#1\eitemize{\begin{itemize}#1\end{itemize}}
\def\benumerate#1\eenumerate{\begin{enumerate}#1\end{enumerate}}

\newenvironment{talign*}
 {\csname align*\endcsname}
 {\endalign}
\newenvironment{talign}
 {\csname align\endcsname}
 {\endalign}

\def\balignst#1\ealignst{\begin{talign*}#1\end{talign*}}
\def\balignt#1\ealignt{\begin{talign}#1\end{talign}}

\let\originalleft\left
\let\originalright\right
\renewcommand{\left}{\mathopen{}\mathclose\bgroup\originalleft}
\renewcommand{\right}{\aftergroup\egroup\originalright}

\def\tinycitep*#1{{\tiny\citep*{#1}}}
\def\tinycitealt*#1{{\tiny\citealt*{#1}}}
\def\tinycite*#1{{\tiny\cite*{#1}}}
\def\smallcitep*#1{{\scriptsize\citep*{#1}}}
\def\smallcitealt*#1{{\scriptsize\citealt*{#1}}}
\def\smallcite*#1{{\scriptsize\cite*{#1}}}

\def\mbb#1{\mathbb{#1}}

\def\mrm#1{\mathrm{#1}}

\newcommand{\norm}[1]{\left\lVert#1\right\rVert}
\newcommand{\snorm}[1]{\lVert#1\rVert}

\theoremstyle{plain}  
\newtheorem*{remark}{\textbf{Remark}}

\def\R{\mathbb{R}}

\def\N{\mathbb{N}}

\def\<{\left\langle} 
\def\>{\right\rangle}

\def\norm#1{\left\|{#1}\right\|} 

\def\E{\mbb{E}} 

\def\bigO#1{\mathcal{O}\left(#1\right)} 
\def\bigOdp#1{\mathcal{O}_{d,\mathbb{P}}\left(#1\right)} 
\def\bigThetdp#1{\Theta_{d,\mathbb{P}}\left(#1\right)} 
\def\littleodp#1{o_{d,\mathbb{P}}\left(#1\right)} 
\def\P{\mbb{P}} 
\def\Parg#1{\P\left({#1}\right)}

\newcommand{\tr}{\operatorname{tr}}
\newcommand{\Tr}{\operatorname{Tr}}
\def\Trarg#1{\Tr\left({#1}\right)} 
\def\trarg#1{\tr\left({#1}\right)} 
\def\Var{\mrm{Var}} 

\newcommand{\Exp}[1]{\operatorname{exp}\left({#1}\right)} 

\DeclareSymbolFont{rsfs}{U}{rsfs}{m}{n}
\DeclareSymbolFontAlphabet{\mathscrsfs}{rsfs}

\providecommand{\argmin}{\mathop\mathrm{arg min}}

\providecommand{\diag}{\mathop\mathrm{diag}}

\ifdefined\nonewproofenvironments\else
\ifdefined\ispres\else
\newtheorem{theo}{Theorem}
\newtheorem{lemm}[theo]{Lemma}
\newtheorem{coro}[theo]{Corollary}

\newtheorem{fact}[theo]{Fact}
\newtheorem{prop}[theo]{Proposition}
\newtheorem{conj}[theo]{Conjecture}

\renewenvironment{proof}{\noindent\textbf{Proof.}\hspace*{.3em}}{\qed\\}
\newenvironment{proof-sketch}{\noindent\textbf{Proof Sketch}
  \hspace*{1em}}{\qed\bigskip\\}
\newenvironment{proof-idea}{\noindent\textbf{Proof Idea}
  \hspace*{1em}}{\qed\bigskip\\}
\newenvironment{proof-of-lemma}[1][{}]{\noindent\textbf{Proof of Lemma {#1}}
  \hspace*{1em}}{\qed\\}
\newenvironment{proof-of-theorem}[1][{}]{\noindent\textbf{Proof of Theorem {#1}}
  \hspace*{1em}}{\qed\\}
\newenvironment{proof-attempt}{\noindent\textbf{Proof Attempt}
  \hspace*{1em}}{\qed\bigskip\\}
\fi

\newtheorem{assumption}{Assumption}

\newenvironment{proofof}[1][{}]{\par\noindent{\bf Proof of {#1}. }  }{\qed\bigskip}   
   
\fi
\makeatletter
\@addtoreset{equation}{section}
\makeatother

\hypersetup{
  colorlinks,
  linkcolor={red!50!black},
  citecolor={blue!50!black},
  urlcolor={blue!80!black}
}

\newcommand{\abs}[1]{\left|#1\right|}

\mathtoolsset{showonlyrefs}

\usepackage{caption}
\usepackage{authblk}
\usepackage{etoc}
\usepackage{wrapfig}

\newcommand{\iid}{\stackrel{\mathrm{\tiny{i.i.d.}}}{\sim}}

\newcommand{\MP}{\text{MP}}
\newcommand{\nnl}{\mathrm{NL}}
\newcommand{\eps}{\varepsilon}
\newcommand{\barPhi}{\bar{\vPhi}}
\newcommand{\hSigmaPhi}{\widehat{\vSigma}_\Phi}
\newcommand{\bSigmaPhi}{\overline{\vSigma}_\Phi}
\newcommand{\tSigmaPhi}{\tilde{\vSigma}_{\Phi}} 
\newcommand{\hSigmaPhii}{\widehat{\vSigma}_{\Phi_0}}
\newcommand{\bSigmaPhii}{\overline{\vSigma}_{\Phi_0}}
\newcommand{\SigmaPhi}{\vSigma_\Phi}
\newcommand{\tlambda}{\tilde{\lambda}}

\renewcommand{\contentsname}{Table of Contents}
\newcommand*\samethanks[1][\value{footnote}]{\footnotemark[#1]}
 
\newcommand{\citep}[1]{\cite{#1}}
\newcommand{\citet}[1]{\cite{#1}}

\makeatletter
\renewcommand{\paragraph}{%
  \@startsection{paragraph}{4}%
  {\z@}{1.5ex \@plus 1ex \@minus .2ex}{-1em}%
  {\normalfont\normalsize\bfseries}%
}
\makeatother

\title{\vspace{-2.5mm}
High-dimensional Asymptotics of Feature Learning: \\How One Gradient Step Improves the Representation}
 
\author{

Jimmy Ba\thanks{University of Toronto and Vector Institute for Artificial Intelligence.  \texttt{\{jba,erdogdu,dennywu\}@cs.toronto.edu}.} ,\,
Murat A.~Erdogdu\samethanks[1] ,\, 
Taiji Suzuki\thanks{University of Tokyo and RIKEN Center for Advanced Intelligence Project. \texttt{taiji@mist.i.u-tokyo.ac.jp}.}  ,\, 
Zhichao Wang\thanks{University of California, San Diego.  \texttt{zhw036@ucsd.edu}.} ,\,
Denny Wu\samethanks[1] ,\,
Greg Yang\thanks{Microsoft Research AI. \texttt{gregyang@microsoft.com}. \vspace{-2mm}} 
}

\begin{document}
\etocdepthtag.toc{mtchapter}
\etocsettagdepth{mtchapter}{subsection}
\etocsettagdepth{mtappendix}{none}

\maketitle 

\vspace{-3mm}  

\begin{abstract}
We study the first gradient descent step on the first-layer parameters $\boldsymbol{W}$ in a two-layer neural network: $f(\boldsymbol{x}) = \frac{1}{\sqrt{N}}\boldsymbol{a}^\top\sigma(\boldsymbol{W}^\top\boldsymbol{x})$, where $\boldsymbol{W}\in\mathbb{R}^{d\times N}, \boldsymbol{a}\in\mathbb{R}^{N}$ are randomly initialized, and the training objective is the empirical MSE loss: $\frac{1}{n}\sum_{i=1}^n (f(\boldsymbol{x}_i)-y_i)^2$. In the proportional asymptotic limit where $n,d,N\to\infty$ at the same rate, and an idealized student-teacher setting, we show that the first gradient update contains a rank-1 ``spike'', which results in an alignment between the first-layer weights and the linear component of the teacher model $f^*$. To characterize the impact of this alignment, we compute the prediction risk of ridge regression on the conjugate kernel after one gradient step on $\boldsymbol{W}$ with learning rate $\eta$, when $f^*$ is a single-index model. We consider two scalings of the first step learning rate $\eta$. For small $\eta$, we establish a Gaussian equivalence property for the trained feature map, and prove that the learned kernel improves upon the initial random features model, but cannot defeat the best linear model on the input. Whereas for sufficiently large $\eta$, we prove that for certain $f^*$, the same ridge estimator on trained features can go beyond this ``linear regime'' and outperform a wide range of random features and rotationally invariant kernels. Our results demonstrate that even one gradient step can lead to a considerable advantage over random features, and highlight the role of learning rate scaling in the initial phase of training.  

\end{abstract}

\section{Introduction}
\label{sec:intro}

We consider the training of a fully-connected two-layer neural network (NN) with $N$ neurons, 
\vspace{-.5mm}
\begin{align} 
f_{\text{NN}}(\vx)  
= \frac{1}{\sqrt{N}} \sum_{i=1}^N a_i\sigma(\langle \vx,\vw_i\rangle)
= \frac{1}{\sqrt{N}}\va^\top\sigma(\vW^\top\vx), 
\label{eq:two-layer-nn}
\end{align} 
where $\vx\in\R^d, \vW\in\R^{d\times N}, \va\in\R^N$, $\sigma$ is the nonlinear activation function applied entry-wise, and the training objective is to minimize the (potentially $\ell_2$-regularized) empirical risk.  
Our analysis will be made in the \textit{proportional asymptotic limit}, i.e., the number of training data $n$, the input dimensionality $d$, and the number of features (neurons) $N$ jointly tend to infinity. Intuitively, this regime reflects the setting where the network width and data size are comparable, which is consistent with practical choices of model scaling.

When the first layer $\vW$ is fixed and only the second layer $\va$ is optimized, we arrive at a kernel model, where the kernel defined by features $\vx\to\sigma(\vW^\top\vx)$ (often called the \textit{hidden representation}) is referred to as the \textit{conjugate kernel} (CK) \citep{neal1995bayesian}. When $\vW$ is randomly initialized, this model is an example of the \textit{random features} (RF) model \citep{rahimi2008random}. 
The training and test performance of RF regression has been extensively studied in the proportional limit
\citep{louart2018random,mei2019generalization}. These precise characterizations reveal interesting phenomena also present in practical deep learning, such as the non-monotonic risk curve \citep{belkin2018reconciling}.     
  
However, RF models do not fully explain the empirical success of neural networks: one crucial advantage of deep learning is the \textit{ability to learn useful features}~\citep{girshick2014rich,devlin2018bert} that ``adapt'' to the learning problem~\citep{suzuki2018adaptivity}. In fact, recent works have shown that such adaptivity enables NNs optimized by gradient descent to outperform a wide range of linear/kernel estimators \citep{allen2019can,ghorbani2019limitations}. 
While many explanations of this separation between NNs and kernel models have been proposed, our starting point is the empirical finding that ``non-kernel'' behavior often occurs in the \textit{early phase} of NN optimization, especially under large learning rate 
\citep{jastrzebski2020break,fort2020deep}.  
The goal of this work is to answer the following question: 
\vspace{-0.5mm}
\begin{center}
{\it Can we precisely capture the presence of feature learning in the early phase of gradient descent training, \\ and demonstrate its improvement over the initial (fixed) kernel in the proportional limit?}  
\end{center} 
\vspace{-2.mm}

\subsection{Contributions}
\label{subsec:contribution}
   
Motivated by the above observations, we investigate a simplified scenario of the ``early phase'' of learning: how \textit{the first gradient step} on the first-layer parameters $\vW$ impacts the representation of the two-layer NN \eqref{eq:two-layer-nn}. Specifically, we consider the regression setting with the squared (MSE) loss, and a student-teacher model in the proportional asymptotic limit; we characterize the prediction risk of the kernel ridge regression estimator on top of the first-layer CK feature $\vx\to\sigma(\vW^\top\vx)$, before and after the gradient descent step\footnote{Some of our results also apply to multiple gradient steps on the first layer $\vW$, which we specify in the sequel. \vspace{-2mm}} on the empirical risk (starting from Gaussian initialization).  
Our findings can be summarized as follows.  
\begin{itemize}[leftmargin=*,topsep=0.5mm, itemsep=0.1mm]
    \item In Section~\ref{sec:spike}, we show that the first gradient step on $\vW$ is approximately rank-1; hence under appropriate learning rate, the updated weight matrix exhibits a 
    information (spike) plus noise (bulk) structure.  
\end{itemize}

\begin{itemize}[resume, before= \vspace*{-1.8mm}, leftmargin=*, topsep=1mm, itemsep=0.1mm]
    \item As a result, the isolated singular vector of the weight matrix aligns with the linear component of target function (teacher) $f^*$, and the top eigenvector of the CK matrix aligns with the training labels $\vy$.   
\end{itemize}

\begin{wrapfigure}{R}{0.37\textwidth}  
\vspace{-4.6mm} 
\centering 
\includegraphics[width=0.36\textwidth]{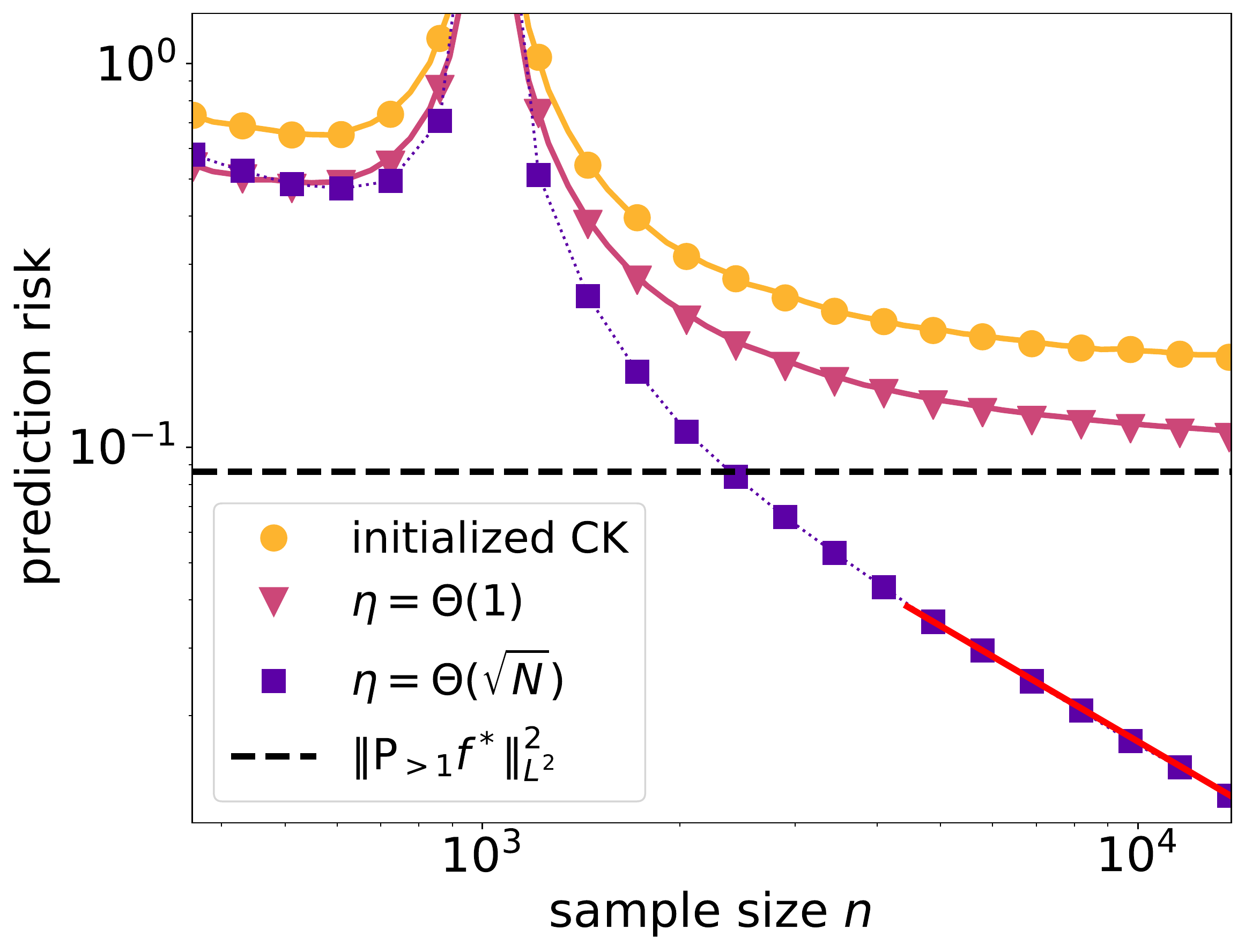} 
\vspace{-2.5mm} 
\caption{\small Prediction risk of ridge regression on trained features (erf) after one gradient step. Dots represent empirical simulations and solid lines are predicted asymptotic values; red line indicates $\Theta(\nicefrac{d}{n})$ rate. }     
\label{fig:intro}     
\vspace{-4mm}    
\end{wrapfigure}

Next in Section~\ref{sec:prediction-risk} we study how the aforementioned alignment improves the kernel. We consider a more specialized setting where the teacher $f^*$ is a single-index model, in which case the prediction risk of a large class of RF/kernel ridge regression estimators is lower-bounded by the $L^2$-norm of the ``nonlinear'' component the teacher $\norm{\textsf{P}_{>1} f^*}_{L^2}^2$, i.e., they can only learn \textit{linear} functions on the input.   
After taking one gradient step on $\vW$, we compute the CK ridge estimator using separate training data, and compare its prediction risk against this linear lower bound. Our analysis will be made under two choices of learning rate scalings (see Figure~\ref{fig:intro}): 
\begin{itemize}[leftmargin=*,topsep=0.5mm, itemsep=0.1mm]
    \item \underline{\textbf{Small lr: $\eta=\Theta(1)$}}. In Section~\ref{subsec:small-lr}, we extend the \textit{Gaussian Equivalence Theorem} (GET) in \citet{hu2020universality} to the updated feature map trained via multiple gradient descent steps on $\vW$ with learning rate $\eta=\Theta(1)$; this allows us to precisely characterize the prediction risk using random matrix theoretical tools. We prove that after one gradient step, the ridge regression estimator on the learned CK features already exhibits nontrivial improvement over the initial RF ridge model, but it remains in the ``linear regime'' and cannot outperform the best linear estimator on the input.  
\end{itemize} 

\begin{itemize}[resume, before= \vspace*{-1.6mm}, leftmargin=*, topsep=1mm, itemsep=0.1mm]
    \item \underline{\textbf{Large lr: $\eta=\Theta(\sqrt{N})$}}. In Section~\ref{subsec:large-lr}, we analyze a larger learning rate that coincides with the \textit{maximal update parameterization} in \citet{yang2020feature}. For certain target functions $f^*$, we prove that kernel ridge regression after one feature learning step can achieve lower risk than the lower bound $\norm{\textsf{P}_{>1} f^*}_{L^2}^2$, and thus outperform a wide range of kernel ridge estimators (including the neural tangent kernel of \eqref{eq:two-layer-nn}).   
\end{itemize} 

\subsection{Related Works}
\label{subsec:related}

\paragraph{Asymptotics of Kernel Regression.} 
A plethora of recent works provided precise performance analysis of RF and kernel models in the proportional limit \citep{mei2019generalization,gerace2020generalisation,dhifallah2020precise,liao2020random,adlam2020neural}. These results typically build upon analyses of the spectrum of kernel matrices, a key ingredient in which is the ``linearization'' of nonlinear random matrices via Taylor expansion \citep{el2010spectrum} or orthogonal 
polynomials \citep{cheng2013spectrum,pennington2017nonlinear}. 

Consequently, a large class of kernel models are essentially linear in the proportional limit \citep{liang2018just,bartlett2021deep}. In the case of RF models, similar property is captured by the Gaussian equivalence theorem \citep{goldt2020modeling,hu2020universality,goldt2021gaussian}, which roughly states that RF estimators achieve the same prediction risk as a (noisy) linear model. For input on unit sphere, \cite{ghorbani2019linearized,mei2021generalization} showed that sample size $n=\Omega(d^2)$ is required to go beyond this ``linear'' regime. 
As we will see in certain settings, such limitation can also be overcome (in the $n\asymp d$ scaling) by training the feature map for one gradient step with sufficiently large learning rate.    

\paragraph{Advantage of NNs over Fixed Kernels.} 
It is well-known that under certain initialization, the learning dynamics of overparameterized NNs can be described by the neural tangent kernel (NTK) \citep{jacot2018neural}. However, the NTK description essentially ``freezes'' the model around its initialization \citep{chizat2018note}, and thus 
does not explain the presence of \textit{feature learning}  in NNs \citep{yang2020feature}.  
     
In fact, various works have shown that deep learning is more powerful than kernel methods in terms of approximation and estimation ability \citep{bach2017breaking,suzuki2018adaptivity,imaizumi2019deep,schmidt2020nonparametric,ghorbani2020neural}. Moreover, in some specialized settings, NNs optimized with gradient-based methods can outperform the NTK (or more generally any kernel estimators) in terms of generalization error \citep{allen2019can,wei2019regularization,ghorbani2019limitations,li2020learning,daniely2020learning,suzuki2020benefit,allen2020backward,refinetti2021classifying,karp2021local,abbe2021staircase} (see \cite[Table 2]{malach2021quantifying} for survey). 
These results often require careful analysis of the landscape (e.g., properties of global optimum) or optimization dynamics; in contrast, our goal is to precisely characterize the first gradient step and demonstrate a similar separation.           

\paragraph{Early Phase of NN Optimization.}
Recent empirical studies suggest that properties of the final trained model is strongly influenced by the early stage of optimization \citep{golatkar2019time,leclerc2020two,pesme2021implicit}, and the NTK evolves most rapidly in the first few epochs \citep{fort2020deep}. 
Large learning rate in the initial steps can impact the conditioning of loss surface \citep{jastrzebski2020break,cohen2021gradient} and potentially improve the generalization performance \citep{li2019towards,lewkowycz2020large}. 
Under structural assumptions on the data, it has been proved that one gradient step with sufficiently large learning rate can drastically decrease the training loss \citep{chatterji2020does}, extract task-relevant features \citep{daniely2020learning,frei2022random}, or escape the trivial stationary point at initialization \citep{hajjar2021training}. While these works also highlight the benefit of one feature learning step\footnote{We however note that the ``early phase'' is not always sufficient: for certain teacher model $f^*$, (stochastic) gradient descent may exhibit a long initial ``search'' stage before nontrivial alignment can be achieved, see \cite{arous2021online,veiga2022phase}.  \vspace{-2.5mm} }, to our knowledge this advantage has not been \textit{precisely} characterized in the proportional regime (where the performance of RF models has been extensively studied).

\section{Problem Setup and Basic Assumptions}
\label{sec:setup}

\paragraph{Notations.}
Throughout this paper, $\|\cdot\|$ denotes the $\ell_2$ norm for vectors and the $\ell_2\to\ell_2$ operator norm for matrices, and $\|\cdot\|_F$ is the Frobenius norm. For matrix $\vM\in\R^{n\times n}$,  $\tr(\vM)=\frac{1}{n}\Tr(\vM)$ is the normalized trace. $\cO_d(\cdot)$ and $o_d(\cdot)$ stand for the standard big-O and little-o notations, where the subscript highlights the asymptotic variable; we write $\tilde{\cO}(\cdot)$ when the (poly-)logarithmic factors are ignored. $\cO_{d,\P}(\cdot)$ (resp.~$o_{d,\P}(\cdot)$) represents big-O (resp.~little-o) in probability as $d\to\infty$. $\Omega(\cdot),\Theta(\cdot)$ are defined analogously. $\Gamma$ is the standard Gaussian distribution in $\R^d$. Given $f:\R^d\to\R$, we denote its $L^p$-norm w.r.t.~$\Gamma$ as $\norm{f}_{L^p(\R^d,\Gamma)}$, which we abbreviate as $\norm{f}_{L^p}$ when the context it clear. $\mu^{\MP}_{\gamma}$ is the Marchenko–Pastur distribution with ratio $\gamma$.    

\subsection{Training Procedure}  

\paragraph{Gradient Descent on the 1st Layer.} 
Given training examples $\{(\vx_i,y_i)\}_{i=1}^n$, we learn the two-layer NN \eqref{eq:two-layer-nn} by minimizing the empirical risk: $\cL(f) = {\frac{1}{n}}\sum_{i=1}^n \ell(f(\vx_i),y_i)$, where $\ell$ is the squared loss $\ell(x,y) = \frac{1}{2}(x-y)^2$. 
As previously remarked, fixing the first layer $\vW$ at random initialization and learning the second layer $\va$ yields RF model, which is a convex problem with a closed-form solution. 
In contrast, we are interested in \textit{learning the feature map (representation)}; hence we first fix $\va$ (at initialization) and perform gradient descent on $\vW$.  
We write the initialized first-layer weights as $\vW_0$, and the weights after $t$ gradient steps as $\vW_t$. The gradient update with learning rate $\eta>0$ is given as: $\vW_{t+1} = \vW_t + \eta \sqrt{N}\cdot\vG_t$, where 
\begin{align}\label{eq:gradient-step-MSE}
\vG_t
:= 
    \frac{1}{n} \vX^\top \left[\left(\frac{1}{\sqrt{N}}\left(\vy - \frac{1}{\sqrt{N}}\sigma(\vX\vW_t)\va\right)\va^\top\right) \odot \sigma'(\vX\vW_t)\right], 
\end{align}  
for $t\in\N$, in which $\odot$ is the Hadamard product, $\sigma'$ is the derivative of $\sigma$ (acting entry-wise), and we denoted the input feature matrix $\vX\in\R^{n\times d}$, and the corresponding label vector $\vy\in\R^n$. We remark that the $\sqrt{N}$-scaling in front of the learning rate $\eta$ is due to the $\frac{1}{\sqrt{N}}$-prefactor in our definition of two-layer NN \eqref{eq:two-layer-nn}.

\paragraph{Ridge Regression for the 2nd Layer.}
After obtaining the updated weights $\vW_1$, we evaluate the quality of the new CK features by computing the prediction risk of the \textit{kernel ridge regression} estimator on top of the first-layer representation. 
Note that if ridge regression is performed on the same data $\vX$, then after one feature learning step, $\vW_1$ is no longer independent of $\vX$, which significantly complicates the analysis. To circumvent this difficulty, we estimate the regression coefficients $\hat{\va}$ using \textit{a new set of training data} $\{\tilde{\vx}_i,\tilde{y}_i\}_{i=1}^n$, which for simplicity we assume to have the same size as the original dataset. This can be interpreted as the representation is ``pretrained'' on separate data before the ridge estimator is learned. 
 
Denote the feature matrix on the fresh training set $\{\tilde{\vX},\tilde{\vy}\}$ as $\vPhi := \frac{1}{\sqrt{N}}\sigma(\tilde{\vX}\vW_1)\in\R^{n\times N}$, the CK ridge regression estimator is given by
$\hat{f}(\vx) =  \frac{1}{\sqrt{N}}\hat{\va}^\top\sigma\left(\vW^\top_1\vx\right)$, where 
$
\hat{\va} = \text{argmin}_{\va}  \,\left\{\frac{1}{n}\norm{\tilde{\vy} - \vPhi\va}^2 + \frac{\lambda}{N} \norm{\va}^2\right\}. 
$

\vspace{-0.6mm}
\subsection{Main Assumptions} 
 
Given a target function (ground truth) $f^*$ and a learned model $\hat{f}$, we evaluate the model performance using the prediction risk: $\cR(\hat{f}) = \E_{\vx}(\hat{f}(\vx) - f^*(\vx))^2 = \snorm{\hat{f} - f^*}_{L^2}^2$, where the expectation is taken over the test data from the same training distribution. 
Our analysis will be made under the following assumptions.  
  
\begin{wrapfigure}{R}{0.35\textwidth}  
\vspace{-4.2mm} 
\centering  
\includegraphics[width=0.34\textwidth]{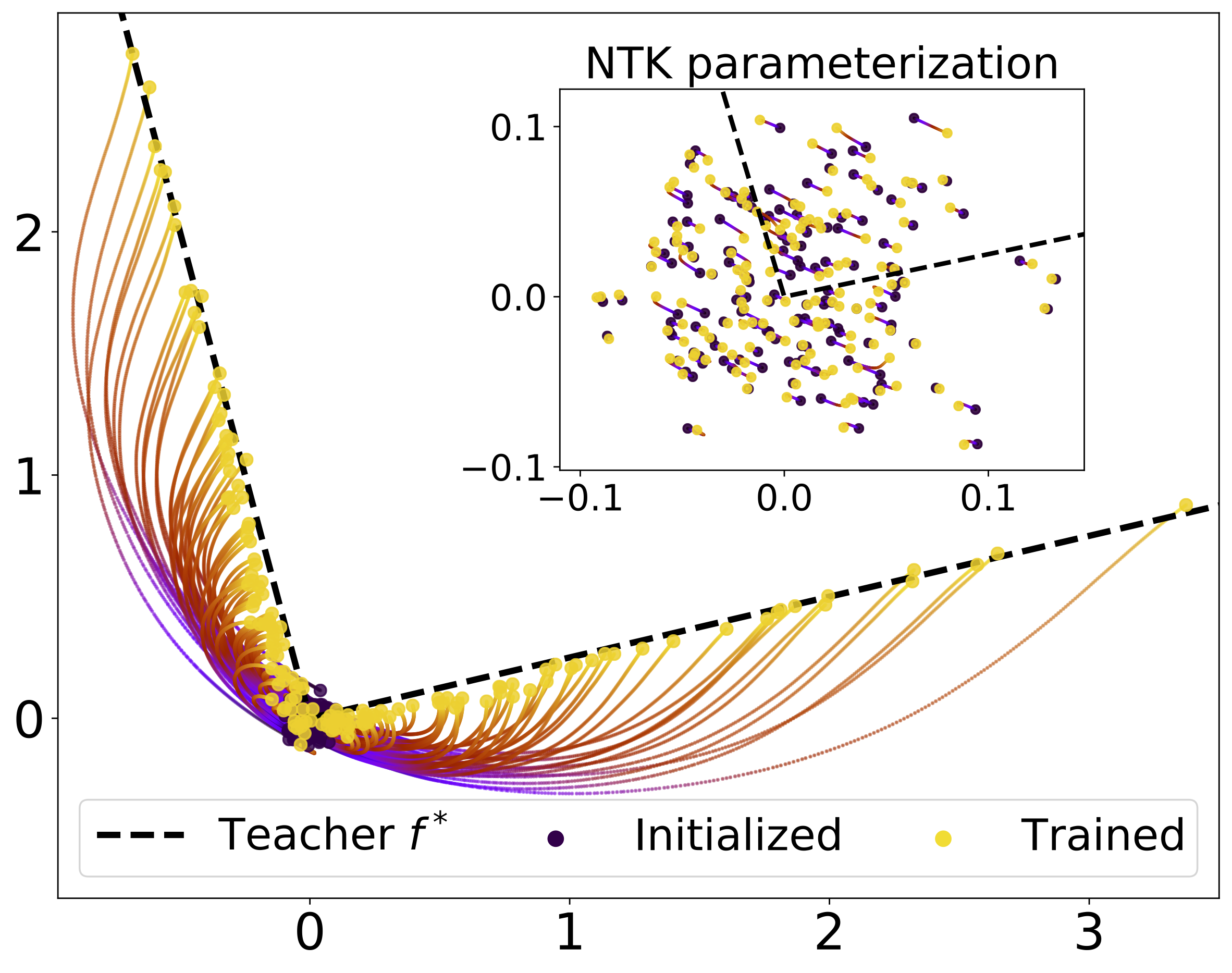}  
\vspace{-2.mm} 
\caption{\small 2D visualization of optimization trajectory under mean-field (main) and NTK (subfigure) parameterizations. $f^*$ consists of two ReLU neurons and the student is a two-layer ReLU neural network. Darker color indicates earlier in training, and vice versa. We set $d=512$, $\psi_1=\psi_2=10$; the models are optimized until both training losses are below $10^{-3}$. }
\label{fig:feature-learning}       
\vspace{-4mm}   
\end{wrapfigure}    

\begin{assumption}
\label{assump:1}~
\begin{enumerate}[leftmargin=*,topsep=0.5mm, itemsep=0.mm]
    \item \textbf{Proportional Limit.} $n,d,N\to\infty$, $n/d\to\psi_1$, $N/d\to\psi_2$, where $\psi_1, \psi_2\in(0,\infty)$. 
        \item \textbf{Student-teacher Setup.} Labels are generated as $y_i = f^*(\vx_i) + \eps_i$, where $\vx_i\iid\cN(0,\vI)$, $\eps_i$ is i.i.d.~sub-Gaussian noise with mean $0$ and variance $\sigma_\eps^2$, and the teacher $f^*$ is $\lambda_\sigma$-Lipschitz with $\norm{f^*}_{L^2} = \Theta_d(1)$.   
    \item \textbf{Normalized Activation.} The nonlinear activation $\sigma$ has $\lambda_\sigma$-bounded first three derivatives almost surely. In addition, the activation function satisfies $\E[\sigma(z)]=0$, 
    $\E[z\sigma(z)]\neq 0$, for $z\sim\cN(0,1)$. 
    \item \textbf{Gaussian Initialization.} $\sqrt{d}\cdot[\vW_0]_{ij}\iid\cN(0,1), ~
    \sqrt{N}\cdot[\va]_j\iid \cN(0,1),$  for all $i\in[d],j\in [N]$. 
\end{enumerate} 
\end{assumption}

\begin{remark}
Following \citet{hu2020universality}, we assume smooth and centered activation to simplify the computation; Section~\ref{sec:prediction-risk} provides empirical evidence that our results hold beyond this condition (see also \cite{loureiro2021learning}). 
We expect that the Gaussian input assumption may be replaced by weaker orthogonality conditions as in \citet{fan2020spectra}.   
\end{remark}

Under Assumption~\ref{assump:1}, increasing the sample size corresponds to enlarging $\psi_1$, and increasing the network width corresponds to enlarging $\psi_2$. The proportional scaling of $n,d,N$ 
(also referred to as the ``linear-width'' regime) 
implies that the model width is not significantly larger than the training set size, 
in contrast to the polynomial overparameterization often required in NTK analyses \citep{du2018gradient}, which may be less realistic for practical settings.   

Importantly, the initialization of our two-layer NN \eqref{eq:two-layer-nn} resembles the \textit{mean-field} parameterization~\citep{mei2018mean,chizat2018global}: the second layer is divided by an additional $\sqrt{N}$-factor compared to the kernel (NTK) scaling --- this ensures that $f_{\mathrm{NN}}(\vx)=\littleodp{1}$ at initialization and enables feature learning (see \cite[Corollary 3.10]{yang2020feature}). 
As an illustrative example, in Figure \ref{fig:feature-learning} we plot the gradient descent trajectory of the first-layer parameters $\vW$ in two coordinates. Observe that under the mean-field parameterization (main figure), the neurons travel away from the initialization and align with the target function (black dashed lines), whereas in the NTK parameterization (subfigure, which omits the $\frac{1}{\sqrt{N}}$-prefactor), the parameters remain close to their initialization and hence do not learn useful features.  
 
\vspace{-0.4mm}
\subsection{Lower Bound for Kernel Ridge Regression}
\label{subsec:kernel-lower-bound}

To illustrate the benefit of \textit{feature learning}, we compare the prediction risk of ridge regression on the trained CK (after one gradient step) against the ridge estimator on the initial RF kernels. Specifically, given training data $\{\vx_i,y_i\}_{i=1}^n$, we consider the following class of kernel models for comparison.  
  
\begin{itemize}[leftmargin=*,topsep=0.4mm, itemsep=0.1mm]
    \item \textbf{Random Features Model.} We introduce two RF kernels associated with the two-layer NN \eqref{eq:two-layer-nn} at initialization: the conjugate kernel (CK) defined by features $\vphi_{\text{CK}}(\vx)=\frac{1}{\sqrt{N}}\sigma(\vW_0^\top\vx)\in\R^N$, and the neural tangent kernel (NTK) \citep{jacot2018neural} defined by features $\vphi_{\text{NTK}}(\vx) = \frac{1}{\sqrt{Nd}}\text{Vec}\big(\sigma'(\vW_0^\top\vx)\vx^\top\big)\in\R^{Nd}$. 
    Given feature map $\text{RF}\in\{\text{CK,NTK}\}$, the RF ridge regression estimator can be written as
    \begin{align}
    \label{eq:RF-estimator}
        \hat{f}_{\text{RF}}(\vx) = \langle\vphi_{\text{RF}}(\vx),\hat{\va}\rangle, ~~
        \hat{\va} = \argmin_{\va\in\R^N} \,\Big\{\frac{1}{n}\sum_{i=1}^n (y_i - \langle\vphi_{\text{RF}}(\vx_i),\va\rangle)^2 + \frac{\lambda}{N} \norm{\va}^2\Big\}. 
    \end{align}
    \item \textbf{Rotationally Invariant Kernel Model.} Consider the inner-product kernel: $k(\vx,\vy) = g\left(\frac{\langle\vx,\vy\rangle}{d}\right)$, and Euclidean distance kernel: $k(\vx,\vy) = g\left(\frac{\norm{\vx-\vy}^2}{d}\right)$, where $g$ satisfies certain smoothness conditions as in \citet{el2010spectrum}. 
    Denote the associated RKHS as $\cH$, and $[\vK]_{ij} = k(\vx_i,\vx_j)$. The kernel ridge estimator is given by   
    \begin{align}
     \label{eq:kernel-estimator}
        \hat{f}_{\text{ker}} = \argmin_{f\in\cH}\Big\{\frac{1}{n}\sum_{i=1}^n (y_i - f(\vx_i))^2 + \lambda\norm{f}_{\cH}^2\Big\} ~\Rightarrow~
        \hat{f}_{\text{ker}}(\vx) =  k(\vx,\vX)^\top\left(\vK + \lambda\vI\right)^{-1}\vy. 
    \end{align}  
\end{itemize}

We denote the prediction risk of the above kernel estimators as $\cR_{\text{CK}}(\lambda), \cR_{\text{NTK}}(\lambda), \cR_{\text{ker}}(\lambda)$, respectively.   
The following lower bound is
a simple combination of known results from \cite{el2010spectrum,hu2020universality,montanari2020interpolation,bartlett2021deep}. 
\begin{prop}[Informal]
\label{prop:ridge-lower-bound}
Under Assumptions \ref{assump:1} and \ref{assump:2}, 
\begin{align}
\label{eq:ridge-lower-bound}
    \inf_{\lambda>0} \min \left\{\cR_{\mathrm{CK}}(\lambda),  \cR_{\mathrm{NTK}}(\lambda), \cR_{\mathrm{ker}}(\lambda)\right\} \ge \norm{\textsf{P}_{>1}f^*}_{L^2}^2 + o_{d,\P}(1),   
\end{align}
where $\textsf{P}_{>1}$ denotes the projector orthogonal to constant and linear functions in $L^2(\R^d,\Gamma)$.
\end{prop}
 
This proposition implies that in the proportional limit,  ridge regression on the RF or rotationally invariant kernels defined above does not outperform the best linear estimator on the input data --- it cannot achieve negligible prediction risk unless the target function is linear (i.e., $\norm{\textsf{P}_{>1}f^*}_{L^2} = 0$). 
In Section~\ref{sec:prediction-risk}, we compare the prediction risk of the ridge estimator on trained features against this lower bound.

\allowdisplaybreaks

\section{How Does One Gradient Step Change the Weights?} 
\label{sec:spike}

In this section, we study the properties of the updated weight matrix $\vW_1$ in the two-layer NN \eqref{eq:two-layer-nn}. We first show that the first gradient step on $\vW$ can be approximated by a rank-1 matrix, which contains information of the training labels $\vy$. Based on this property, we provide a signal (spike) plus noise (bulk) decomposition of $\vW_1$, and prove that the isolated singular vector is aligned to the linear component of the teacher $f^*$.  

\subsection{Almost Rank-1 Property of the Gradient Matrix}
\label{subsec:rank-1}

We utilize the orthogonal decomposition of the activation function $\sigma$ (note that $\sigma$ is normalized by Assumption~\ref{assump:1} so that $\E[\sigma(z)]=0$). Define the coefficients
\begin{align}
    \mu_1 = \E[z\sigma(z)], 
\quad 
\mu_2 = \sqrt{\E[\sigma(z)^2] - \mu_1^{2}}, \quad \mathrm{where~} z\sim\cN(0,1). \label{def:mu_1mu_2}
\end{align}
This implies that $\sigma(z) = \mu_1 z + \sigma_{\perp}(z)$, where $\E[\sigma_{\perp}(z)] = \E[z\sigma_{\perp}(z)] = 0$, and $\E[\sigma_\perp(z)^2] = \mu_2^{2}$. 
When $\mu_1\neq 0$ (again due to Assumption~\ref{assump:1}), we have the following characterization of the first gradient step $\vG_0$ in \eqref{eq:gradient-step-MSE}.

\begin{prop}\label{thm:W1-W0}
Define $\vG_0 = \frac{1}{\eta \sqrt{N}}(\vW_1 - \vW_0)$ and a rank-1 matrix $\vA := \frac{\mu_1}{n\sqrt{N}} \vX^\top\vy\va^\top$. 
Under Assumption \ref{assump:1}, there exist some constants $c, C>0$ such that for all large $n,N,d$, with probability at least $1-ne^{-c\log^2n}$, 
\begin{align*}
   \norm{\vG_0-\vA} \le \frac{C\log^2n}{\sqrt{n}}\cdot\norm{\vG_0}. 
\end{align*}
\end{prop} 
Proposition~\ref{thm:W1-W0} suggests that the first-step gradient can be approximated in operator norm by a rank-1 matrix $\vA$; thus, when the learning rate is reasonably large, we expect a ``spike'' to appear in the updated weight matrix $\vW_1$. 
Intuitively, since this rank-1 direction relates to the label vector $\vy$, the resulting $\vW_1$ may be ``aligned'' to the target function $f^*$. 
This intuition is confirmed in the next subsection. 

\paragraph{Scaling of Learning Rate $\eta$.}  Before we analyze the alignment property, it is important to specify an appropriate learning rate $\eta$ such that change in the first-layer weights after one gradient descent step is neither insignificant nor unreasonably large. 
From Assumption~\ref{assump:1} we know that for proportional $n,d,N$, the initial weight matrix satisfies $\norm{\vW_0}=\Theta_{d,\P}(1), \norm{\vW_0}_F=\Theta_{d,\P}(\sqrt{d})$, and due to Proposition~\ref{thm:W1-W0}, the first gradient step satisfies $\sqrt{N}\norm{\vG_0}=\Theta_{d,\P}(1), \sqrt{N}\norm{\vG_0}_F=\Theta_{d,\P}(1)$.   
 
In other words, if we write $\eta=\Theta(N^{\alpha})$, then $\alpha\ge 0$ is required so that the change in the weight matrix is non-negligible (one may verify that for $\eta=o_d(1)$, the test performance of kernel ridge regression remains unchanged after one GD step). On the other hand, when $\alpha>1/2$, the gradient ``overwhelms'' the initialized parameters $\vW_0$, and the preactivation feature $\langle\vx,\vw_i\rangle$ in the NN \eqref{eq:two-layer-nn} becomes unbounded as $N\to\infty$.  
This motivates us to consider the following two regimes of learning rate scaling. 
\begin{align}
    \text{\underline{Small lr}: } &\eta=\Theta(1) ~\Rightarrow~ \norm{\vW_1-\vW_0}\asymp \norm{\vW_0}
    \label{eq:lrsmall} 
    \\
    \text{\underline{Large lr}: } &\eta=\Theta(\sqrt{N}) ~\Rightarrow~ \norm{\vW_1-\vW_0}_F\asymp \norm{\vW_0}_F
    \label{eq:lrlarge}
\end{align}
The following subsection and Section~\ref{subsec:small-lr} consider the setting where $\eta=\Theta(1)$, which is parallel to common practice in NN optimization\footnote{Heuristically speaking, the updated NN under $\eta=\Theta(1)$ remains close to the ``kernel regime'', in the sense that each neuron does not travel far away from the initialization, i.e., as $N\to\infty$,  $\big|[\vW_1-\vW_0]_{ij}\big| \ll \big|[\vW_0]_{ij}\big|$ for all $i,j$ with high probability.  }.   
Whereas in Section~\ref{subsec:large-lr} we analyze the larger step size $\eta=\Theta(\sqrt{N})$, which resembles the learning rate scaling in the maximal update parameterization in \citet{yang2020feature}; in particular, using Lemma~\ref{lemm:gradient-norm} in Appendix~\ref{app:gradient-norm} one can easily verify that given data point $\vx\sim\cN(0,\vI)$, the change in each coordinate of the feature vector is roughly of the same order as its initialized magnitude, that is, for $i\in [N]$, $\big|\sigma(\vW_1^\top\vx)-\sigma(\vW_0^\top\vx)\big|_i \asymp \big|\sigma(\vW_0^\top\vx)\big|_i = \tilde{\Theta}(1)$ with probability 1 as $N\to\infty$.

\subsection{Alignment with the Target Function}
\label{subsec:alignment}
Under Assumption~\ref{assump:1}, we may utilize the following orthogonal decomposition of the target function $f^*$, 
\begin{align}
    f^*(\vx) = \mu_0^* + \mu_1^*\langle\vx,\vbeta_*\rangle + \textsf{P}_{>1} f^*(\vx), ~~
    \mu_1^*\vbeta_* = \E[\vx f^*(\vx)], 
    \label{eq:orthogonal-decomposition}
\end{align} 
where $\textsf{P}_{>1}$ is the projector orthogonal to constant and linear functions in $L^2(\R^d,\Gamma)$, which implies that $\E[\textsf{P}_{>1} f^*(\vx)]=0, \E[\vx\textsf{P}_{>1} f^*(\vx)]=\mathbf{0}$ (e.g., see \cite[Section 4.3]{bartlett2021deep}). As $d\to\infty$, quantities defined in \eqref{eq:orthogonal-decomposition} satisfy
$\norm{\vbeta_*}=1, ~
\norm{\textsf{P}_{>1}f^*}_{L^2} \to \mu_2^*$, where $\mu_0^*, \mu_1^*, \mu_2^*$ are bounded constants.  
Intuitively, $\mu_0^*,\mu_1^*$, and $\mu_2^*$ can be interpreted as the ``magnitude'' of the constant, linear, and nonlinear components of $f^*$, respectively.

\paragraph{A Spiked Model for $\vW_1$.} 
When $\eta=\Theta(1)$ in \eqref{eq:lrsmall}, we show a BBP phase transition (named after Baik, Ben Arous, Péché \citep{baik2005phase}) for the leading singular value of $\vW_1$, and quantify the alignment between the corresponding singular vector $\vu_1$ and the linear component of target function $\vbeta_*$.
It is worth noting that in our analysis, the signal  $\vbeta_*$ is ``hidden'' in the rank-one perturbation $\vA$ defined in Proposition \ref{thm:W1-W0}; thus our setting is different from the usual low-rank signal-plus-noise models  (e.g. \cite{benaych2011eigenvalues,benaych2012singular,capitaine2018limiting}), and the alignment we aim to quantify $|\langle\vu_1,\vbeta_*\rangle|$ does not directly follow from classical results on the BBP transition.   
 
\begin{theo}\label{thm:alignment} 
Given Assumption~\ref{assump:1} and fixed $\eta=\Theta(1)$, we define $\bar{\mu} = \lim_{d\to\infty} \norm{f^*}_{L^2(\R^d,\Gamma)}$, 
and
\begin{equation}\label{def:theta_12}
    \theta_1:= \sqrt{\bar{\mu}^2\psi_1^{-1} + \mu_1^{*2}}\cdot\mu_1\eta,
    \quad
    \theta_2:=\mu_1\mu_1^{*}\eta.
\end{equation} 
Then the leading singular value $s_1(\vW_1)$ and the corresponding left singular vector $\vu_1$ satisfy 
\begin{equation}
	s_1(\vW_1)\to \sqrt{\frac{(1+\theta_1^2)(\psi_2+\theta_1^2)}{\theta_1^2}}, \quad
	|\langle\vu_1,\vbeta_*\rangle|^2\to\frac{\theta_2^2}{\theta_1^2}\left(1-\frac{\psi_2+\theta_1^2}{\theta_1^2(\theta_1^2+1)}\right),
\end{equation}
if $\theta_1>\psi_2^{1/4}$; otherwise, $s_1( \vW_1)\to1+\sqrt{\psi_2}$ and $|\langle\vu_1,\vbeta_*\rangle|\to 0$, in probability, as $n,N,d\to \infty$. 
\end{theo}
\begin{remark}
While the above proposition only describes the isolated singular value/vector, due to the almost rank-1 property of $\vG_0$, one can easily verify that the limiting spectrum of first-layer weights, namely the ``bulk'', remains unchanged after the gradient update, and for any fixed $i>1$, $\abs{s_i(\vW_1) - s_i(\vW_0)} = o_{d,\P}(1)$. 
\end{remark} 

\begin{wrapfigure}{R}{0.365\textwidth}  
\vspace{-0.5mm}
\centering  
\includegraphics[width=0.355\textwidth]{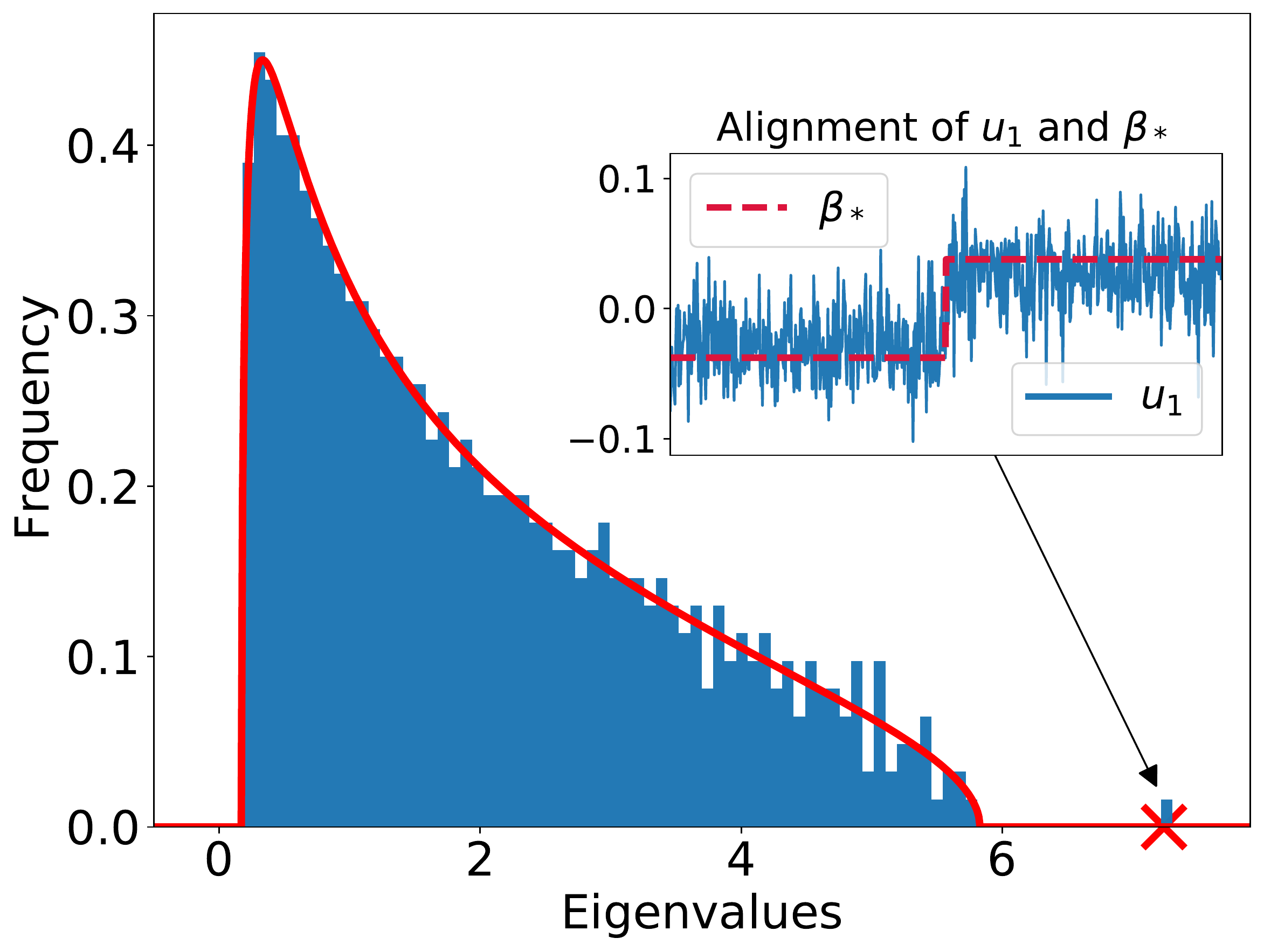} 
\vspace{-2.5mm} 
\caption{\small Main: empirical singular values of $\vW_1$ (blue) vs.~analytic prediction (red). Subfigure: overlap between $\vu_1$ and the teacher vector $\vbeta_* \propto [-\mathbf{1}_{d/2};\mathbf{1}_{d/2}]\in\R^d$. We set $\sigma=\text{tanh}$, $f^*(\vx)=\text{ReLU}(\langle\vx,\vbeta_*\rangle)$, $\eta=2$, $\psi_1=4, \psi_2=2$, and $\sigma_\eps=0.2$.  } 
\label{fig:W-spike}     
\vspace{-5.5mm}   
\end{wrapfigure}  

We make the following observations. Beyond the threshold $\theta_1>\psi_2^{1/4}$, increasing the learning rate $\eta$ enlarges the leading singular value (spike) $s_1(\vW_1)$. 
As for the overlap, one can numerically verify $|\langle\vu_1,\vbeta_*\rangle|^2$ is upper-bounded by $\frac{\theta_2^4 - \psi_2}{\theta_2^2(\theta_2^2+1)}<1$ (obtained when $\psi_1 = n/d \to\infty$), from which we deduce that better alignment is achieved when we take a bigger step, or when the nonlinearity $\sigma$ and target $f^*$ have larger linear components (i.e., larger $\mu_1,\mu_1^*$).    
 
Theorem~\ref{thm:alignment} is numerically verified in Figure~\ref{fig:W-spike}. Observe that after one gradient step with $\eta=\Theta(1)$, the bulk of the spectrum of $\vW$ remains unchanged and is given by the Marchenko-Pastur law (red), but a spike may appear (prediction from Theorem~\ref{thm:alignment} is indicated by marker ``$\times$'') when $\eta$ exceeds a certain threshold; furthermore, the corresponding singular vector $\vu_1$ aligns with the linear component $\vbeta_*$ of the target function, as shown in the subfigure (see also Figure~\ref{fig:appendix}(a)).  
We investigate the impact of this alignment on the performance of kernel ridge regression in Section~\ref{sec:prediction-risk}.

\paragraph{A Spiked Model for CK?} 
While our result only characterizes the weight matrix $\vW$, it may also reveal interesting properties of the CK matrix. 
In particular,
\cite[Lemma 5]{hu2020universality} in combination with Lemma~\ref{lemm:gradient-norm} imply that for odd activation $\sigma$, the expected feature matrix (after one gradient step with $\eta=\Theta(1)$) satisfies
$$
    \norm{\vSigma_\Phi - \overline{\vSigma}_\Phi} \overset{\P}{\to} 0, \quad
    \text{where\, }
    \vSigma_\Phi = \E_{\vx}\left[\sigma(\vW_1^\top\vx)\sigma(\vx^\top\vW_1)\right], ~
    \overline{\vSigma}_\Phi = \mu_1^2\vW_1^\top\vW_1 + \mu_2^2\vI. 
$$
Consequently, Theorem~\ref{thm:alignment} implies the same BBP transition for $\vSigma_\Phi$. When the population $\vSigma_\Phi$ contains a spike, it is natural to expect the empirical CK matrix to exhibit a similar transition, which we conjecture that the Gaussian equivalence property (see Section~\ref{subsec:GET}) can precisely capture. 

\begin{conj} 
\label{conj:CK}
Assume $\sigma$ is an odd function\footnote{The odd activation $\sigma$ ensures that the initialized $\mathbf{CK}_0$ does not contain ``uniformative'' spikes -- see \citet{benigni2022largest}. \vspace{-2mm}} in addition to Assumption~\ref{assump:1}, and $\eta=\Theta(1)$. 
Given new training data/labels $\tilde{\vX},\tilde{\vy}$ (independent of $\vW_1$), define $\vPhi = \frac{1}{\sqrt{N}}\sigma(\tilde{\vX}\vW_1),
\bar{\vPhi} = \frac{1}{\sqrt{N}}\left(\mu_1\tilde{\vX}\vW_1 + \mu_2\vZ\right),$ 
where $[\vZ]_{i,j}\iid\cN(0,1)$, and denote the left leading singular vectors of $\vPhi,\bar{\vPhi}$ as $\vu_1,\bar{\vu}_1$, respectively. We conjecture
$$\abs{s_i(\vPhi) - s_i(\bar{\vPhi})} = o_{d,\P}(1), ~ \forall i\in [n]; \quad
|\langle\vu_1,\tilde{\vy}/\norm{\tilde{\vy}}\rangle|^2 =  |\langle\bar{\vu}_1,\tilde{\vy}/\norm{\tilde{\vy}}\rangle|^2 + o_{d,\P}(1).$$    
\end{conj}  
 
\begin{wrapfigure}{R}{0.375\textwidth}  
\vspace{-5.mm}  
\centering  
\includegraphics[width=0.365\textwidth]{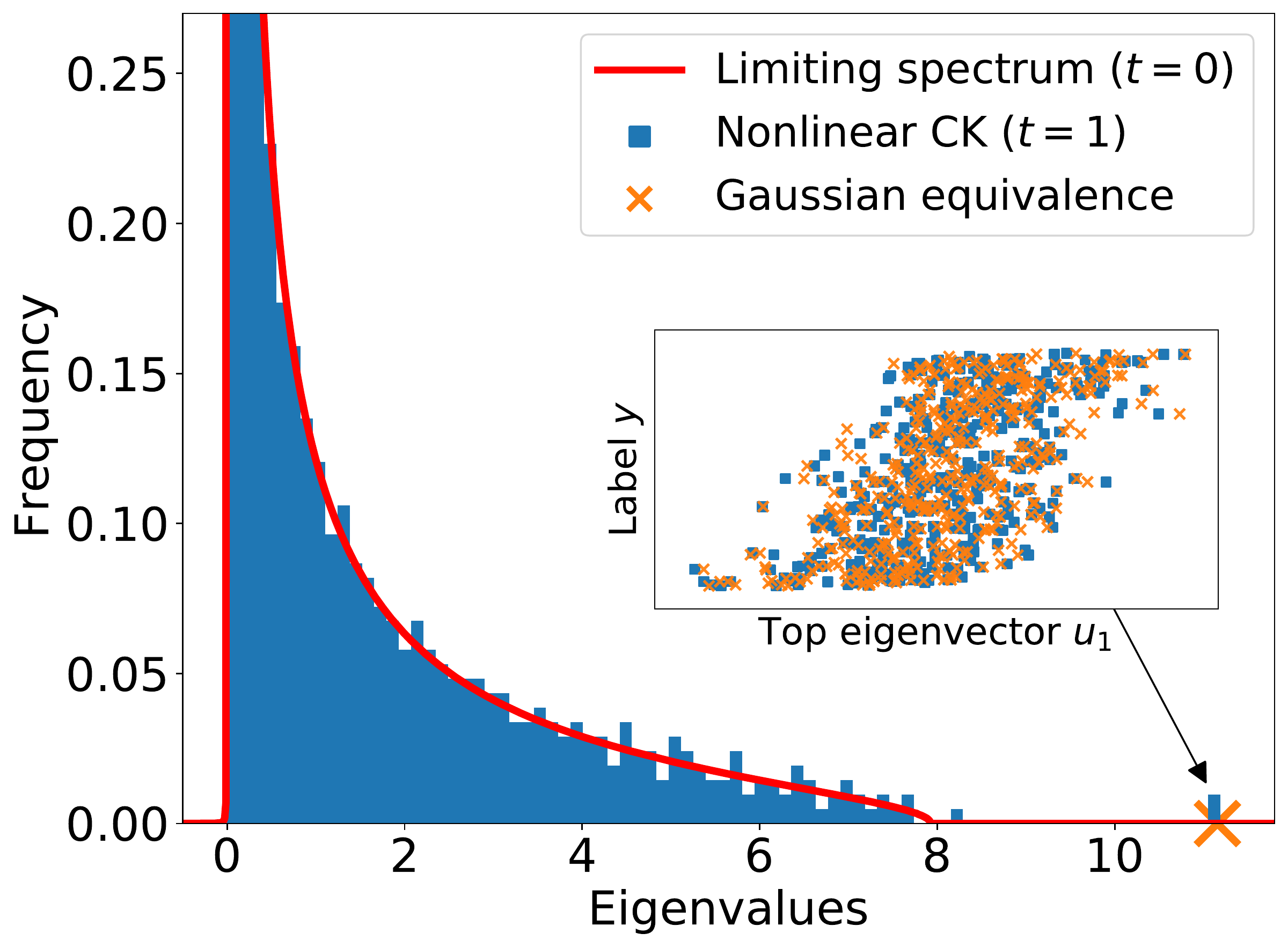}  
\vspace{-2.5mm} 
\caption{\small Main: CK spectrum after one gradient step on $\vW$. Subfigure: projection of training labels $\tilde{\vy}$ onto top PC of $\mathbf{CK}_1$.  Quantities computed from the nonlinear features $\vPhi$ are colored blue, and the conjectured Gaussian equivalent predictions are in orange.  
We set $\sigma=\text{SoftPlus}$, $f^*(\vx)=\text{tanh}(\langle\vx,\vbeta_*\rangle)$, $\eta=2$, $\psi_1=1.5, \psi_2=1.25$.   }
\label{fig:CK-spike}    
\vspace{-4mm}   
\end{wrapfigure}      
The conjecture predicts both the eigenvalues of the CK matrix and the overlap between its spike eigenvector and the training labels. 
In Figure \ref{fig:CK-spike} we plot the eigenvalue histogram of the CK matrix after one gradient step with $\eta=\Theta(1)$, which we denote as $\mathbf{CK}_1 = \vPhi\vPhi^\top$. Observe that the bulk of the spectrum remains unchanged compared to $\mathbf{CK}_0$, which can be analytically computed (red). 
On the other hand, similar to $\vW_1$, an isolated eigenvalue (spike) appears in $\mathbf{CK}_1$, the location of which can be predicted by the Gaussian equivalent model in Conjecture~\ref{conj:CK} (marker ``$\times$''). 
 
Furthermore, in our student-teacher setting, we observe that the isolated eigenvector (top principal component PC) of $\mathbf{CK}_1$ correlates with the training labels $\tilde{\vy}$ -- this is also captured by the Gaussian equivalent model, as shown in the subfigure of Figure~\ref{fig:CK-spike}. This demonstrates that the alignment phenomenon reported in
\cite[Figure 3]{fan2020spectra} already occurs \textit{after one gradient step}. We note that similar alignment between the training labels and the principal components of the (trained) NTK has also been empirically observed \citet{chen2020label,ortiz2021can}, and it is argued that such overlap may improve optimization or generalization.

\section{Do the Learned Features Improve Generalization?}
\label{sec:prediction-risk}

Thus far we have shown that after one gradient step, the first-layer weights align with the linear component of the teacher model. Intuitively, since the learned feature map $\vx\rightarrow\sigma(\vW_1^\top\vx)$ ``adapts'' to the teacher $f^*$, we may expect the ridge regression estimator on the trained CK to achieve better performance. 
In this section we confirm this intuition in a concrete example: we consider the setting where $f^*$ is a single-index model, and compare the CK prediction risk before and after one gradient descent step on $\vW$.   

\begin{assumption}[\textbf{Single-index/one-neuron Teacher}]
\label{assump:2}\,
$f^*(\vx) = \sigma^*(\langle\vx,\vbeta_*\rangle),$ where $\vbeta_*\in\R^d$ is a deterministic signal with $\norm{\vbeta_*} = 1$, and $\sigma^*$ is Lipschitz with $\mu_0^* = 0$, $\mu_1^*\neq 0$ defined in \eqref{eq:orthogonal-decomposition}.  
 
\end{assumption}
\begin{remark}
The single-index setting has been extensively studied in the proportional regime \citep{gerace2020generalisation,dhifallah2020precise,hu2020universality}, and it is an instance of the ``hidden manifold model'' \cite{goldt2020modeling}.
However, most prior works only considered training the coefficients $\va$ on top of \textit{fixed} feature map (e.g., defined by randomly initialized $\vW_0$), and such RF models cannot learn a single-index $f^*$ efficiently in high dimensions \citep{yehudai2019power}.   
\end{remark}  
 
As stated in Section~\ref{subsec:kernel-lower-bound}, the RF ridge estimator defined by the two-layer NN \eqref{eq:two-layer-nn} has $\Omega(1)$ prediction risk unless $\sigma^*$ is a linear function. 
Here our goal is to demonstrate that the trained CK model can \textit{outperform} the initial RF and potentially the kernel lower bound \eqref{eq:ridge-lower-bound}. 
We first introduce the Gaussian equivalence property which will be useful in the computation of prediction risk.  
 
\subsection{The Gaussian Equivalence Property} 
\label{subsec:GET}

The Gaussian equivalence theorem (GET) implies that the prediction risk of a nonlinear kernel model can be the same as that of a noisy linear model. Specifically, recall the prediction risk of the ridge estimator:
\begin{align}  
    \cR_{\mathrm{F}}(\lambda) =\E_{\vx}\big(\langle\vphi_{\mathrm{F}}(\vx),\hat{\va}_\lambda\rangle - f^*(\vx)\big)^2, ~
    \hat{\va}_\lambda = \text{argmin}_{\va} \Big\{\frac{1}{n}\sum_{i=1}^n (y_i - \langle\vphi_{\mathrm{F}}(\vx_i),\va\rangle)^2 + \frac{\lambda}{N}\norm{\va}^2\Big\},  
    \label{eq:ridge-risk}
\end{align}
where $\mathrm{F}\in\{\mathrm{CK},\mathrm{GE}\}$ indicates the choice of feature map, which is either the nonlinear CK feature $\vphi_{\mathrm{CK}}(\vx) = \frac{1}{\sqrt{N}}\sigma(\vW^\top\vx)$, or the Gaussian equivalent (GE) feature $\vphi_{\mathrm{GE}}(\vx) = \frac{1}{\sqrt{N}}\left(\mu_1\vW^\top\vx + \mu_2\vz\right)$ where $\vz\sim\cN(0,\vI)$ independent of $\vx$, $\vW$. In the following, we take $\vW$ to be the updated weights after one or more GD steps.

The Gaussian equivalence refers to the universality phenomenon $\cR_{\mathrm{CK}}(\lambda) \approx \cR_{\mathrm{GE}}(\lambda)$. 
For RF models~\eqref{eq:RF-estimator}, the GET has been rigorously proved in \cite{hu2020universality,montanari2022universality}. Furthermore, \citet{goldt2021gaussian,loureiro2021learning} provided empirical evidence that such equivalence holds in much more general feature maps, including the representation of certain pretrained NNs (e.g., see \cite[Figure 4]{loureiro2021learning}). Since our setting goes beyond RF model and cannot be covered by prior results, we first establish the GET for our \textit{trained} feature map under small learning rate.    
 
\begin{theo} 
\label{thm:GET}
Given Assumptions \ref{assump:1}, \ref{assump:2}, and in addition assume the activation $\sigma$ is an odd function. 
If the learning of $\vW_{\!t}$ in \eqref{eq:gradient-step-MSE} and estimation of $\hat{\va}_\lambda$ in \eqref{eq:ridge-risk} are performed on independent training data $\vX$ and $\tilde{\vX}$, respectively, then for any fixed $t\in\N$, the GET holds after the first-layer weights are optimized for $t$ gradient steps with learning rate $\eta=\Theta(1)$; that is, for trained CK feature $\vphi_{\mathrm{CK}}(\vx) = \frac{1}{\sqrt{N}}\sigma(\vW_{\!t}^\top\vx)$ and $\lambda>0$, 
\begin{align}
    \abs{\cR_{\mathrm{CK}}(\lambda) - \cR_{\mathrm{GE}}(\lambda)} = o_{d,\P}(1). 
    \label{eq:GET}
\end{align}  
\end{theo} 
This is to say, for learning rate $\eta=\Theta(1)$, the Gaussian equivalent model provides an accurate description of the prediction risk of ridge regression (on the trained CK) at any fixed time step $t$, although most of our analysis deals with $t=1$. 
The important observation is that even though the trained weights $\vW_t$ are no longer i.i.d., the Gaussian equivalence property can still hold when $\vW_t - \vW_0$ remains ``small''  (in some norm, see \eqref{eq:orthogonality-condition} for details), which entails that the neurons remain nearly orthogonal to one another. 
 
\paragraph{Implications of Gaussian Equivalence.}
Under the GET, we can equivalently compute $\cR_{\mathrm{GE}}(\lambda)$, the prediction risk of ridge regression on noisy Gaussian features $\vphi_{\mathrm{GE}}$, which can be characterized using standard tools such as the Gaussian comparison inequalities \citep{gordon1988milman,thrampoulidis2015regularized}. 
Theorem~\ref{thm:GET} is empirically validated in Figure \ref{fig:main}(c), in which we run gradient descent on $\vW$ with small learning rate for 50 steps, and compute the prediction risk of the ridge regression estimator on the CK at each step; observe that empirical values match the analytic predictions\footnote{In Figure \ref{fig:main}(c), we compute certain quantities involved in the optimization problem (e.g., spectrum of $\vW$) using finite-dimensional matrices, following \citet{loureiro2021learning}; hence our analytic curves are not entirely ``asymptotic''. \vspace{-2mm}} in the early phase of training. We however emphasize that Theorem~\ref{thm:GET} does not allow for the number of training steps $t$ to grow with the training set size $n$; in fact, in Appendix \ref{app:experiment} we empirically observe that the GET may fail if we train the first-layer weights longer.    
  
On the other hand, the GET also implies that the kernel estimator is essentially ``linear'' in high dimensions. For the squared loss, it is straightforward to verify that the Gaussian equivalent model cannot learn the nonlinear component of the target function $\textsf{P}_{>1}f^*$ as follows.
\begin{fact} 
Under the same assumptions as Theorem~\ref{thm:GET},  $\cR_{\mathrm{GE}}(\lambda) \ge \norm{\textsf{P}_{>1}f^*}_{L^2}^2$ for any $\psi_1, \psi_2$ and  $\lambda>0$. 
\label{fact:GET-lower-bound}
\end{fact}
Hence, when $\eta=\Theta(1)$, even though training the first-layer $\vW$ for just one step leads to non-trivial improvement over the initial RF ridge estimator (which we precisely quantify in Section~\ref{subsec:small-lr}), the learned CK cannot outperform the best linear model on the input features. 
In other words, to (possibly) learn a nonlinear $f^*$, the trained feature map needs to violate the GET. 
In the case of one gradient step on $\vW$, this amounts to using a sufficiently large step size, which we analyze in Section~\ref{subsec:large-lr}.    

\subsection{$\eta=\Theta(1)$: Improvement Over the Initial CK} 
\label{subsec:small-lr}

While the Gaussian equivalence property allows us to compute the asymptotic prediction risk after multiple gradient steps with $\eta=\Theta(1)$, the precise expressions can be opaque and not amenable to interpretation or quantitative characterization. 
Fortunately for the first gradient step, the risk calculation can be simplified by the rank-1 approximation of the gradient matrix $\vG_0$ shown in Section~\ref{subsec:rank-1}.  
Therefore, in this subsection we focus on $t=1$ and analyze how the trained features improves over the initialized RF.  
To quantify the discrepancy in the prediction risk \eqref{eq:ridge-risk}, we write $\cR_0(\lambda)$ as the prediction risk of the initialized RF ridge regression estimator (on the feature map $\vx\to\sigma(\vW_0^\top\vx)$), and $\cR_1(\lambda)$ as the prediction risk of the ridge estimator on the new feature map $\vx\to\sigma(\vW_1^\top\vx)$ after one feature learning step.

Importantly, due to the alignment between the trained features and the teacher model $f^*$ demonstrated in Section~\ref{sec:spike}, we cannot simply apply a rotation invariance argument (e.g., \cite[Lemma 9.2]{mei2019generalization}) to remove the dependency on the true parameters $\vbeta_*$ and reduce the prediction risk to trace of certain rational functions of the kernel matrix; in other words, knowing the spectrum (or the Stieltjes transform) of the CK is not sufficient.  Instead, we utilize the GET and the almost rank-1 property of $\vG_0$ in Proposition~\ref{thm:W1-W0}, which, in combination with techniques from operator-valued free probability theory \citep{mingo2017free}, enables us to obtain the asymptotic expression of the difference in the prediction risk before and after one gradient step.  
 
\begin{theo}\label{thm:R0-R1} 
Under the same assumptions as Theorem~\ref{thm:GET} and $\eta=\Theta(1)$, we have  
$$
    \cR_0(\lambda) - \cR_1(\lambda) \overset{\P}{\to} \delta(\eta,\lambda,\psi_1,\psi_2)\ge 0, 
$$
where $\delta(\eta,\lambda,\psi_1,\psi_2)$ is defined by \eqref{delta_formula} in Appendix~\ref{sec:linear_pencil}. $\delta$ is a non-negative function of $\eta,\lambda,\psi_1,\psi_2\in(0,+\infty)$ with parameters $\mu_1^*,\mu_1,\mu_2$, and it vanishes if and only if (at least) one of $\mu_1^*, \mu_1$ and $\eta$ is zero. 
\end{theo}
\begin{remark}
Performance of the initial RF ridge estimator $\cR_0(\lambda)$ has been characterized by many prior works (e.g., \citep{gerace2020generalisation,mei2019generalization}); hence the precise asymptotics of $\delta$ provided in Theorem~\ref{thm:R0-R1} allows us to explicitly compute the asymptotic prediction risk of the CK model after one feature learning step $\cR_1(\lambda)$.  
\end{remark} 
 
Theorem~\ref{thm:R0-R1} confirms our intuition that training the first-layer parameters improves the CK model, as shown in Figure~\ref{fig:main}(a)(b). Remarkably, this improvement ($\delta>0$) holds for any $\psi_1,\psi_2\in(0,\infty)$, that is, taking one gradient step (with learning rate $\eta=\Theta(1)$) is \textit{always} beneficial, even when the training set size $n$ is small. 
Moreover, we do not require the student and teacher models to have the same nonlinearity
--- a non-vanishing decrease in the prediction risk of CK ridge regression is present as long as $\mu_1,\mu_1^*\neq 0$. 
On the other hand, the GET (in particular Fact \ref{fact:GET-lower-bound}) also implies an upper bound on the possible improvement: $\delta\le\cR_0(\lambda)-\mu_2^{*2}$ as $n,d,N\to\infty$; this is to say, the trained CK remains in the ``linear'' regime. 

Now we consider the following special cases where the expression of $\delta$ can be further simplified. 

\paragraph{Large Sample Limit.} 
We first analyze the setting where the sample size $n$ is larger than any constant times $d$, that is, we let $n,d,N\to\infty$ proportionally, and then take the limit $\psi_1\to\infty$. 
In this regime, since a large number of training data is used to compute the gradient for the first-layer parameters, we intuitively expect the benefit of feature learning to be more pronounced, and a larger step may be more beneficial.

\begin{prop}\label{prop:large-sample-size} 
Consider the large sample regime: $\psi_1\to\infty$, $\psi_2\in (0,\infty)$. 
Under the same assumptions as Theorem~\ref{thm:GET} and $\eta=\Theta(1)$, $\lim_{\psi_1\to\infty}\delta(\eta,\lambda,\psi_1,\psi_2)$, defined in \eqref{eq:delta_case1}, is $(i)$ non-negative, $(ii)$ vanishing if and only if one of $\mu_,\mu_1^*,\eta$ is zero, and $(iii)$ increasing with respect to the learning rate $\eta$. 
\end{prop} 
Proposition~\ref{prop:large-sample-size} predicts that the prediction risk $\cR_1(\lambda)$ further decreases as we use a larger learning rate $\eta$, which is empirically verified in Figure~\ref{fig:main}(a). We note that the large learning rate setting ($\eta=\Theta(\sqrt{N})$) in Section~\ref{subsec:large-lr} cannot be covered by the above proposition by increasing $\eta$, as here $\eta$ does not scale with $N$.
 
\paragraph{Large Width Limit.} We also address the highly overparameterized regime, i.e., $\psi_2\to\infty$. In this limit, the initialized CK model approaches the kernel ridge regression estimator, the prediction risk of which is still lower bounded by $\norm{\textsf{P}_{>1}f^*}_{L^2}^2$ due to Proposition~\ref{prop:ridge-lower-bound}. The following proposition indicates that the advantage of one-step feature learning becomes negligible in this large width setting. 
\begin{prop}\label{prop:large-width} 
Consider the large width regime: $\psi_1\in (0,\infty)$, $\psi_2\to\infty$. 
Then under the same assumptions as Theorem~\ref{thm:GET}  and $\eta=\Theta(1)$, we have $\lim_{\psi_2\to\infty}\delta(\eta,\lambda,\psi_1,\psi_2)=0$.  
\end{prop} 
Proposition~\ref{prop:large-width} agrees with Figure~\ref{fig:main}(b), where we see that the risk improvement is more prominent when the width $N$ is not too large. 
One explanation is that as $\psi_2=N/d$ increases, the initial CK already achieves lower prediction risk (e.g., see \cite[Figure 4]{mei2019generalization}), so the benefit of feature learning becomes less significant.  
 
\begin{figure}[t]  
\centering
\begin{minipage}[t]{0.328\linewidth}
\centering
{\includegraphics[width=1\textwidth]{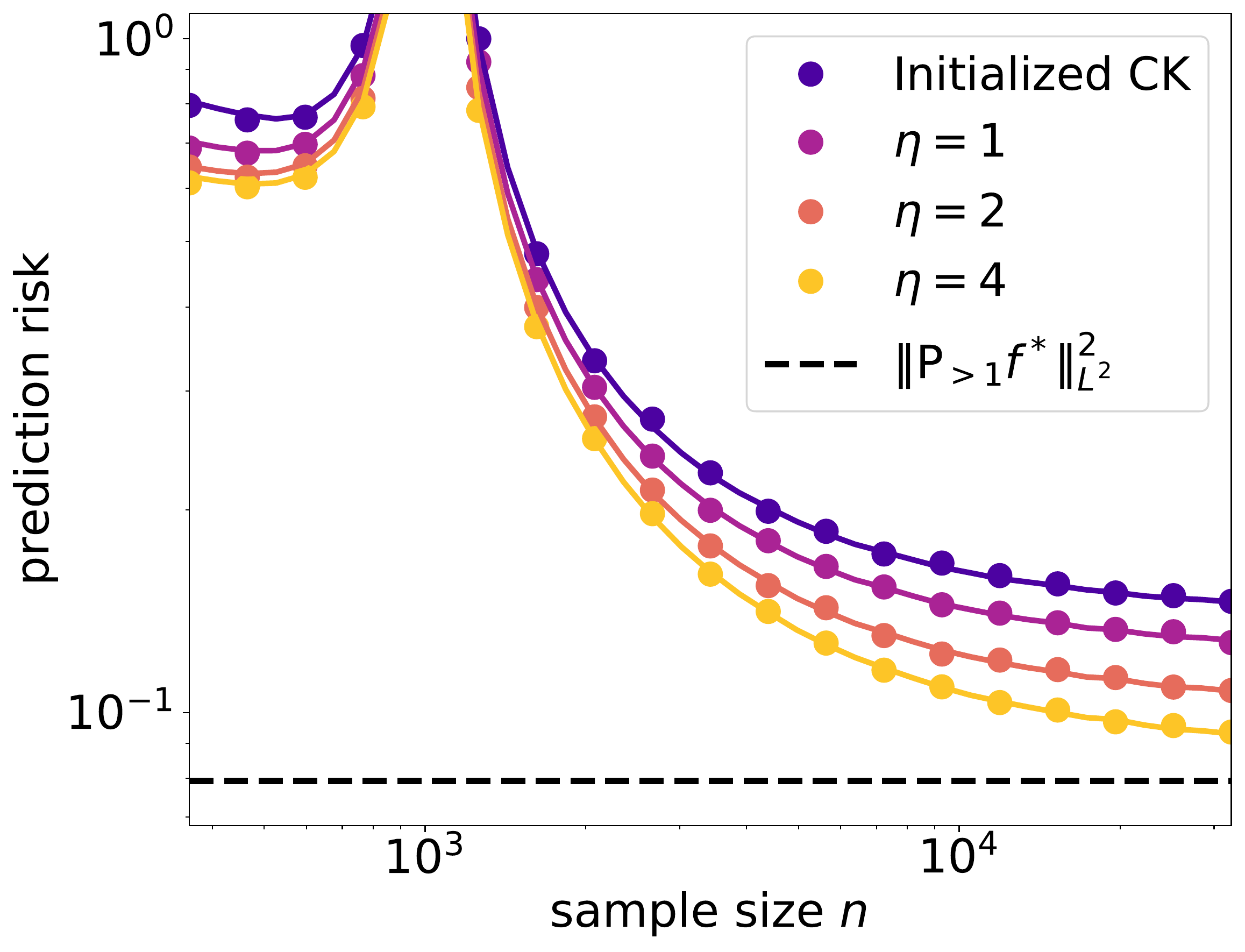}}  \\
\small (a) One step (risk vs.~sample size). 
\end{minipage}
\begin{minipage}[t]{0.328\linewidth}
\centering 
{\includegraphics[width=0.978\textwidth]{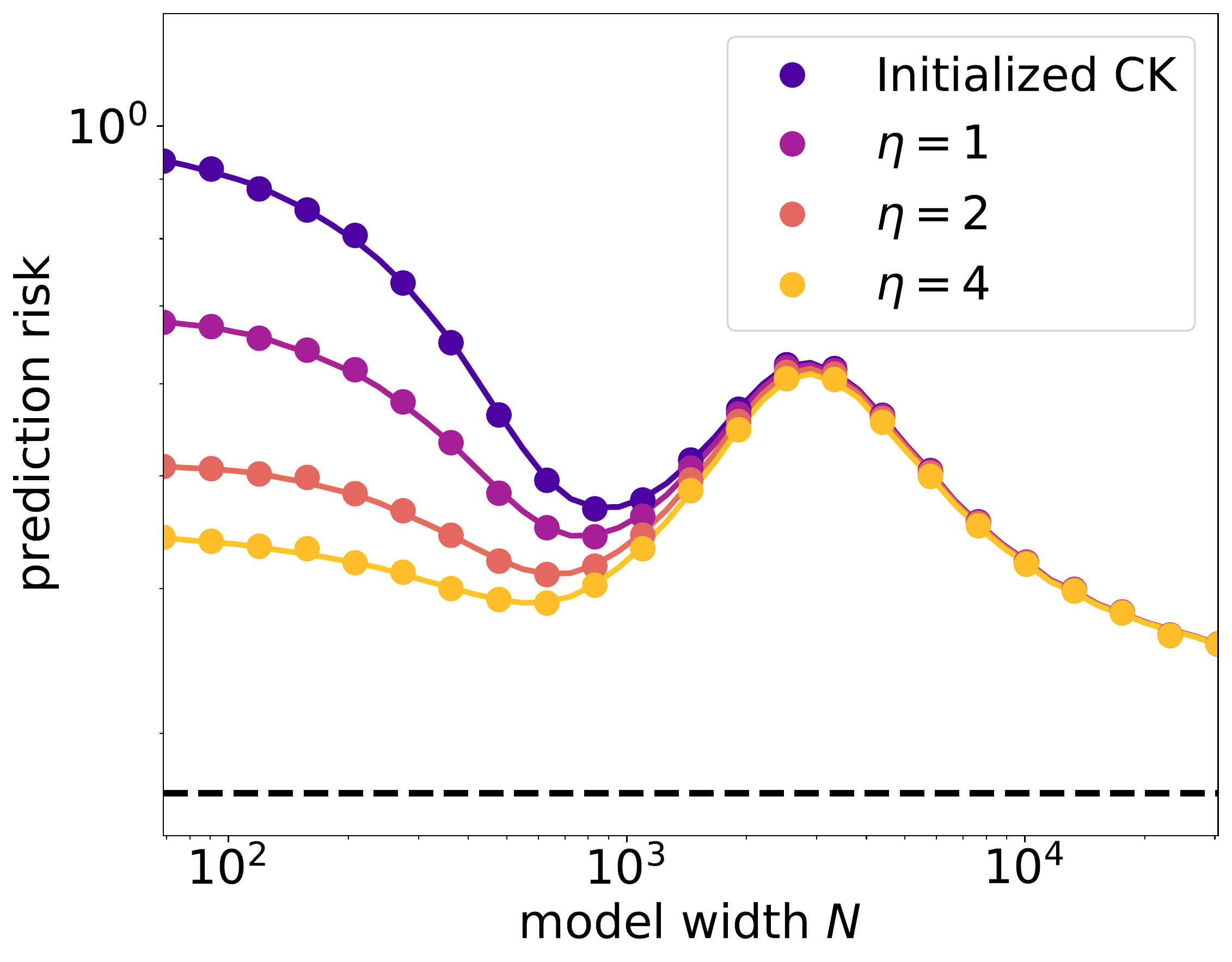}} \\ 
\small (b) One step (risk vs.~width). 
\end{minipage} 
\begin{minipage}[t]{0.328\linewidth}
\centering
{\includegraphics[width=1\textwidth]{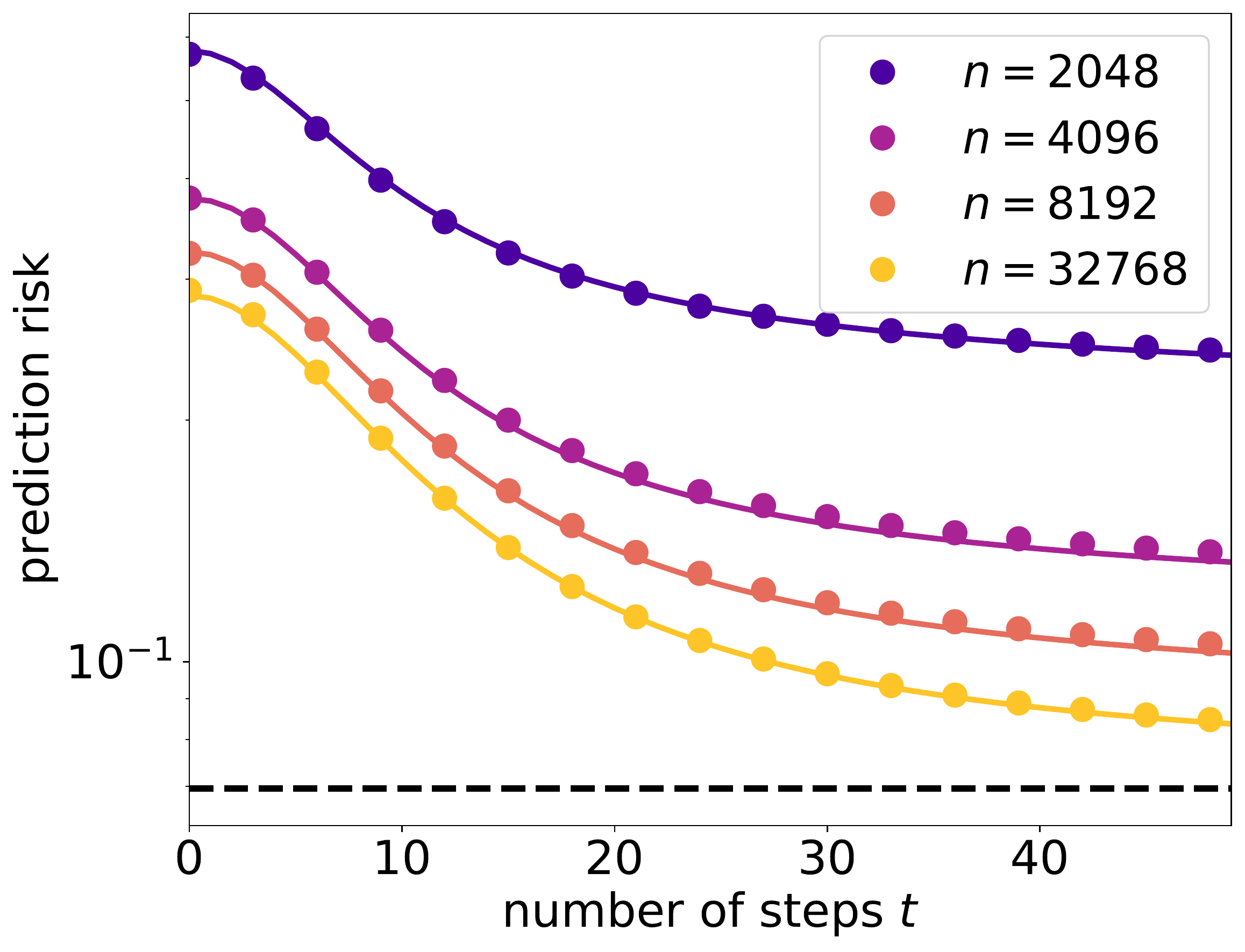}} \\ 
\small (c) Multiple steps (risk vs.~GD steps).  
\end{minipage}  
 \caption{\small Prediction risk of ridge regression on trained features ($\eta=\Theta(1)$): dots represent empirical simulations ($d=512$, averaged over 50 runs) and solid curves are asymptotics predicted by the GET; dashed black line corresponds to the kernel lower bound \eqref{eq:ridge-lower-bound}. 
 (a) $\sigma=\text{tanh}, \sigma^*=\text{SoftPlus}$; we set $\psi_2=2$,  $\lambda=10^{-4}$,  $\sigma_\eps=0.25$. 
 (b) $\sigma=\text{tanh}, \sigma^*=\text{ReLU}$; we set $\psi_1=5$,  $\lambda=10^{-2}$,  $\sigma_\eps=0.1$. 
 (c) $\sigma=\text{ReLU}, \sigma^*=\text{tanh}$; we set $\eta=0.2$, $\psi_2=2$, $\lambda=10^{-3}$.}   
\label{fig:main}  
\end{figure}

\subsection{$\eta=\Theta(\sqrt{N})$: Improvement Over the Kernel Lower Bound} 
\label{subsec:large-lr}
 
Now we take one gradient step with large learning rate $\eta=\Theta(\sqrt{N})$, which matches the asymptotic order of the Frobenius norm of the gradient $\vG_0$ and that of the initialized weight matrix $\vW_0$ as in \eqref{eq:lrlarge}. 
Note that after absorbing the prefactors, this learning rate scaling is analogous to the \textit{maximum update parameterization} \citep{yang2020feature}, which admits a feature learning limit; specifically, the change in each coordinate of the feature vector $[\sigma(\vW^\top\vx)]_i$ is $\tilde{\Theta}_{d,\P}(1)$, which has roughly the same magnitude as its value at initialization. 
  
Due to the large step size, the columns of the updated weight matrix $\vW_1$ are no longer near-orthogonal, which is an important property used in existing analyses of the Gaussian equivalence (e.g., see Proposition~\ref{prop:GET-perturbed} or \cite[Equation~(66)]{hu2020universality}). Indeed, we will see that in this regime, the ridge regression estimator on the trained CK features is no longer ``linear'' and can potentially outperform the kernel lower bound \eqref{eq:ridge-lower-bound} in the proportional limit. 
However, in the absence of GET, it is difficult to derive the precise asymptotics of the CK model. 
As an alternative, in this subsection we establish an \textit{upper bound} on the prediction risk $\cR_1(\lambda)$, which we then compare against the kernel ridge lower bound. 

\paragraph{Existence of ``Good'' Solution.} 
Given the trained first-layer weights $\vW_1$, we first construct a second-layer $\tilde{\va}$ for which the prediction risk can be easily upper-bounded. For a pair of nonlinearities $(\sigma,\sigma^*)$, we introduce a scalar quantity $\tau^*$ which is the optimum of the following minimization problem: 
\begin{align}
    \tau^* 
:=
    \inf_{\kappa\in\R}\, \E_{\xi_1}\left[\big(\sigma^*(\xi_1) - \E_{\xi_2}\sigma(\kappa\xi_1 + \xi_2) \big)^2\right], 
\label{eq:tau*}
\end{align} 
where $\xi_1,\xi_2\iid\cN(0,1)$. We write $\kappa^*$ as an optimal value at which $\tau^*$ is attained (when $\tau^*$ is not achieved by finite $\kappa$, the same argument holds by introducing a small tolerance factor $\epsilon>0$ in $\tau^*$; see Appendix~\ref{app:oracle-estimator}). 
Roughly speaking, $\tau^*$ approximates the prediction risk of a specific student model which takes the form of an average over \textit{subset} of neurons (after one feature learning step); in particular, the first term on the RHS of \eqref{eq:tau*} containing $\sigma^*$ corresponds to the teacher $f^*$, and the second term $\E_{\xi_2}$ represents the constructed student model.  
The following lemma shows that we can find
some $\tilde{\va}$ on the trained CK features whose prediction risk is approximately $\tau^*$, under the additional assumption that the activation function $\sigma$ is bounded. 
\begin{lemm}[Informal]
Given Assumptions \ref{assump:1} and \ref{assump:2}, and assume further that $\sigma$ is bounded. Then after one gradient step on $\vW$ with $\eta=\Theta(\sqrt{N})$, there exists some second-layer coefficients $\tilde{\va}$ such that the constructed student model $\tilde{f}(\vx) = \frac{1}{\sqrt{N}}\tilde{\va}^\top\sigma(\vW_1^\top\vx)$ achieves prediction risk ``close'' to $\tau^*$ when $\psi_1=n/d$ is large.  
\label{lemm:optimal-estimator}
\end{lemm}
It is worth noting that the definition of $\tau^*$ does not involve the specific value of learning rate $\eta$. This is because for any choice of $\eta=\Theta(\sqrt{N})$, due to the Gaussian initialization of $a_i$, we can find a subset of weights that receive a ``good'' learning rate (with high probability) such that the corresponding neurons are useful in learning the teacher model. 
In addition, observe that $\tau^*$ is a simple Gaussian integral which can be numerically or analytically computed (see Appendix~\ref{app:oracle-estimator} for some examples). For instance, when $\sigma=\sigma^*=\text{erf}$, one can easily verify that $\kappa^*=\sqrt{3}$ and $ \tau^*=0$.    

\paragraph{Prediction Risk of Ridge Regression.}
Having established the existence of a ``good'' student model $\tilde{f}$ that achieves prediction risk close to $\tau^*$ defined in \eqref{eq:tau*}, we can now prove an upper bound for the prediction risk of the ridge regression estimator on trained CK features $\cR_1(\lambda)$ in terms of $\tau^*$.  

\begin{theo}
\label{thm:risk-large-lr}
Under the same assumptions as Lemma~\ref{lemm:optimal-estimator}, after one gradient step on $\vW$ with $\eta=\Theta(\sqrt{N})$, there exist constants $C, \psi_1^*>0$ such that for any $n/d>\psi_1^*$, the ridge regression estimator \eqref{eq:ridge-risk} satisfies
\begin{align}
    \cR_1(\lambda) \le 10\tau^* + C\Big(\sqrt{\tau^*}\cdot\sqrt{\tfrac{d}{n}} + \tfrac{d}{n}\Big),   
\end{align} 
with probability $1$ as $n,d,N\to\infty$, if we choose the ridge penalty: $n^{\eps-1}<N^{-1}\lambda<n^{-\eps}$ for some small $\eps>0$. 
\end{theo} 

\begin{wrapfigure}{R}{0.355\textwidth}   
\vspace{-5mm}
\centering  
\includegraphics[width=0.345\textwidth]{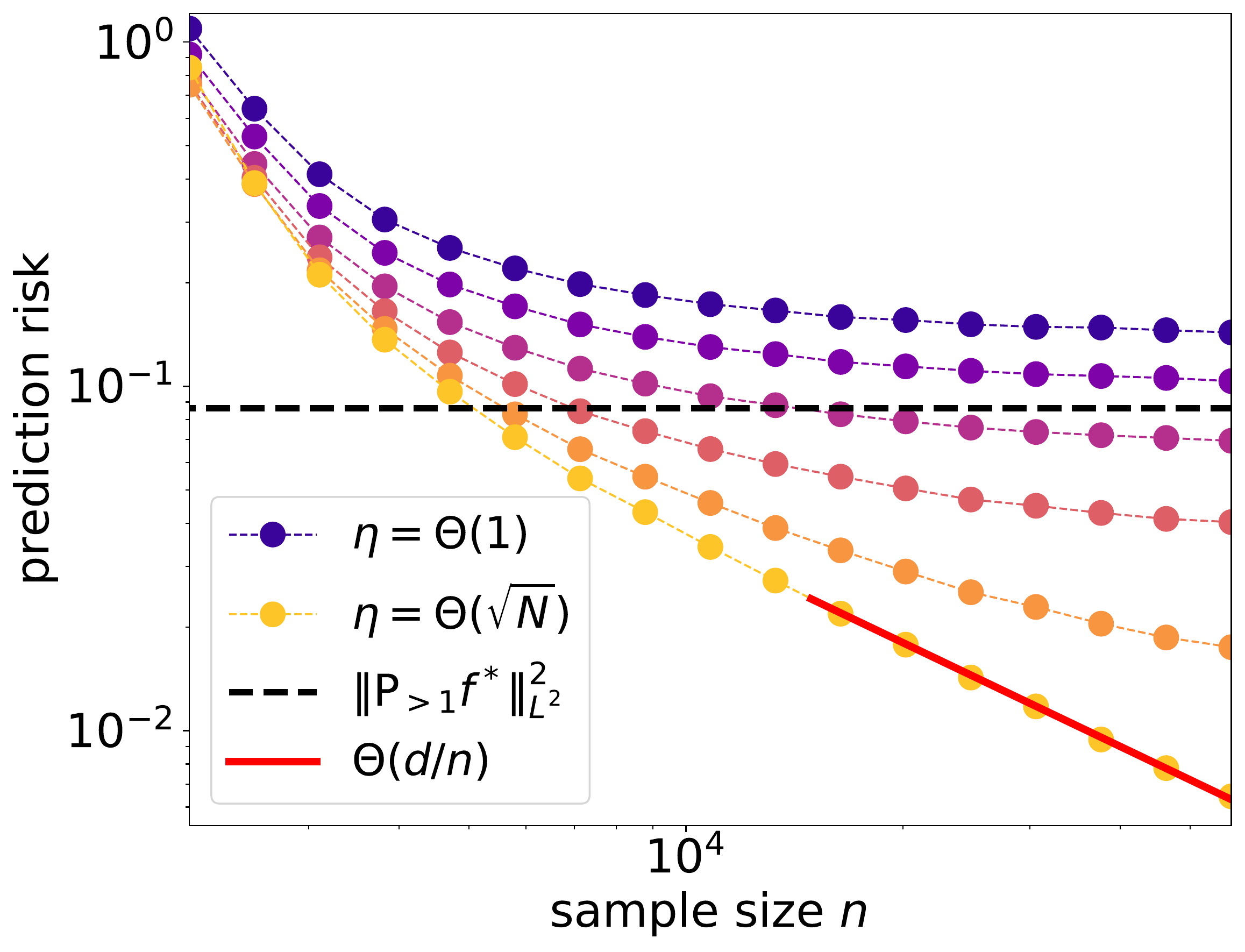}  
\vspace{-2.5mm}  
\caption{\small Prediction risk of ridge regression on CK trained for one step (empirical simulation, $d=1024$): brighter color represents larger step size scaled as $\eta=N^{\alpha}$ for $\alpha\in[0,1/2]$.  We choose $\sigma=\sigma^*=\text{erf}$, $\psi_2=2$, $\lambda=10^{-3}$, and $\sigma_\eps=0.1$. }
\label{fig:large-lr}    
\vspace{-9mm}   
\end{wrapfigure}   
   
While Theorem~\ref{thm:risk-large-lr} does not provide exact expression of the prediction risk,
the upper bound still allows us to compare the prediction risk of CK ridge regression before and after one large gradient step. In particular, if $\norm{\mathsf{P}_{>1} f^*}_{L^2}^2 \ge 10\tau^*$ (the constant $10$ is not optimized), we know that the trained CK can outperform the kernel lower bound \eqref{eq:ridge-lower-bound} (hence also the initialized CK) in the proportional limit, when the ratio $\psi_1=n/d$ is sufficiently large. 
The following corollary provides two examples of this separation (see Figure \ref{fig:large-lr}). 

\begin{coro}
Under the same conditions as Theorem~\ref{thm:risk-large-lr}, there exists some constant $\psi_1^*$ such that for any $\psi_1>\psi_1^*$, the following holds with probability 1 when $n,d,N\to\infty$ proportionally: 
\begin{itemize}[leftmargin=*,topsep=1mm, itemsep=0mm]
    \item For $\sigma=\sigma^*=\mathrm{erf}$,  $\cR_1(\lambda)= \bigO{\nicefrac{d}{n}}$, which vanishes if $\psi_1$ is large. 
    \item For $\sigma=\sigma^*=\mathrm{tanh}$, we have $\cR_1(\lambda)<\norm{\mathsf{P}_{>1} f^*}_{L^2}^2$. 
\end{itemize}
\end{coro}
In the two examples outlined above, training the features by taking one large gradient step on the first-layer parameters can lead to substantial improvement in the performance of the CK model. In fact, the new ridge regression estimator may outperform a wide range of kernel models outlined in Section~\ref{subsec:kernel-lower-bound}. 
However, we emphasize that this separation is only present in specific pairs of $(\sigma,\sigma^*)$ for which $\tau^*$ is small enough. 
In general settings, learning a good representation likely requires more than one gradient step (even if $f^*$ is a simple single-index model).

\section{Discussion and Conclusion}
\label{sec:conclusion}
We investigated how the conjugate kernel of a two-layer neural network \eqref{eq:two-layer-nn} benefits from feature learning in an idealized student-teacher setting, where the first-layer parameters $\vW$ are updated by one gradient descent step on the empirical risk.  
Based on the approximate low-rank property of the gradient matrix, we established a signal-plus-noise decomposition for the updated weight matrix $\vW_1$, and quantified the improvement in the prediction risk of conjugate kernel ridge regression under two different scalings of first-step learning rate $\eta$. 
To the best of our knowledge, this is the first work that rigorously characterizes the precise asymptotics of kernel models (defined by neural networks) in the presence of feature learning. 

We outline a few limitations of our current analysis as well as future directions.  

\begin{itemize}[leftmargin=*,topsep=0.5mm, itemsep=0.36mm]
    \item \textbf{Dependence between $\vW_1$ and $\vX$.} One of our crucial assumptions is that the trained weight matrix $\vW_1$ is \textit{independent} of the data $\tilde{\vX}$ on which the CK is computed. 
    While this does not cover the important scenario where feature learning and kernel evaluation are performed on the same data, our setting is very natural in the analysis of pretrained models or transfer learning, which would be an interesting extension. 
    \item \textbf{Scaling of Learning Rate.} Our findings in Section~\ref{sec:prediction-risk} illustrate that $\eta\!=\!\Theta(1)$ and $\eta\!=\!\Theta(\sqrt{N})$ result in drastically different behavior. One natural question to ask is whether there exists a ``phase transition'' in between the two regimes (see Figure~\ref{fig:large-lr}) that dictates whether the GET holds. Interestingly, \citet{refinetti2021classifying} showed that instead of breaking the near-orthogonality of weight matrix $\vW$ (via large gradient step), one can also introduce sufficiently large low-rank shifts to the input $\vX$ to enable the initial RF estimator to fit a nonlinear $f^*$. Intuitively, this may be due to the ``dual'' relation of $\vX$ and $\vW$ in the CK model. 
    \item \textbf{Rigorous Analysis of CK Spike.} In Section~\ref{subsec:alignment} we put forward a Gaussian equivalence hypothesis on the isolated eigenvalue/eigenvector of the trained CK matrix (see Figure~\ref{fig:CK-spike}); understanding whether and when such property holds is an important research direction.  
\end{itemize}

\bigskip

\subsection*{Acknowledgement}
 
The authors would like to thank (in alphabetical order) Konstantin Donhauser, Zhou Fan, Hong Hu, Masaaki Imaizumi, Ryo Karakida, Bruno Loureiro, Yue M. Lu, Atsushi Nitanda, Sejun Park, Ji Xu, Yiqiao Zhong for discussions and feedback on the manuscript.  

JB was supported by NSERC Grant [2020-06904], CIFAR AI Chairs program, Google Research Scholar Program and Amazon Research Award. 
MAE was supported by NSERC Grant [2019-06167], Connaught New Researcher Award, CIFAR AI Chairs program, and CIFAR AI Catalyst grant. TS was partially supported by JSPS KAKENHI (20H00576) and JST CREST. ZW was supported by NSF Grant DMS-2055340. Part of this work was completed when DW interned at Microsoft Research (hosted by GY). 

\bigskip

{

\fontsize{10}{11}\selectfont     

\bibliography{citation}
\bibliographystyle{alpha}

}

\newpage
{
\renewcommand{\contentsname}{Table of Contents}
\tableofcontents
}
 
\newpage
\appendix

\allowdisplaybreaks

\section{Background and Additional Results} 
 
\subsection{Additional Experiments}
\label{app:experiment}

\begin{figure}[htb!]  
\vspace{-1mm}  
\centering
\begin{minipage}[t]{0.388\linewidth}
\centering 
{\includegraphics[width=1\textwidth]{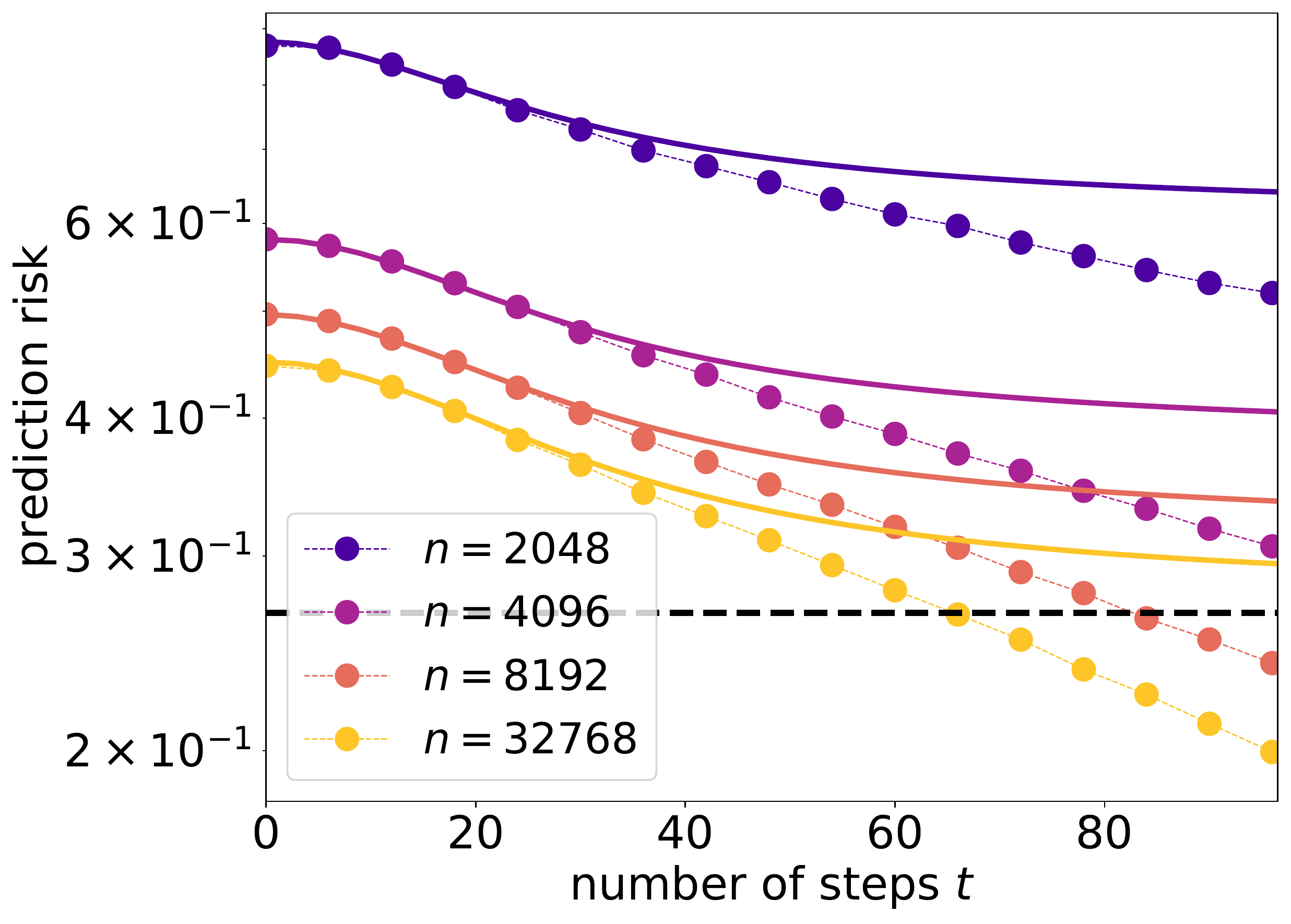}}  \\
\small (a) Failure of GET prediction. 
\end{minipage} 
\begin{minipage}[t]{0.4\linewidth}
\centering 
{\includegraphics[width=1\textwidth]{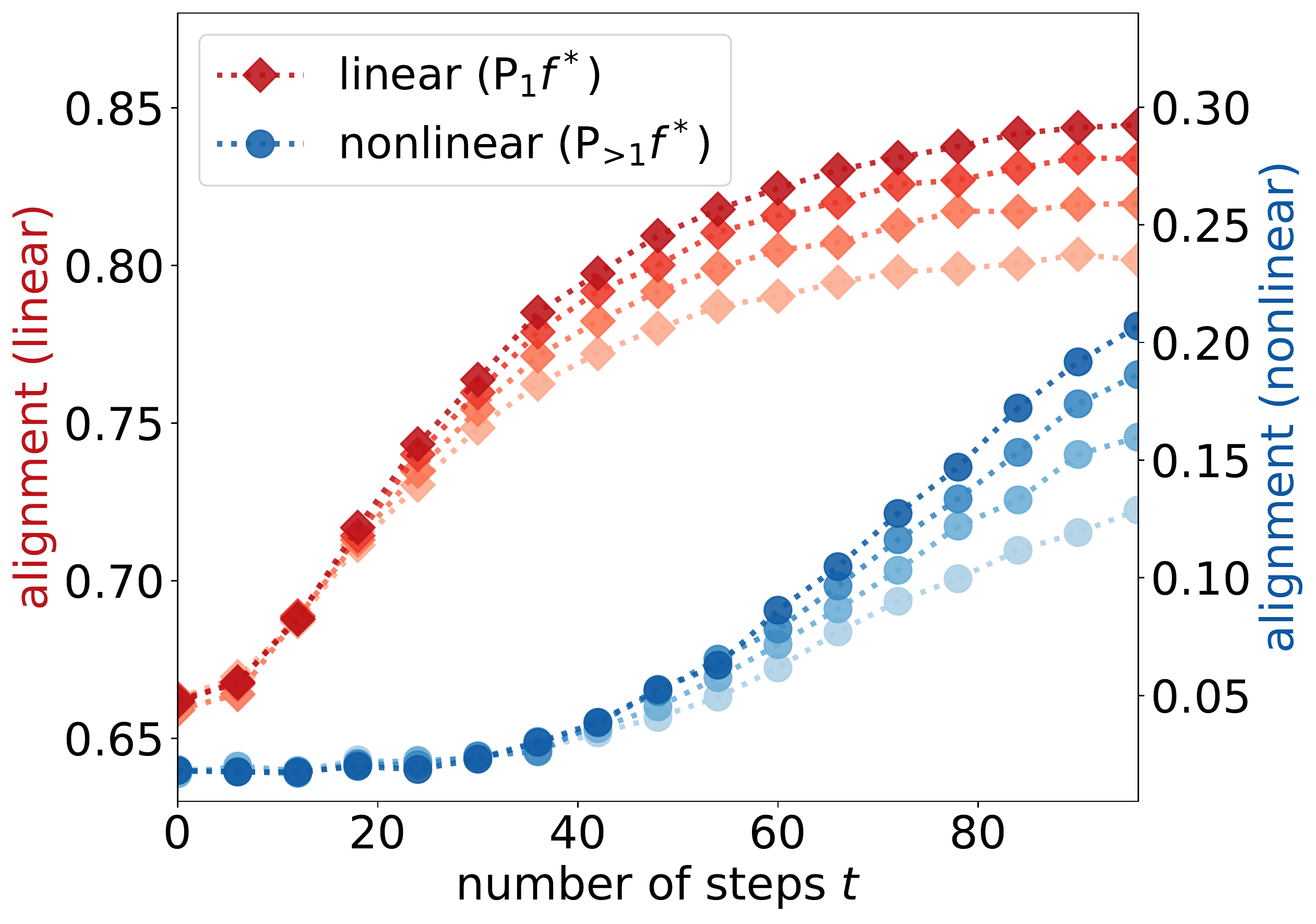}}  \\
\small (b) Alignment with teacher $f^*$.  
\end{minipage}
\caption{\small 
We choose $\sigma=\sigma^*=\text{ReLU}$, and set $\psi_2=2$, $d=512$, $\lambda=10^{-4}$, and $\eta = 0.1$. 
(a) Prediction risk of ridge regression on trained CK features: dots represent empirical simulations (averaged over 50 runs) and solid curves are asymptotic values predicted by the GET; dashed black line corresponds to the kernel ridge lower bound in \eqref{eq:ridge-lower-bound}. 
(b) Alignment between the student $f_\lambda^t$ and the linear (red) and nonlinear (blue) components of the teacher model \eqref{eq:alignment-ft}. Darker colors correspond to larger sample size in $n=\{2^{11}, 2^{12}, 2^{13}, 2^{15}\}$.  }    
\label{fig:GET-failure}  
\end{figure} 

\paragraph{Failure Cases of GET.} 
It is worth noting that Theorem~\ref{thm:GET} does not apply to the setting where $t$ scales with $n,d,N$. 
Because of our mean-field parameterization, the first-layer weight $\vW$ needs to travel sufficiently far away from initialization to achieve small training loss (see Figure~\ref{fig:feature-learning}). Hence in our experimental simulations (where $n,d,N$ are large but finite), as the number of steps $t$ or learning rate $\eta$ increases, we expect the Gaussian equivalence predictions to become inaccurate at some point. This transition is empirically demonstrated in Figure~\ref{fig:GET-failure}(a). Observe that for larger $t$, the GET predictions \textit{overestimate} the test loss; one possible explanation is that the trained kernel can learn nonlinear functions (which we show in Section~\ref{subsec:large-lr} for one gradient step with $\eta=\Theta(\sqrt{N})$ and specific choices of $f^*$), which the GET cannot capture.  

We provide additional empirical evidence on this explanation in Figure~\ref{fig:GET-failure}(b). 
To track the learning of the linear and nonlinear components of $f^*$, we recall the orthogonal decomposition:
\begin{align}
    f^*(\vx) = \underbrace{\mu_0^* + \mu_1^*\langle\vx,\vbeta_*\rangle}_{f^*_{\mathrm{L}}(\vx)} + \underbrace{\textsf{P}_{>1} f^*(\vx)}_{f^*_{\mathrm{NL}}(\vx)}.
\end{align}
Denote the CK ridge regression estimator on the feature map after $t$ gradient steps $\vx\to\sigma(\vW_t^\top\vx)$ as $f_\lambda^t$. We estimate the following alignment quantities (we normalize $f^*_{\mathrm{L}}$ and $f^*_{\mathrm{NL}}$ to have unit $L^2$-norm):
\begin{align}
    \text{\textit{Linear} component:~} \left\langle f^*_{\mathrm{L}}, f_\lambda^t\right\rangle_{L^2(\R^d,\Gamma)}. \quad~
    \text{\textit{Nonlinear} component:~} \left\langle f^*_{\mathrm{NL}}, f_\lambda^t\right\rangle_{L^2(\R^d,\Gamma)}. \quad 
    \label{eq:alignment-ft}
\end{align}
In Figure~\ref{fig:GET-failure}(b), we observe that the student model $f_\lambda^t$ first aligns with the linear component of the teacher model $f^*_{\mathrm{L}}$; on the other hand, when the student model begins to learn the nonlinear component $f^*_{\mathrm{NL}}$ (at $\sim$30 gradient steps), the Gaussian equivalent predictions (Figure~\ref{fig:GET-failure}(a)) overestimate the prediction risk. 

\paragraph{Singular Vector Alignment (Theorem~\ref{thm:alignment}).} In Figure~\ref{fig:appendix}(a), we compute the overlap between the leading eigenvector of $\vW_1$ and the linear component of the teacher model $\vbeta_*$. Observe that the empirical simulations (dots) closely match the analytic predictions of Theorem \ref{thm:alignment} (solid curves). Also, note that increasing the learning rate $\eta$ or the sample size $\psi_1=n/d$ both lead to greater alignment with the teacher model. 

\paragraph{Large Learning Rate (SoftPlus).}
In Figure~\ref{fig:appendix}(b) we repeat the large learning rate experiment in Section~\ref{subsec:large-lr} for a different nonlinearity $\sigma=\sigma^*=\text{SoftPlus}$, for which $\tau^*\approx 0.03>0$, and hence the upper bound in Theorem~\ref{thm:risk-large-lr} is \textit{non-vanishing}. In this case, we observe that the prediction risk of the CK ridge regression model (after one feature learning step) is also non-vanishing even when the step size is large; this indicates that although we do not provide precise asymptotic characterization in Theorem~\ref{thm:risk-large-lr}, the upper-bounding quantity $\tau^*$ in \eqref{eq:tau*} has predictive power on the actual prediction risk. 

\paragraph{Kernel Target Alignment.} 
In Section~\ref{subsec:alignment}, we observed that the trained CK aligns with training labels. Here we provide additional empirical evidence by tracking the \textit{Kernel Target Alignment} (KTA) \citep{cristianini2001kernel} between the CK and training labels during training. Specifically, we compute the following quantity at each gradient step $t$, which takes value between 0 and 1,
\begin{align}
    \mathrm{KTA} = \frac{\langle\mathbf{CK}_t,\vy\vy^\top\rangle}{\norm{\mathbf{CK}_t}_F\norm{\vy}^2}, 
    \label{eq:KTA}
\end{align}  
where $\mathbf{CK}_t$ denotes the CK matrix defined by $\vW_t$. 
Figure~\ref{fig:appendix}(c) shows the KTA for two-layer NN under our mean-field parameterization and also the NTK parameterization (which omits the $\frac{1}{\sqrt{N}}$-prefactor in \eqref{eq:two-layer-nn}). We optimize the first-layer weights $\vW$ until the training loss reaches $10^{-2}$ for both settings, and compute the KTA on the training and test data at gradient step. Observe that the trained CK in the mean-field model aligns with both the training and test labels (purple), whereas the NN in the kernel regime does not exhibit such alignment (orange). 

\begin{figure}[htb!]  
\centering
\begin{minipage}[t]{0.325\linewidth}
\centering 
{\includegraphics[width=0.99\textwidth]{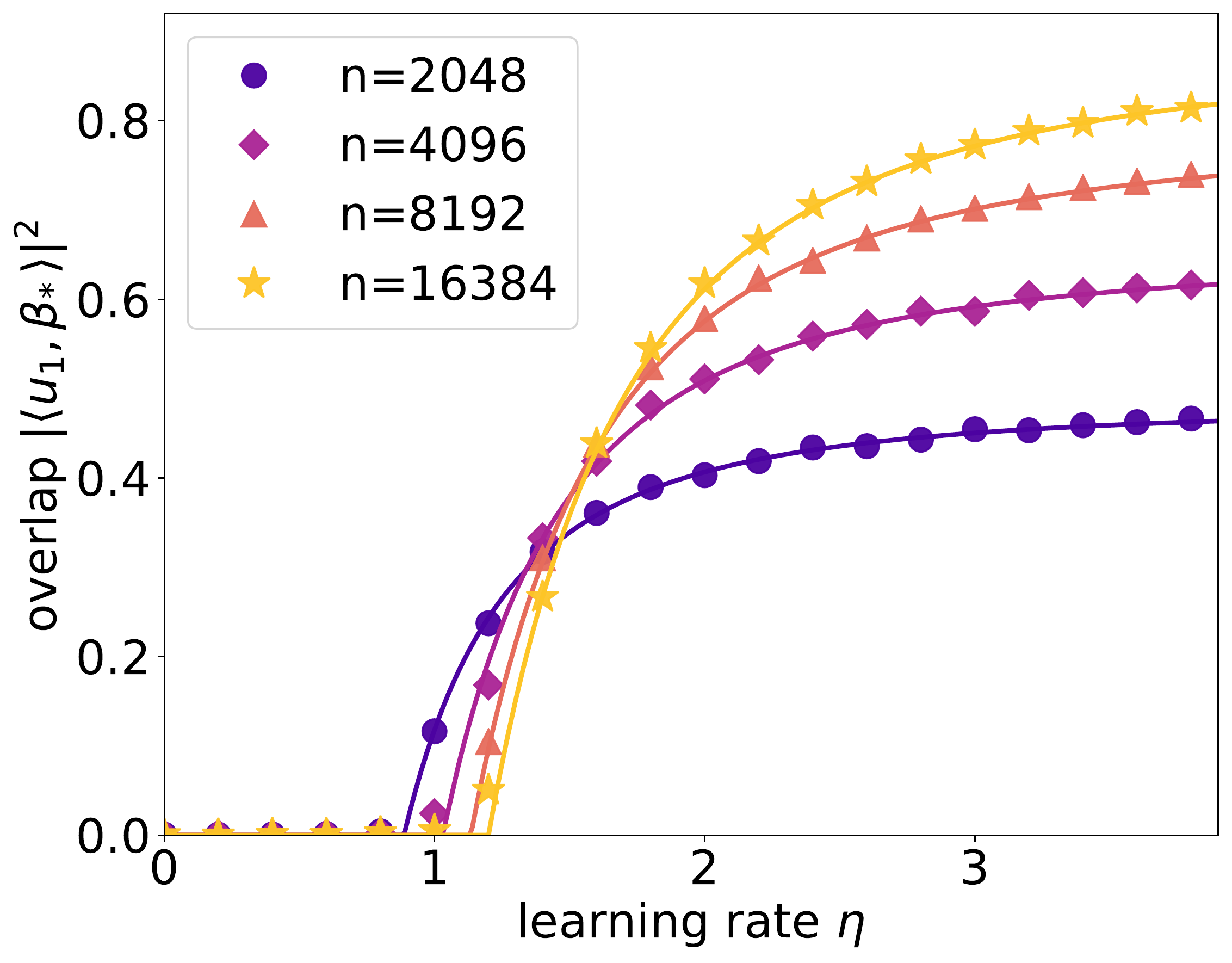}}  \\
\small (a) Alignment of singular vector. 
\end{minipage}
\begin{minipage}[t]{0.33\linewidth}
\centering 
{\includegraphics[width=1\textwidth]{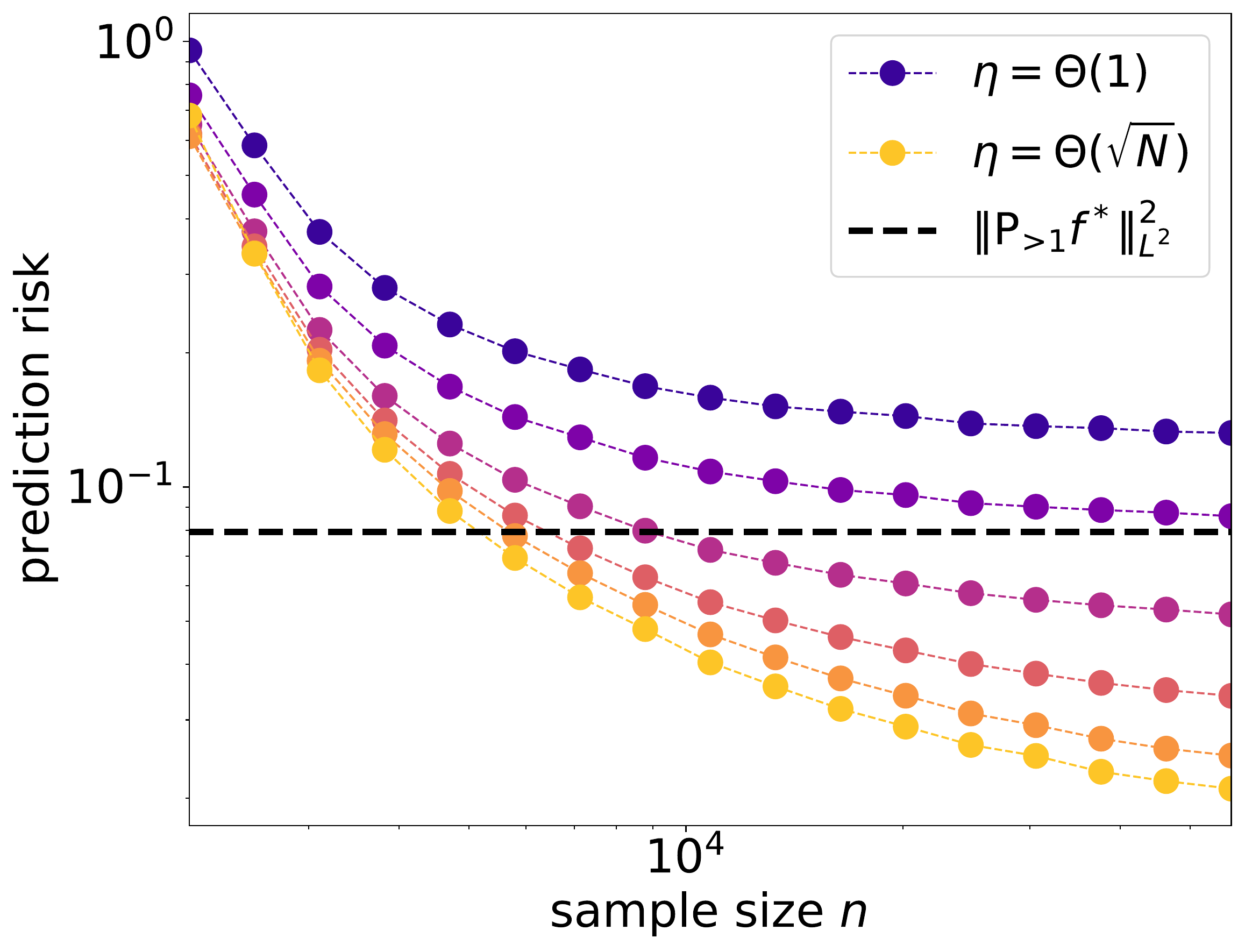}}  \\
\small (b) Prediction risk ($\eta=\Theta(N^{\alpha})$). 
\end{minipage} 
\begin{minipage}[t]{0.33\linewidth} 
\centering 
{\includegraphics[width=1\textwidth]{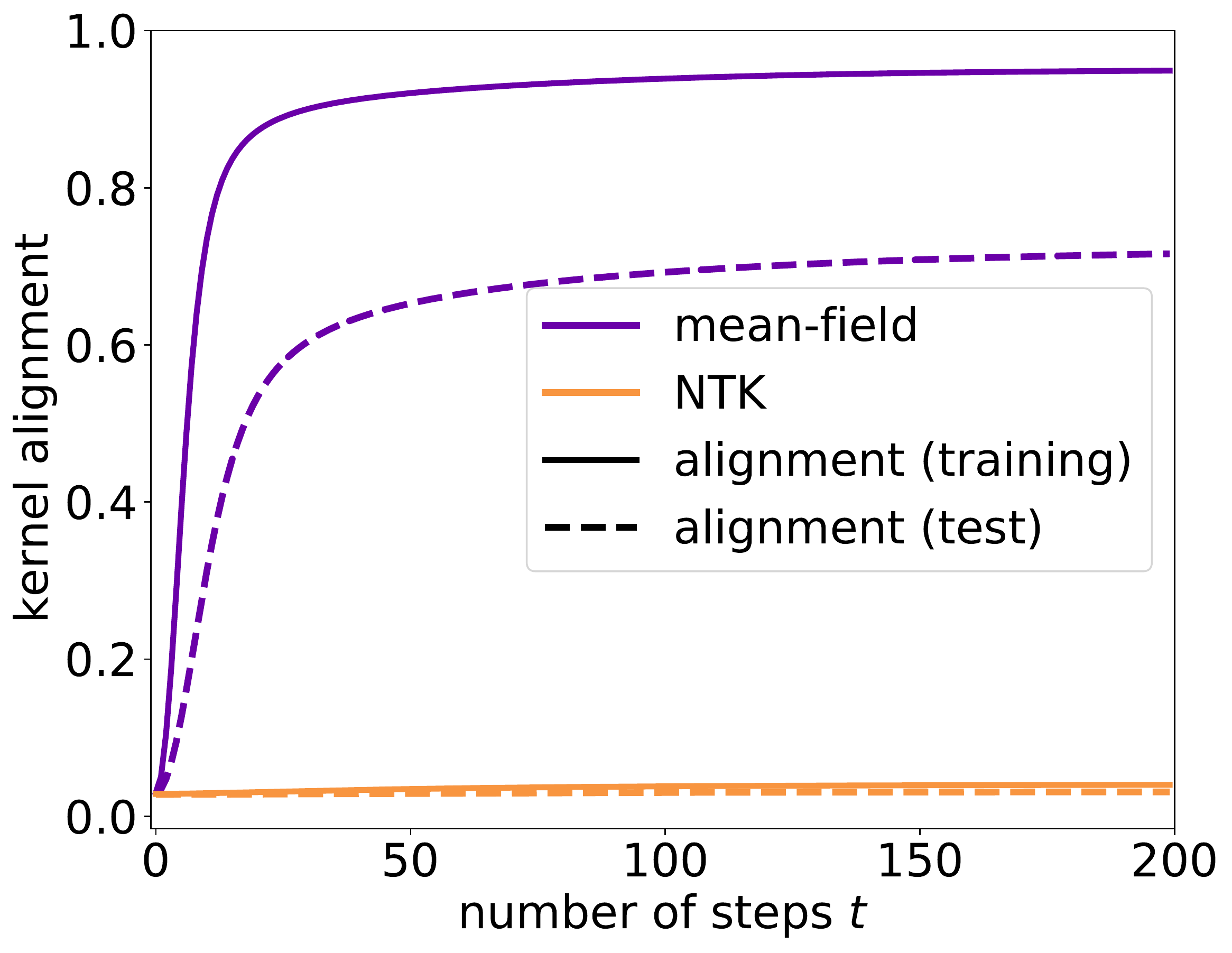}}  \\
\small (c) Kernel-target alignment. 
\end{minipage} 
\caption{\small 
(a) Alignment between leading singular vector of $\vW_1$ and the teacher model:  $\abs{\vu_1^\top\vbeta_*}^2$ vs.~learning rate $\eta=\Theta(1)$. 
Dots represent empirical simulations and solid lines are asymptotic values from Theorem~\ref{thm:alignment}. 
We set $\sigma=\sigma'=\text{tanh}$, $\psi_2=2$, $d=2048$. 
(b) Prediction risks of ridge regression on the trained CK after one gradient step (empirical simulation, $d=1024$): brighter color represents larger step size scaled as $\eta=N^{\alpha}$ for $\alpha\in[0,1/2]$. We choose $\sigma=\sigma^*=\text{SoftPlus}$, $\psi_2=2$, $\lambda=10^{-3}$, and $\sigma_\eps=0.1$.
(c) KTA \eqref{eq:KTA} between the trained CK and the labels (training and test) vs.~the number of gradient steps on weight $\vW$. We set $\psi_1=\psi_2=2$, $\sigma=\text{ReLU}$, $\sigma^*=\text{tanh}$. }    
\label{fig:appendix}  
\end{figure} 

\subsection{Additional Related Works}

\paragraph{The Kernel Regime and Beyond.}
  
The neural tangent kernel (NTK) \cite{jacot2018neural} describes the learning dynamics of wide neural network under specific parameter scaling. 
Such description is based on \textit{linearizing} the NN around its initialization, and the limiting kernel can be computed for various architectures \citep{arora2019exact,yang2020tensor}. Thanks to strong convexity of the kernel objective, global convergence rate guarantees of gradient descent can be established \cite{du2018gradient,ji2019polylogarithmic}. 
As mentioned in Section~\ref{subsec:related}, this first-order Taylor expansion fails to explain the \textit{adaptivity} of NNs; therefore, recent works also analyzed higher-order approximations of the training dynamics \citep{dyer2019asymptotics,huang2020dynamics}. Noticeably, a quadratic model (i.e., second-order approximation) can outperform kernel (NTK) estimators in certain settings \citep{allen2018learning,bai2019beyond}. 

In contrast to the aforementioned local approximations (via Taylor expansion and truncation), the \textit{mean-field} regime (e.g., \citep{nitanda2017stochastic,mei2018mean,chizat2018global}) deals with a different scaling limit under which the evolution of parameters can be described by some partial differential equation (for comparison between regimes see \citet{woodworth2020kernel,geiger2020disentangling}). While the mean-field limit can capture the presence of feature learning \citep{chizat2020implicit,nguyen2021analysis}, \textit{quantitative} guarantees often require additional conditions such as KL regularization \citep{nitanda2022convex,chizat2022mean}.  
Note that our parameterization \eqref{eq:two-layer-nn} mirrors the mean-field scaling, but we circumvent the difficulty of analyzing the nonlinear PDE because only the ``early phase'' (one gradient step) is considered.     

Finally, we highlight two concurrent papers that studied the mean-field dynamics of two-layer NNs (under one-pass SGD) in the high-dimensional asymptotic regime, and showed learnability results for certain target functions. \citet{abbe2022merged} established a separation between NNs and kernel methods in learning ``staircase-like'' functions on hypercube; \citet{veiga2022phase} analyzed how the model width and step size impact the learning of a well-specified two-layer NN teacher model.  

\paragraph{Spectrum of Kernel Random Matrices.}

Kernel matrices in the proportional regime was first analyzed by \cite{el2010spectrum} through Taylor expansion, and later their limiting spectra were fully described by \cite{cheng2013spectrum,do2013spectrum,fan2019spectral}. 
As an extension of kernel random matrices, the CK matrix has also been studied in \cite{pennington2017nonlinear,peche2019note,benigni2019eigenvalue,benigni2022largest} and \cite{louart2018random,fan2020spectra,wang2021deformed}, using the moment method and the Stieltjes transform method, respectively. 
In addition, the spectrum and concentration behavior of the NTK matrix were elaborated in \cite{montanari2020interpolation,fan2020spectra,wang2021deformed}. 
We remark that based on these prior results on the NTK of two-layer NNs, one can check our large learning rate $\eta=\Theta(\sqrt{N})$ satisfies $\sqrt{N}\eta\cdot\lambda_{\max}(\vF)=\bigThetdp{1}$, where $\vF$ is the \textit{Fisher information matrix}; heuristically speaking, this means that the chosen step size is not unreasonably large (under first-order approximation of the landscape).

\subsection{Linearity of Kernel Ridge Regression}
\label{app:kernel-lower-bound}

As previously mentioned, our kernel ridge regression lower bound (Proposition~\ref{prop:ridge-lower-bound}) is a simple combination of existing results, which we briefly outline below.  

\paragraph{Linear Regression on Input.}  
We first discuss the prediction risk of the ridge regression estimator on the input features. 
Recall that under Assumptions \ref{assump:1} and \ref{assump:2}, we may write: $f^*(\vx) = \mu_1^*\langle\vx,\vbeta_*\rangle + \textsf{P}_{>1} f^*(\vx)$. 
Given the ridge regression estimator on the input features: $\hat{\vtheta}_{\text{Lin}} \triangleq (\vX^\top\vX + \lambda n\vI_d)^{-1}\vX^\top\vy$, we have the following bias-variance decomposition, 
\begin{align}
    \cR_{\mathrm{Lin}}(\lambda)
=& 
    \underbrace{\E_{\vx}\left(f^*(\vx) - \vx^\top(\vX^\top\vX + \lambda n\vI_d)^{-1}\vX^\top f^*(\vX)\right)^2}_{\text{Bias}}
    + 
    \underbrace{\sigma_\varepsilon^2 \Trarg{(\vX^\top\vX + \lambda n\vI_d)^{-2}\vX^\top\vX}}_{\text{Variance}} + \littleodp{1}.
\end{align}
Following a similar computation as \cite[Theorem 4.13]{bartlett2021deep} and using the asymptotic formulae in \cite{dobriban2018high,wu2020optimal}, we can derive the following expression,
\begin{align}
    \cR_{\mathrm{Lin}}(\lambda)  
\overset{\P}{\to}
    \frac{\bar{m}'(-\lambda)}{\bar{m}^2(-\lambda)}\cdot\frac{\mu_1^{*2}}{(1+\bar{m}(-\lambda))^2} +  (\sigma_\varepsilon^2 + \mu_2^{*2})\cdot \left(\frac{\bar{m}'(-\lambda)}{\bar{m}^2(-\lambda)} - 1\right) +  \mu_2^{*2},  
    \label{eq:risk-linear}
\end{align} 
where $\bar{m}(-\lambda)>0$ is the Stieltjes transform of the limiting eigenvalue distribution of $\frac{1}{n}\vX\vX^\top$.  
Observe that $\cR_{\mathrm{Lin}}(\lambda)\ge\mu_2^{*2}$. 
In addition, as shown in \cite{dobriban2018high,wu2020optimal}, the optimal ridge regularization and the corresponding prediction risk can be written as  
\begin{align}
    \lambda_{\text{opt}} = \frac{\sigma_\eps^2 + \mu_2^{*2}}{\psi_1\mu_1^{*2}}, \quad
    \cR_{\mathrm{Lin}}(\lambda_{\mathrm{opt}})   \overset{\P}{\to} \frac{\sigma_\eps^2 + \mu_2^{*2}}{\lambda_{\text{opt}} \bar{m}(-\lambda_{\text{opt}})} - \sigma_\eps^2.  
    \label{eq:optimal-risk-linear}
\end{align}

\paragraph{Lower Bound for RF/Kernel Ridge Regression.} 
First note that for RF models \eqref{eq:RF-estimator}, the lower bound $\mu_2^{*2}$ is directly implied by the GET \citep{hu2020universality} under Assumptions~\ref{assump:1} and \ref{assump:2} (see Fact~\ref{fact:GET-lower-bound}). 
For inner-product kernels\footnote{Similar result can also be shown for Euclidean distance kernels following the analysis in \cite[Theorem 2.2]{el2010spectrum}.} in \eqref{eq:kernel-estimator}, if $g:\R\to\R$ is a smooth function in a neighborhood of $0$, then the same lower bound can be obtained from \cite[Theorem 4.13]{bartlett2021deep} (observe that the bias term is lower bounded by $\norm{\mathsf{P}_{>1}f^*}_{L^2}^2$). Finally, for the (first-layer) NTK, the kernel ridge regression estimator is given as
\begin{align}
   \hat{f}_{\text{NTK}}(\vx) = & \vg^\top(\vK+\lambda\vI)^{-1}\vy, \\
    \mathrm{where~} \vg_i =& \frac{1}{Nd}\sum_{k=1}^N\langle\vx,\vx_i\rangle\sigma'(\langle\vx,\vw_k\rangle)\sigma'(\langle\vx_i,\vw_k\rangle),\\ \mathrm{and~} \vK_{ij} = & \frac{1}{Nd}\sum_{k=1}^N\langle\vx_i,\vx_j\rangle\sigma'(\langle\vx_i,\vw_k\rangle)\sigma'(\langle\vx_j,\vw_k\rangle).  
\end{align}
Define the orthogonal decomposition $\sigma'(z) = b_0 + \sigma'_\perp(z)$, where $b_0 = \mu_1 = \E[\sigma'(z)]$, $b_1^2=\E[\sigma'(z)^2] - b_0^2$, for $z\sim\cN(0,1)$.   
Similar to \cite{adlam2020neural,montanari2020interpolation}, we make the following ``lineaized'' substitutions: 
\begin{align}
    \vg \approx \bar{\vg} \triangleq \frac{1}{d}\cdot b_0^2\vX\vx, \quad 
    \vK \approx \bar{\vK} \triangleq \frac{1}{d}\cdot b_0^2 \vX\vX^\top + b_1^2 \vI.  
\end{align} 
The error of this linear approximation has been studied in \cite[Lemma B.8]{montanari2020interpolation} and \cite[Theorem 2.7]{wang2021deformed}, which, together with \cite[Theorem 4.13]{bartlett2021deep},  entail the following equivalence under Assumption \ref{assump:1}, 
\begin{align}
    \cR_{\mathrm{NTK}}(\lambda) = \cR_{\mathrm{Lin}}\left(\frac{\lambda + b_1^2}{b_0^2\psi_1}\right) + \littleodp{1},  
\end{align} 
where $\cR_{\mathrm{Lin}}$ is the prediction risk of the ridge regression estimator on the input features defined in \eqref{eq:risk-linear}. 
Hence, the linear lower bound \eqref{eq:ridge-lower-bound} directly applies; in fact, the prediction risk is lower-bounded by the optimal ridge regression estimator on the input \eqref{eq:optimal-risk-linear}.

\paragraph{Kernel Lower Bound under Polynomial Scaling.} 
For high-dimensional input $\vx$ uniform on sphere or hypercube, \citet{ghorbani2019linearized,mei2021generalization} showed that RF and kernel ridge estimators can learn at most a degree-$k$ polynomial when $n=\bigO{d^{k+1-\varepsilon}}$; for the proportional scaling, this implies our lower bound $\norm{\mathsf{P}_{>1} f^*}_{L^2}^2$ (but under different input assumptions). \citet{donhauser2021rotational} provided a similar result for more general data distributions and a class of rotation invariant kernels based on power series expansion, but the dependence on $k$ is not sharp enough to cover the linear lower bound in Proposition~\ref{prop:ridge-lower-bound}.    

\bigskip 
\section{Proof for the Weight Matrix}
 
\subsection{Norm Control of Gradient Matrix}
\label{app:gradient-norm}

In this section we establish a few important properties of the gradient matrix defined in \eqref{eq:gradient-step-MSE}. For simplicity, we derive the results for the squared loss, but one may check that the same characterization holds for any differentiable loss function $\ell$ with Lipschitz derivative (w.r.t.~both arguments), such as the logistic loss, for which the gradient update on the first layer at step $t$ is given by 
\begin{align}
\label{eq:gradient-step}
\vW_{t+1} - \vW_{t} 
&=
    -\eta \sqrt{N}\cdot\sum_{i=1}^n \frac{1}{n}\partial_2\ell\Big(y_i, \sigma(\vx^\top\vW_t)\frac{\va}{\sqrt{N}}\Big)\cdot\vx_i\Big(\sigma'(\vx_i^\top\vW_t)\odot\frac{\va^\top}{\sqrt{N}}\Big),   
\end{align} 
where $\partial_2$ refers to the partial derivative w.r.t.~the second argument in $\ell$. In the following, for any $t\in\N$ and $i\in[N]$, we will always use $\vw_i^t$ with both subscript and superscript to indicate the $i$-th column of the weight matrix $\vW_t$ at time step $t$.

For our later analysis, a key quantity to control is the entry-wise $2$-$\infty$ matrix norm defined as $$\|\vM\|_{2,\infty}:=\max_{1\le i\le N}\|\vm_i\|,$$for any matrix $\vM\in\R^{d\times N}$ with the $i$-th column $\vm_i\in\R^d$ and $1\le i\le N$. It is straightforward to verify that 
\begin{equation}\label{eq:2_infty_norm}
    \|\vM\|_{2,\infty}\le \|\vM\|\le \|\vM\|_F\le \sqrt{N} \|\vM\|_{2,\infty}.
\end{equation}
In addition, for the Hadamard product with rank-1 matrix, we have the following property. 
\begin{fact}\label{fact:Hadamard1}
	For $\vm\in\R^m, \vn\in\R^n, \vM\in\R^{m\times n}$, we can write
	$\vm\vn^\top\odot\vM=\diag (\vm)\vM\diag (\vn)$, and
	\[\norm{\vm\vn^\top\odot\vM}\le \norm{\diag (\vm)}\cdot\|\vM\|\cdot\norm{\diag(\vn)}= \norm{\vm}_{\infty}\|\vM\|\norm{\vn}_{\infty}.\] 
\end{fact}

\subsubsection{Norm Bounds for the First Gradient Step} 
We begin with the first gradient step. Recall the definition of the gradient matrix under the squared loss (we omit the learning rate $\eta$ and prefactor $\sqrt{N}$): 
\begin{align} 
    \vG_0 
&=  
    -\frac{1}{n}\vX^\top \left[\left(\frac{1}{\sqrt{N}}\left(\frac{1}{\sqrt{N}}\sigma(\vX\vW_0)\va - \vy\right)\va^\top\right) \odot \sigma'(\vX\vW_0)\right] \label{eq:decomposition_gradient}\\
&= 
    \underbrace{\frac{1}{n}\cdot\frac{\mu_1}{\sqrt{N}}\vX^\top\vy\va^\top}_{\vA} + 
    \underbrace{\frac{1}{n}\cdot\frac{1}{\sqrt{N}}\vX^\top\left(\vy\va^\top\odot\sigma'_\perp(\vX\vW_0)\right)}_{\vB} 
    - \underbrace{\frac{1}{n}\cdot\frac{1}{N}\vX^\top\left(\sigma(\vX\vW_0)\va\va^\top\odot\sigma'(\vX\vW_0)\right)}_{\vC},
\end{align} 
where we utilized the orthogonal decomposition: $\sigma'(z) = \mu_1 + \sigma'_\perp(z)$. Due to Stein's lemma, we know that $\E[z\sigma(z)] = \E[\sigma'(z)] = \mu_1$, and hence $\E[\sigma'_\perp(z)]=0$ for $z\sim\cN(0,1)$. 
The following lemma provides norm control for the above decomposition. 
\begin{lemm}
Assume that $f^*\in L^2(\R^d,\Gamma)$, and both $f^*$ and $\sigma$ are Lipschitz functions. Then 
\begin{itemize}[topsep=0mm, itemsep=0.1mm]
    \item[(i)] $\E\|\vA\|_{2,\infty}\le \E\|\vA\|\le \E\|\vA\|_F\le C\sqrt{\frac{d}{nN}+\frac{1}{N}}$, 
    \item[(iii)]  $\E\norm{\vC}\le\E\norm{\vC}_F \le \frac{C}{N}\sqrt{1+\frac{d}{n}}$.
\end{itemize}
Furthermore, we have the following probability bounds. 
\begin{itemize}[topsep=0mm, itemsep=0.1mm]
    \item[(i)] 
    $\Parg{\|\vA\|_F\ge C\left(\sqrt{\frac{d}{nN}}+\sqrt{\frac{1}{N}}\right)}\le C'\left(e^{-cn}+e^{-cN}\right),$ \\
    $\Parg{\|\vA\|_F\le C\sqrt{\frac{d}{nN}}}\le C' \left(e^{-c\min\left\{\frac{nd^2}{(n^2+d^2)},\frac{nd}{n+d}\right\}}+e^{-cN}+e^{-cn}\right),$ and \\
    $\Parg{\|\vA\|_{2,\infty}\ge C\frac{(\sqrt{n}+\sqrt{d})\log n}{N\sqrt{n}}}\le C'\left(e^{-c\frac{(\sqrt{n}+\sqrt{d})^2}{n}\log^2 n}+e^{-cn}+Ne^{-c\log^2 n}\right).$ 
    \item[(ii)] 
    $\Parg{\|\vB\|\ge C\frac{(\sqrt{n}+\sqrt{d})(\sqrt{n}+\sqrt{N})\log^2 n}{n\sqrt{Nd}}}\le C'\left((n+N)e^{-c\log^2n}+e^{-(\sqrt{n}+\sqrt{d})^2}+e^{-cN}+e^{-cd}\right),$\\
    $\Parg{\norm{\vB}_F\ge C\frac{\sqrt{n} + \sqrt{d}}{\sqrt{nN}}} \le C'\left(e^{-cn}+e^{-cN}+e^{-c(\sqrt{n}+\sqrt{d})^2}\right).$ 
    \item[(iii)]  $\P\left(\norm{\vC}_F \ge C\frac{(\sqrt{d}+\sqrt{n})\log n \log N}{\sqrt{n}N} \right) \le 
    C'\left(Ne^{-cN}+ne^{-cd}+ne^{-c\log^2 n} + Ne^{-c\log^2 N}\right)$.  
\end{itemize}
Here all constants $C,C',c>0$ only depend on $\lambda_\sigma$, $\mu_1$, $\sigma_\eps$ and $\|f^*\|_{L^2(\R^d,\Gamma)}$.
\label{lemm:gradient-norm}    
\end{lemm}
\begin{remark}
In Lemma~\ref{lemm:gradient-norm}, we do not use the proportional scaling in Assumption~\ref{assump:1} to simplify the expressions. This is because the dependence on $n,d,N$ needs to be tracked separately in some of our calculations. 
\end{remark}

\begin{proof} 
We analyze the three matrices of interest separately. 

\paragraph{Part $(i)$.} We first upper-bound $\|\vA\|_F^2$. Notice that
\begin{align}
    \frac{n\sqrt{N}}{\mu_1}\|\vA\|_F\le& \|\vX^\top f^*(\vX)\va^\top\|_F+\|\vX^\top\vvarepsilon \va^\top\|_F\nonumber\\
    \le &\|\vX\|(\|f^*(\vX)\|+\|\vvarepsilon\|)\|\va\|.\label{eq:vA_control}
\end{align}
We know that Gaussian random matrices and vectors satisfy
\begin{align}
    \E\|\vvarepsilon\|^2=\sigma_{\eps}^2n,&\quad\E\|f^*(\vX)\|^2=n\|f^*\|_{L^2(\R^d,\Gamma)}^2,\label{eq:expecation_square_1}\\
    \E\|\va\|^2=1,&\quad\E\|\vX\|^2\le C_0(n+d),\label{eq:expecation_square_2}
\end{align}
where the last inequality is from \cite[Exercise 4.6.2]{vershynin2018high}. Based on Cauchy-Schwarz inequality, we can employ \eqref{eq:expecation_square_1} and \eqref{eq:expecation_square_2} to obtain
\begin{align*}
    \E\|\vA\|_{2,\infty}\le \E\|\vA\|\le \E\|\vA\|_F\le C_1\sqrt{\frac{d}{nN}+\frac{1}{N}},
\end{align*}
where constant $C_1>0$ only depends on $\mu_1$, $\sigma_\eps$ and $\|f^*\|_{L^2(\R^d,\Gamma)}$. As for the probability bound, we use the Lipschitz concentration property (e.g., see \cite[Theorem 5.2.2]{vershynin2018high}) of $\|\va\|$, $\|\vvarepsilon\|$ and $\|f^*(\vX)\|$, and apply \cite[Corollary 7.3.3]{vershynin2018high} for $\|\vX\|$ to obtain
\begin{align}
    \Parg{\|\vvarepsilon\|\ge \sigma_\eps\sqrt{n}}\le 2e^{-cn},\quad \Parg{\left|\|\va\|-1\right|\ge \frac{1}{2}}\le & 2e^{-cN},\label{eq:gaussian_norm}\\
    \Parg{\left|\|f^*(\vX)\|- \|f^*\|_{L^2(\R^d,\Gamma)}\sqrt{n}\right|\ge\frac{1}{2}\|f^*\|_{L^2(\R^d,\Gamma)}\sqrt{n} }\le & 2e^{-cn},\label{eq:f*X_norm}\\
    \Parg{\|\vX\|\ge \sqrt{n}+\sqrt{d}+t}\le & 2e^{-ct^2},\label{eq:gaussian_matrix_norm}
\end{align}
for any $t\ge 0$. Hence, from \eqref{eq:vA_control}, we arrive at
\[\Parg{\|\vA\|_F\ge \sqrt{\frac{d}{nN}}+\sqrt{\frac{1}{N}}+t}\le 4\left(e^{-cn}+e^{-cN}+e^{-ct^2nN}\right).\]
Note that the same probability bounds also applies to $\|\vA\|$ and $\|\vA\|_{2,\infty}$. Thus, we may take $t=\sqrt{\frac{1}{N}}$ to obtain the desired result. 
Now we provide lower bounds for $\vvarepsilon^\top\vX\vX^\top\vvarepsilon $ and $f^*(\vX)^\top\vX\vX^\top\vvarepsilon$. First, we define events $\cA_1$, $\cA_2$ and $\cA_3$ by
\[\cA_1:=\left\{\left|\Trarg{\vX\vX^\top}-nd\right|\le \frac{nd}{2}\right\},\quad\cA_2:=\left\{\|\vX\|\le \sqrt{d}+2\sqrt{n}\right\},\]
\[ \cA_3:=\left\{\|f^*(\vX)\|\le\frac{1}{2}\|f^*\|_{L^2(\R^d,\Gamma)}\sqrt{n}\right\}.\]
We know that $\Parg{\cA_1},\Parg{\cA_2},\Parg{\cA_3}\ge 1-2e^{-cn}$ by Bernstein's inequality, the Lipschitz Gaussian concentration inequality, and \cite[Corollary 7.3.3]{vershynin2018high}. Condition on $\cA_1\cap \cA_2$, by the Hanson-Wright inequality,  
\[\Parg{\vvarepsilon^\top\vX\vX^\top\vvarepsilon\le \frac{\sigma_\eps^2}{2}nd-t~\Big|~\cA_1\cap \cA_2}\le 2e^{-c\min\left\{\frac{t^2}{n(n^2+d^2)},\frac{t}{n+d}\right\}}.\]
Choosing $t=\sigma_\eps^2nd/4$, we have
\begin{equation}\label{eq:lower_bound_A1}
    \Parg{\vvarepsilon^\top\vX\vX^\top\vvarepsilon\le \frac{\sigma_\eps^2}{8}nd}\le 2e^{-c\min\left\{\frac{nd^2}{(n^2+d^2)},\frac{nd}{n+d}\right\}}+4e^{-cn}.
\end{equation} 
Similarly, by the general Hoeffding inequality, one can easily see that
\[\Parg{\left|f^*(\vX)^\top\vX\vX^\top\vvarepsilon\right|\ge t~\Big|~\cA_2\cap\cA_3}\le 2e^{-\frac{ct^2}{n(n^2+d^2)}}.\]
Thus, again, by \eqref{eq:gaussian_norm}, we obtain
\begin{equation}\label{eq:lower_bound_A2}
    \Parg{\left|f^*(\vX)^\top\vX\vX^\top\vvarepsilon\right|\ge \frac{\sigma_\eps^2}{32}nd}\le 2e^{-\frac{cnd^2}{(n^2+d^2)}}+2e^{-cN}+4d^{-cn}.
\end{equation}
Also, since the operator norm has the following lower bound
\begin{align*}
    \|\vX^\top\vy\va^\top\|=\|\va\|\|\vX^\top\vy\|= \|\va\|\left(\vy^\top\vX\vX^\top\vy\right)^{1/2}
    \ge \|\va\|\left(\vvarepsilon^\top\vX\vX^\top\vvarepsilon+2f^*(\vX)^\top\vX\vX^\top\vvarepsilon\right)^{1/2},
\end{align*}
by \eqref{eq:gaussian_norm}, \eqref{eq:lower_bound_A1} and \eqref{eq:lower_bound_A2}, we arrive at
\[\Parg{\frac{n^2N}{\mu_1^2}\|\vA\|^2\le\frac{\sigma_\eps^2}{16}nd}\le 16 \left(e^{-c\min\left\{\frac{nd^2}{(n^2+d^2)},\frac{nd}{n+d}\right\}}+e^{-cN}+e^{-cn}\right).\]
As for the last inequality on $\|\vA\|_{2,\infty}$, by definition we know that
\[\|\vA\|_{2,\infty}\le \frac{\mu_1}{n\sqrt{N}}\|\vX\|\left(\|f^*(\vX)\|+\|\vvarepsilon\|\right)\|\va\|_{\infty}.\]
The desired result can be obtained from the tail bound on the sup-norm of Gaussian random vector, $\Parg{\|\va\|_{\infty}\le t/\sqrt{N}}\ge 1-2Ne^{-ct^2}$, in combination with \eqref{eq:gaussian_norm}, \eqref{eq:f*X_norm} and \eqref{eq:gaussian_matrix_norm}.

\paragraph{Part $(ii)$.} As a result of Fact~\ref{fact:Hadamard1}, we have
\begin{equation}\label{eq:vB_control}
    \|\vB\|\le  \frac{1}{n\sqrt{N}}\|\vX\|\|\va\|_{\infty}(\|f^*(\vX)\|_{\infty}+\|\vvarepsilon\|_{\infty})\|\sigma'_\perp(\vX\vW_0)\|.
\end{equation}

We first control the operator norm of the random feature matrix $\sigma'_\perp(\vX\vW_0)$. 
Since $\sigma'_\perp$ is centered, \cite[Lemma D.4]{fan2020spectra} implies that 
\begin{equation}
    \Parg{\|\sigma_\perp'(\vX\vW_0)\|\ge C(\sqrt{n}+\sqrt{N})\lambda_\sigma B, \cA_B}\le 2e^{-cN},
\end{equation} where event $\cA_B$ is defined by 
\[\cA_B:=\left\{\|\vW_0\|\le B, \sum_{i=1}^N(\|\vw^0_i\|^2-1)^2\le B^2\right\},\]given any constant $B>0$.
Hence, we have
\begin{align}
    \Parg{\|\sigma_\perp'(\vX\vW_0)\|\ge C(\sqrt{n}+\sqrt{N})\lambda_\sigma B}\le 2e^{-cN}+\Parg{ \cA_B^c}. \label{eq:CK_norm}
\end{align}
Next, we estimate the failure probability of event $\cA_B^c$. By Bernstein's inequality, for any $t\ge 0,$ we have 
\begin{equation}
    \Parg{|\|\vw^0_1\|^2-1|^2\ge t^2}\le 2e^{-cd\min\{t^2,t\}},
\end{equation}
where we write $\vw_1^0$ as the first column of $\vW_0$ (and similarly for all $\vw_i^0$). Following the proof of Proposition 3.3 in \cite{fan2020spectra}, we can obtain that
\begin{align}
    \Parg{\sum_{i=1}^N\left(\|\vw^0_i\|^2-1\right)^2\ge 4t^2}\le 2e^{N\log 5-cd\min\{t^2,t\}}\label{eq:sum_vW_columns},
\end{align}
for any $t\ge0$. Besides, inequality \eqref{eq:gaussian_matrix_norm} implies that for any $t\ge 0$, 
\begin{equation}
    \Parg{\|\vW_0\|\le c'\sqrt{\frac{N}{d}}}\ge 1-2e^{-cd}.
\end{equation}
By choosing $t=c'\sqrt{\frac{N}{d}}$ in \eqref{eq:sum_vW_columns} and $B:=c'\sqrt{\frac{N}{d}}$ for sufficient large $c'>0$, we can claim that there exists sufficient large constant $c>0$ such that 
\[\Parg{\cA_B^c}\le 2e^{-cd}+2e^{-cN}.\]
Combining \eqref{eq:CK_norm} and the above inequality, we have
\begin{equation}\label{eq:CK_norm1}
    \Parg{\|\sigma_\perp'(\vX\vW_0)\|\ge C(\sqrt{n}+\sqrt{N})\sqrt{\frac{N}{d}}}\le 4e^{-cN}+2e^{-cd}.
\end{equation}
In addition, the following tail bound is due to property of (sub-)Gaussian random variables: 
\begin{align}
    \Parg{\|\va\|_{\infty}\le t_1/\sqrt{N}}\ge 1-2Ne^{-ct_1^2},\quad\Parg{\|\vvarepsilon\|_{\infty}\le t_2}\ge 1-2ne^{-ct_2^2},\label{eq:gaussian_max}
\end{align}for any $t_1,t_2\ge 0$.
Because $f^*$ is Lipschitz, $f^*(\vX)$ is a sub-Gaussian random vector with similar tail bound
\[\Parg{\|f^*(\vX)\|_{\infty}\le t_2}\ge 1-2ne^{-ct_2^2}.\]
Let $t_1=t_2=\log n$. Applying all these three tail bounds \eqref{eq:CK_norm1} and \eqref{eq:gaussian_matrix_norm}, \eqref{eq:vB_control} gives us the first part of the probability bound in $(ii)$. 
As for the second part, following the observation
\[\|\vB\|_F\le \frac{\mu_1}{n\sqrt{N}}\|\vX\|\|\vy\va^\top\odot\sigma'_\perp(\vX\vW_0)\|_F\le \frac{\mu_1\lambda_\sigma}{n\sqrt{N}}\|\vX\|\|\vy\|\|\va\|,\]
we can adopt \eqref{eq:gaussian_norm}, \eqref{eq:f*X_norm} and \eqref{eq:gaussian_matrix_norm} to conclude the second probability bound. 

\paragraph{Part $(iii)$.}
Finally, we analyze the lower-order term $\vC$. 
Recall the definitions $\vX=[\tilde{\vX},\ldots,\vx_n]^\top$, $\vW_0=[\vw^0_1,\ldots,\vw^0_N]$ and $\va=[a_1,\ldots,a_N]^\top$. 
We first observe that
\begin{align}
    &\E\|\sigma(\vX\vW_0)\va\va^\top\odot\sigma'(\vX\vW_0)\|_F^2\le \lambda_\sigma^2\sum_{j=1}^n\sum_{k=1}^N\E\left(\sum_{i=1}^N a_ia_k\sigma(\vx_j^\top\vw_i^0)\right)^2,\nonumber\\
    =& \lambda_\sigma^2\sum_{j=1}^n\sum_{k=1}^N\sum_{i,l=1}^N\E\left[ a_la_ia_k^2\sigma(\vx_j^\top\vw_i^0)\sigma(\vx_j^\top\vw_l^0)\right]=\lambda_\sigma^2\sum_{j=1}^n\sum_{k=1}^N\sum_{i=1}^N\E\left[ a_i^2a_k^2\sigma(\vx_j^\top\vw_i^0)^2\right],\nonumber\\
    \le &\frac{C'}{N^2}\sum_{j=1}^n\sum_{k=1}^N\sum_{i=1}^N\E\left[\sigma(\vx_j^\top\vw_i^0)^2\right]\le C''n,\label{eq:part3_last}
\end{align}
where the last inequality can be deduced by
\begin{align*}
    &\E[\sigma(\vx^\top\vw)^2]=\E_{\vw}\left[\E_{\vx}[\sigma(\vx^\top\vw)^2]\right]=\E_{\vw}\left[\E_{z}[\sigma(\|\vw\|z)^2]\right]\\
    \le& 2\E_{\vw}\left[\E_{z}(\sigma(\|\vw\|z)-\sigma(z))^2\right]+2\E_{\vw}\left[\E_{z}[\sigma(z)^2]\right]\\
     \le& 2\lambda^2_\sigma\E_{\vw}\left[(\|\vw\|-1)^2\right]+2\E_{\vw}\left[\E_{z}[\sigma(z)^2]\right]\le 4\lambda^2_\sigma+\E_{z}[\sigma(z)^2],
\end{align*}
which is uniformly bounded by a constant. Therefore, by \eqref{eq:expecation_square_2} and \eqref{eq:part3_last}, we get
\[\E\|\vC\|]\le \E\norm{\vC}_F\le \frac{1}{nN}\E[\|\vX\|^2]^{\frac{1}{2}}\E[\|\sigma(\vX\vW_0)\va\va^\top\odot\sigma'(\vX\vW_0)\|_F^2]^{\frac{1}{2}}\le \frac{C_3}{N}\sqrt{1+\frac{d}{n}}.\]
As for the tail control, because of Fact \ref{fact:Hadamard1}, we consider the following upper-bound, 
\begin{equation}\label{eq:vC_control}
    \norm{\vC}\le \norm{\vC}_F \le \frac{1}{nN}\norm{\vX}\norm{\sigma(\vX\vW_0)\va}_\infty\norm{\va}_\infty\norm{\sigma'(\vX\vW_0)}_F
\le
    \frac{\lambda_\sigma}{\sqrt{nN}}\norm{\vX}\norm{\sigma(\vX\vW_0)\va}_\infty\norm{\va}_\infty, 
\end{equation}
where the last inequality is due to $|\sigma'|$ being upper-bounded by $\lambda_\sigma$. 

To control $\norm{\sigma(\vX\vW_0)\va}_\infty$, 
note that since $\va$ is centered by Assumption~\ref{assump:1}, we can apply Bernstein inequality for $\va$ and $\vW$ conditioned on the event $\cM:=\left\{\left|\|\vx_i\|/\sqrt{d}-1\right|\le \nicefrac{1}{2},~i\in [n]\right\}$. Conventionally, we denote $\norm{\cdot}_{\psi_2}$ as the sub-Gaussian norm. Since $\big\|\|\vx_i\|-\sqrt{d}\big\|_{\psi_2}$ is bounded by some absolute constant (\cite[Theorem 3.1.1]{vershynin2018high}), we know that
\begin{equation}
    \Parg{\cM}\ge 1-ne^{-cd}.
\end{equation}

Notice that for any $j\in [n]$,
$\sigma(\vx_j^\top\vW_0)\va = \sum_{i=1}^N a_i\sigma(\vx_j^\top\vw_i^0)$ 
is the sum of $N$ independent and centered sub-Exponential random variables, where, in terms of \cite[Lemma D.5]{fan2020spectra}, the sub-Exponential norm $\norm{\cdot}_{\psi_1}$ of each term is bounded by the sub-Gaussian norm of the entries as follows, 
\[\|a_i\sigma(\vx_j^\top\vw_i^0) \|_{\psi_1}\le \|a_i\|_{\psi_2}\|\sigma(\vx_j^\top\vw_i^0) \|_{\psi_2}\le \frac{C\lambda_\sigma}{\sqrt{N}}\frac{\|\vx_j\|}{\sqrt{d}}\le \frac{3C\lambda_\sigma}{2\sqrt{N}},\]
for some absolute constant $C$. Thus, by Bernstein inequality \cite[Theorem 2.8.1]{vershynin2018high}, for each $j\in [n]$,
\begin{align}
    \Parg{ |\sigma(\vx_j^\top\vW_0)\va|\ge \log n}\le 2e^{-c(\log n)^2}.
\end{align}
Then we take the union over all $\vx_j$ and obtain $\norm{\sigma(\vX\vW_0)\va}_\infty\le \log n$ with probability at least $1-2ne^{-c(\log n)^2}$. Hence, by \eqref{eq:gaussian_matrix_norm}, \eqref{eq:gaussian_max} and \eqref{eq:vC_control}, we get
\begin{align}
    \Parg{\|\vC\|_F\ge \frac{ (\sqrt{d}+\sqrt{n}+t)\log n \log N}{\sqrt{n}N}} \le 2ne^{-c(\log n)^2} + 2N e^{-c(\log N)^2} + ne^{-cd} + 2e^{-ct^2} + 2Ne^{-cN}. 
\end{align}
Part $(iii)$ is established by choosing $t=\sqrt{d}$. 
This concludes the proof of the lemma. 
\end{proof}

Proposition~\ref{thm:W1-W0} is a direct consequence of the above norm bounds.
\begin{proofof}[Proposition~\ref{thm:W1-W0}]
Notice that $\vG_0-\vA=\vB+\vC$. In the proportional regime, by Lemma~\ref{lemm:gradient-norm}, there exist universal constants $C,c>0$ such that
\[\Parg{\|\vG_0-\vA\|\le C\frac{\log^2 n}{n}}\ge 1-ne^{-c\log^2 n}.\]
On the other hand, part $(i)$ in Lemma~\ref{lemm:gradient-norm} implies that 
\[\Parg{\|\vA\|\ge \frac{C}{\sqrt{n}}}\ge 1-e^{-cn},\]
for some constant $c,C>0$. Here we used the fact $\|\vA\|=\|\vA\|_F$ because it is a rank-one matrix. Conditioning on the two events stated above, we have
\begin{align}
    \|\vG_0-\vA\|\le \frac{C}{\sqrt{n}}\frac{\log^2 n}{\sqrt{n}}\le \frac{\log^2 n}{\sqrt{n}}\|\vA\|\le \frac{\log^2 n}{\sqrt{n}}\left(\|\vG_0\|+\|\vG_0-\vA\|\right).
\end{align}
As long as $n$ is sufficiently large such that $\frac{\log^2 n}{\sqrt{n}}<\frac{1}{2}$, we can obtain
\[\Parg{\|\vG_0-\vA\|\le \frac{2\log^2 n}{\sqrt{n}} \|\vG_0\|}\ge 1-ne^{-c\log^2 n}- e^{-cn},\]
which completes the proof.

\end{proofof}  
 
\subsubsection{Decomposition of Matrix A}  
Using the orthogonal decomposition \eqref{eq:orthogonal-decomposition}, we can further decompose the rank-1 matrix $\vA$ as follows
\begin{align}
    \vA = \underbrace{\frac{1}{n}\cdot\frac{\mu_1\mu_1^*}{\sqrt{N}}\vX^\top\vX\vbeta_*\va^\top}_{\vA_1} + 
    \underbrace{\frac{1}{n}\cdot\frac{\mu_1}{\sqrt{N}}\vX^\top\left(\mu_0^*\boldsymbol{1} + \textsf{P}_{>1} f^*(\vX)+\vvarepsilon\right)\va^\top}_{\vA_2}, \label{eq:A_1+A_2}
\end{align}
where we denote $\textsf{P}_{>1} f^*(\vX):=[\textsf{P}_{>1} f^*(\vx_1),\ldots,\textsf{P}_{>1} f^*(\vx_n)]^\top\in\R^n$. Similar to the previous Lemma~\ref{lemm:gradient-norm}, we have the following norm bound.  
\begin{lemm}
\label{lemm:gradient-norm2}
Assume that target function $f^*\in L^4(\R^d,\Gamma)$ is a Lipschitz function. We have  
\begin{itemize}[topsep=0mm, itemsep=0.1mm]
    \item[(i)] $\E\|\vA_1\|_F\le \frac{C}{\sqrt{N}}\left(1+\frac{d}{n}\right)$ and $\Parg{\|\vA_1\|_F\ge C\left(\frac{1}{\sqrt{N}}+\frac{d}{n\sqrt{N}}\right)}\le C'(e^{-cN}+e^{-cn})$;   
    \item[(ii)] $\E\norm{\vA_2}_F \le C\sqrt{\frac{d}{Nn}}$, and when $n\ge d$,
    \begin{equation}\label{eq:A_2_F}
        \Parg{\|\vA_2\|_F^2\ge \frac{Cd}{nN}}\le C'(e^{-c\sqrt{n}}+e^{-cN}+ne^{-cd}+d^{-1}),
    \end{equation}
\end{itemize} 
for some constants $C,C',c>0$ that only depend on $\mu_1$, $\sigma_\eps$ and $f^*$. 
\end{lemm}
\begin{proof} 
For simplicity, we denote $\textsf{P}_{>1} f^*(\vx)$ by $f^*_{\mathrm{NL}}(\vx)$ and $\textsf{P}_{>1} f^*(\vX) $ by $\vf^*_{\mathrm{NL}}\in\R^n$.
\paragraph{Part $(i)$. }
The expectation follows from \eqref{eq:expecation_square_2} and the following inequality,
\[\|\vA_1\|_F\le\frac{\mu_1\mu_1^*}{n\sqrt{N}} \|\vX\|^2\|\vbeta_*\|\|\va\|=\frac{\mu_1\mu_1^*}{n\sqrt{N}} \|\vX\|^2 \|\va\|.\]
The probability bound also follows from the same argument as Lemma~\ref{lemm:gradient-norm}. 
 
\paragraph{Part $(ii)$.} 
Following the proof of part $(i)$ in Lemma \ref{lemm:gradient-norm}, we can further decompose $\norm{\vA}_F$ into
\begin{equation}
   \|\vA_2\|_F\le \frac{\mu_1}{n\sqrt{N}}\|\vX\|\|\va\|\left(\mu_0^*\sqrt{n}+\|\vf^*_{\mathrm{NL}}\|+\|\vvarepsilon\|\right).
\end{equation}
Since $\textsf{P}_{>1} f^*$ is a Lipschitz function as well, we can again apply the Lipschitz concentration \eqref{eq:f*X_norm}. 
Hence, combining \eqref{eq:gaussian_matrix_norm}, \eqref{eq:gaussian_norm} and \eqref{eq:f*X_norm}, one can conclude the bound on the expectation of $\|\vA_2\|_F$. 

For the tail bound, we consider matrices
\[
\vA_2':=\frac{\mu_1}{n\sqrt{N}}\vX^\top\vf^*_{\mathrm{NL}}\va^\top, ~\vA_2'':=\frac{\mu_1}{n\sqrt{N}}\vX^\top\vvarepsilon\va^\top, ~\vA_2''':=\frac{\mu_0^*\mu_1}{n\sqrt{N}}\vX^\top\boldsymbol{1}\va^\top,
\]  
whose squared Frobenius norms are given by
\begin{align}
  \|\vA_2'\|_F^2=&\frac{\mu_1^2}{n^2N}\va^\top\va  \vf^{*\top}_{\mathrm{NL}}\vX\vX^\top\vf^*_{\mathrm{NL}},\quad 
  \|\vA_2''\|_F^2= \frac{\mu_1^2}{n^2N}\va^\top\va\vvarepsilon^\top\vX\vX^\top\vvarepsilon, \quad
  \|\vA_2'''\|_F^2=\frac{\mu_0^{*2}\mu_1^2}{n^2N}\va^\top\va\boldsymbol{1}^\top\vX\vX^\top\boldsymbol{1}.
\end{align}
Recall that \eqref{eq:gaussian_norm} implies 
\begin{equation}\label{eq:a_norm}
  \Parg{\|\va\|^2\ge 4}\le  2e^{-cN}.
\end{equation}

Let us first address $\vA_2''$. 
Due to \eqref{eq:a_norm}, it suffices to control $\vvarepsilon^\top\vX\vX^\top\vvarepsilon$, whose expectation with respect to $\vvarepsilon$ is $\sigma^2_\eps\Tr(\vX\vX^\top)$, and $\E[\Tr(\vX\vX^\top)]=nd$. Recalling the Lipschitz Gaussian concentration for $\|\vX\|_F$ and \eqref{eq:gaussian_norm}, we know that for some constant $c>0$, $\Parg{\cA_{\eps}}\ge 1-4e^{-cd}$, where $\cA_{\eps}:=\{\|\vX\|_F\le \sqrt{nd},\|\vX\|\le \sqrt{n}+\sqrt{d}\}$. This directly implies that 
\[\|\vX\vX^\top\|_F\le \|\vX\|_F\|\vX\|\le \sqrt{nd}\left(\sqrt{n}+\sqrt{d}\right),\]
conditioned on event $\cA_{\eps}$. Thus, the Hanson-Wright inequality (Theorem 6.2.1 \cite{vershynin2018high}) indicates that 
\begin{align}
      &\Parg{\vvarepsilon^\top\vX\vX^\top\vvarepsilon\ge t+ 4\sigma^2_{\eps} nd}\le \Parg{\vvarepsilon^\top\vX\vX^\top\vvarepsilon\ge t+ \sigma^2_{\eps} nd~\Big|~ \cA_{\eps}}+\Parg{\cA_{\eps}^c}\\
      \le & \Parg{\left|\vvarepsilon^\top\vX\vX^\top\vvarepsilon-\sigma^2_{\eps}\|\vX\|_F^2\right|\ge t~\Big|~ \cA_{\eps}}+ 4e^{-cd}
      \le 2e^{-c\min\left\{\frac{t^2}{nd(n+d)},\frac{t}{\left(\sqrt{d}+\sqrt{n}\right)^2}\right\}}+4e^{-cd}.
\end{align}
Thus, by choosing $t=nd$ and employing \eqref{eq:a_norm}, we have 
\begin{equation}\label{eq:A_2''}
      \Parg{\|\vA_2''\|_F^2\ge \frac{Cd}{nN}}\le 6e^{-cd}+2e^{-cN}, 
\end{equation}
where we simplified the expression using the assumption that $n\ge d$. 

Next we analyze $\|\vA_2'\|_F^2$. Notice that $\E[f^*_{\mathrm{NL}}(\vx_1)]=0$ and $\E[\vx_1f^*_{\mathrm{NL}}(\vx_1)]=\mathbf{0}$. Since $\vf^*_{\mathrm{NL}}$ is a random vector with independent mean-zero sub-Gaussian coordinates and $f^*\in L^2(\R^d,\Gamma)$, we know that
\begin{equation}\label{eq:f_perp_concentration}
    \Parg{\left|\frac{1}{n}\|\vf^*_{\mathrm{NL}}\|^2-\|f^*_{\mathrm{NL}}\|^2_{L^2(\R^d,\Gamma)}\right|\le n^{-1/4}}\ge 1-2e^{-c\sqrt{n}}.
\end{equation} 
We can further decompose $\|\vA_2'\|_F^2$ into two parts:
  \begin{equation}
      \|\vA_2'\|_F^2=\frac{\mu_1^2}{n^2N}\|\va\|^2\underbrace{\sum_{i\neq j}^n\vx_i^\top\vx_jf^*_{\mathrm{NL}}(\vx_i)f^*_{\mathrm{NL}}(\vx_j)}_{J_1}+\frac{\mu_1^2}{n^2N}\|\va\|^2\underbrace{\sum_{i=1}^n\|\vx_i\|^2(f^*_{\mathrm{NL}}(\vx_i))^2}_{J_2}.
  \end{equation}
Since $\E[\vx_if^*_{\mathrm{NL}}(\vx_i)]=\mathbf{0}$ for $1\le i\le n$, we deduce that $\E[J_1]=0$ and
\begin{align}
    \Var(J_1)=&\sum_{i\neq j}^n\E\left[(\vx_i^\top\vx_j)^2\left(f^*_{\mathrm{NL}}(\vx_i)f^*_{\mathrm{NL}}(\vx_j)\right)^2\right]+2\sum_{i\neq j\neq k}^n\E\left[\vx_i^\top\vx_j\vx_k^\top\vx_jf^*_{\mathrm{NL}}(\vx_i)f^*_{\mathrm{NL}}(\vx_k)\left(f^*_{\mathrm{NL}}(\vx_j)\right)^2\right]\\
    &+\sum_{i\neq j\neq k\neq l}^n\E\left[\vx_i^\top\vx_j\vx_k^\top\vx_lf^*_{\mathrm{NL}}(\vx_i)f^*_{\mathrm{NL}}(\vx_k)f^*_{\mathrm{NL}}(\vx_l)f^*_{\mathrm{NL}}(\vx_j)\right]\\
    =& \sum_{i\neq j}^n\E\left[(\vx_i^\top\vx_j)^2\left(f^*_{\mathrm{NL}}(\vx_i)f^*_{\mathrm{NL}}(\vx_j)\right)^2\right]\le n^2\E[(\vx_1^\top\vx_2)^4]^{1/2}\|f^*_{\mathrm{NL}}\|^4_{L^4(\R^d,\Gamma)}.
    \label{eq:J_1}
\end{align}
On the other hand, Bernstein's inequality \cite[Theorem 2.8.1]{vershynin2018high} indicates that for all $1\le i\neq j\le n$,
\begin{equation}\label{eq:x_ix_j}
      \Parg{|\vx_i^\top\vx_j|\ge t},~\Parg{|\|\vx_i\|^2-d|\ge t}\le 2e^{-c\min\left\{\frac{t^2}{d},t\right\}},
\end{equation}
which yields the sub-exponential condition for $\vx_1^\top\vx_2$. Thus, based on \cite[Theorem 2.3]{boucheron2013concentration}, we can obtain moment bounds for $\vx_1^\top\vx_2$, namely $\E[(\vx_1^\top\vx_2)^4]\lesssim d^2$, whence $ \Var(J_1)\le Cn^2d$ for some constant $C>0$. 
By Chebyshev's inequality, we deduce that 
\begin{equation}\label{eq:J_1_cheby}
    \Parg{\Big|\frac{\mu_1^2}{n^2N}\sum_{i\neq j}^n\vx_i^\top\vx_jf^*_{\mathrm{NL}}(\vx_i)f^*_{\mathrm{NL}}(\vx_j)\Big|>t}\le \frac{Cd}{n^2N^2t^2},
\end{equation}for all $t>0$. As for $J_2$, we apply \eqref{eq:f_perp_concentration} and \eqref{eq:x_ix_j} to all $\|\vx_i\|$. Letting $t=d$ in \eqref{eq:x_ix_j} and taking union bounds for all $1\le i\le n$, we obtain
\begin{equation}
    \Parg{\|\vx_i\|^2\le 2d,~\forall 1\le i\le n}\ge 1-2ne^{-cd}.
    \label{eq:xi-norm-union}
\end{equation}
Hence, the above equation and \eqref{eq:f_perp_concentration} lead the following bound
\begin{equation}\label{eq:J_2}
    \Parg{\frac{\mu_1^2}{n^2N} \sum_{i=1}^n\|\vx_i\|^2(f^*_{\mathrm{NL}}(\vx_i))^2\le \frac{Cd}{nN}}\ge 1-2e^{-c\sqrt{n}}-ne^{-cd}.
\end{equation} Therefore, by letting $t=\frac{d}{nN}$ in \eqref{eq:J_1_cheby} and combining \eqref{eq:a_norm} and \eqref{eq:J_2}, we can conclude that
\begin{equation}
    \Parg{\|\vA_2'\|_F^2\ge  \frac{Cd}{nN}}\le 2e^{-c\sqrt{n}}+ne^{-cd}+2e^{-cN}+\frac{c}{d},
    \label{eq:A2''-tail}
\end{equation}
for some constant $C,c>0$. 

Finally for $\vA_2'''$, we may employ a similar decomposition as $\vA_2'$,
\begin{align}
    \norm{\vA_2'''}_F^2 = \frac{\mu_0^{*2}\mu_1^2}{n^2N}\|\va\|^2\underbrace{\sum_{i\neq j}^n\vx_i^\top\vx_j}_{K_1}+\frac{\mu_0^{*2}\mu_1^2}{n^2N}\|\va\|^2\underbrace{\sum_{i=1}^n\|\vx_i\|^2}_{K_2}.
\end{align}
Following the same computation as \eqref{eq:J_1} and \eqref{eq:J_1_cheby}, we may control the tail of $K_1$ via Chebyshev's inequality due to $\E[K_1] = 0$, $\Var(K_1)\lesssim n^2d$. Whereas the bound on $K_2$ follows from \eqref{eq:a_norm} and the union bound \eqref{eq:xi-norm-union}. We omit the details for this part.
Combining these estimates, we know that $\norm{\vA_2'''}_F^2$ also obeys the same tail bound as \eqref{eq:A2''-tail}. 
The proof of \eqref{eq:A_2_F} is completed by combining the above calculations. 

\end{proof}

\subsubsection{Multiple Gradient Steps} 
Finally, we show via induction that the estimates in Lemma~\ref{lemm:gradient-norm} still hold after $t$ gradient steps with learning $\eta=\Theta(1)$ for fixed $t\in\N$ (note that we do not scale the number of steps $t$ jointly with $n,d,N$). For this lemma, we directly consider the proportional limit, that is, we do not keep track of the exact constants and dependence on $n,d,N$ separately for simplicity. 
 
\begin{lemm}
\label{lemm:gradient-norm-multi}
Under Assumption~\ref{assump:1}, given any fixed $t\in\N$ and learning rate $\eta=\Theta(1)$, the weight matrix after $t$ gradient steps $\vW_t$ defined in \eqref{eq:gradient-step-MSE} satisfies: 
\begin{align}
    \Parg{\norm{\vW_{t}-\vW_0}\ge C} &\le \Exp{-cN}; \\
    \Parg{\norm{\vW_{t}-\vW_0}_{2,\infty} \ge \frac{C\log^2 N}{\sqrt{N}}} &\le \Exp{-c\log^2 N}; \\
    \Parg{\norm{\vW_{t}-\vW_0}_F\ge C} &\le \Exp{-cN},  
\end{align}  
for some positive constants $c,C$. 
\end{lemm}
\begin{proof}
For the induction hypothesis, we assume that (under the proportional scaling in Assumption~\ref{assump:1}) after $t$ gradient steps with learning rate $\eta=\Theta(1)$, the weight matrix satisfies 
$\Parg{\norm{\vW_{t}-\vW_0}\ge C} \le \Exp{-cN}$, 
$\Parg{\norm{\vW_{t}-\vW_0}_{2,\infty} \ge \frac{C\log^2 N}{\sqrt{N}}} \le \Exp{-c\log^2 N}$, 
$\Parg{\norm{\vW_{t}-\vW_0}_F\ge C} \le \Exp{-cN}$. Our goal is to show that the same high-probability statements also hold for $\vW_{t+1}$  (for some different constants $c',C'$). 

We first control the difference in the prediction of the trained neural network compared to the initialized model. 
Following the same argument as \cite[Setion 6.6.1]{oymak2019towards}, we know that
\begin{align}
    \norm{f_t(\vX)} \le \norm{f_0(\vX)} + \norm{f_t(\vX) - f_0(\vX)} \lesssim \norm{f_0(\vX)} + \frac{1}{\sqrt{N}}\norm{\va}\norm{\vX}\norm{\vW_t - \vW_0}_F. 
\label{eq:ft_diff}
\end{align}
Note that $\norm{\vW_t - \vW_0}_F = \cO(1)$ with high probability due to the induction hypothesis. 
We now compute the next gradient update $\vG_t$ (we drop the learning rate $\eta=\Theta(1)$).    
\begin{align}
    \vG_t  
&=
    -\frac{1}{n}\vX^\top \left[\left(\frac{1}{\sqrt{N}}\left(\frac{1}{\sqrt{N}}\sigma(\vX\vW_t)\va - \vy\right)\va^\top\right) \odot \sigma'(\vX\vW_t)\right] \\
&= 
    \underbrace{\frac{1}{n}\cdot\frac{\mu_1}{\sqrt{N}}\vX^\top\left(\vy - f_t(\vX)\right)\va^\top}_{\vA^t} + 
    \underbrace{\frac{1}{n}\cdot\frac{1}{\sqrt{N}}\vX^\top\left(\left(\vy - f_t(\vX)\right)\va^\top\odot\sigma'_\perp(\vX\vW_t)\right)}_{\vB^t}. 
\end{align}

For $\vA^t$, following the same argument as Lemma~\ref{lemm:gradient-norm}, we have
\begin{align}
    \norm{\vA^t} = \norm{\vA^t}_F &\lesssim \frac{1}{n\sqrt{N}}\norm{\vX}(\norm{\vy}+\norm{f_t(\vX)})\norm{\va}. \\
    \norm{\vA^t}_{2,\infty} &\lesssim \frac{1}{n\sqrt{N}}\norm{\vX}(\norm{\vy}+\norm{f_t(\vX)})\norm{\va}_{\infty}. 
\end{align}
Now recall that $\norm{f_0(\vX)} \le \frac{1}{\sqrt{N}}\norm{\va}\norm{\sigma(\vX\vW_0)}$. 
Combining the norm control of $\va$ in \eqref{eq:gaussian_norm}, \eqref{eq:gaussian_max}, the norm control of $\vy$ due to \eqref{eq:f*X_norm} and \eqref{eq:gaussian_max}, the operator norm bound on $\vX$ and $\norm{\sigma(\vX\vW_0)}$ given in \eqref{eq:gaussian_matrix_norm} and \eqref{eq:CK_norm1} (where we applied \cite[Lemma D.4]{fan2020spectra} to the matrix $\sigma(\vX\vW_0)$, since $\sigma$ is centered), and the upper bound on $\norm{f_t(\vX)}$ given in \eqref{eq:ft_diff}, we arrive at
\begin{align}
    \Parg{\norm{\vA^t} \ge \frac{C'}{\sqrt{N}}} \le \Exp{-c'N}, \quad 
    \Parg{\norm{\vA^t}_{2,\infty} \ge \frac{C'\log N}{\sqrt{N}}} \le \Exp{-c'\log^2 N},
\end{align} 
for large enough $N$ and constants $c',C'>0$.  
Similarly for $\vB^t$, we have
\begin{align}
    \norm{\vB^t}_{2,\infty}&\le\norm{\vB^t} \lesssim \frac{1}{n\sqrt{N}}\norm{\vX}(\norm{\vy}_\infty + \norm{f_t(\vX)}_{\infty})\norm{\va}_{\infty}\norm{\sigma'_\perp(\vX\vW_t)}, \\
    \norm{\vB^t}_F &\lesssim \frac{1}{n\sqrt{N}}\norm{\vX}(\norm{\vy} + \norm{f_t(\vX)})\norm{\va}\max_{i,j}\abs{\sigma'_\perp(\vX\vW_t)}_{i,j}. 
\end{align} 
Again using \cite[Setion 6.6.1]{oymak2019towards}, we have
\begin{align}
\norm{\sigma'_\perp(\vX\vW_t)} 
\le&~ \norm{\sigma'_\perp(\vX\vW_0)} + 
\norm{\sigma'_\perp(\vX\vW_t) - \sigma'_\perp(\vX\vW_0)} \\
\lesssim&~ 
\norm{\sigma'_\perp(\vX\vW_0)} + \norm{\vX}\norm{\vW_t - \vW_0}_F. 
\end{align} 
Thanks to the norm control of $\vX$ in \eqref{eq:gaussian_matrix_norm}, the norm control of $\vy$ and $f_t(\vX)$ from \eqref{eq:gaussian_norm}, \eqref{eq:f*X_norm},  \eqref{eq:gaussian_max}, and \eqref{eq:ft_diff}, and the operator norm of the CK matrix given in \eqref{eq:CK_norm1}, we get 
\begin{align} 
    \Parg{\norm{\vB^t} \ge \frac{C' \log^2 N}{N}} \le \Exp{-c'\log^2 N}, \quad
    \Parg{\norm{\vB^t}_F \ge \frac{C'}{\sqrt{N}}} \le \Exp{-c'N},   
\end{align}
for large enough $N$. 
Consequently, given the induction hypothesis, we know that for the next time step $(t+1)$ with learning rate $\eta=\Theta(1)$, there exist some constants $c',C'$ such that
\begin{align}
    \Parg{\norm{\vW_{t+1}-\vW_0}\ge C'} &\le \Exp{-c'N}; \\
    \Parg{\norm{\vW_{t+1}-\vW_0}_{2,\infty} \ge \frac{C'\log^2 N}{\sqrt{N}}} &\le \Exp{-c'\log^2 N}; \\
    \Parg{\norm{\vW_{t+1}-\vW_0}_F\ge C'} &\le \Exp{-c'N}. 
\end{align}
Note that constants $c',C'>0$ may depend on $t$ but do not rely on $n,d,N$. 
We therefore conclude that the above statements hold true for any finite $t\in\N$. 

\end{proof} 

\subsection{Calculation of Alignment with Target Function}
In this section we prove Theorem \ref{thm:alignment}. We first characterize certain quadratic forms which will appear in many parts of our analysis. 

\subsubsection{Concentration of Quadratic Forms}
The following lemma is a direct adaptation from Lemma 2.7 and Lemma A.1 in \cite{bai1998no}. We also refer readers to section B.5 in \cite{bai2010spectral} for more details.
\begin{lemm}\label{lemm:quadratic_E}
Given any deterministic matrix $\vD\in\R^{d\times d}$ and $\vx\sim\cN (0,\vI)$ in $\R^d$, we have that for any $p\ge 1$,
\begin{equation}
    \E\left[\left|\vx^\top\vD\vx-\Tr\vD\right|^p\right]\le C_p\left( 3^{p/2}+(2p-1)!!\right)\left(\Trarg{ \vD\vD^\top}\right)^{p/2},
\end{equation}
where $C_p>0$ is a universal constant. Furthermore, if $\vD$ is a non-negative definite matrix, then we have
\begin{equation}
    \E\left[\left|\vx^\top\vD\vx\right|^p\right]\le C_p\left((\Tr\vD)^p+(2p-1)!!\Trarg{\vD^p}\right).
\end{equation}
\end{lemm}

Equipped with Lemma~\ref{lemm:quadratic_E}, we introduce a quadratic concentration lemma specialized to our setting. 

\begin{lemm}\label{lemm:quadratic}
Define $\vu:=\frac{\eta\mu_1}{n}\vX^\top\vy$ where $\vy=f^*(\vX)+\vvarepsilon$. Under the same assumptions as Theorem \ref{thm:alignment}, consider any deterministic matrix $\vD\in\R^{d\times d}$ with $\|\vD\|\le C$ uniformly for some constant $C>0$. Then, as $n/d\to \psi_1$ proportionally, we have that
\begin{align}
    \left|\vu^\top\vD\vu-\left(\theta_1^2-\theta_2^2\right)\tr\vD-\theta_2^2\vbeta_*^\top\vD\vbeta_*\right|,~\left|\vbeta_*^\top\vD\vu-\theta_2\vbeta_*^\top\vD\vbeta_*\right|\overset{\P}{\to} 0,
\end{align}where $\theta_1$ and $\theta_2$ are defined in \eqref{def:theta_12}. In addition, recalling that the nonlinear part of the target function is given as $f^*_{\mathrm{NL}}(\vx):=f^*(\vx) - \mu_0^* -\mu_1^*\langle\vx,\vbeta_*\rangle$, we have that 
\begin{equation}\label{eq:quadratic_nonli}
    \left|\frac{1}{n}\vbeta_*^\top\vD\vX^\top f^*_{\mathrm{NL}}(\vX)\right|\overset{\P}{\to} 0.
\end{equation}
\end{lemm}
\begin{proof}
We first consider the concentration for $ \vu^\top\vD\vu$. Note that $\vX^\top=[\vx_1,\ldots,\vx_n]$ has i.i.d.\ columns. Hence, we can expand the first quadratic form as follows:
\begin{equation}
    \vu^\top\vD\vu=\frac{\eta^2\mu_1^2}{n^2}\sum_{i,j=1}^n \vx_j^\top\vD\vx_i(f^*(\vx_i)+\eps_i)(f^*(\vx_j)+\eps_j).
\end{equation}Denote $\vv_i:=\vx_i(f^*(\vx_i)+\eps_i)$, for $1\le i\le n.$ By condition \eqref{eq:orthogonal-decomposition} for $f^*$, all vectors $\vv_i$ are i.i.d.~with $\E[\vv_i]=\mu_1^*\vbeta_*$. Let us first compute the expectation of this quadratic form
\begin{align}
    \E[\vu^\top\vD\vu]=&\frac{\eta^2\mu_1^2}{n^2}\sum_{i,j=1}^n\left(\E\left[\vx^\top_i\vD\vx_j f^*(\vx_i)f^*(\vx_j)\right]+\E\left[\eps_i\eps_j\vx^\top_i\vD\vx_j f^*(\vx_i)f^*(\vx_j)\right]\right)\nonumber\\
    =& \frac{\eta^2\mu_1^2}{n} \left( \E[\vx_1^\top\vD\vx_1 f^{*2}(\vx_1)]+\sigma_\eps^2\Tr\vD\right)+\frac{\eta^2\mu_1^2}{n^2}\sum_{i\neq j}\mu_1^{*2}\vbeta_*^\top\vD\vbeta_*.\label{eq:uDu}
\end{align}
Notice that
\begin{align}
    &\left|\E[\vx_1^\top\vD\vx_1 f^{*2}(\vx_1)]-\E[f^{*2}(\vx_1)]\Tr\vD\right|
    \le\E\left[\left|\vx_1^\top\vD\vx_1-\Tr\vD\right|f^{*2}(\vx_1)\right]\nonumber\\
    \le ~& \E\left[\left|\vx_1^\top\vD\vx_1-\Tr\vD\right|^2\right]^{1/2}\E\left[f^{*4}(\vx_1)\right]^{1/2}
    \overset{(i)}{\le} C\|\vD\|_F\le C\sqrt{d}\|\vD\|,\label{eq:xDx}
\end{align}where $(i)$ is due to Lemma~\ref{lemm:quadratic_E} and $\|f^*\|_{L^4(\R^d,\Gamma)}\le C_f$ uniformly for some constant $C_f>0$, since $f^*$ is Lipschitz by Assumption~\ref{assump:1}. Therefore, recalling the definitions of $\theta_1$ and $\theta_2$ in \eqref{def:theta_12}, we arrive at
\begin{equation}
\E[\vu^\top\vD\vu]-\left(\theta_1^2-\theta_2^2\right)\tr\vD-\theta_2^2\vbeta_*^\top\vD\vbeta_*\to 0,
\end{equation}as $n\to \infty$ and $n/d\to\psi_1$. Next, we claim that this quadratic form $\vu^\top\vD\vu$ concentrates around its expectation in $L^2$. For $i\neq j \in [n]$, denote $q_{ij}:=\vv_i^\top\vD\vv_j-\mu_1^{*2}\vbeta_*^\top\vD\vbeta_*$ and $q_{ii}:=\vv_i^\top\vD\vv_i-\left(\bar\mu^2+\sigma_\eps^2\right)\Tr\vD$, hence $\E[q_{ij}]=0$ and, by \eqref{eq:uDu} and \eqref{eq:xDx}, $\frac{1}{n}\E[q_{ii}]\to 0$ as $n\to \infty$. In particular, 
\begin{align}
    \E[q_{12}q_{13}]=&-\left(\mu_1^{*2}\vbeta_*^\top\vD\vbeta_*\right)^2+\mu_1^{*2}\vbeta_*^\top\vD\E[\vv_1\vv_1^\top]\vD\vbeta_*\\
    = & \mu_1^{*2}\E\left[\vx_1^\top\vD^\top\vbeta_*\vbeta_*^\top\vD\vx_1\left(f^{*2}(x_1)+\eps_1^2\right)-\mu_1^{*2}\vbeta_*^\top\vD\vbeta_*\vbeta_*^\top\vD\vbeta_*\right]\\
    \le  & \mu_1^{*4}\|\vD\|^2+\mu_1^{*2}\E\left[\left(\vx_1^\top\vD^\top\vbeta_*\vbeta_*^\top\vD\vx_1\right)^2\right]^{1/2}\left(\E\left[f^{*4}(x_1)\right]^{1/2}+\E\left[\eps_1^4\right]^{1/2}\right)
    \le  C\|\vD\|^2,\label{eq:q1213}
\end{align}
where the last inequality is from Lemma~\ref{lemm:quadratic_E} and the uniform boundedness of $\|f^*\|_{L^4(\R^d,\Gamma)}$. In addition, since $\left|\Tr\vD\right|\le d\|\vD\|$, according to Lemma~\ref{lemm:quadratic_E}, we get
\begin{align}
    \E[q_{11}^2]&=\E\left[\left(\left(f^*(\vx_1)+\eps_1\right)^2\left(\vx_1^\top\vD\vx_1-\Tr\vD\right)+\Tr\vD \cdot \left(\left(f^*(\vx_1)+\eps_1\right)^2-\bar\mu^2-\sigma_\eps^2\right)\right)^2\right]\\
    &\le 2\E[\left(\vx_1^\top\vD\vx_1-\Tr\vD\right)^4]^{1/2}\E[\left(f^*(\vx_1)+\eps_1\right)^8]^{1/2}+2\E[\left(f^*(\vx_1)+\eps_1\right)^4]\left(\Tr\vD\right)^2\\
    &\le C\left(\Trarg{\vD\vD^\top}+\left(\Tr\vD\right)^2\right)\le Cd^2\|\vD\|,\label{eq:q11}
\end{align}
where we used the fact that $\|f^*\|_{L^8(\R^d,\Gamma)}$ is uniformly bounded (which is due to Lipschitz assumption on $f^*$). Analogously, following the above estimations, we have
\begin{align}
    \E[q_{12}^2]= & \Var(\vv_1^\top\vD\vv_2)\le \E\left[\vx_1^\top\vD^\top\vx_2\vx_2^\top\vD\vx_1\left(f^*(\vx_1)+\eps_1\right)^2\left(f^*(\vx_2)+\eps_2\right)^2\right]\\
    \le ~& \E\left[\left(\vx_1^\top\vD^\top\vx_2\vx_2^\top\vD\vx_1\right)^2\right]^{1/2}\E\left[\left(f^*(\vx_1)+\eps_1\right)^4\right]
    \le C \E_{\vx_2}\left[\E_{\vx_1}\left[\left(\vx_1^\top\vD^\top\vx_2\vx_2^\top\vD\vx_1\right)^2\right]\right]^{1/2}\\
    \overset{(ii)}{\lesssim }~& \E_{\vx_2}\left[(\Tr\vD^\top\vx_2\vx_2^\top\vD)^2+\Tr\vD^\top\vx_2\vx_2^\top\vD\vD^\top\vx_2\vx_2^\top\vD\right]^{1/2}=\E_{\vx_2}\left[2(\vx_2^\top\vD^\top\vD\vx_2)^2\right]^{1/2}\\
    \overset{(iii)}{\lesssim } &\left(\left(\Trarg{\vD^\top\vD}\right)^2+\Trarg{(\vD^\top\vD)^2}\right)\lesssim d\|\vD\|^2,\label{eq:q12}
\end{align}
where both $(ii)$ and $(iii)$ are deduced from Lemma~\ref{lemm:quadratic_E} since $\vD^\top\vx_2\vx_2^\top\vD $ and $\vD^\top\vD$ are both semi-positive definite. Combining the bounds for $ \E[q_{12}q_{13}]$, $\E[q_{12}^2]$ and $\E[q_{11}^2]$, we  conclude that
\begin{align}
    &\E\left[\left|\vu^\top\vD\vu-\left(\theta_1^2-\theta_2^2\right)\tr\vD-\theta_2^2\vbeta_*^\top\vD\vbeta_*\right|^2\right]\nonumber\\
    \le ~& 2\E\left[\left|\frac{\eta^2\mu_1^2}{n^2}\sum_{i\neq j} \vv_i^\top\vD\vv_j-\theta_2^2\vbeta_*^\top\vD\vbeta_*\right|^2\right]+2\E\left[\left|\frac{\eta^2\mu_1^2}{n^2}\sum_{i=1}^n\vv_i^\top\vD\vv_i-\left(\theta_1^2-\theta_2^2\right)\tr\vD\right|^2\right]\nonumber\\
    \le~ &\frac{C\eta^4\mu_1^4}{n^4} \E\left[\Big(\sum_{i\neq j} q_{ij}\Big)^2\right] +\frac{C\eta^4\mu_1^4}{n^4}\E\left[\sum_{i=1}^n\left|\vv_i^\top\vD\vv_i-\left(\bar\mu^2+\sigma_\eps^2\right)\Tr\vD\right|^2\right]\\
    \overset{(iv)}{\le} ~& C\eta^4\mu_1^4\left(\frac{1}{n^2}\E[q_{12}^2]+\frac{1}{n}\E[q_{12}q_{13}]+\frac{1}{n^3}\E[q_{11}^2]\right)\lesssim \frac{1}{n}\to 0,\label{eq:moments_cross2}
\end{align}
as $n\to \infty$, where $(iv)$ is obtained by Lemma 2.2 in \cite{bai1998no} because $\vv_i^\top\vD\vv_i$ are i.i.d.~for $i\in[n]$, and the last inequality \eqref{eq:moments_cross2} follows from \eqref{eq:q1213}, \eqref{eq:q11} and \eqref{eq:q12}. This yields the convergence in probability.

For the second part $\vbeta_*^\top\vD\vu$, notice that
\begin{align}
    \vbeta_*^\top\vD\vu-\theta_2\vbeta_*^\top\vD\vbeta_*=  \frac{\eta\mu_1}{n}\sum_{i=1}^n\left(\vbeta_*^\top\vD\vx_i(f^*(\vx_i)+\eps_i)-\mu_1^*\vbeta_*^\top\vD\vbeta_*\right)
\end{align}
is a sample mean of i.i.d.~centered random variables. Therefore by Lemma~\ref{lemm:quadratic_E}, we have
\begin{align}
    &\E\left[\left|\vbeta_*^\top\vD\vu-\theta_2\vbeta_*^\top\vD\vbeta_*\right|^2\right]\le \frac{\eta^2\mu_1^2}{n}\E\left(\vbeta_*^\top\vD\vx_i(f^*(\vx_i)+\eps_i) \right)^2\\
    \le ~& \frac{\eta^2\mu_1^2}{n} \E\left[\left(\vbeta_*^\top\vD\vx_i \right)^4\right]^{1/2}\E\left[\left( (f^*(\vx_i)+\eps_i) \right)^4\right]^{1/2}\le \frac{C}{n}\E\left[\left(\vx_i^\top\vD^\top\vbeta_*\vbeta_*^\top\vD\vx_i \right)^2\right]^{1/2}\\
    \lesssim~ & \frac{1}{n}\left(\left(\Trarg{\vD^\top\vbeta_*\vbeta_*^\top\vD}\right)^2+\Trarg{\vD^\top\vbeta_*\vbeta_*^\top\vD\vD^\top\vbeta_*\vbeta_*^\top\vD}\right)^{1/2} \lesssim \frac{\|\vD\|^2}{n}\to 0.
\end{align}
Hence, this $L^2$ convergence completes the proof. 
Finally, note that the proof of \eqref{eq:quadratic_nonli} is identical to the above calculation, where we can apply $\E[\vx_if^*_{\mathrm{NL}}(\vx_i)]=\mathbf{0}$ for $i\in [n]$.

\end{proof}

\subsubsection{Analysis of Spike in Weight Matrix}

\paragraph{Useful Lemmas.}
Observe that the limiting eigenvalue distribution of $\vW_0^\top\vW_0$ is the Marchenko–Pastur distribution $\mu_{\psi_2}^{\MP}$ with parameter $\psi_2$; let $m(z)$ be the Stieltjes transform of $\mu_{\psi_2}^{\MP}$. 
Also, we denote the limiting eigenvalue distribution for $\vW_0\vW_0^\top$ by $\bar\mu_{\psi_2}^{\MP}$ whose Stieltjes transform is $\bar{m}(z)$, which is referred to as the \textit{companion} transform. The relation between $m(z)$ and $\bar{m}(z)$ is given as 
\begin{equation}\label{eq:companinon_m(z)}
    \psi_2m(z)=\bar{m}(z)+\frac{1-\psi_2}{z}.
\end{equation}
Moreover, $m(z)$ is uniquely determined by the fixed-point equation 
\begin{equation}\label{eq:fixedpoint_m(z)}
    z\psi_2m^2(z)-(1-\psi_2-z)m(z)+1=0,
\end{equation}
for $z\in\mathbb{C}\setminus \text{supp}(\mu_{\psi_2}^{\MP})$. For more details on Stieltjes transform of $\mu_{\psi_2}^{\MP}$, we refer to \cite[Chapter 3]{bai2010spectral}.
\begin{lemm}\label{lemm:bound_Q}
Following the above notions, we define the resolvent $\vQ_0(z):=(\vW_0\vW_0^\top-z\vI)^{-1}$, for 
\begin{equation}\label{def:set}
    z\in \Omega_{\epsilon}:=\left\{z\in\mathbb{C}: \text{Re} (z)>\big(1+\sqrt{\psi_2}\big)^2+\epsilon\right\},
\end{equation}
and any small $\epsilon>0$. Then, under the same assumptions of Theorem \ref{thm:alignment}, for all sufficiently large $N$, $\|\vQ_0(z)\|\le 2/\epsilon$ uniformly for all $z\in\Omega_\epsilon$. 
\end{lemm}
\begin{proof}
Write $z=x+iy$, with $x>\big(1+\sqrt{\psi_2}\big)^2+\epsilon$. For any $i\in [N]$, let $\lambda_i$ be the $i$-th eigenvalue of $\vW_0\vW_0^\top$. Then we have
\[\left|\frac{1}{z-\lambda_i}\right|\le \frac{1}{\sqrt{(x-\lambda_i)^2+y^2}}\le \frac{1}{|x-\lambda_i|}\le \frac{1}{|x-\lambda_1|}.\]
By Theorem 5.11 in \cite{bai2010spectral}, for sufficiently large $N$, $\left|\lambda_1-\big(1+\sqrt{\psi_2}\big)^2\right|\le \epsilon/2$. Hence, $1/|z-\lambda_i|\le \epsilon/2$ for $i\in [N]$ and $\|\vQ_0(z)\|\le 2/\epsilon$ for all $z\in\Omega_\epsilon$. 

\end{proof}

The following lemma characterizes the asymptotics of certain quantities in terms of the Stieltjes transform which will be useful in the subsequent analysis. 

\begin{lemm}\label{lemm:uniform_converge}
Recall the definition $\vu=\frac{\eta\mu_1}{n}\vX^\top\vy$. Under the same assumptions as Theorem \ref{thm:alignment}, for any $\epsilon>0$, 
\begin{align} 
    &\va^\top \vW_0^\top\vQ_0(z)\vu\to 0,\quad \vbeta_*^\top\vQ_0(z)\vW_0\va\to 0,\label{eq:a_concentration}\\
    &\vu^\top\vQ_0(z)\vu \to \theta_1^2\bar{m}\left(z\right),\quad  \vbeta_*^\top\vQ_0(z)\vu\to \theta_2\bar m(z),\quad \vu^\top\vQ_0(z)^2\vu\to \theta_1^2 \bar m'(z) \label{eq:u_concentration} \\
    & \va^\top\vW_0^\top\vQ_0(z)\vW_0\va \to 1+zm(z),\quad \va^\top\vW_0^\top\vQ_0(z)^2\vW_0\va\to m(z)+zm'(z), \label{eq:a_concentration2} 
\end{align}
in probability as $n/d\to\psi_1 $ and $N/d\to\psi_2$, uniformly on any compact subset of $\Omega_\epsilon$ defined in \eqref{def:set}, where scalars $\theta_1$ and $\theta_2$ are defined in Theorem \ref{thm:alignment}.
\end{lemm}
\begin{proof}
Firstly note that \eqref{eq:a_concentration} directly follows from Lemma \ref{lemm:bound_Q} and Hoeffding's inequality for $\va$. 
The remaining concentration statements will be established by applying Lemma \ref{lemm:quadratic} to different choices of $\vD$ and the Hanson-Wright inequality for $\va$. 
In particular, due to Lemma \ref{lemm:bound_Q}, we know that $\|\vQ_0(z)\|$, $\|\vW_0^\top\vQ_0(z)\vW_0\|$, $\|\vQ_0(z)^2\|$ and $\|\vW_0^\top\vQ_0(z)^2\vW_0\|$ are all uniformly bounded on $\Omega_\epsilon$ for large $N$. Take $\vD=\vQ_0(z)$ in Lemma \ref{lemm:quadratic}, we obtain that
\[ \left|\vu^\top\vQ_0(z)\vu-\left(\theta_1^2-\theta_2^2\right)\tr\vQ_0(z)-\theta_2^2\vbeta_*^\top\vQ_0(z)\vbeta_*\right|\overset{\P}{\to }0,\]
uniformly for $z\in \Omega_\epsilon$. 
Moreover, we may treat $\vbeta_*$ as uniformly distributed on $\mathbb{S}^{d-1}$ due to the rotational invariance of $\vW_0\vW_0^\top$. Thus, $|\vbeta_*^\top\vQ_0(z)\vbeta_*-\tr \vQ_0(z)|\to 0$ in probability uniformly for $z\in \Omega_\epsilon$ (the detailed statement will be elaborated by Lemma \ref{lem:Titilde} and \ref{lemm:quadraticbeta} in Section \ref{sec:linear_pencil}). 
By the Marchenko–Pastur law \cite[Theorem 3.10]{bai2010spectral} and \eqref{eq:companinon_m(z)}, we can conclude that $\vu^\top\vQ_0(z)\vu$ converges to $\theta_1^2\bar{m}\left(z\right)$ in probability for any  $z\in\Omega_\epsilon$. 
In addition, Arzelà-Ascoli theorem implies that this convergence in probability holds uniformly on any compact subset of $\Omega_\epsilon$. 
One can analogously verify the remaining statements in \eqref{eq:u_concentration}. Finally, note that with Lemma \ref{lemm:bound_Q}, the Hanson-Wright inequality for normalized Gaussian vector $\va$ enables us to establish \eqref{eq:a_concentration2} directly. To complete the proof, we point out that
\begin{align}
    &\tr\left(\vW_0^\top\vQ_0(z)\vW_0\right)\overset{\P}{\to} \frac{1}{\psi_2}\int \frac{x}{x-z}d \bar \mu_{\psi_2}^\MP(x)=1+zm(z),\\
    &\tr\left(\vW_0^\top\vQ_0(z)^2\vW_0\right)\overset{\P}{\to} \frac{1}{\psi_2}\int \frac{x}{(x-z)^2}d \bar \mu_{\psi_2}^\MP(x)=m(z)+zm'(z),
\end{align}uniformly on any compact subset of $\Omega_\epsilon$, due to \eqref{eq:companinon_m(z)}, Lemma 7.4 of \cite{dobriban2018high} and Lemma 2.14 of \cite{bai2010spectral}.  

\end{proof}

Finally, we recall the following control of singular values (also referred to as Weyl's inequality).
\begin{lemm}[Theorem A.46 of \citep{bai2010spectral}]\label{lemm:singular_diff}
	Let $\vC$ and $\vD\in\R^{n\times m}$ be two complex matrices with singular values $s_1(\vC)\ge s_2(\vC)\ge \cdots\ge s_{r}(\vC)$ and $s_1(\vD)\ge s_2(\vD)\ge \cdots\ge s_{r}(\vD)$ where $r=\min\{n,m\}$. Then for any $1\le k\le r$, the difference in the $k$-th singular values of $\vC$ and $\vD$ satisfies
	\[\left| s_k(\vC)-s_k(\vD)\right|\le \|\vC-\vD\|.\] 
\end{lemm}

We are now ready to prove Theorem \ref{thm:alignment}. 
\begin{proofof}[Theorem \ref{thm:alignment}] 
Denote $\tilde{\vW}_1:=\vW_0+\eta\sqrt{N}\vA$, where $\eta\sqrt{N}\vA=\frac{\eta\mu_1}{n}\vX^\top\vy\va^\top$. Lemma \ref{lemm:gradient-norm} implies that as $n/d\to\psi_1$ and $N/d\to\psi_2$, $\|\vW_1-\tilde{\vW}_1\|\to 0$ almost surely. Thanks to Lemma \ref{lemm:singular_diff}, the top $i$-th singular value $s_i(\vW_1)$ coincides with $s_i(\tilde{\vW}_1)$ asymptotically for any fixed $i\ge 1$. Hence we first prove Theorem \ref{thm:alignment} for $\tilde{\vW}_1$ in lieu of the original $\vW_1$; this is equivalent to considering the leading eigenvalue $\hat\lambda:=s_1(\tilde{\vW}_1)^2$ of $\tilde{\vW}_1\tilde{\vW}_1^\top$ and its corresponding eigenvector denoted as $\tilde{\vu}_1$. Note that $s_1(\tilde{\vW}_1)$ is the leading singular value of a rectangular Gaussian random matrix $\vW_0$ plus an independent rank-one perturbation $\eta\sqrt{N}\vA$. 

From the definition of $\tilde{\vW}_1$, we have the decomposition
\begin{equation}\label{eq:tildeW_1}
    \tilde{\vW}_1\tilde{\vW}_1^\top=\vW_0\vW_0^\top +\begin{bmatrix}
\vu & \vW_0\va
\end{bmatrix}\begin{bmatrix}
\|\va\|^2 &1\\
1& 0
\end{bmatrix}\begin{bmatrix}
\vu^\top\\
\va^\top\vW_0^\top
\end{bmatrix}.
\end{equation}
Also, by \cite[Section 6.2.1.]{benaych2011eigenvalues} and \cite[Section 5.2]{bai2010spectral}, we know that $\liminf s_1(\tilde{\vW}_1)\ge 1+\sqrt{\psi_2}$, $s_1( \vW_0)\to 1+\sqrt{\psi_2}$, and $s_i(\vW_0), s_i(\tilde{\vW}_1)\to 1+\sqrt{\psi_2}$ for any fixed $i>1$.

To analyze the isolated eigenvalue and the corresponding eigenvector for $\tilde{\vW}_1\tilde{\vW}_1^\top$, we follow the approach in \citep{benaych2011eigenvalues,benaych2012singular}. It is straightforward to verify that the isolated eigenvalue of $\tilde{\vW}_1\tilde{\vW}_1^\top$ outside the spectrum of $\vW_0\vW_0^\top$ is the solution $x\in\R$ to the following equation:
\begin{align*}
    \det \vQ_0(x)\left(\tilde{\vW}_1\tilde{\vW}_1^\top-x\vI\right)=0.
\end{align*}
By \eqref{eq:tildeW_1}, the equality $\det(\vA\vB) = \det(\vA) \det(\vB)$, and the Sylvester’s determinant identity $\det(\vI + \vA\vB) = \det(\vI + \vB\vA)$ for $\vA, \vB$ of appropriate dimensions, the above equations has the same solution as 
\begin{align*}
P_n(x):=\det\left(\vI+\begin{pmatrix}
\|\va\|^2 &1\\
1& 0
\end{pmatrix}\begin{pmatrix}
\vu^\top\\
\va^\top\vW_0^\top
\end{pmatrix}\vQ_0(x)\begin{pmatrix}
\vu & \vW_0\va
\end{pmatrix}\right)=0,
\end{align*}where $P_n(x)$ is the determinant of a 2-by-2 matrix. Notice that $\|\va\|\to 1$ almost surely as $N\to\infty$. Hence in terms of Lemma \ref{lemm:uniform_converge}, we know that given $\epsilon>0$, as $n,d,N\to\infty$ proportionally, $P_n(z)\to P(z)$ in probability uniformly on any compact subset of $\Omega_\epsilon$, where $P(z):=1-\theta_1^2zm(z)\bar m(z)$. Next we establish the convergence of the roots of $P_n(z)$ to the roots of $P(z)$ on $\Omega_\epsilon \cap \R$. Due to \eqref{eq:companinon_m(z)} and \eqref{eq:fixedpoint_m(z)}, we know that
\begin{align}\label{eq:P(z)}
    P(z)=1+\theta_1^2(1+zm(z)).
\end{align}
Now we compute the root of $P(z)$ on $\left((1+\sqrt{\psi_2})^2,+\infty\right)$. From \eqref{eq:fixedpoint_m(z)} we have
\[z=\frac{zm(z)}{zm(z)+1}-\psi_2zm(z),\]
which implies that the root of $P(z)$ on the real line is given as $\lambda_0:=\frac{(1+\theta_1^2)(\psi_2+\theta_1^2)}{\theta_1^2}$. Based on the expression of $m(z)$ and \eqref{eq:fixedpoint_m(z)}, we have 
\[\lim_{x	\searrow (1+\sqrt{\psi_2})^2} xm(x)=-\frac{1+\sqrt{\psi_2}}{\sqrt{\psi_2}}.\]
Note that $zm(z)$ is an increasing mapping from $\left((1+\sqrt{\psi_2})^2,+\infty\right)$ to $\left(-\frac{1+\sqrt{\psi_2}}{\sqrt{\psi_2}},-1\right)$. Thus, as long as $\theta_1>\psi_2^{1/4}$, by selecting a sufficient small $\epsilon>0$, we can obtain that $\lambda_0$ is a root in $\Omega_\epsilon \cap \R$. By Hurwitz's theorem and the uniform convergence of $P_n(z)$, we know that the root of $P_n(z)$, which is exactly the isolated eigenvalue $\hat\lambda$ for $\vW_1\vW_1^\top$ outside the spectrum of $\vW_0^\top\vW_0$, is converging to $\lambda_0$ in probability. 
On the other hand, if $\theta_1\le\psi_2^{1/4}$, there is no root of $P(z)$ in $((1+\sqrt{\psi_2})^2,+\infty)$, hence the largest eigenvalue $\hat\lambda$ for $\vW_1\vW_1^\top$ is no greater than $(1+\sqrt{\psi_2})^2+\epsilon$, for any $\epsilon>0$; because of the lower bound on $\hat\lambda$, in this case $\hat\lambda\to (1+\sqrt{\psi_2})^2$ which is at the right-edge of the support of $\mu_{\psi_2}^{\MP}$.

Next we consider the isolated eigenvector $\tilde{\vu}_1$. By definition, $\hat\lambda\tilde{\vu}_1=\vW_1\vW_1^\top \tilde{\vu}_1$, which, together with \eqref{eq:tildeW_1}, yields the identity
\begin{align}
    \mathbf{0}=(\tilde{\vW}_0\tilde{\vW}_0^\top-\hat\lambda\vI)\tilde{\vu}_1+\begin{bmatrix}
\vu & \vW_0\va
\end{bmatrix}\begin{bmatrix}
\|\va\|^2 &1\\
1& 0
\end{bmatrix}\begin{bmatrix}
\vu^\top\\
\va^\top\vW_0^\top
\end{bmatrix}\tilde{\vu}_1.
\end{align}
Since $\hat\lambda$ does not reside in the spectrum of $\vW_0\vW_0^\top$, we can further write
\begin{align}
    \mathbf{0}=\left(\vI+\vQ_0(\hat\lambda)\begin{bmatrix}
\vu & \vW_0\va
\end{bmatrix}\begin{bmatrix}
\|\va\|^2 &1\\
1& 0
\end{bmatrix}\begin{bmatrix}
\vu^\top\\
\va^\top\vW_0^\top
\end{bmatrix}\right)\tilde{\vu}_1.\label{eq:tilde_u_identity}
\end{align}
By multiplying $[\vu,\vW_0\va]^\top$ from the left hand side of the above equality, we arrive at $\mathbf{0}=\vM_n(\hat\lambda)\begin{bmatrix}
\hat v_1\\
\hat v_2
\end{bmatrix}$, where the 2-by-2 matrix is given as
\begin{align}
    \vM_n(z):=\vI+\begin{bmatrix}
\vu^\top\vQ_0(z)\vu &\vu^\top\vQ_0(z)\vW_0\va\\
\va^\top\vW_0^\top\vQ_0(z)\vu& \va^\top\vW_0^\top\vQ_0(z)\vW_0\va
\end{bmatrix}\begin{bmatrix}
\|\va\|^2 &1\\
1& 0
\end{bmatrix},
\end{align}and $\hat v_1:=\vu^\top\tilde{\vu}_1$, $\hat v_2:=\va^\top\vW_0^\top\tilde{\vu}_1$. This implies that the vector $ \begin{bmatrix}
\hat v_1\\
\hat v_2
\end{bmatrix}$ belongs to the kernel of $2\times 2$ matrix $\vM_n(\hat\lambda)$. Hence, the relation between $\hat v_1$ and $\hat v_2$ is determined by 
\begin{equation}\label{eq:equation1}
    \left(\|\va\|^2\va^\top\vW_0^\top\vQ_0(\hat\lambda)\vu+\va^\top\vW_0^\top\vQ_0(\hat\lambda)\vW_0\va\right)\hat v_1+\left(1+\va^\top\vW_0^\top\vQ_0(\hat\lambda)\vu\right)\hat v_2=0.
\end{equation}
Moreover, by definition we have
\begin{align}
    \hat\lambda\tilde{\vu}_1=\vW_1\vW_1^\top \tilde{\vu}_1=\vW_0\vW_0^\top\tilde{\vu}_1+\left(\hat v_1\|\va\|^2+\hat v_2\right)\vu+ \hat v_1\cdot \vW_0\va.
\end{align}
With $\|\tilde{\vu}_1\|^2=1$, the above implies that
\begin{align}
    1=&\left(\left(\hat v_1\|\va\|^2+\hat v_2\right)\vu^\top+ \hat v_1\cdot \va^\top\vW_0^\top\right)\vQ_0(\hat\lambda)^2\left(\left(\hat v_1\|\va\|^2+\hat v_2\right)\vu+ \hat v_1\cdot \vW_0\va\right)\\
    =&\left(\hat v_1\|\va\|^2+\hat v_2\right)^2\vu^\top\vQ_0(\hat\lambda)^2\vu+\hat v_1^2\cdot \va^\top\vW_0^\top\vQ_0(\hat\lambda)^2\vW_0\va+2\hat v_1\left(\hat v_1\|\va\|^2+\hat v_2\right)\vu^\top\vQ_0(\hat\lambda)^2\vW_0\va.\label{eq:equation2}
\end{align}
In addition, multiplying $\vbeta_*^\top$ from the left hand side of \eqref{eq:tilde_u_identity} yields
\begin{equation}
    \tilde{\vu}_1^\top\vbeta_*=-\left(\hat v_1\left(\|\va\|^2\vbeta_*^\top\vQ_0(\hat\lambda)\vu+\vbeta_*^\top\vQ_0(\hat\lambda)\vW_0\va\right)+\hat v_2\cdot \vbeta_*^\top\vQ_0(\hat\lambda)\vu\right).\label{eq:equation3}
\end{equation}

Our goal is to describe the asymptotic behavior of $\tilde{\vu}_1^\top\vbeta_*$ using \eqref{eq:equation1}, \eqref{eq:equation2} and \eqref{eq:equation3}. As a side remark, our quantity of interest $\tilde\vu_1^\top\vbeta_*$ is different from the eigenvector alignments $\tilde\vu_1^\top\vu$ addressed in prior works \cite{benaych2012singular}, so we further introduce \eqref{eq:equation3} for our purpose.

\paragraph{Case I: $\theta_1>\psi_2^{1/4}$.}
We first consider the scenario $\theta_1>\psi_2^{1/4}$, where $\hat \lambda\to \lambda_0$ in probability and $\lambda_0$ is outside the support of $\mu_{\psi_2}^{\MP}$. For sufficiently small $\epsilon>0$ and all large $n$, $\hat\lambda\in\Omega_\epsilon$, and thus Lemma \ref{lemm:uniform_converge} gives
\begin{equation}\label{eq:all_qudratic}
    \begin{aligned}
    &\va^\top\vW_0^\top\vQ_0(\hat\lambda)\vu\to 0,\quad \vbeta_*^\top\vQ_0(\hat\lambda)\vW_0\va\to 0,\quad \vu^\top\vQ_0(\hat\lambda)^2\vW_0\va\to 0,\\
    &\va^\top\vW_0^\top\vQ_0(\hat\lambda)^2\vW_0\va\to m(\lambda_0)+\lambda_0 m'(\lambda_0), \quad \vu^\top\vQ_0(\hat\lambda)^2\vu\to \theta_1^2 \bar m'(\lambda_0),\\
    &\vbeta_*^\top\vQ_0(\hat\lambda)\vu\to\theta_2\bar m(\lambda_0),\quad \va^\top\vW_0^\top\vQ_0(\hat\lambda)\vW_0\va\to 1+\lambda_0m(\lambda_0), \quad \vu^\top\vQ_0(\hat\lambda)\vu\to\theta_1^2\bar m(\lambda_0),
\end{aligned}
\end{equation}
in probability as $n,d,N\to\infty$ proportionally. Notice that both $\hat v_1$ and $\hat v_2$ are uniformly bounded by some constants with high probability since $\|\va\|\to 1$, $\|\vW_0\|\to (1+\sqrt{\psi_2})$ almost surely, and Lemma \ref{lemm:quadratic} implies $\|\vu\|\to\theta_1$ in probability as $n,d,N\to\infty$.
Therefore, when $n/d\to\psi_1$ and $N/d\to\psi_2$,  \eqref{eq:equation1} and \eqref{eq:equation2} provide the limits of $\hat v_1$ and $\hat{v}_2$, which we denote by $v_1$ and $v_2$, respectively. More precisely,  
\begin{align}
    v_2 = -\left(1+\lambda_0 m(\lambda_0)\right) v_1,\quad
    v_1^2 = \frac{1}{\theta_1^2\bar m'(\lambda_0)\lambda_0^2 m(\lambda_0)^2+m(\lambda_0)+\lambda_0m'(\lambda_0)}.
\end{align}
Now by \eqref{eq:equation3}, we know that
\begin{align}
    (\tilde{\vu}_1^\top\vbeta_*)^2\overset{\P}{\to } \theta_2^2\bar m(\lambda_0)^2(v_1+v_2)^2=\frac{\theta_2^2\bar m(\lambda_0)^2\lambda_0^2 m(\lambda_0)^2}{\theta_1^2\bar m'(\lambda_0)\lambda_0^2 m(\lambda_0)^2+m(\lambda_0)+\lambda_0m'(\lambda_0)}.
\end{align}
Thus, we can apply formula \eqref{eq:companinon_m(z)}, condition $P(\lambda_0)=0$ in \eqref{eq:P(z)} and the following well-known facts of the Stieltjes transform (e.g., see \citep{bai2010spectral}) with $\lambda_0=(1+\theta_1^2)(\psi_2+\theta_1^2)/\theta_1^2$:
\begin{align*}
    \lambda_0 m(\lambda_0)=&\frac{-1}{\theta_1^2}-1,&\quad& m(\lambda_0)=\frac{-1}{\psi_2+\theta_1^2},\\
    \lambda_0 \bar{m}(\lambda_0)= & \frac{-\psi_2}{\theta_1^2}-1, &\quad& \bar{m}(\lambda_0)= \frac{-1}{1+\theta_1^2},\\
    \lambda_0 m'(\lambda_0)=& \frac{\theta_1^2(\theta_1^2+1)}{(\psi_2+\theta_1^2)(\theta_1^4-\psi_2)},&\quad& \lambda_0 \bar m'(\lambda_0)=\frac{\theta_1^2(\theta_1^2+\psi_2)}{(1+\theta_1^2)(\theta_1^4-\psi_2)},
\end{align*} to conclude that asymptotic alignment between the linear component $\vbeta_*$ of the teacher model and leading left singular vector $\tilde{\vW}_1$ satisfies
\begin{align}
    (\tilde{\vu}_1^\top\vbeta_*)^2 \overset{\P}{\to}  \frac{\theta_2^2(\theta_1^4-\psi_2)}{\theta_1^4(\theta_1^2+1)}=\frac{\theta_2^2}{\theta_1^2}\left(1-\frac{\psi_2+\theta_1^2}{\theta_1^2(\theta_1^2+1)}\right),\label{eq:aligment}
\end{align}
as $n,d,N\to\infty$ proportionally and when $\theta_1^4>\psi_2$. This establishes the alignment between $\vbeta_*$ and $\tilde{\vu}_1$. 
Now we return to the left singular vector $\vu_1$ of the original $\vW_1$ using Davis-Kahan $\sin\theta$ \cite[Theorem 4.4]{stewart1990matrix}:
\[\|\vu_1-\tilde{\vu}_1\|\le \frac{\sqrt{2}\|\vW_1-\tilde{\vW}_1\|}{\delta-\|\vW_1-\tilde{\vW}_1\|},\]
where $\delta:=s_1(\tilde{\vW}_1)-s_2(\tilde{\vW}_1)$. When $\theta_1^4>\psi_2$, $s_1(\tilde{\vW}_1)$ will stay outside of the bulk whereas $s_2(\tilde{\vW}_1)$ will stick to right edge of the bulk. Therefore, $\delta$ has a uniform lower bound and eventually $\|\vu_1-\tilde{\vu}_1\|\to 0$, which implies that $(\vu_1^\top\vbeta_*)^2$ has the same limit as $(\tilde{\vu}_1^\top\vbeta_*)^2$. 

\paragraph{Case II: $\theta_1^4\le \psi_2$.} 
On the other hand, if $\theta_1^4\le \psi_2$, we have proved that $\hat \lambda $ is approaching to the right-edge of the bulk of $\mu_{\psi_2}^{\MP}$. In fact, with probability one, $ \lambda_1<\hat\lambda$, where $\lambda_1$ is the largest eigenvalue of $\vW_0\vW_0^\top$; this is because $\det \vM_n(z)=\left(\vu^\top\vQ_0(z)\vW_0\va+1\right)^2+\vu^\top\vQ_0(z)\vu\left(\|\va\|^2-\va^\top\vW_0^\top\vQ_0(z)\vW_0\va\right)$ satisfies 
\[\lim_{z\to+\infty }\det \vM_n(z)=1, \quad \lim_{z	\searrow \lambda_1 }\det \vM_n(z)=-\infty,\]
while $\det \vM_n(\hat\lambda)=0$. Hence by the definition of $M_n(z)$, 
\begin{equation}
    \det \vM_n(\hat\lambda)=\left(1+\va^\top\vW_0^\top\vQ_0(\hat\lambda)\vu\right)^2-\hat\lambda\cdot\va^\top\vW_0^\top\vQ_0(\hat\lambda)\vW_0\va\cdot\vu^\top\vQ_0(\hat\lambda)\vu =0.\label{eq:det=0}
\end{equation}
Also, by the Cauchy–Schwarz inequality, we have
\[\hat\lambda=\frac{\left(1+\va^\top\vW_0^\top\vQ_0(\hat\lambda)\vu\right)^2}{\va^\top\vW_0^\top\vQ_0(\hat\lambda)\vW_0\va\cdot\vu^\top\vQ_0(\hat\lambda)\vu}\le \left(1+\frac{1}{\va^\top\vW_0^\top\vQ_0(\hat\lambda)\vu}\right)^2.\]
Since $\hat\lambda\to (1+\sqrt{\psi_2})^2$, given any small $\epsilon\in(0,\sqrt{\psi_2})$, for all sufficiently large $N$, we have 
\[(1+\sqrt{\psi_2}-\epsilon)^2\le \left(1+\frac{1}{\va^\top\vW_0^\top\vQ_0(\hat\lambda)\vu}\right)^2,\]
which indicates that $\va^\top\vW_0^\top\vQ_0(\hat\lambda)\vu\in \left(\frac{1}{-2-\sqrt{\psi_2}+\epsilon},\frac{1}{\sqrt{\psi_2}-\epsilon}\right)$. Therefore, for all large $N$, $|\va^\top\vW_0^\top\vQ_0(\hat\lambda)\vu|$ is bounded by some universal constant related to $\psi_2$. Then by \eqref{eq:det=0}, we can conclude $\va^\top\vW_0^\top\vQ_0(\hat\lambda)\vW_0\va\cdot\vu^\top\vQ_0(\hat\lambda)\vu$ is also asymptotically bounded by some constant.  On the other hand, since all eigenvalues of $\vW_0\vW_0^\top$ is smaller than $\hat\lambda$, we obtain
\[-\vu^\top\vQ_0(\hat\lambda)\vu>\frac{1}{\hat\lambda}.\]
This directly implies $-\va^\top\vW_0^\top\vQ_0(\hat\lambda)\vW_0\va$ has a constant upper bound for all large $N$. Following the proofs of \cite[Theorem 2.3]{benaych2011eigenvalues} and \cite[Theorem 2.10]{benaych2012singular} (with slight modifications of Lemma A.2 and Proposition A.3 in \cite{benaych2011eigenvalues}), it is straightforward to control the following quadratic forms by verifying the weak convergence of certain weighted spectral measures in combination with the Portmanteau theorem:
\begin{align}
    \liminf_{n\to\infty}\vu^\top\vQ_0(\hat\lambda)^2\vu\ge & \lim_{z	\searrow (1+\sqrt{\psi_2})^2} \theta_1^2 \bar m'(z)=+\infty,\label{eq:infinity1}\\
    \liminf_{n\to\infty} \va^\top\vW_0^\top\vQ_0(z)^2\vW_0\va \ge & \lim_{z	\searrow (1+\sqrt{\psi_2})^2} \big(m(z)+zm'(z)\big)=+\infty,\label{eq:infinity2}\\
    \liminf_{n\to\infty} -\va^\top\vW_0^\top\vQ_0(z)\vW_0\va\ge &\lim_{z	\searrow (1+\sqrt{\psi_2})^2} \big(-zm(z)-1\big)= \frac{1}{\sqrt{\psi_2}}.
\end{align}
Consequently, we know that  $-\va^\top\vW_0^\top\vQ_0(z)\vW_0\va\in \left(1/\sqrt{\psi_2},C(1+\sqrt{\psi_2})^2\right)$ for some constant $C>0$ and all large $N$. Hence $-\vu^\top\vQ_0(\hat\lambda)\vu\in \left(1/(1+\sqrt{\psi_2})^2,C\sqrt{\psi_2}\right)$. Now notice that 
\begin{align}
    \left(\vbeta_*^\top\vQ_0(\hat\lambda)\vu\right)^2\le \left(-\vbeta_*^\top\vQ_0(\hat\lambda)\vbeta_*\right)\left(-\vu^\top\vQ_0(\hat\lambda)\vu\right). 
\end{align}
In addition, by definition we have
\[-\va^\top\vW_0^\top\vQ_0(z)\vW_0\va=-\|\va\|^2+\hat\lambda \va^\top\left(\hat\lambda\vI-\vW_0^\top\vW_0\right)^{-1}\va.\]
By the same rotation invariance argument as in Lemma \ref{lemm:uniform_converge}, we may assume $\vbeta_*\sim\text{Unif}(\mathbb{S}^{d-1})$. 
Thus we can apply the Hanson-Wright inequality for $\va$ and $\vbeta_*$ to show that $-\vbeta_*^\top\vQ_0(\hat\lambda)\vbeta_*$ and $\va^\top\left(\hat\lambda\vI-\vW_0^\top\vW_0\right)^{-1}\va$ have comparable limits as $N,n,d\to\infty$ proportionally. This directly implies that $-\vbeta_*^\top\vQ_0(\hat\lambda)\vbeta_*$ also has a uniform upper bound for all large $d$. 
We conclude that as $n,d,N\to\infty$ proportionally, $|\vbeta_*^\top\vQ_0(\hat\lambda)\vu|$ is eventually bounded by some constant from above. 

On the other hand, as $n/d\to\psi_1$ and $N/d\to\psi_2$, from \eqref{eq:equation2}, \eqref{eq:infinity1} and \eqref{eq:infinity2} we know that $\hat v_1,\hat v_2\to 0$. This allows us to conclude that $\tilde{\vu}_1^\top\vbeta_*\to 0$. 
Finally, to translate the result back to $\vu_1$ which is the left singular vector of $\vW_1$, we denote $\vR:=\vW_1\vW_1^\top-\tilde\vW_1\tilde\vW_1^\top$. Recall that Lemma \ref{lemm:gradient-norm} ensures $\|\vR\|\to 0$ almost surely. Hence, we may repeat above computations for $\hat\lambda\vu_1=\vW_1\vW_1^\top\vu_1$, where with a slight abuse of notation we still denote $\hat\lambda $ as the largest eigenvalue of $\vW_1\vW_1^\top$. In this case, \eqref{eq:equation3} needs to be modified as
\[ \vu_1^\top\vbeta_*=-\left(\hat v_1\left(\|\va\|^2\vbeta_*^\top\vQ_0(\hat\lambda)\vu+\vbeta_*^\top\vQ_0(\hat\lambda)\vW_0\va\right)+\hat v_2\cdot \vbeta_*^\top\vQ_0(\hat\lambda)\vu\right)-\vbeta_*^\top\vQ_0(\hat\lambda)\vR\vu_1. \]
Since $\|\vR\|\to 0$, by a simple adaptation of Lemma A.2 in \cite{benaych2011eigenvalues}, one can directly verify $\vbeta_*^\top\vQ_0(\hat\lambda)\vR\vu_1\to 0$, as $N,n,d\to \infty$ proportionally. Hence we conclude that $\vu_1^\top\vbeta_*\to 0$ if $\theta_1^4\le \psi_2$. 
The theorem is established by combining the above cases. 

\end{proofof}

\bigskip
\section{Proof for Small Learning Rate ($\eta=\Theta(1)$)}

\subsection{Gaussian Equivalence for Trained Feature Map}

\paragraph{The Gaussian Equivalence Property.} To validate Theorem~\ref{thm:GET}, we follow the proof strategy of \citet{hu2020universality}, which established the GET for RF models using the Lindeberg approach and leave-one-out arguments \citep{el2018impact}.
We remark that concurrent to our work, \citet{montanari2022universality} proved the Gaussian equivalence property for a larger model class under an assumed central limit theorem, which is verified for two-layer RF or NTK models, and thus cannot directly imply our results on the trained features.  

We first introduce the notations used in this section. Given weight matrix $\vW$ and input $\vx$, we define the feature vector $\vphi_{\vx} = \frac{1}{\sqrt{N}}\sigma(\vW^\top\vx)\in\R^N$; similarly, given training data matrix $\tilde{\vX}\in\R^{n\times d}$, the kernel feature matrix is given as $\vPhi = \frac{1}{\sqrt{N}}\sigma(\tilde{\vX}\vW)\in\R^{n\times N}$. 
Also, the linearized noisy Gaussian feature can be written as: $\bar{\vphi}_{\vx} = \frac{1}{\sqrt{N}}\left(\mu_1\vW^\top\vx + \mu_2\vz\right)$, and the corresponding matrix $\bar{\vPhi} = \frac{1}{\sqrt{N}}\left(\mu_1\tilde{\vX}\vW + \mu_2\vZ\right)$, where $\vz,[\vZ]_i\iid\cN(0,\vI)$ for $i\in [n]$. 
We emphasize that in our analysis $\vW$ does not depend on $\tilde{\vX}$; for notational simplicity, in this subsection we omit the accent in $\tilde{\vX}$.  

We establish the Gaussian equivalence property (Theorem~\ref{thm:GET}) for kernel regression with respect to certain trained feature map under general convex loss $\ell$ satisfying Assumption (A.4) in \citet{hu2020universality}. 
Consider the estimators obtained from $\ell_2$-regularized empirical risk minimization:
\begin{align}
    &\hat{\va} \triangleq  \text{arg\,min}_{\va} \,\Big\{\frac{1}{n}\sum_{i=1}^n \ell(y_i,  \langle\va,\vphi_i\rangle) + \frac{\lambda}{N} \norm{\va}_2^2\Big\}, \label{eq:RFRR_nonlin} \\
    &\bar{\va} \triangleq  \text{arg\,min}_{\va} \,\Big\{\frac{1}{n}\sum_{i=1}^n \ell(y_i,  \langle\va,\bar{\vphi}_i\rangle) + \frac{\lambda}{N} \norm{\va}_2^2\Big\},\label{eq:RFRR_lin}
\end{align} 
where we abbreviated $\vphi_i = \vphi_{\vx_i} =  \frac{1}{\sqrt{N}}\sigma(\vW^\top\vx_i), \bar{\vphi}_i = \bar{\vphi}_{\vx_i} =   \frac{1}{\sqrt{N}}\left(\mu_1\vW^\top\vx_i + \mu_2\vz_i\right)$ for $i\in [n]$.

In our setting, the first-layer weight $\vW$ is no longer the initialized random matrix $\vW_0$. However, we can still write the weight matrix as a perturbed version of $\vW_0$, i.e., $\vW = \vW_0 + \vDelta$, where $\vDelta\in\R^{d\times N}$ corresponds to the update to the weights (possibly multiple gradient steps as in Lemma~\ref{lemm:gradient-norm-multi}) that is \textit{independent of} the training data $\vX$ for ridge regression (e.g., the weight matrix and the ridge regression estimator are trained on separate data). We aim to show that under suitable conditions on $\vDelta$, the Gaussian equivalence theorem holds for the kernel model defined by the perturbed features $\vx\to\frac{1}{\sqrt{N}}\sigma(\vx^\top\vW)$. Throughout this section, we take $\vW$ to be the trained first layer $\vW_t$ for $t\in\N$ according to \eqref{eq:gradient-step-MSE}. 

Define the set of weight matrices perturbed from the Gaussian initialization $\vW_0$ as
\begin{align}
\label{eq:orthogonality-condition}
    \cW: = \left\{\vW = \vW_0 + \vDelta \in\R^{d\times N} :\, \norm{\vDelta} = \bigO{1}, ~ \norm{\vDelta}_{2,\infty} = \bigO{\frac{\text{polylog}\,d}{\sqrt{d}}}\right\}.
\end{align}
Note that for learning rate $\eta=\Theta(1)$, we can verify that $\cW$ is a high-probability event after any finite number of gradient steps, as characterized in Lemma~\ref{lemm:gradient-norm} and \ref{lemm:gradient-norm-multi}.  The following proposition is a reformulation and extension of \cite[Theorem 1]{hu2020universality}, stating that the Gaussian equivalence property holds as long as $\vW$ remains ``close'' to the initialization $\vW_0$. 
\begin{prop}
\label{prop:GET-perturbed}
Under Assumptions \ref{assump:1} and \ref{assump:2}, and $\P(\cW) \ge 1 - \Exp{-c\log^2 N}$ for some $c>0$, we have that as $n,d,N\to\infty$ proportionally,  
$$\E_{\vx} \left(f^*(\vx) - \langle\vphi_{\vx}, \hat{\va}\rangle\right)^2 = (1+o_{d,\P}(1))\cdot \E_{\vx} \left(f^*(\vx) - \langle\bar{\vphi}_{\vx}, \bar{\va}\rangle\right)^2,$$
where $\hat{\va}$ and $\bar\va$ are defined in \eqref{eq:RFRR_nonlin} and \eqref{eq:RFRR_lin}.
\end{prop}
From Proposition~\ref{prop:GET-perturbed} we know that Theorem~\ref{thm:GET} holds if the optimized weight matrix $\vW$ falls into the set $\cW$ with sufficiently high probability. This condition is in turn verified by Lemma~\ref{lemm:gradient-norm} and \ref{lemm:gradient-norm-multi}. 
Also note that in our setting of MSE loss and $\lambda>0$, the RHS of the above equation is bounded in probability.

\paragraph{Central Limit Theorem for Trained Features.}

Recall the single-index teacher assumption: $y_i = \sigma^*(\langle\vx_i,\vbeta^*\rangle) + \varepsilon_i$ for $i\in[n]$.  
Observe that for $\vW\in\cW$, the following near-orthogonality condition between the neurons holds with high probability
\begin{align}
    \norm{\vW} = \bigO{1}, \quad\text{and } 
    \max_{i\neq j}\,\left\{\langle\vw_i,\vw_j\rangle, \langle\vw_i,\vbeta_*\rangle\right\} = \bigO{\frac{\text{polylog}\,d}{\sqrt{d}}}.
\label{eq:near-orth} 
\end{align}

Importantly, for $\vW$ satisfying the near-orthogonality condition \eqref{eq:near-orth}, we can utilize the following central limit theorem from \citet{hu2020universality} derived via Stein's method.  
\begin{prop}[Theorem 2 in \citep{hu2020universality}]
Given Assumptions \ref{assump:1} and \ref{assump:2}, suppose that the activation $\sigma$ is an odd function. Let $\{\varphi_d(x;y)\}$ be a sequence of two-dimensional test functions, where $|\varphi_d(x;y)|,|\varphi_d'(x;y)|\le B_d(y)\left(1 + |x|\right)^K$ for some function $B_d$ and constant $K\ge 1$, then for $\vW$ satisfying \eqref{eq:near-orth}, and fixed vectors $\valpha\in\R^N, \vbeta\in\R^d$ with $\norm{\vbeta}=1$, we have  
\begin{align}
    &\abs{\E\varphi_d\left(\vphi_{\vx}^\top\valpha; \vx^\top\vbeta\right) -  \E\varphi_d\left(\bar{\vphi}_{\vx}^\top\valpha; \vx^\top\vbeta\right)} = 
    \bigO{\frac{\mathrm{polylog} N}{\sqrt{N}}\E[B_d(z)^{4}]^{1/4}\left(1+\norm{\valpha}_\infty^2 + \big(\tfrac{1}{\sqrt{N}}\norm{\valpha}\big)^{K'}\right)},
\end{align}
where $z\sim\cN(0,1)$, and $K'$ only depends on constant $K$. 
\label{prop:CLT}
\end{prop} 
 
We remark that our assumption of odd activation in Theorem~\ref{thm:GET} is required by the above Proposition~\ref{prop:CLT}, and we believe it could be removed with some extra work. Also, to verify the GET, we take $\varphi_d$ to be the test function defined in \cite[Equation (50)]{hu2020universality}. 
In our case, by \cite[Lemma 25]{hu2020universality} we know that there exists a function $B$ satisfying the growth condition such that $\E[B(z)^4]$ is bounded. Therefore, in order to apply Proposition~\ref{prop:CLT} and obtain the Gaussian equivalence theorem (see derivation in \cite[Section 2]{hu2020universality}),
we only need to control the $\ell_2$-norm and $\ell_\infty$-norm of certain vector $\valpha$ of interest. The following subsection establishes the required norm bound.

\paragraph{Norm Control Along the Interpolation Path.}
Following \cite{hu2020universality}, we construct an interpolating sequence between the nonlinear and linear features model. For any $0\le k\le n$, we define
\begin{align}  
    \!\!\!\!\!\! \vg_k^* \triangleq\text{arg\,min}_{\vg\in\R^N} \,\Bigg\{\sum_{i=1}^k \ell(y_i,  \langle\vg,\bar{\vphi}_i\rangle) +\! \sum_{j=k+1}^n \ell(y_j, \langle\vg,\vphi_j\rangle) + \frac{n}{N}\left(\lambda\norm{\vg}_2^2 + Q(\vg)\right)\Bigg\},   
\label{eq:interpolation-objective}
\end{align}  
where we introduce a perturbation term $$Q(\vg) \triangleq \gamma_1\vg^\top\left(\mu_1^2\vW^\top\vW + \mu_2^2\vI\right)\vg + \gamma_2\mu_1\sqrt{N}\vbeta_*^\top\vW\vg.$$ 
Note that when $\gamma_1=\gamma_2=0$, setting $k=0$ recovers the estimator on nonlinear features $\hat{\va}$, and similarly, setting $k=n$ gives the estimator on the linear Gaussian features $\bar{\va}$. 

We remark that the perturbation $Q(\vg)$ allows us to compute the prediction risk by taking the derivative of the objective w.r.t.~$\gamma_1,\gamma_2$ around 0 --- see \cite[Proposition 1]{hu2020universality} for details. 
Note that when $\norm{\vW}=\Theta(1)$, we may choose $\gamma^* = \frac{N}{n}\cdot\frac{\lambda/4}{\mu_1^2\norm{\vW}^2 + \mu_2^2}>0$ such that for $\abs{\gamma_1}\le\gamma^*, \abs{\gamma_2}\le 1$, the overall objective \eqref{eq:interpolation-objective} is $\frac{\lambda}{2}$-strongly convex (i.e., the strongly-convex regularizer dominates the concave part of $Q(\vg)$ when $\gamma_1<0$).  

While most of the statements in \citet{hu2020universality} hold for deterministic weight matrices satisfying \eqref{eq:near-orth}, the $\ell_\infty$-norm bound relies on the (sub-)Gaussian property of $\vW$ and thus only applies to RF models. The following lemma establishes a high probability upper bound on the $\ell_\infty$-norm of $\vg_k^*$ on our trained feature map. 
\begin{lemm} 
\label{lemm:sup-norm}
Given Assumptions \ref{assump:1} and \ref{assump:2}, if we further assume that $1-\P(\cW) \le \Exp{-c\log^2 N}$ for some constant $c>0$, then there exists some constant $c'>0$ such that for any $0\le k\le n$, 
\begin{align}
    \Parg{\norm{\vg_k^*}_\infty \ge \mathrm{polylog}N} \le \Exp{-c'\log^2 N}.
\end{align}
\end{lemm} 
\begin{proof}
We follow the proof of \cite[Lemma 23]{hu2020universality} and first analyze one coordinate of $\vg_k^*$ defined by \eqref{eq:interpolation-objective}, which WLOG we select to be the last coordinate. 
For concise notation, we instead augment the weight matrix with an $(N+1)$-th column and study the corresponding $[\vg_k^*]_{N+1}$. Denote the weight vector $\vw_{N+1} = \vw^0_{N+1}+\vdelta_{N+1}$, where $\vw^0_{N+1}$ is the $(N+1)$-th column of the initialized $\vW_0$, and $\vdelta$ is the perturbation (i.e., gradient update for $\vW$). 

To further simplify the notation, we define $\vr_i\in\R^N$, where $\vr_i=\frac{1}{\sqrt{N}}\left(\mu_1\vW^\top\vx_i + \mu_2\vz_i\right)$, $\vz_i\iid\cN(0,\vI)$ for $i\le k$, and $\vr_i= \frac{1}{\sqrt{N}}\sigma(\vW^\top\vx_i)$ for $k<i\le n$. 
Recall that $\vW = \vW_0+\vDelta$, in which the initialization $[\vW_0]_{i,j} = \cN(0,d^{-1})$; we denote the $i$-th feature vector at initialization $\vW_0$ by $\vr^0_i$. 
In addition, we define $\vf\in\R^n$ to represent the feature vector at the last coordinate, i.e., $f_i=[\vf]_i = \frac{1}{\sqrt{N}}\left(\mu_1\vx_i^\top\vw_{N+1} + \mu_2 z_i\right)$, $z_i\iid\cN(0,1)$ for $i\le k$, and $f_i=[\vf]_i = \frac{1}{\sqrt{N}}\sigma(\vx_i^\top\vw_{N+1})$ for $k<i\le n$; similarly, we introduce a superscript in $\vf^0\in\R^n$ to denote the features produced by the initial $\vw^0_{N+1}$. 
 
The $(N+1)$-th coordinate of interest, which we denote as $u^*$, can be written as the solution to the following optimization problem,
\begin{align}
    u^* &= \argmin_u \min_{\vg} \sum_{i=1}^n \ell\left(\vr_i^\top\vg + f_i u; y_t\right) 
    + \frac{n}{N}\left(\lambda\norm{\vg}^2 + Q(\vg) + \lambda u^2 + q(u) + \left(2\gamma_1\mu_1^2\vw_{N+1}^\top\vW\vg \right)u\right),  
\end{align}  
where we defined 
$$
q(u) = \gamma_1\left(\mu_1^2\norm{\vw_{N+1}}^2 + \mu_2^2\right)u^2 + \gamma_2\left(\mu_1 \sqrt{N}\vbeta_*^\top\vw_{N+1}\right)u. 
$$
By \cite[Equation (249)]{hu2020universality}, we know that for $\vW\in\cW$, 
\begin{align}
    |u^*| \lesssim \frac{1}{\lambda} \abs{2\gamma_1\mu_1^2\vw_{N+1}^\top\vW\vg_k^* + \gamma_2\mu_1\sqrt{N}\vbeta_*^\top\vw_{N+1} + \sum_{i=1}^n \ell'\left(\vr_i^\top\vg_k^*; y_i\right)f_i}. \label{eq:u*bound}
\end{align}
We control each term on the right hand side of \eqref{eq:u*bound} separately. Note that $\vW\in\cW$ implies that $\norm{\vdelta_{N+1}} = \bigO{\frac{\text{polylog} d}{\sqrt{d}}}$ due to the definition \eqref{eq:orthogonality-condition}. Since $\abs{\vbeta_*^\top\vw_{N+1}} \le \abs{\vbeta_*^\top\vw^0_{N+1}} + \norm{\vdelta_{N+1}}\norm{\vbeta_*}$, by combining \cite[Equation (252)]{hu2020universality} and our assumption that $\norm{\vbeta_*} = 1$, we know that for some constant $c_1>0$ and large $N$,  
\begin{align}
    \Parg{\abs{\sqrt{N}\vbeta_*^\top\vw_{N+1}}\ge\text{polylog} N} \le \Exp{-c_1\log^2 N}. 
\end{align}
Similarly, $\abs{\vw_{N+1}^\top\vW\vg_k^*}\le\abs{\vw^{0^\top}_{N+1}\vW\vg_k^*} + \norm{\vdelta_{N+1}}\norm{\vW\vg_k^*}$, and therefore by \cite[Lemma 17]{hu2020universality} (note that the lemma only requires $\vW$ to satisfy \eqref{eq:near-orth}), we have 
\begin{align}
    \Parg{\abs{\vw_{N+1}^\top\vW\vg_k^*}\ge\text{polylog} N} \le \Exp{-c_2\log^2 N},
\end{align}for some constant $c_2>0$. 

To control the sum of $\ell'$ in \eqref{eq:u*bound}, for simplicity we define $\vtheta^*\in\R^n$, where $\theta_i^* =[\vtheta^*]_i= \ell'\left(\vr_i^\top\vg_k^*; y_i\right)$ for $i\in [n]$. 
Notice that $\abs{\vx_i^\top\vw_j - \vx_i^\top\vw_j^0} = \abs{\vx_i^\top\vdelta_j}$. Due to the assumed independence between $\vX$ and $\vDelta$, and the assumption on $\P(\cW)$, we know that $\abs{\vx_i^\top\vdelta_j}\lesssim\norm{\vdelta_j}\cdot\log N = \bigO{\frac{\text{polylog}N}{\sqrt{N}}}$ with high probability. 
In addition, since the activation function $\sigma$ is Lipschitz, for $k<i\le n$, we may take a union bound over the weight vectors $\vw_j$ and obtain
\begin{align}
    \Parg{\sqrt{N}\norm{\vr_i - \vr_i^0}\ge\text{polylog}N} \le  N\cdot\Exp{-c_3\log^2 N},
    \label{eq:r_i-substitution}
\end{align}
for some $c_3>0$. The case where $i\le k$ (i.e., the features are linear) follows from the exact same argument. Also, because of $\abs{\vr_i^\top\vg_k^*} \le \abs{\vr_i^{0\top}\vg_k^*} + \norm{\vr_i - \vr_i^0}\norm{\vg_k^*}$, we know that \cite[Equation (257)]{hu2020universality}, \cite[Lemma 17]{hu2020universality}, and \eqref{eq:r_i-substitution} together ensure that
\begin{align}
    \Parg{\abs{\theta_i^*}\ge\text{polylog}N}\le\Exp{-c_4\log^2 N}, 
    \label{eq:theta-bound}
\end{align} 
for some constant $c_4>0$ and large enough $N$. Now we can control $\abs{\sum_{i=1}^n \ell'\left(\vr_i^\top\vg_k^*; y_i\right)f_i}$ in \eqref{eq:u*bound}. Again using the Lipschitz property of activation $\sigma$, we get
\begin{align}
    &\abs{\sum_{i=1}^n \ell'\left(\vr_i^\top\vg_k^*; y_i\right)(f_i-f_i^0 + f_i^0)} \\
\le~& 
    \abs{\sum_{i=1}^n \ell'\left(\vr_i^\top\vg_k^*; y_i\right)f^0_i} + \abs{\sum_{i=1}^n \ell'\left(\vr_i^\top\vg_k^*; y_i\right)\left(f_i - f^0_i\right)} \\
\lesssim~&
    \abs{\sum_{i=1}^n \theta_i^* f^0_i} + \frac{1}{\sqrt{N}}\sum_{i=1}^n \abs{\theta_i^*}\cdot\abs{\vx_i^\top\vdelta_{N+1}}.
\end{align}
Given \eqref{eq:theta-bound} (which implies that $\frac{1}{\sqrt{N}}\norm{\vtheta^*}=\bigO{\text{polylog}N}$ with high probability), it has been shown in \cite[Proof of Lemma 23]{hu2020universality} that $\Parg{\abs{\sum_{i=1}^n \theta_i^* f^0_i} \ge \text{polylog}N}\le\Exp{-c_5\log^2 N}$ for some constant $c_5>0$. Hence, by taking union bound over the failure events $\abs{\theta_i^*}\ge\text{polylog}N$ and $\sqrt{N}\cdot\abs{\vx_i^\top\vdelta_{N+1}}\ge\text{polylog}N$, we arrive at the following high probability upper bound on $u^*$ in terms of \eqref{eq:u*bound}:
\begin{align}
    \Parg{\abs{u^*}\ge\text{polylog}N} \le n^2 N\cdot\Exp{-c_6\log^2 N},
\end{align}for some constant $c_6>0$ and all large $N$.
Finally, since the assumption on $\norm{\vDelta}_{2,\infty}$  implies  control of $\norm{\vdelta_i}$ for all $i\in[ N]$, we complete the proof by a union bound over the $N$ coordinates.  

\end{proof}

\paragraph{Putting Things Together.}
Denote the optimal value of objective \eqref{eq:interpolation-objective} by
\begin{align}  
   \!\!\!\!R_k^* \triangleq\text{min}_{\vg\in\R^N} \,\Bigg\{\frac{1}{n}\sum_{i=1}^k \ell(y_i,  \langle\vg,\bar{\vphi}_i\rangle) + \frac{1}{n}\sum_{j=k+1}^n \ell(y_j, \langle\vg,\vphi_j\rangle) + \frac{1}{N}\left(\lambda\norm{\vg}_2^2 + Q(\vg)\right)\Bigg\}   
\end{align} 
From \cite[Section 2.3]{hu2020universality}, we know that Proposition~\ref{prop:CLT} and Lemma~\ref{lemm:sup-norm} imply that for any $\vW\in\cW$ and $1\le k\le n$, the discrepancy due to one swap can be bounded as
$$\Big|\E\left[\psi(R_k^*)\right] - \E\left[\psi(R_{k-1}^*)\right]\Big| = \bigO{\frac{\text{polylog}N}{N^{3/2}}},$$
for bounded test function $\psi$ with bounded first and second derivatives. 
As there are $n=\Theta(N)$ total swaps to be made, we can obtain the desired Gaussian equivalence  (\cite[Theorem 1]{hu2020universality}) if the failure probability $(1-\P(\cW))$ is sufficiently small. Hence we can conclude Proposition~\ref{prop:GET-perturbed}. 

\begin{proofof}[Theorem~\ref{thm:GET}]
Finally, we establish Theorem~\ref{thm:GET} by verifying that in our setting the event $\cW$ occurs with high probability. For one gradient step on the squared loss with learning rate $\eta=\Theta(1)$, Lemma~\ref{lemm:gradient-norm} together with $\norm{\vbeta_*}=1$ entail that for proportional $n,d,N$, there exists some constant $c,C>0$ such that
\begin{align}
    \Parg{\norm{\vW_1}\ge C} &\le \Exp{-cd}, \\
    \Parg{\max_{i\neq j} \abs{\langle\vw^1_{i},\vw^1_{j}\rangle} \ge \frac{C\log^2 d}{\sqrt{d}}} &\le \Exp{-c\log^2 d}, \\
    \Parg{\max_{i} \abs{\langle\vw^1_{i},\vbeta_*\rangle} \ge \frac{C\log^2 d}{\sqrt{d}}} &\le \Exp{-c\log^2 d},
\end{align}
where $\vw^1_i$ stands for the $i$-th column of $\vW_1$ for $i\in [N]$. 
For multiple gradient steps with $\eta=\Theta(1)$, Lemma~\ref{lemm:gradient-norm-multi} implies similar tail probability bounds. 
In addition, under Assumptions \ref{assump:1} and \ref{assump:2}, when $\lambda>0$, it is straightforward to verify that prediction risk of the Gaussian equivalent model $\cR_{\textrm{GE}}(\lambda)\overset{\P}{\to} C_{\lambda}$ for some finite constant $C_\lambda>0$ as $n,N,d\to\infty$ proportionally. 
Theorem~\ref{thm:GET} therefore follows from Proposition~\ref{prop:GET-perturbed} (or equivalently, Equation (16) in \cite[Theorem 1]{hu2020universality}).  

\end{proofof}

\subsection{Prediction Risk of the Gaussian Equivalent Model} 
Now we compute the prediction risk of the CK ridge estimator on the feature map after one gradient step $\vx\to\sigma(\vW_1^\top\vx)$. 
We restrict ourselves to the squared loss, the optimal solution of which is given by:
\begin{align}
    \hat{\va} = \text{arg\,min}_{\va}
=
    \left(\vPhi^\top\vPhi + \frac{\lambda n}{N}\vI\right)^{-1} \vPhi^\top\tilde{\vy},
\end{align} 
where $\vPhi = \frac{1}{\sqrt{N}}\sigma(\tilde{\vX}\vW_1)\in\R^{n\times N}$, $\tilde{\vX}\in\R^{n\times d}$ denotes a new batch of training data independent of $\vW_1$, and $\tilde{\vy}=f^*(\tilde{\vX})+\tilde{\vvarepsilon} \in\R^n$ is the corresponding training labels (following the same Assumption~\ref{assump:1}). Also, recall the following Gaussian covariates model:  
\begin{align}
    &\barPhi \triangleq \frac{1}{\sqrt{N}}\left(\mu_1\tilde{\vX}\vW_1 + \mu_2 \vZ \right) \in\R^{n\times N}; \quad
    \bar{\va}\triangleq \left(\bar{\vPhi}^\top\bar{\vPhi} + \frac{\lambda n}{N}\vI\right)^{-1} \bar{\vPhi}^\top\tilde{\vy}. 
\end{align}
where $[\vZ]_{ij}\sim\cN(0,1)$ independent of $\tilde{\vX}$ and $\vW_1$. 
Due to the Gaussian equivalence property \eqref{eq:GET}, we can analyze the prediction risk of the Gaussian covariates model, which we denote as $\cR_{\mathrm{GE}}(\lambda)$. 

\paragraph{Bias-variance Decomposition.}
The following lemma simplifies the prediction risk $\cR_{\mathrm{GE}}(\lambda)$ and separates the bias (due to learning the teacher $f^*$) and variance (due to the label noise $\tilde\vvarepsilon$).  
\begin{lemm}\label{lemm:bias-variance}
Under Assumptions~\ref{assump:1} and \ref{assump:2}, we have
\begin{align}
    \cR_{\mathrm{GE}}(\lambda) -
    \left(B_1 + B_2 + V\right)\overset{\P}{\to} 0,
\end{align}
where the bias and variance terms are given as
\begin{align} 
    B_1  
=&~
    \mu_1^{*2} + \mu_2^{*2} - \frac{2\mu_1\mu_1^{*}}{\sqrt{N}}\vbeta_*^\top\vW_1\left(\hSigmaPhi + \tlambda\vI\right)^{-1} \bar{\vPhi}^\top\vf^*. \label{eq:B_1}\\
    B_2 
=&~  
    \vf^{*\top}\bar{\vPhi}\left(\hSigmaPhi + \tlambda\vI\right)^{-1}\bSigmaPhi\left(\hSigmaPhi + \tlambda\vI\right)^{-1} \bar{\vPhi}^\top\vf^*.  \label{eq:B_2}\\
    V 
=&~ 
    \sigma_\eps^2\Trarg{\left(\hSigmaPhi + \tlambda\vI\right)^{-1} \hSigmaPhi\left(\hSigmaPhi + \tlambda\vI\right)^{-1}\bSigmaPhi}.  \label{eq:V}
\end{align}
and we defined $\tlambda = \frac{\lambda n}{N}, \hSigmaPhi = \barPhi^\top\barPhi, \bSigmaPhi = \frac{1}{N}\left(\mu_1^2\vW_1^\top\vW_1 + \mu_2^2\vI\right)$, and $[\vf^*]_i = f^*(\tilde{\vx}_i)$.  
\end{lemm}
\begin{proof}
First note that $\cR_{\mathrm{GE}}$ is given by \cite[Equation (57)]{hu2020universality}:
\begin{align}
    &\cR_{\mathrm{GE}} =  \E_{\vx}\left(\sigma^*(\vx^\top\vbeta_*) - \bar{\vphi}_{\vx}^\top\bar{\va}\right)^2 \\
=&  
    \E_{z_1,z_2}\left[\left(\sigma^*(z_1) - \frac{\mu_1}{\sqrt{N}}\vbeta_*^\top\vW_1\bar{\va}\cdot z_1  
    +\sqrt{\frac{1}{N}\bar{\va}^\top(\mu_1^2\vW_1^\top\vW_1 + \mu_2^2\vI -  \mu_1^2\vW_1^\top\vbeta_*\vbeta_*^\top\vW_1)\bar{\va}}\cdot z_2\right)^2\right]
\label{eq:GET-risk}
\end{align}
where $z_1,z_2\iid\cN(0,1)$. 
Because of the independence between $z_1,z_2$, we only need to show the following as $n,d,N\to\infty$ proportionally: 
\begin{align} 
    &\frac{1}{\sqrt{N}}\vbeta_*^\top\vW_1\left(\hSigmaPhi + \tlambda\vI\right)^{-1} \bar{\vPhi}^\top\tilde{\vvarepsilon} \overset{\P}{\to} 0, \label{eq:betaeps}\\  
    & \vf^{*\top}\bar{\vPhi}\left(\hSigmaPhi + \tlambda\vI\right)^{-1}\bSigmaPhi\left(\hSigmaPhi + \tlambda\vI\right)^{-1} \bar{\vPhi}^\top\tilde{\vvarepsilon} \overset{\P}{\to}
    0. \label{eq:betaeps2}
\end{align}
Both equations directly follow from the general Hoeffding inequality for $\tilde\vvarepsilon$ (e.g., see Theorem 2.6.3 \cite{vershynin2018high}) since both $\norm{\vbeta_*^\top\vW_1\left(\hSigmaPhi + \tlambda\vI\right)^{-1} \bar{\vPhi}^\top}$ and $\norm{\sqrt{N}\cdot\vf^{*\top}\bar{\vPhi}\left(\hSigmaPhi + \tlambda\vI\right)^{-1}\bSigmaPhi\left(\hSigmaPhi + \tlambda\vI\right)^{-1} \bar{\vPhi}^\top}$ are bounded by some constant with high probability when $\lambda>0$.  

\end{proof}

Also, the risk lower bound for the Gaussian equivalent model is a direct consequence of \eqref{eq:GET-risk}. 
\begin{proofof}[Fact \ref{fact:GET-lower-bound}]
Under Assumptions \ref{assump:1} and \ref{assump:2}, we may write $\sigma^*(z) = \mu_1^* z + \sigma_\perp^*(z)$, where $\E_z[z\sigma_\perp^*(z)] = 0, \E_z[\sigma_\perp^*(z)^2] = \mu_2^{*2}$ for $z\sim\cN(0,1)$. 
Hence from \eqref{eq:GET-risk} we know that
\begin{align}
    \cR_{\mathrm{GE}} 
\ge  
    \E_{z_1}\left(\sigma^*(z_1) - \frac{\mu_1}{\sqrt{N}}\vbeta_*^\top\vW_1\bar{\va}\cdot z_1\right)^2
=
    \left(\mu_1^* - \frac{\mu_1}{\sqrt{N}}\vbeta_*^\top\vW_1\bar{\va}\right)^2 + \mu_2^{*2}.
\end{align}
This implies that $\cR_{\mathrm{GE}} 
\ge \norm{\textsf{P}_{>1}f^*}_{L^2}^2 = \mu_2^{*2}$ with probability one as $d\to\infty$.  

\end{proofof}

In the following sections, we compare the bias and variance terms given in \eqref{eq:B_1}, \eqref{eq:B_2} and \eqref{eq:V} before and after one feature learning step. We first simplify the calculation by showing that the values of these equations remain asymptotically unchanged if we remove certain low-order terms.

\paragraph{Stability of the Bias and Variance.}

We now control the errors in the bias and variance terms after ignoring the lower-order terms in the weight matrix. 

Recall that $\vW_1 = \vW_0+\eta\sqrt{N}\vG_0$; we introduce $\tilde{\vW}: = \vW_0+\eta\sqrt{N}\vA$, in which we ignored the terms $\vB$ and $\vC$ in the gradient matrix \eqref{eq:gradient-step}. 
We also denote the corresponding CK features and kernel matrix as $\tilde{\vPhi}:=\frac{1}{\sqrt{N}}\left(\mu_1\tilde{\vX}\tilde{\vW} + \mu_2 \vZ \right)$, $\tilde{\vSigma}_\Phi := \tilde{\vPhi}^\top\tilde{\vPhi}$, and the bias terms as $\tilde{B}_1, \tilde{B}_2$ (parallel to \eqref{eq:B_1} and \eqref{eq:B_2}).
Finally, we write the initial random feature matrix as $\bar{\vPhi}_0:=\frac{1}{\sqrt{N}}\left(\mu_1\tilde{\vX}\vW_0 + \mu_2 \vZ \right)$, $\hSigmaPhii := \bar{\vPhi}_0^\top\bar{\vPhi}_0$, and refer to the variance of the initialized RF ridge estimator as $V_0$. 

\begin{lemm}\label{lemm:perturbation-stability}
Given Assumptions \ref{assump:1}, \ref{assump:2} and $\lambda>0$. Then for $\eta=\Theta(1)$, we have $$|B_1-\tilde{B}_1|= o_{d,\P}(1),~ |B_2-\tilde{B}_2|= o_{d,\P}(1), ~|V-V_0| = o_{d,\P}(1).$$ 
\end{lemm}
\begin{proof}
To start with, recall that the operator norms of all matrices $\vW_1,\vW_0,\tilde\vW,\bar\vPhi,\bar\vPhi_0$ and $\tilde\vPhi$ are uniformly bounded by some constants with high probability. We first consider the change in Frobenius norm of first-layer $\vW$ to analyze the difference between $V$ and $V_0$. 
By Lemma~\ref{lemm:gradient-norm}, standard calculation yields:
\begin{align} 
    \norm{\vW_1^\top\vW_1 - \vW_0^\top\vW_0}_F = \cO_{d,\P}(1); \quad
    \norm{\bar{\vPhi} - \bar{\vPhi}_0}_F = \cO_{d,\P}(1); \quad
    \norm{\hSigmaPhi - \hSigmaPhii}_F = \cO_{d,\P}(1).
\end{align}
Utilizing the above estimates, we obtain 
\begin{align}
    &\norm{\left(\hSigmaPhi + \tlambda \vI\right)^{-1} \bar{\vPhi}^\top - \left(\hSigmaPhii + \tlambda \vI\right)^{-1} \bar{\vPhi}_0^\top}_F \\
\le\,&
    \norm{\bar{\vPhi}-\bar{\vPhi}_0}_F\norm{\left(\hSigmaPhi + \tlambda \vI\right)^{-1}} 
    + \norm{\bar{\vPhi}_0}\norm{\left(\hSigmaPhi + \tlambda\vI\right)^{-1} - \left(\hSigmaPhii + \tlambda\vI\right)^{-1}}_F 
\overset{(i)}{=}
    \bigOdp{1}. 
\end{align}
where $(i)$ is due to our assumption that $\lambda>0$. 
Denote $\vM := \left(\hSigmaPhi + \tlambda\vI\right)^{-1} \bar{\vPhi}^\top$ and likewise $\vM_0:=\left( \hSigmaPhii + \tlambda\vI\right)^{-1} \bar{\vPhi}_0^\top$. Then we have
\begin{align} 
    &|V - V_0| 
\overset{(ii)}{\lesssim}\, 
    \frac{1}{N}\left|\Trarg{\vM\vM^\top \left(\mu_1^2\vW_1^\top\vW_1 + \mu_2^2\vI\right) - \vM_0\vM_0^\top \left(\mu_1^2\vW_0^\top\vW_0 + \mu_2^2\vI\right)}\right| \\
\lesssim~&  
    \frac{1}{N}\norm{\vM_0\vM_0^\top}_F\cdot \norm{\vW_1^\top\vW_1 - \vW_0^\top\vW_0}_F + \frac{1}{N}\norm{\vM\vM^\top - \vM_0\vM_0^\top}_F\cdot\norm{\mu_1^2\vW_1^\top\vW_1 + \mu_2^2\vI}_F 
=
    o_{d,\P}(1),
\end{align}  
as $n,d,N\to\infty$ at comparable rate, where we dropped the constant $\sigma_\eps^2$ in $(ii)$.  

For the bias terms, we consider perturbation on $\vW_1$ in the operator norm. Again, Lemma~\ref{lemm:gradient-norm} entails that 
\begin{align} 
    \norm{\vW_1^\top\vW_1 - \tilde{\vW}^\top\tilde{\vW}} = \littleodp{1}; \quad
    \norm{\bar{\vPhi} - \tilde{\vPhi}} = \littleodp{1}; \quad
    \norm{\hSigmaPhi - \tSigmaPhi} = \littleodp{1}.
\end{align}
Define $\tilde\vM:=\left(\tilde\vSigma_{\vPhi} + \tlambda\vI\right)^{-1} \tilde{\vPhi}^\top$. Following the same procedure, we obtain the operator norm control
\begin{align}
    \norm{\vM-\tilde{\vM}}
=~&
    \norm{\left(\hSigmaPhi + \tlambda \vI\right)^{-1} \bar{\vPhi}^\top - \left(\tSigmaPhi + \tlambda \vI\right)^{-1} \tilde{\vPhi}^\top} \\
\le~&
    \snorm{\bar{\vPhi}-\tilde{\vPhi}}\norm{\left(\hSigmaPhi + \tlambda\vI\right)^{-1}} 
    + \snorm{\tilde{\vPhi}}\norm{\left(\hSigmaPhi + \tlambda\vI\right)^{-1} - \left(\tSigmaPhi + \tlambda\vI\right)^{-1}}
=
    \littleodp{1}.    
\end{align}
Based on this result, it is straightforward to show that 
\begin{align}
    &|B_1 - \tilde{B}_1| 
\lesssim
    \norm{\vW_1-\tilde{\vW}}\snorm{\tilde{\vM}} + \norm{\vW_1}\norm{\vM - \tilde{\vM}}
=
    \littleodp{1}. 
\end{align}
Similarly, for $B_2$, we have
\begin{align}
   |B_2 - \tilde{B}_2| 
\lesssim ~&
    \frac{1}{N}\norm{\vf^*}^2 \cdot \norm{\vM^\top\left(\mu_1^2\vW_1^\top\vW_1 + \mu_2^2\vI\right)\vM - \tilde{\vM}^\top\left(\mu_1^2\tilde{\vW}^\top\tilde{\vW} + \mu_2^2\vI\right)\tilde{\vM}} \\
\overset{(iii)}{\lesssim}~&
    \bigOdp{1} \cdot \left((\norm{\vM}+\snorm{\tilde{\vM}}) \norm{\vW_1^\top\vW_1}\norm{\vM-\tilde{\vM}}+ \snorm{\tilde{\vM}}^2\norm{\vW_1^\top\vW_1-\tilde{\vW}^\top\tilde{\vW}}\right) 
=
    \littleodp{1},  
\end{align}
where in $(iii)$ we used the fact that $\sigma^*$ is Lipschitz and $\norm{\vbeta_*}=1$ (for example see \cite[Lemma A.12]{bartlett2021deep}). 
The statement is proved by combining all the above calculations. 

\end{proof}

Lemma~\ref{lemm:perturbation-stability} entails that the variance term in the risk does not change after one gradient step with $\eta=\Theta(1)$, and for the bias terms, we may consider the rank-1 approximation of the gradient matrix studied in Proposition~\ref{thm:W1-W0} instead. In the following section, we use this property to simplify the risk expressions.

\subsection{Precise Characterization of Prediction Risk}\label{sec:linear_pencil} 

Now we compute the asymptotic expressions of the bias and variance terms defined in Lemma~\ref{lemm:bias-variance}. 
As previously remarked, due to the dependence between the feature matrix $\vPhi$ and the teacher $\vbeta_*$, we cannot naively employ a rotation invariance argument to simplify the calculation (as in \cite{mei2019generalization}). 
Instead, based on the Gaussian equivalence property, we first make use of the Woodbury formula to separate the low-rank terms in the risk expressions. 
In particular, because of Lemma~\ref{lemm:gradient-norm} and Lemma~\ref{lemm:perturbation-stability}, we may simply consider the rank-one approximation of the first-step gradient: $\vW_1=\vW_0+\vu\va^\top$, where $\vu=\frac{\mu_1\eta}{n}\vX^\top\vy$ and $\vy=f^*(\vX)+\vvarepsilon$ satisfying Assumptions~\ref{assump:1} and \ref{assump:2}. Notice here $\vu$, $\tilde{\vX}$, $\vW_0$ and $\va$ are mutually independent. To distinguish the terms in the CK ridge regression estimator using the initial weights $\vW_0$ and the trained weights $\vW_1$, in this section we denote
\begin{equation}\label{def:notions}
  \begin{aligned}
    &\barPhi:=\frac{1}{\sqrt{N}}\left(\mu_1\tilde{\vX}\vW_1 + \mu_2 \vZ \right), & & \vPhi_0:=\frac{1}{\sqrt{N}}\left(\mu_1\tilde{\vX}\vW_0 + \mu_2 \vZ \right),\\
    &\hSigmaPhi:=\barPhi^\top \barPhi,& &\hSigmaPhii:=\vPhi_0^\top \vPhi_0\in\R^{N\times N},\\
    &\vR:=\left(\hSigmaPhi + \tlambda\vI\right)^{-1},&&\vR_0:=\left(\hSigmaPhii + \tlambda\vI\right)^{-1},\\
    &\bSigmaPhi:=\frac{1}{N}\left(\mu_1^2\vW_1^\top\vW_1 + \mu_2^2\vI\right), &&\bSigmaPhii:=\frac{1}{N}\left(\mu_1^2\vW_0^\top\vW_0 + \mu_2^2\vI\right).
\end{aligned}  
\end{equation}
Also, we write $\vf^*:=f^*(\tilde{\vX})\in\R^n$, which can be decomposed into
 \begin{equation}\label{eq:f*fNL}
     \vf^* = \mu_1^*\tilde{\vX}\vbeta_* + \vf^*_{\mathrm{NL}},
 \end{equation}
where $[\vf^*_{\mathrm{NL}}]_i = \mathsf{P}_{>1} f^*\left(\tilde{\vx}_i\right)$ (recall that $\mu_0^* = 0$ by Assumption~\ref{assump:2}). Furthermore, we introduce the following terms which will be important in the decomposition of the prediction risk:
\begin{align}
    &T_1:=\va^\top\vR_0\va,&\quad  &T_2:=\frac{\mu_1^2}{N}\vu^\top\tilde{\vX}^\top\vPhi_0\vR_0\vPhi_0^\top\tilde{\vX}\vu,\\
    &T_3:=\frac{\mu_1^2}{N}\vu^\top\tilde{\vX}^\top\tilde{\vX}\vu,&\quad &T_4:=\mu_1^*\vbeta_*^\top\vu,\\
    &T_5:=\frac{\mu_1^2\mu_1^*}{N}\vbeta_*^\top\tilde{\vX}^\top\tilde{\vX}\vu,&\quad  &\tilde{T}_5:=\frac{\mu_1^2}{N}\vf^{*\top}_{\mathrm{NL}}\tilde{\vX}\vu,\\ &T_6:=\frac{\mu_1\mu_1^{*}}{2\sqrt{N}}\vbeta_*^\top\left(\vW_0\vR_0\vPhi_0^\top\tilde{\vX}+\tilde{\vX}^\top\vPhi_0\vR_0\vW_0^\top\right)\vu,&\quad &\tilde{T}_6:=\frac{\mu_1 }{2\sqrt{N}}\vf^{*\top}_{\mathrm{NL}}\vPhi_0\vR_0\vW_0^\top\vu,\label{def:Tis}\\
    & T_7 := \frac{\mu_1^2\mu_1^{*}}{N}\vu^\top\tilde{\vX}^\top\vPhi_0\vR_0\vPhi_0^\top\tilde{\vX}\vbeta_*,&\quad &\tilde{T}_7 := \frac{\mu_1^2}{N}\vu^\top\tilde{\vX}^\top\vPhi_0\vR_0\vPhi_0^\top\vf^{*}_{\mathrm{NL}},\\ &T_8:=\frac{N}{\mu_1^2}\va^\top \vR_0\bSigmaPhii\vR_0\va,&\quad &T_9:=\frac{\mu_1}{2\sqrt{N}}\vu^\top\left(\vW_0\vR_0\vPhi_0^\top\tilde{\vX}+\tilde{\vX}^\top\vPhi_0\vR_0\vW_0^\top\right)\vu, \!\!\\
    &T_{11}:=\vu^\top\tilde{\vX}^\top\vPhi_0\vR_0\bSigmaPhii\vR_0\vPhi_0^\top\tilde{\vX}\vu,&\quad &T_{10}:=\|\vu\|^2,\\
    &T_{12}:=\mu_1^*\vu^\top\tilde{\vX}^\top\vPhi_0\vR_0\bSigmaPhii\vR_0\vPhi_0^\top\tilde{\vX}\vbeta_*,&\quad &\tilde{T}_{12}:= \vu^\top\tilde{\vX}^\top\vPhi_0\vR_0\bSigmaPhii\vR_0\vPhi_0^\top\vf^{*}_{\mathrm{NL}}.
\end{align}
In the following subsections we will characterize the limiting value of each $T_i$ as $n,d,N\to\infty$ proportionally. 

\subsubsection{Concentration and Simplification}

In the following lemma, we show that each $T_i$ will concentrate around some $T_i^0$ given by
\begin{equation}\label{def:Ti0s}
    \begin{aligned}
    &T_1^0:=\tr\vR_0,&\quad  &T_2^0:=\frac{\mu_1^2}{N}\theta_1^2\trarg{\tilde{\vX}^\top\vPhi_0\vR_0\vPhi_0^\top\tilde{\vX}},\\
    &T_3^0:=\frac{\mu_1^2\theta_1^2}{N}\trarg{\tilde{\vX}^\top\tilde{\vX}},&\quad 
    &T_4^0:=\mu_1^*\theta_2,\\
    &T_5^0:=\frac{\mu_1^2\mu_1^*\theta_2}{N}\trarg{\tilde{\vX}^\top\tilde{\vX}},&\quad & T_6^0:=\frac{\mu_1\mu_1^{*}\theta_2}{\sqrt{N}}\trarg{\vW_0\vR_0\vPhi_0^\top\tilde{\vX}},\\
    & T_7^0 := \frac{\mu_1^2\mu_1^{*}\theta_2}{N}\trarg{\tilde{\vX}^\top\vPhi_0\vR_0\vPhi_0^\top\tilde{\vX}},&\quad &T_8^0:=\frac{N}{\mu_1^2}\trarg{\vR_0\bSigmaPhii\vR_0},\\
    &T_9^0:=\frac{\mu_1\theta_1^2}{\sqrt{N}}\trarg{\vW_0\vR_0\vPhi_0^\top\tilde{\vX}},
    &\quad &T_{10}^0:=\theta_1^2,\\
    &T_{11}^0:=\theta_1^2\trarg{\tilde{\vX}^\top\vPhi_0\vR_0\bSigmaPhii\vR_0\vPhi_0^\top\tilde{\vX}},&\quad &T_{12}^0:=\mu_1^*\theta_2\trarg{\tilde{\vX}^\top\vPhi_0\vR_0\bSigmaPhii\vR_0\vPhi_0^\top\tilde{\vX}},
\end{aligned}
\end{equation}
where scalars $\theta_1$ and $\theta_2$ are defined in Theorem \ref{thm:alignment}.
In what follows, we first prove that each $\tilde{T}_i$ in \eqref{def:Tis} vanishes in probability for $i=5,6,7,12$. 
Then we extend Lemma~\ref{lemm:quadratic} to cover the case where matrix $\vD$ is also random to establish the concentrations for all $T_i$'s in \eqref{def:Tis}. 

\begin{lemm}\label{lem:Titilde}
Under Assumptions \ref{assump:1} and \ref{assump:2}, as $n,d,N\to \infty$ proportionally, we have 
\begin{equation}
   |\tilde{T}_5|,|\tilde{T}_6|,|\tilde{T}_7|,|\tilde{T}_{12}|\overset{\P}{\to}0.
\end{equation}
\end{lemm}
\begin{proof}
To simplify the presentation, we take $\vD\in\R^{n\times d}$, which involves independent Gaussian random matrices $\vW_0$, $\tilde{\vX}$ and $\vZ$, to be any of the following matrices:
\begin{equation}\label{eq:all_D}
    \frac{\mu_1^2}{\sqrt{N}}\tilde{\vX},~\frac{\mu_1}{2}\vPhi_0\vR_0\vW_0,~\frac{\mu_1^2}{\sqrt{N}}\vPhi_0\vR_0\vPhi_0^\top\tilde{\vX},~\frac{1}{\sqrt{N}}\vPhi_0\vR_0(\mu_1^2\vW_0^\top\vW_0+\mu_2^2\vI)\vR_0\vPhi_0^\top\tilde{\vX}.
\end{equation}
It is clear that all the $\tilde{T}_i$ of interest can be written as  $\frac{1}{\sqrt{N}}\vf^{*\top}_{\nnl}\vD\vu$ for different choices of $\vD$. As a first observation, one can verify that $\|\vD\|\le C$ with high probability, for some constant $C>0$, as $n,N,d\to\infty$ proportionally. Moreover, by Lemma A.12 in \cite{bartlett2021deep}, with probability at least $1-Cn^{-1/4}$, we have 
\[\left|\|\vf^*_{\nnl}\|^2/n-\|f^*_{\nnl}\|^2_{L^2}\right|\le n^{-3/8}.\] Hence, $\frac{1}{\sqrt{N}}\|\vf^*_{\nnl}\|\le C$ uniformly with high probability. Notice that 
\begin{equation}\label{eq:NLDu}
    \frac{1}{\sqrt{N}}\vf^{*\top}_{\nnl}\vD\vu=\frac{\mu_1\eta}{n\sqrt{N}}\left(\vf^{*\top}_{\nnl}\vD\vX^\top f^*(\vX)+\vf^{*\top}_{\nnl}\vD\vX^\top\vvarepsilon\right).
\end{equation} 
Since $\vvarepsilon$ is independent with all other random variables, we can easily show the second term $\frac{\mu_1\eta}{n\sqrt{N}}\vf^{*\top}_{\nnl}\vD\vX^\top\vvarepsilon$ is negligible asymptotically. 
In particular, notice that $\vX/\sqrt{n}$ is also bounded by some constant with high probability. Thus, by the general Hoeffding inequality,
\[\Parg{\left|\frac{\mu_1\eta}{n\sqrt{N}}\vf^{*\top}_{\nnl}\vD\vX^\top\vvarepsilon\right|\ge n^{-1/4}}\le 2e^{-c\sqrt{n}},\]
which implies that the second term in \eqref{eq:NLDu} converges to zero in probability. Hence we only need to control 
\begin{equation}\label{eq:fNLf}
    \frac{\mu_1\eta}{n\sqrt{N}}\vf^{*\top}_{\nnl}\vD\vX^\top f^*(\vX)=\frac{\mu_1\eta}{n\sqrt{N}}\sigma^{*}_{\nnl}(\vbeta_*^\top\tilde{\vX}^\top)\vD\vX^\top \sigma^*(\vX\vbeta_*),
\end{equation}
where $\sigma^*_{\nnl}(x):=\sigma^*(x)-\mu_1^*x$. Our analysis can be divided into the following three steps. 
 
\paragraph{Step 1: Simplification via Rotation Invariance.} 
Following an argument similar to Lemma~9.1 in \cite{mei2019generalization} and Lemma 6.1 in \cite{montanari2020interpolation}, we claim that for \eqref{eq:fNLf}, it suffices to consider $\vbeta_*$ uniform on sphere and independent of $\vX,\tilde{\vX},\vW_0$ and $\vZ$, that is,  $\vbeta_*\sim\text{Unif}(\mathbb{S}^{d-1})$. 
In particular, given Haar-distributed orthogonal matrix $\vO\in\R^{d\times d}$, we apply this random rotation to $\vW_0, \tilde{\vX}^\top, \vX^\top$ and $\vbeta_*$ respectively in \eqref{eq:fNLf}. 
Since $\vO\vW_0\overset{d}{=}\vW_0$, $\vX\vO\overset{d}{=}\vX$, and $\tilde{\vX}\vO\overset{d}{=}\tilde{\vX}$, after replacing the deterministic $\vbeta_*$ by the random $\vbeta_*\sim\text{Unif}(\mathbb{S}^{d-1})$, one can easily verify that the quantities of interest in \eqref{eq:fNLf} are unchanged in distribution. 
For example consider the quantity $\tilde{T}_5$. Denote
$G(\vX,\tilde{\vX},\vbeta_*) = \vf^{*\top}_{\nnl}\tilde{\vX}\vX^\top f^*(\vX)$; we see that $G(\vX,\tilde{\vX},\vbeta_*) = G(\vX\vO,\tilde{\vX}\vO,\vO\vbeta_*)\overset{d}{=} G(\vX,\tilde{\vX},\vO\vbeta_*)$, and thus we can assume $\vbeta_*$ is uniform on sphere.  
Computation of the remaining terms follow from similar procedure. 
Hence, without loss of generality, in the following, we view $\vbeta_*$ as a uniformly random vector on unit sphere $\mathbb{S}^{d-1}$. 
For properties of random orthogonal matrices, we refer the reader to \cite{meckes2019random}.

\paragraph{Step 2: Nonlinear Hanson-Wright Inequality.} Now, condition on the event 
\[\mathcal{E}:=\{\|\vD\|,\|\tilde{\vX}\|/\sqrt{N},\|\vX\|/\sqrt{N}\le C\},\]
for some constant $C>0$, we can apply the nonlinear Hanson-wright inequality (\cite[Theorem 3.4.]{wang2021deformed}) for $\vbeta_*\sim\text{Unif}(\mathbb{S}^{d-1})$. In particular, we can rewrite \eqref{eq:fNLf} as a quadratic form:
\[\frac{\mu_1\eta}{n\sqrt{N}}\vf^{*\top}_{\nnl}\vD\vX^\top f^*(\vX)=\frac{\mu_1\eta}{n\sqrt{N}}\vq^\top\begin{bmatrix}
\mathbf{0} & \vD\vX^\top\\
\mathbf{0} & \mathbf{0}
\end{bmatrix}\vq,\]
where we denoted $\vq:=\begin{bmatrix}
\sigma^{*}_{\nnl}\big(\vbeta_*^\top\tilde{\vX}^\top\big),\,  \sigma^{*}\big(\vbeta_*^\top\vX^\top\big) 
\end{bmatrix}^\top\in\R^{2n}$. From Corollary 5.4 of \cite{meckes2019random} we know the vector $\vbeta_*\sim\text{Unif}(\mathbb{S}^{d-1})$ satisfies 
\begin{equation}\label{eq:flip_beta}
    \Parg{\left|f(\vbeta_*)-\E[f(\vbeta_*)]\right|>t}\le e^{-cdt^2},
\end{equation}
for 1-Lipschitz function $f:\mathbb{S}^{d-1}\to\R$; this is to say, $\vbeta_*$ satisfies the convex concentration property with parameter $1/\sqrt{d}$ defined in \cite{adamczak2015note}. Since both $\sigma^*$ and $\sigma^*_{\nnl}$ are $\lambda_\sigma$-Lipschitz, based on \cite[Theorem 2.5]{adamczak2015note} (or analogously \cite[Theorem 3.4]{wang2021deformed}), only considering the randomness of $\vbeta_*$, we have 
\begin{align}
    &\Parg{\left|\frac{\mu_1\eta}{n\sqrt{N}}\vf^{*\top}_{\nnl}\vD\vX^\top f^*(\vX)-\frac{\mu_1\eta}{n\sqrt{N}}\Tr\vD\vX^\top\vPsi\right|\ge\frac{t}{n\sqrt{N}}}\\
    \le & ~2\exp \left(-\frac{1}{C} \min \left\{ \frac{d^2t^2}{4\lambda_{\sigma}^4 \left(\|\vX\|^4+\|\tilde{\vX}\|^4\right)\|\vD\vX^\top\|_F^2}, \frac{dt}{\lambda_{\sigma}^2 \left(\|\vX\|^2+\|\tilde{\vX}\|^2\right)\|\vD\vX^\top\|}\right\}\right)\\ &  +2\exp \left( -\frac{dt^2}{16\lambda_{\sigma}^2\left(\|\vX\|^2+\|\tilde{\vX}\|^2\right)\|\vD\vX^\top\|^2 \|\E_{\vbeta_*} [\vq]\|^2  }\right),
\end{align}
where $\vPsi:=\E_{\vbeta_*}\left[\sigma^*(\vX\vbeta_*)\sigma^{*}_{\nnl}(\vbeta_*^\top\tilde{\vX}^\top)\right]\in\R^{n\times n} $ is an ``expected'' kernel matrix. By the Lipschitz property of $\sigma^*$, we have $\|\E_{\vbeta_*} [\vq]\|^2 \le \|\vX\|^2+\|\tilde{\vX}\|^2\le Cd$ under the event $\mathcal{E}$. Moreover, note that under $\mathcal{E}$, $\|\vX\|,\|\tilde{\vX}\|\le C\sqrt{d}$, $\|\vD\|\le C$, and $\|\vX\|_F^2\le Cnd$ for some constant $C>0$. Setting $t=nN^{1/4}$ in the above probability bound, one can obtain 
\begin{equation}
    \Parg{\left|\frac{\mu_1\eta}{n\sqrt{N}}\vf^{*\top}_{\nnl}\vD\vX^\top f^*(\vX)-\frac{\mu_1\eta}{n\sqrt{N}}\Tr\vD\vX^\top \vPsi\right|\ge\frac{1}{N^{1/4}}}\le 4e^{-c\sqrt{N}},
\end{equation}
for some constant $c>0$. Therefore, \eqref{eq:fNLf} concentrates around $\frac{\mu_1\eta}{n\sqrt{N}}\Tr\vD\vX^\top\vPsi$ with high probability. Thus it remains to show that $\frac{\mu_1\eta}{n\sqrt{N}}\Tr\vD\vX^\top\vPsi$ is vanishing in probability. 

\paragraph{Step 3: Estimations of Expected Kernel.} 
Notice that on the event $\mathcal{E}$,
\[\left|\frac{\mu_1\eta}{n\sqrt{N}}\Tr\vD\vX^\top \vPsi\right|\le \frac{\mu_1\eta}{n\sqrt{N}}\|\vD\vX^\top\|_F\|\vPsi\|_F\le \frac{\mu_1\eta}{n\sqrt{N}}\|\vD\|\|\vX^\top\|_F\|\vPsi\|_F\le \frac{C}{\sqrt{N}}\| \vPsi\|_F.\]
Hence we aim to control the entry-wise magnitude of $\vPsi$. We denote columns of $\vX^\top$ and $\tilde{\vX}^\top$ by $\vX^\top:=[\vx_1,\ldots,\vx_n]$ and $\tilde{\vX}^\top:=[\tilde{\vx}_1,\ldots,\tilde{\vx}_n]$, respectively. Notice that for $1\le i,j\le n$, entry of the expected kernel matrix $\vPsi$ is given as
\[\vPsi_{i,j}=\E_{\vbeta_*}\left[\sigma^*(\vbeta_*^\top\vx_i)\sigma^{*}_{\nnl}(\vbeta_*^\top\tilde{\vx}_j)\right].\]
Define the event $\cM:=\left\{\left|\|\vx_i\|/\sqrt{d}-1\right|,~ \left|\|\tilde{\vx}_i\|/\sqrt{d}-1\right|,~ \langle \vx_i,\tilde{\vx}_j\rangle\le C\log d /\sqrt{d},~i,j\in [n]\right\}$. One can verify that $\cM$ holds with high probability based on concentration of Gaussian random vectors. 
Conditioned on event $\cM$, we aim to control $|\vPsi_{i,j}|$ for $1\le i,j\le n$. In the following, we follow the arguments in \cite[Appendix A.4]{montanari2020interpolation} to establish the desired claim. 
Without loss of generality, we may change the coordinate and take the direction of $\tilde{\vx}_j$ as $\ve_1$, namely $\tilde{\vx}_j=\|\tilde{\vx}_j\|\ve_1$. Hence, $\vbeta_*\sim\text{Unif}(\mathbb{S}^{d-1})$ can be rewritten as $\vbeta_*=w_1\ve_1+\sqrt{1-w_1^2}\vw$, where $w_1:=\vbeta_*^\top\ve_1$, and $\vw:=[0,\tilde{\vw}^\top]^\top\in\R^d$ with $\tilde{\vw}\sim\text{Unif}(\mathbb{S}^{d-2})$ independent of $w_1$. Note that $w_1\in[-1,1]$ and $\sqrt{d}w_1$ converges weakly to the standard Gaussian distribution. 
Analogous to the proof of Lemma A.5 in \cite{montanari2020interpolation}, we make the decomposition
\[\vbeta_*^\top\vx_i=\underbrace{w_1 \left\langle \vx_i,\frac{\tilde{\vx}_j}{\|\tilde{\vx}_j\|}\right\rangle}_{\delta_d}+\underbrace{\frac{\sqrt{1-w_1^2}\|\vx_i\|}{\sqrt{d-1}}}_{1+\epsilon_d}\underbrace{\left\langle \frac{\vx_i}{\|\vx_i\|},\sqrt{d-1}\vw\right\rangle}_{\xi_d}.\]
Therefore, we have for any $i,j\in[n]$,
\begin{align}
    \vPsi_{i,j}= \E_{w_1}\left[\sigma_\nnl^*(\|\tilde{\vx}_j\|w_1)\E_{\vw}\left[\sigma^*\left(\delta_d+(1+\epsilon_d)\xi_d\right)\right]\right].\label{eq:psi_ij}
\end{align} 
Now condition on event $\cM$, the Lipschitz property of $\sigma^*$ entails that
\begin{align}
    &\left|\E_{w_1}\left[\sigma_\nnl^*(\|\tilde{\vx}_j\|w_1)\delta_d\right]-\E_{w_1}\left[\sigma_\nnl^*(\sqrt{d}w_1)\sqrt{d}w_1\right]\left\langle \frac{\vx_i}{\sqrt{d}},\frac{\tilde{\vx}_j}{\|\tilde{\vx}_j\|}\right\rangle\right|\\
    \le ~& \E_{w_1}[dw_1^2]\left|\left\langle \frac{\vx_i}{\sqrt{d}},\frac{\tilde{\vx}_j}{\|\tilde{\vx}_j\|}\right\rangle\right|\cdot\left|1-\frac{\|\tilde{\vx}_j\|}{\sqrt{d}}\right|
    \lesssim  \frac{(\log d)^2}{d},\label{eq:delta_control}
\end{align}
where we also applied Lemma 4.9 of \cite{montanari2020interpolation} to obtain $|\E_{w_1}[dw_1^2]-1|\lesssim (\log d)^2/d$. Similarly,
\[
\left|\E_{w_1}\left[\sigma_\nnl^*(\|\tilde{\vx}_j\|w_1)\delta_d\right]\right|\lesssim \frac{(\log d)^2}{d}+\E_{w_1}\left[\sigma_\nnl^*(\sqrt{d}w_1)\sqrt{d}w_1\right]\frac{\log d}{\sqrt{d}}\lesssim\frac{(\log d)^2}{d}, \]
since $\E_{w_1}\left[\sigma_\nnl^*(\sqrt{d}w_1)\sqrt{d}w_1\right]=\E_{\xi\sim\cN (0,1)}\left[\sigma_\nnl^*(\xi)\xi\right]+O_d\left((\log d)^2/d\right)$ and $\E_{\xi\sim\cN (0,1)}\left[\sigma_\nnl^*(\xi)\xi\right]=0$. Next, note that $\left|1-\sqrt{1-w_1^2}\right|\le \frac{1}{2}w_1^2$, and $\epsilon_d=1-\|\vx_i\|/\sqrt{d-1}+\left(1-\sqrt{1-w_1^2}\right)\|\vx_i\|/\sqrt{d-1}$. Thus under event $\cM$, 
\begin{align}
   &\left|\E_{w_1}\left[\sigma_\nnl^*(\|\tilde{\vx}_j\|w_1)\epsilon_d\right]\right|\label{eq:epsilon_control}\\
     \le ~&\left|1-\frac{\|\vx_i\|}{\sqrt{d-1}}\right|\left|\E_{w_1}\left[\sigma_\nnl^*(\|\tilde{\vx}_j\|w_1)-\sigma_\nnl^*(\sqrt{d}w_1)\right]\right|\\
    &+\left|1-\frac{\|\vx_i\|}{\sqrt{d-1}}\right|\left|\E_{w_1}\left[ \sigma_\nnl^*(\sqrt{d}w_1)\right]\right|+\frac{\|\vx_i\|}{\sqrt{d-1}}\left|\E_{w_1}\left[|\sigma_\nnl^*(\|\tilde{\vx}_j\|w_1)|\cdot\left|1-\sqrt{1-w_1^2}\right|\right]\right|\\
    \le ~&\frac{C(\log d)^2}{d}\lambda_\sigma\E[|\sqrt{d}w_1|]+\frac{C(\log d)^3}{d^{3/2}}+\frac{C}{2d}\E[(\sqrt{d}w_1)^4]^{1/2}\E_{w_1}[\sigma_\nnl^{*2}(\|\tilde{\vx}_j\|w_1)]^{1/2}\lesssim \frac{(\log d)^2}{d},
\end{align}
where we used $\E_{\xi\sim\cN (0,1)}\left[\sigma_\nnl^*(\xi) \right]=0$ and Lemma A.9 of \cite{montanari2020interpolation}. Consequently, by \eqref{eq:psi_ij},
\begin{align}
   &\left|[\vPsi]_{i,j}-\E_{w_1}\left[\sigma_\nnl^*(\|\tilde{\vx}_j\|w_1)\right]\E_{\vw}\left[\sigma^*\left(\xi_d\right)\right]\right|\\
   \lesssim ~& \left|\E_{w_1}\left[\sigma_\nnl^*(\|\tilde{\vx}_j\|w_1)\delta_d\right]\right|+\left|\E_{w_1}\left[\sigma_\nnl^*(\|\tilde{\vx}_j\|w_1)\epsilon_d\right]\E_{\vw}[\xi_d]\right|+\E_{w_1,\vw}\left[\left(\delta_d+\epsilon_d\xi_d\right)^2\right]\\
   \lesssim~ & \frac{(\log d)^2}{d}\left(1+\E_{\vw}[|\xi_d|]\right)+\E_{w_1}[\delta_d^2]+\E_{w_1}[\epsilon_d^2]\E_{\vw}[\xi_d^2]\lesssim  \frac{(\log d)^2}{d},
\end{align}
where we repeatedly applied Lemma A.9 of \cite{montanari2020interpolation}, \eqref{eq:delta_control}, and \eqref{eq:epsilon_control}. Meanwhile, notice that $\E_{\vw}\left[\sigma^*\left(\xi_d\right)\right]=\E_{\xi\sim\cN (0,1)}\left[\sigma^*(\xi) \right]+O_d((\log d)^2/d)$. We conclude that with high probability, the following holds
\[ \max_{1\le i,j\le n}|[\vPsi]_{i,j}|\lesssim \frac{(\log d)^2}{d}.\]
Thus, as $d\to\infty$, $\|\vPsi\|_F/\sqrt{N}\lesssim (\log d)^2/\sqrt{d}$ with high probability, which finally implies that $\frac{\mu_1\eta}{n\sqrt{N}}\Tr\vD\vX^\top \vPsi$ converges to zero in probability as $n,d,N\to\infty$ proportionally. This completes the proof.
\end{proof}

Finally, we use the following simplification of quadratic forms to obtain the desired $T^0_i$. 

\begin{lemm}\label{lemm:quadraticbeta}
Consider a random matrix $\vD\in\R^{d\times d}$ that does not rely on $\vbeta_*$ and is rotational invariant in distribution, namely $\vD\overset{d}{=}\vO^\top\vD\vO$ for any random rotational matrix $\vO\in\R^{d\times d}$. Assume that $\|\vD\|\le C$ with high probability for some universal constant $C>0$. Then as $d\to \infty$,  
\begin{equation}
    \left|\vbeta_*^\top\vD\vbeta_*-\tr\vD\right|\overset{\P}{\to }0,
\end{equation}
Also as a corollary, we have $\left|T_i-T^0_i\right|\overset{\P}{\to}0$ as $n,d,N\to \infty$ proportionally for all $1\le i\le 12$, where $T_i, T^0_i$ are defined in \eqref{def:Tis} and \eqref{def:Ti0s}.
\end{lemm}
\begin{proof}
Given any rotational matrix $\vO\in\R^{d\times d}$ following the Haar distribution, notice that $\vbeta_*^\top\vD\vbeta_*=\vbeta_*^{'\top}\vO^\top\vD\vO\vbeta_*'\overset{d}{=}\vbeta_*^{'\top}\vD\vbeta_*'$, where $\vbeta_*':=\vO^\top\vbeta_*$. Consequently, we can equivalently take $\vbeta_*$ to be a random vector uniformly distributed on the unit sphere $\mathbb{S}^{d-1}$. Again recall that $\vbeta_*\sim\text{Unif}(\mathbb{S}^{d-1})$ satisfies the convex concentration property 
$$\Parg{\left|f(\vbeta_*)-\E[f(\vbeta_*)]\right|>t}\le e^{-cdt^2}$$ for 1-Lipschitz $f$. Therefore, conditioned on the event $\|\vD\|\le C$, by Theorem 2.5 in \cite{adamczak2015note}, one can conclude that $\left|\vbeta_*^\top\vD\vbeta_*-\E_{\vbeta_*}[\vbeta_*^\top\vD\vbeta_*]\right|\overset{\P}{\to }0$. Finally, note that $\E_{\vbeta_*}[\vbeta_*^\top\vD\vbeta_*]=\tr\vD$ because the covariance of the uniform random vector on $\mathbb{S}^{d-1}$ is $\frac{1}{d}\vI$; this concludes the proof. 
Convergence of each $T_i$ to the corresponding $T_i^0$ follows from a direct application of Lemma~\ref{lemm:quadratic} and this lemma. 

\end{proof}

\subsubsection{Risk Calculation via Linear Pencils}

In this section, we derive analytic expressions of the terms $T_i$ defined in \eqref{def:Tis} as $n,d,N\to\infty$ proportionally. In particular, the exact values are described by self-consistent equations defined in the following proposition. 
\begin{prop}\label{prop:tau_i}
Given Assumption \ref{assump:1} and $\lambda>0$. For each $T_i$ defined in \eqref{def:Tis} and $1\le i \le 12,$ we have
\[T_{i} \to \tau_i,\]
in probability, as $n/d\to\psi_1$ and $N/d\to\psi_2$, where $\tau_i$'s are defined as follows
\begin{align}
    &\tau_1:= \frac{\psi_1}{\psi_2}m_1+\left(\frac{\psi_2}{\psi_1}-1\right)\frac{1}{\lambda},&\quad  &\tau_2:= \mu_1^2\theta_1^2 \frac{\psi_1}{\psi_2}\left(1-\lambda\frac{\psi_1}{\psi_2}m_2\right),&\quad  &\tau_3:=\mu_1^2\theta_1^2\frac{\psi_1}{\psi_2},\\
    & \tau_4:=\mu_1^*\theta_2,&\quad  &\tau_5:=\mu_1^2\mu_1^*\theta_2\frac{\psi_1}{\psi_2},&\quad &\tau_6:= \mu_1^{*}\theta_2\left(1-\frac{m_2}{m_1}\right),\\
    &\tau_7 := \mu_1^2\mu_1^{*}\theta_2\frac{\psi_1}{\psi_2}\left(1-\lambda\frac{\psi_1}{\psi_2}m_2\right),&\quad &\tau_8:=\frac{m_1+ \frac{\psi_1}{\psi_2}\lambda m_1'}{ \left(\mu_1\frac{\psi_1}{\psi_2}\lambda m_1\right)^2},&\quad &\tau_9:=\theta_1^2\left(1-\frac{m_2}{m_1}\right),\\
   &\tau_{10}:=\theta_1^2,&\quad  &\tau_{11}:=\theta_1^2\left(1-\frac{2m_2}{m_1}-\frac{m_2'}{m_1^2}\right),&\quad &\tau_{12}:=\mu_1^*\theta_2\left(1-\frac{2m_2}{m_1}-\frac{m_2'}{m_1^2}\right).  
\end{align}
All scalars $\tau_i$'s are only determined by parameters $\psi_1,\psi_2,\eta,\mu_1,\mu_2,\lambda,$ and $m_1,m_2,m_1',m_2'$. Here, $m_1:=m_1\left(\lambda\frac{\psi_1}{\psi_2}\right)$, $m_1':=m_1'\left(\lambda\frac{\psi_1}{\psi_2}\right)$, $m_2:=m_2\left(\lambda\frac{\psi_1}{\psi_2}\right)$ and $m_2':=m_2'\left(\lambda\frac{\psi_1}{\psi_2}\right)$, where $m_1(z)$ and $m_2(z)\in\mathbb{C}^+\cup \R_{+}$ are the solutions to the following self-consistent equations for $z\in\mathbb{C}^+\cup \R_{+}$,
\begin{align}
    \frac{1}{\psi_1}(m_1(z)-m_2(z))(\mu_2^2m_1(z)+\mu_1^2m_2(z))+\mu_1^2m_1(z)m_2(z)\left(zm_1(z)-1\right)=& ~0,\label{eq:fixedpoint1}\\
    \frac{\psi_2}{\psi_1}\left(\mu_1^2m_1(z)m_2(z)+\frac{1}{\psi_1}(m_2(z)-m_1(z))\right)+\mu_1^2m_1(z)m_2(z)\left(zm_1(z)-1\right)=& ~0.\label{eq:fixedpoint2}
\end{align}
\end{prop}

\begin{proof}
First note that due to Lemma \ref{lemm:quadraticbeta}, it suffices to consider the limits of $T_i^0$ instead. 
Convergence of $T_3,T_4,T_5,$ and $T_{10}$ directly follows from Lemma \ref{lemm:quadratic}, Lemma \ref{lemm:quadraticbeta} and the Marchenko-Pastur law for $\frac{1}{n}\tilde{\vX}^\top\tilde{\vX}$. In addition, $T_7,T_9$ and $T_{12}$ are analogous to $T_2,T_6$ and $T_{11}$, respectively. To characterize the remaining $T_1,T_2,T_6,T_8$ and $T_{11}$, we adopt the \textit{linear pencil} method in basis of operator-valued free probability theory \cite{far2006spectra,helton2007operator,mingo2017free,helton2018applications}. 
Specifically, the linear pencil allows us to relate the quantities of interest to the trace of certain large block matrices; in our case, variants of $T_1,T_2,T_6,T_8,T_{11}$ have already appeared in prior constructions from \cite{adlam2020neural,bodin2021model,tripuraneni2021covariate}, which we build upon in the following calculation.

For $z\in\mathbb{C}^+\cup \R_{+}$, let us define $\vR_0(z):=\left(\hSigmaPhii + z\vI\right)^{-1}$ and $\bar{\vR}_0(z):=\left(\vPhi_0\vPhi_0^\top + z\vI\right)^{-1}\in\R^{n\times n}$. Note that due to the Gaussian equivalent property, as $n,N,d\to\infty$ at comparable rate, the limit of $\tr \bar{\vR}_0(z)$ is exactly the Stieltjes transform of the limiting spectrum of the (nonlinear) CK, namely $\vPhi\vPhi^\top\in\R^{n\times n}$, evaluated at $-z$. 
We denote $m_1(z):=\lim_{n\to\infty} \tr\bar{\vR}_0(z)$.  Similarly, the limit of $\tr\vR_0(z)$ is the \textit{companion} Stieltjes transform of $m_1(z)$, as $\vPhi_0\vPhi_0^\top$ and $\vPhi_0^\top\vPhi_0$ have the same non-zero eigenvalues. 
We denote $\tau_1(z):=\lim_{n\to\infty} \tr\vR_0(z)$. 
The defined Stieltjes transforms will be evaluated at $z=\frac{\psi_1}{\psi_2}\lambda$. 
Also recall the the following relationship between $\vR_0(z)$ and $\bar{\vR}_0(z)$, 
\begin{equation}\label{eq:tau_1m1}
    \tau_1(z)=\frac{\psi_1}{\psi_2}m_1(z)+\left(1-\frac{\psi_1}{\psi_2}\right)\frac{1}{z}.
\end{equation}
Analogously, we introduce the following quantities: for any $z\in\mathbb{C}^+\cup \R_{+}$, as $n,N,d\to\infty$ proportionally
\begin{align*}
    &m_2(z):= \lim_{n\to\infty}\tr\left(\frac{1}{d}\tilde{\vX}\tilde{\vX}^\top\bar{\vR}_0(z)\right),\quad &\tau_2(z):&= \lim_{n\to\infty} \tr\left(\frac{1}{n}\tilde{\vX}^\top\vPhi_0\vR_0(z)\vPhi_0^\top\tilde{\vX}\right),\\
    &\tau_6(z):= \lim_{n\to\infty}\frac{1}{\sqrt{N}}\tr\left(\vW_0\vR_0(z)\vPhi_0^\top\tilde{\vX}\right),\quad &\tau_8(z):&=\lim_{n\to\infty}\tr\left(\vR_0(z)\left(\mu_1^2\vW_0^\top\vW_0+\mu_2^2\vI\right)\right).
\end{align*}
It is straightforward to verify all the above limits exist and are finite. 
Finally, in the following analysis we will repeatedly make use of the following identities: 
\begin{align}
    \vR_0(z)\vPhi_0^\top=&\vPhi_0^\top\bar\vR_0(z),\label{eq:RphiPhiR}\\
    \tilde{\vX}^\top\vPhi_0\vR_0(z)\vPhi_0^\top\tilde{\vX}=& \tilde{\vX}^\top \tilde{\vX}-z\tilde{\vX}^\top\bar\vR_0(z) \tilde{\vX}.\label{eq:XPHIRPHIX}
\end{align}

\paragraph{Analysis of $T_1$ and $T_2$.} 
Note that $m_1(z)$ and $m_2(z)$ defined in \eqref{eq:fixedpoint1} and \eqref{eq:fixedpoint2} have been characterized in prior works, such as Proposition 1 of \cite{adlam2020neural}; in particular, since we are only interested in the CK, we can simply set $\sigma_{W_2}=0$ in \cite{adlam2020neural} (which considered the sum of the CK and the first-layer NTK). Therefore, from \eqref{eq:tau_1m1} we obtain $\tau_1=\tau_1\left(\frac{\psi_1}{\psi_2}\lambda\right)$ from $m_1:=m_1\left(\frac{\psi_1}{\psi_2}\lambda\right)$. 
As for the limit of $T_2$, \eqref{eq:XPHIRPHIX} indicates
\begin{equation}
    \tau_2(z)=1-zm_2(z),
\end{equation}
since $\tr\left(\tilde{\vX}\tilde{\vX}^\top/d\right)\to 1$, as $n,N,d\to\infty$. Thus, 
\begin{align}
    \tau_2 = \mu_1^2\theta_1^2\frac{\psi_1}{\psi_2}\left(1-z m_2(z)\right),
\end{align}with $z=\frac{\psi_1}{\psi_2}\lambda$. $\tau_7$ can also be derived in similar fashion.

\paragraph{Analysis of $T_6$.} 
For $T_6,$ we utilize the computations in Appendix I.6.1 of \cite{tripuraneni2021covariate} by setting the covariance $\Sigma=\vI$. More precisely, based on Equations (S370) and (S418) in \cite{tripuraneni2021covariate}, 
\begin{align}
\mu_1\tau_6(z)=1-G_{6,6}^{K^{-1}}=\frac{z\mu_1^2\psi_1m_1(z)\tau_1(z)}{1+z\mu_1^2\psi_1m_1(z)\tau_1(z)}\overset{(i)}{=}z\mu_1^2\psi_1m_1(z)\tau_1(z)\overset{(ii)}{=}1-\frac{m_2(z)}{m_1(z)},\label{eq:m2/m1}
\end{align}
where $(i)$ and $(ii)$ are both due to \eqref{eq:fixedpoint2} and \eqref{eq:tau_1m1}. Hence we obtain the formulae of $\tau_6$ and $\tau_9$.

\paragraph{Analysis of $T_8$.} Recall the following derivative trick of the Stieltjes transform, 
\begin{equation}
    \frac{\partial}{\partial z}\tr\left(\vR_0(z)\left(\mu_1^2\vW_0^\top\vW_0+\mu_2^2\vI\right)\right)=-\tr\left(\vR_0(z)\left(\mu_1^2\vW_0^\top\vW_0+\mu_2^2\vI\right)\vR_0(z)\right).
\end{equation}
Define $\tilde{\tau}(z) = \lim_{n\to\infty}\tr\left(\vR_0(z)\left(\mu_1^2\vW_0^\top\vW_0+\mu_2^2\vI\right)\right)$. Following the proof of Lemma 7.4 in \cite{dobriban2018high}, we can apply Lemma 2.14 of \cite{bai2010spectral} and Vitali’s theorem to claim that 
\begin{equation}
    -\tilde{\tau}'(z)=\lim_{n\to\infty} N\tr\left(\vR_0(z)\bSigmaPhii\vR_0(z)\right).
\end{equation}
Hence, we need to first calculate $\tilde{\tau}(z)$. By definition of $\tau_1(z)$, Equations (S376) and (S412) in \cite{tripuraneni2021covariate}, 
\begin{align}
    \tilde{\tau}(z)=&~\mu_2^2\tau_1(z)+\frac{\mu_1^2\tau_1}{1+z\mu_1^2\psi_1m_1(z)\tau_1(z)}\\
    \overset{(iii)}{=}&~ \tau_1(z)\left(\mu_2^2+\mu_1^2\frac{m_2(z)}{m_1(z)}\right)\\
     \overset{(iv)}{=}&~ \frac{\mu_1^2}{z}\left(\frac{\psi_1}{\psi_2}zm_1(z)+\left(1-\frac{\psi_1}{\psi_2}\right) \right)\left(\frac{\mu_2^2}{\mu_1^2}+\frac{m_2(z)}{m_1(z)}\right)\\
     \overset{(v)}{=}&~ \frac{1}{\psi_1 z}\frac{m_1(z)-m_2(z)}{m_1(z)m_2(z)}\left(\frac{\mu_2^2}{\mu_1^2}+\frac{m_2(z)}{m_1(z)}\right) \overset{(vi)}{=}\frac{1}{zm_1(z)}-1,
\end{align}
where $(iii)$ and $(v)$ are due to \eqref{eq:fixedpoint2} and \eqref{eq:tau_1m1}, $(iv)$ comes from \eqref{eq:tau_1m1}, $(vi)$ is based on \eqref{eq:fixedpoint1}. We arrive at $\tau_8=-\frac{1}{\mu_1^2}\tilde{\tau}'(z)$ with $z=\frac{\psi_1}{\psi_2}\lambda$.

\paragraph{Analysis of $T_{11}$.} Once again by setting $\sigma_{W_2}=0$ in \cite{adlam2020neural}, we can directly employ the computation of $E_{32}$ in Section S4.3.4 of \cite{adlam2020neural} to derive $\tau_{11}$ and $\tau_{12}$. In particular, due to \eqref{eq:RphiPhiR}, 
\begin{align}
    E_{32}=&\tr\left(\tilde{\vX}^\top\bar\vR_0(z)\left(\frac{\mu_2^2}{N}\vPhi_0\vPhi_0^\top+\frac{\mu_1^2}{N}\vPhi_0\vW_0^\top\vW_0\vPhi_0^\top\right)\bar\vR_0(z)\tilde{\vX}\right)\\
    =& \tr\left(\tilde{\vX}^\top\vPhi_0\vR_0(z)\left(\frac{\mu_2^2}{N}\vI+\frac{\mu_1^2}{N} \vW_0^\top\vW_0\right)\vR_0(z)\vPhi_0^\top\tilde{\vX}\right)= \tr\left(\tilde{\vX}^\top\vPhi_0\vR_0(z)\bSigmaPhii\vR_0(z)\vPhi_0^\top\tilde{\vX}\right).
\end{align}
Note that Equation (S148) in \cite{adlam2020neural} established that  
\begin{equation}
    \lim_{n\to\infty}E_{32}=1-\frac{2m_2(z)}{m_1(z)}-\frac{m_2'(z)}{m_1(z)^2},
\end{equation} whence, letting $z=\frac{\psi_1}{\psi_2}\lambda$, we obtain $\tau_{11}$ and $\tau_{12}$. 

\end{proof}

Having obtained the asymptotic expressions of each term in the decomposition of the prediction risk, we can now compute the difference in the prediction risk of CK ridge regression before and after one gradient descent step, i.e., $\cR_0(\lambda)-\cR_1(\lambda)$ in Theorem~\ref{thm:R0-R1}. The following statement is the complete version of Theorem~\ref{thm:R0-R1}.

\begin{theo}\label{thm:risk_general}
Given Assumptions \ref{assump:1} and \ref{assump:2}, consider $\psi_1,\psi_2\in (0,+\infty)$. Fix $\eta=\Theta(1)$ and $\lambda>0$. Denote $\cR_0(\lambda)$ and $\cR_1(\lambda)$ as the prediction risk of CK ridge regression in \eqref{eq:ridge-risk} using initial weight $\vW_0$ and first-step updated $\vW_1$, respectively. Then the difference between these two prediction risk values satisfies
\[\cR_0(\lambda)-\cR_1(\lambda)\overset{\P}{\to} \delta(\eta,\lambda,\psi_1,\psi_2),\]
where $\delta$ is a non-negative function of $\eta,\lambda,\psi_1$ and $\psi_2\in(0,+\infty)$ with parameters $\mu_1^*,\mu_1,\mu_2$ given as
\begin{align}
    \delta(\eta,\lambda,\psi_1,\psi_2)=\frac{\tau_1(\tau_7-\tau_5)(\tau_4+\tau_{12}-2\tau_6)}{\tau_1(\tau_2-\tau_3)-1}-\frac{\tau_1(\tau_7-\tau_5)(\tau_4+\tau_{12}-2\tau_6)+(\tau_7-\tau_5)^2\tau_8}{\left(\tau_1(\tau_2-\tau_3)-1\right)^2}.\label{delta_formula}
\end{align}
Here the scalars $\tau_i$'s are defined in Proposition \ref{prop:tau_i}. Furthermore, $\delta(\eta,\lambda,\psi_1,\psi_2)=0$ if and only if at least one of $\mu_1^*, \mu_1$ and $\eta$ is zero.
\end{theo}
\begin{proof}
Due to Lemma~\ref{lemm:perturbation-stability} (or the decomposition \eqref{eq:B_1}, \eqref{eq:B_2} and \eqref{eq:V}), we can see that variance $V$ is unchanged after one gradient descent step with $\eta=\Theta(1)$. 
Hence we only need to analyze the changes in \eqref{eq:B_1} and \eqref{eq:B_2}. 
Also, due to Lemma~\ref{lemm:perturbation-stability} and the proof of Theorem \ref{thm:alignment}, we can ignore $\vB$ and $\vC$ in $\vW_1$ and take $\vW_1:=\vW_0+\vu\va^\top$, where $\vu=\frac{\mu_1\eta}{n}\vX^\top\vy$ and $\vy=f^*(\vX)+\vvarepsilon$, without changing the bias terms.

\paragraph{Separation of Low-rank Terms.} 
First note that if $\mu_1=0$, then $\vu=\mathbf{0}$ and therefore $\cR_0(\lambda)=\cR_1(\lambda)$ as $n\to\infty$.  
In the following, we take $\mu_1\neq 0$ which implies that $\theta_1$ defined in Theorem~\ref{thm:alignment} will not vanish. 
Now we aim to extract the low-rank perturbation $\vu\va^\top$ from bias terms \eqref{eq:B_1} and \eqref{eq:B_2}. We adhere to the notions in \eqref{def:notions}, \eqref{def:Tis} and \eqref{def:Ti0s} and define $D:=T_1(T_2-T_3)-1$. Similar to \cite[Lemma C.1]{mei2019generalization}, we use the following linearization trick to separate the gradient step $\vu\va^\top$ from the matrices $\vR,\barPhi,\bSigmaPhi$ and $\vW_1$. 

Define $\vb:=\frac{\mu_1}{\sqrt{N}}\tilde{\vX}\vu$ and $\vc:=\vPhi_0^\top\vb$; observe that $T_2=\vc^\top\vR_0\vc$, $T_3=\vb^\top\vb$, and
\begin{equation}
    \hSigmaPhi=\hSigmaPhii+\begin{bmatrix}
    \va &\vc
    \end{bmatrix}
    \begin{bmatrix}
    T_3 & 1\\
    1 &0
    \end{bmatrix}
    \begin{bmatrix}
    \va^\top \\
    \vc^\top
    \end{bmatrix}.
\end{equation} 
Therefore, by the Sherman-Morrison-Woodbury formula and Hanson-Wright inequality, we have 
\begin{equation}
    \vR=\vR_0-\vDelta_{aa}-\vDelta_{cc}+\vDelta_{ac}+\vDelta_{ca}+o_{d,\P}(1), 
\label{eq:def-R0-R1}
\end{equation}
where we further defined
\begin{align*}
    &\vDelta_{aa}:=\frac{T_2-T_3}{D}\vR_0\va\va^\top\vR_0, \quad \vDelta_{cc}:=\frac{T_1}{D}\vR_0\vc\vc^\top\vR_0,\\
    &\vDelta_{ca}:=\frac{1}{D}\vR_0\vc\va^\top\vR_0, \quad \vDelta_{ac}:=\frac{1}{D}\vR_0\va\vc^\top\vR_0.
\end{align*}

We decompose the subtracted term in \eqref{eq:B_1} into 
\begin{align*}
    B_{1,1}:&=- \frac{2\mu_1\mu_1^{*2}}{\sqrt{N}}\vbeta_*^\top\vW_1\vR \bar{\vPhi}^\top \tilde{\vX}\vbeta_*,&\quad & B_{1,1}^{\nnl}:=- \frac{2\mu_1\mu_1^{*2}}{\sqrt{N}}\vbeta_*^\top\vW_1\vR \bar{\vPhi}^\top \tilde{\vX}\vf^*_{\nnl},\\
    B_{1,1}^0:&=- \frac{2\mu_1\mu_1^{*2}}{\sqrt{N}}\vbeta_*^\top\vW_0\vR_0 \vPhi_0^\top\tilde{\vX}\vbeta_*,&\quad & B_{2}^0:=\vf^{*\top}\tilde{\vX}^\top\vPhi_0\vR_0\bSigmaPhii\vR_0 \vPhi_0^\top\tilde{\vX}\vf^*,
\end{align*}
where $\vf^*_{\nnl}$ and $\vf^*$ are defined in \eqref{eq:f*fNL}. By repeatedly applying the Hanson-Wright inequality (since $\va$ is centered and independent of all other terms) and Lemma \ref{lem:Titilde}, we can obtain
\begin{align}
    B_{1,1}=B_{1,1}^0-2T_1(T_7-T_5)(T_4-T_6)/D+o_{d,\P}(1),\quad B_{1,1}^{\nnl}=o_{d,\P}(1).
\end{align}
Now we denote $\vDelta_{ua}:=\frac{\mu_1^2}{N}\vW_0^\top\vu\va^\top$, $\vDelta_{au}:=\vDelta_{ua}^\top$ and $\vDelta_{aua}:=\frac{\mu_1^2T_{10}}{N}\va\va^\top$. Hence, \[\bSigmaPhi=\bSigmaPhii+\vDelta_{ua}+\vDelta_{au}+\vDelta_{aua}.\] 
Analogously, we can decompose $B_2$ in \eqref{eq:B_2} as follows
\begin{align}
    B_2=~&\mu_1^{*2}\vbeta_*^\top\tilde{\vX}^\top\bar\vPhi\vR\bSigmaPhi\vR \bar\vPhi^\top\tilde{\vX}\vbeta_*\\
    =~& \mu_1^{*2}\vbeta_*^\top\tilde{\vX}^\top \vPhi_0\left(\vR_0-\vDelta_{cc}+\vDelta_{ca}\right)\bSigmaPhi\left(\vR_0-\vDelta_{cc}+\vDelta_{ac}\right)  \vPhi_0^\top\tilde{\vX}\vbeta_*\\
    &+ 2\mu_1^{*2}\vbeta_*^\top\tilde{\vX}^\top \vPhi_0\left(\vR_0-\vDelta_{cc}+\vDelta_{ca}\right)\bSigmaPhi\left(\vR_0-\vDelta_{aa}+\vDelta_{ca}\right)  \va\vb^\top\tilde{\vX}\vbeta_*\\
    &+\mu_1^{*2}\vbeta_*^\top\tilde{\vX}^\top \vb\va^\top\left(\vR_0-\vDelta_{aa}+\vDelta_{ac}\right)\bSigmaPhi\left(\vR_0-\vDelta_{aa}+\vDelta_{ca}\right)  \va\vb^\top\tilde{\vX}\vbeta_*+o_{d,\P}(1)\\
    =~& B_2^0+\frac{2T_1(T_7-T_5)(T_6-T_{12})}{D}+\frac{(T_7-T_5)^2(T_1^2T_{11}+T_1^2T_{10}+T_8-2T_1^2T_9)}{D^2}+o_{d,\P}(1),
\end{align}where we repeatedly make use of Lemma~\ref{lem:Titilde} and the concentration for $\va$ to simplify the computations. Therefore, one can obtain
\begin{align}
     &\cR_0(\lambda)- \cR_1(\lambda)=B_{1,1}^0-B_{1,1}+B_2^0-B_2\label{eq:R_0-R_1}\\
     =& \frac{2T_1(T_7-T_5)(T_4+T_{12}-2T_6)}{D}-\frac{(T_7-T_5)^2(T_1^2T_{11}+T_1^2T_{10}+T_8-2T_1^2T_9)}{D^2}+o_{d,\P}(1).\nonumber
\end{align}
On the other hand, from Proposition \ref{prop:tau_i} we know that 
\begin{align}
     &\cR_0(\lambda)- \cR_1(\lambda)\overset{\P}{\to} \underbrace{\frac{2\tau_1(\tau_7-\tau_5)(\tau_4+\tau_{12}-2\tau_6)}{\tau_1(\tau_2-\tau_3)-1}-\frac{(\tau_7-\tau_5)^2(\tau_1^2\tau_{11}+\tau_1^2\tau_{10}+\tau_8-2\tau_1^2\tau_9)}{\left(\tau_1(\tau_2-\tau_3)-1\right)^2}}_{\triangleq \delta(\eta,\lambda,\psi_1,\psi_2)}, 
\end{align}
where the right hand side is the quantity of interest $\delta(\eta,\lambda,\psi_1,\psi_2)$ defined in Theorem~\ref{thm:R0-R1}. Also observe the following equivalences from Proposition \ref{prop:tau_i}, 
\begin{align}
    \mu_1^*\theta_2\left(\tau_2-\tau_3\right)=\theta_1^2\left(\tau_7-\tau_5\right),\quad \mu_1^*\theta_2\left(\tau_{11}+\tau_{10}-2\tau_9\right)=&\theta_1^2\left(\tau_{4}+\tau_{12}-2\tau_6\right).
\end{align}Hence, we can simplify $ \delta(\eta,\lambda,\psi_1,\psi_2)$ as follows
\begin{align}
    &\delta(\eta,\lambda,\psi_1,\psi_2)=-\frac{(\tau_7-\tau_5)^2\tau_8}{\left(\tau_1(\tau_2-\tau_3)-1\right)^2}\\
    &+\frac{\tau_1^2(\tau_2-\tau_3)(\tau_7-\tau_5)(\tau_4+\tau_{12}-2\tau_6)-\tau_1 (\tau_7-\tau_5)(\tau_4+\tau_{12}-2\tau_6)-\tau_1^2(\tau_7-\tau_5)^2(\tau_4+\tau_{12}-2\tau_6)}{\left(\tau_1(\tau_2-\tau_3)-1\right)^2}\\
    =& -\frac{(\tau_7-\tau_5)^2\tau_8}{\left(\tau_1(\tau_2-\tau_3)-1\right)^2}+\frac{\tau_1(\tau_7-\tau_5)(\tau_4+\tau_{12}-2\tau_6)}{\tau_1(\tau_2-\tau_3)-1}-\frac{\tau_1(\tau_7-\tau_5)(\tau_4+\tau_{12}-2\tau_6)}{\left(\tau_1(\tau_2-\tau_3)-1\right)^2}\\
    =&  \frac{\tau_1(\tau_7-\tau_5)(\tau_4+\tau_{12}-2\tau_6)}{\tau_1(\tau_2-\tau_3)-1}-\frac{\tau_1(\tau_7-\tau_5)(\tau_4+\tau_{12}-2\tau_6)+(\tau_7-\tau_5)^2\tau_8}{\left(\tau_1(\tau_2-\tau_3)-1\right)^2}.\label{eq:delta_final}
\end{align}

\paragraph{Non-negativity of $\delta(\eta,\lambda,\psi_1,\psi_2)$.} 
Finally, we validate that the function $\delta(\eta,\lambda,\psi_1,\psi_2)$ is non-negative on variables $\eta,\lambda,\psi_1$ and $\psi_2\in(0,+\infty)$. Observe that the formula of $\delta(\eta,\lambda,\psi_1,\psi_2)$ in \eqref{eq:delta_final} is decomposed into two parts. 
From Proposition \ref{prop:tau_i} we know that $\tau_1$ and $m_1$ are the limits of $\tr\vR_0(z)$ and $\tr\bar\vR_0(z)$ evaluated at $z=\psi_1\lambda/\psi_2$; this indicates that $\tau_1\in (0,\psi_2/\lambda\psi_1]$ is non-negative. For the same reason, $m_2\in (0,\psi_2/\lambda\psi_1]$ and  $-m_1',-m_2'\in  (0,\psi_2^2/\lambda^2\psi_1^2]$. Also due to Proposition \ref{prop:tau_i}, we have
\begin{equation}\label{eq:simplify_taui}
    \begin{aligned}
&\tau_2-\tau_3=-\mu_1^2\theta_1^2\left(\frac{\psi_1}{\psi_2}\right)^2\lambda m_2\le 0,&\quad& \tau_7-\tau_5=-\mu_1^2\mu_1^*\theta_2\left(\frac{\psi_1}{\psi_2}\right)^2\lambda m_2\le 0,\\
&\tau_4 +\tau_{12}-2\tau_6=- \mu_1^*\theta_2\frac{m_2'}{m_1^2}\ge 0,&\quad& \tau_{11} +\tau_{10}-2\tau_9=- \theta_1^2\frac{m_2'}{m_1^2}\ge 0,\\
& \tau_8=\frac{1}{m_1}\frac{1}{\mu_1^2\lambda^2}\left(\frac{\psi_2}{\psi_1}\right)^2+\frac{m_1'}{m_1^2}\left(\frac{\psi_2}{\psi_1}\right)\frac{1}{\mu_1^2\lambda}.
\end{aligned}
\end{equation}
Therefore, $\tau_1(\tau_7-\tau_5)(\tau_4+\tau_{12}-2\tau_6)\le 0$ and $\tau_1(\tau_2-\tau_3)-1\le -1$. This entails that the first part of $\delta(\eta,\lambda,\psi_1,\psi_2)$ is non-negative: 
\[\frac{\tau_1(\tau_7-\tau_5)(\tau_4+\tau_{12}-2\tau_6)}{\tau_1(\tau_2-\tau_3)-1}\ge 0.\]
As for the second part, it suffices to evaluate $\Delta:=\tau_1 (\tau_4+\tau_{12}-2\tau_6)+(\tau_7-\tau_5)\tau_8$ since
\begin{equation}
    -\frac{\tau_1(\tau_7-\tau_5)(\tau_4+\tau_{12}-2\tau_6)+(\tau_7-\tau_5)^2\tau_8}{\left(\tau_1(\tau_2-\tau_3)-1\right)^2}=\frac{(\tau_5-\tau_7)\Delta}{\left(\tau_1(\tau_2-\tau_3)-1\right)^2}.
\end{equation}
Plugging in quantities in \eqref{eq:simplify_taui} with $z=\lambda\psi_1/\psi_2$, we have
\begin{align}
    \Delta=&-\mu_1^*\theta_2 \left(\frac{\psi_1}{\psi_2}\frac{m_2}{zm_1^2}\left(m_1+zm_1'\right)+\frac{\tau_1m_2'}{m_1^2}\right)\\
    =& -\frac{\mu_1^*\theta_2}{zm_1^2}\left(\frac{\psi_1}{\psi_2}\left(m_1m_2+zm_2m_1'+zm_1m_2'\right)+\left(1-\frac{\psi_1}{\psi_2}\right)m_2'\right)\\
    = & -\left.\frac{\mu_1^*\theta_2}{zm_1^2} \frac{d}{dz}\right\vert_{z=\lambda\psi_1/\psi_2}\left(\frac{\psi_1}{\psi_2} zm_1(z)m_2(z)+\left(1-\frac{\psi_1}{\psi_2}\right)m_2(z)\right)\\
    \overset{(i)}{=}&  -\left.\frac{\mu_1^*\theta_2}{zm_1^2\psi_1\mu_1^2} \frac{d}{dz}\right\vert_{z=\lambda\psi_1/\psi_2}\left(1-\frac{m_2(z)}{m_1(z)}\right)\\
    \overset{(ii)}{=}&  -\left.\frac{\mu_1^*\theta_2}{zm_1^2 \mu_1^2} \frac{d}{dz}\right\vert_{z=\lambda\psi_1/\psi_2}z\mu_1^2 m_1(z)\tau_1(z),
\end{align}
where $(i)$ and $(ii)$ are due to \eqref{eq:fixedpoint2} and \eqref{eq:m2/m1}, respectively. By Lemma A.1 in \cite{tripuraneni2021covariate}, we know function $z\mu_1^2 m_1(z)\tau_1(z)$ has non-positive derivative when $z>0$. This implies that $\Delta\ge 0$ and hence the second part of $\delta(\eta,\lambda,\psi_1,\psi_2)$ is also non-negative. 

Finally, we note that when $\mu_1^*=0$, the function $\delta(\eta,\lambda,\psi_1,\psi_2)=0$. This is because $$\tau_7-\tau_5=-\mu_1^2\mu_1^*\theta_2\psi_1^2\lambda m_2/\psi_2^2=0,$$ 
when $\mu_1^*=0$. Whereas when $\eta=0$, we know that $\theta_1=\theta_2=0$, which entails $\delta(\eta,\lambda,\psi_1,\psi_2)$ is also vanishing. 
Also observe that in \eqref{eq:simplify_taui}, $m_1,m_2,m_1',m_2',\tau_1$ are all positive. Hence we conclude that if $\delta(\eta,\lambda,\psi_1,\psi_2)=0$, then at least one of $\eta,\mu_1\mu_1^*$ must be zero.

\end{proof}

\subsection{Analysis of Special Cases}
\label{sec:special_cases}

While the previous subsection provides explicit formulae of $\delta$, the expressions are rather complicated due to the self-consistent equations \eqref{eq:fixedpoint1} and \eqref{eq:fixedpoint2}. 
In this section we consider two special cases: the large sample limit $\psi_1\to\infty$ and the large width limit $\psi_2\to\infty$, where the calculation simplifies and enables us to further characterize properties of $\delta$. 
In both cases, we start with Theorem \ref{thm:risk_general} and take one of aspect ratios ($\psi_1$ or $\psi_2$) to infinity. 

\subsubsection{Case I: Large sample limit}

In this subsection we prove Proposition \ref{prop:large-sample-size}. We introduce two positive parameters
\begin{align}
    s_1:=\int\frac{1}{\mu_1^2x+\mu_2^2+\lambda}d\mu^{\MP}_{\psi_2}(x),\quad 
    s_2:=\int\frac{1}{\left(\mu_1^2x+\mu_2^2+\lambda\right)^2}d\mu^{\MP}_{\psi_2}(x),\label{eq:tau_1}
\end{align}
where $\mu^{\MP}_{\psi_2}$ is Marchenko–Pastur distribution with rate $\psi_2\in(0,\infty)$. Now we consider the large-sample limit: $\psi_1\to\infty, \psi_2\in(0,\infty)$. The following statement is the formal version of Proposition \ref{prop:large-sample-size}, and compared to the general result (Theorem \ref{thm:risk_general}), this special case admits a more explicit formula only determined by $s_1$ and $s_2$.

\begin{theo}[Large sample limit]\label{thm:case1}
Under the same assumptions as Theorem \ref{thm:risk_general} and take $\psi_1\to\infty$. Then the difference between the prediction risks before and after one feature learning step $\cR_0(\lambda) - \cR_1(\lambda)$  satisfies
\begin{equation}\label{eq:delta_case1}
    \lim_{\psi_1\to\infty}~\lim_{n,d,N\to\infty}\left(\cR_0(\lambda)-\cR_1(\lambda)\right)=:\delta(\eta,\lambda,\infty,\psi_2)=\mu_1^{*2}\left(\frac{AB}{A+1}+\frac{C}{(A+1)^2}\right),
\end{equation}
in probability, where
\begin{align}
A:=&\mu_1^2\theta_2^2s_1(1+\psi_2(\mu_2^2+\lambda)s_1-\psi_2),\label{eq:Alimit}\\
B:=&1-\psi_2+\psi_2\lambda(\mu_2^2+\lambda)s_2+\mu_2^2\psi_2s_1,\label{eq:Blimit}\\
C:=& \lambda\mu_1^2\theta_2^2(1+\psi_2(\mu_2^2+\lambda)s_1-\psi_2)\left(2(\mu_2^2+\lambda)\psi_2s_1s_2-\psi_2s_1^2+s_2(1-\psi_2)\right).\label{eq:Climit}
\end{align} 
In this case $\delta(\eta,\lambda,\infty,\psi_2)$ is a non-negative function of $\eta,\lambda,\psi_2\in (0,+\infty)$, and $\delta=0$ if and only if one of $\mu_1,\mu_1^*,\eta$ is zero. Furthermore, $\delta(\eta,\lambda,\infty,\psi_2)$ is increasing with respect to the learning rate $\eta\ge 0$. 
\end{theo}
\begin{proof}
Following Theorem \ref{thm:risk_general}, it suffices to consider the limit of $\delta(\eta,\lambda,\psi_1,\psi_2)$ when $\psi_1\to\infty$. This reduces to simplifying the asymptotics of $\tau_i$'s defined in Proposition \ref{prop:tau_i}, as $\delta(\eta,\lambda,\psi_1,\psi_2)$ is determined by $\tau_i$'s in Theorem \ref{thm:risk_general}.  We aim to prove the following:
\begin{equation}\label{eq:tau_i_new}
    \begin{aligned}
    &\frac{\psi_1}{\psi_2}\tau_1\to s_1
    ,&\quad&  \frac{\psi_2}{\mu_1^2\psi_1}\tau_2\to \theta_2^2\psi_2(1-(\mu_2^2+\lambda)s_1),\\
    &\tau_6\to\mu_1^{*}\theta_2\psi_2 (1-(\mu_2^2+\lambda)s_1),&\quad& \frac{\psi_2}{\mu_1^2\psi_1}\tau_7 \to\mu_1^{*}\theta_2\psi_2 (1-(\mu_2^2+\lambda)s_1),\\
    &\frac{\mu_1^2\psi_1^2}{\psi_2^2}\tau_8\to s_1-\lambda s_2,&\quad& \tau_9\to \theta_2^2 \psi_2(1-(\mu_2^2+\lambda)s_1),\\
    &\tau_{11}\to\theta_2^2\psi_2(1-(\mu_2^2+2\lambda)s_1+\lambda(\mu_2^2+\lambda)s_2),&\quad& \tau_{12}\to\mu_1^{*}\theta_2\psi_2(1-(\mu_2^2+2\lambda)s_1+\lambda(\mu_2^2+\lambda)s_2),
\end{aligned}
\end{equation}
as $\psi_1\to\infty$, where $s_1$ and $s_2$ are defined in \eqref{eq:tau_1}. The trivial cases when $\mu_1,\eta=0$ have been studied in Theorem \ref{thm:risk_general}. So, WLOG, we assume $\mu_1,\eta>0$ in the following derivations. 

Recall the definitions of $\tau_1(z),m_1(z)$ and $m_2(z)$. One can easily see that $m_1,m_2\to 0$ as $\psi_1\to\infty$. For any $z\ge 0$, \eqref{eq:fixedpoint1} and \eqref{eq:fixedpoint2} can be written as follows
\begin{align}
     &\psi_1(zm_1(z)-1)+\psi_1\mu_1^2zm_1(z)\tau_1(z)+\frac{\mu_2^2}{\mu_1^2}\left(1-\frac{1}{1+\psi_1\mu_1^2zm_1(z)\tau_1(z)}\right)=0,\label{eq:m1_def}\\
     & m_2(z)=\frac{m_1(z)}{1+\psi_1\mu_1^2m_1(z)\left(\frac{\psi_1}{\psi_2}(zm_1(z)-1)+1\right)}=\frac{m_1(z)}{1+\psi_1\mu_1^2zm_1(z)\tau_1(z) }.\label{eq:m2_def}
\end{align}
Notice that $\tau_1(z)=\lim\tr\vR_0(z)$, and based on \cite[Theorem 3.4]{fan2020spectra}, for any $z\ge 0$,
\begin{equation}
    \tr\vR_0(z)=\frac{N}{n}\tr\left(\frac{1}{n}\left(\mu_1\tilde{\vX}\vW_0 + \mu_2 \vZ \right)\left(\mu_1\tilde{\vX}\vW_0 + \mu_2 \vZ \right)+\frac{N}{n}z\vI\right)^{-1},
\end{equation}
where, by exchanging $\vW_0$ and $\tilde{\vX}$, the limit of the right hand side coincides with the Stieltjes transform $\mu^{\MP}_{\psi_2/\psi_1}\boxtimes \left(\mu_2^2+\mu_1^2\cdot\mu^{\MP}_{\psi_2}\right)$ at $-z$, as $n/d\to\psi_1$ and $N/d\to\psi_2$. Denote this Stieltjes transform at point $-z$ by $\tilde{m}(z)$. Then as $n/d\to\psi_1$ and $N/d\to\psi_2$, we have
\begin{equation}\label{tau_1m_tilde}
    \tau_1(z)= \frac{\psi_2}{\psi_1}\tilde{m}\left(\frac{\psi_2}{\psi_1}z\right).
\end{equation}Therefore, when $\psi_1\to\infty$, the measure $\mu^{\MP}_{\psi_2/\psi_1}\boxtimes \left(\mu_2^2+\mu_1^2\cdot\mu^{\MP}_{\psi_2}\right)$ reduces to a deformed Marchenko–Pastur law $ \left(\mu_2^2+\mu_1^2\cdot\mu^{\MP}_{\psi_2}\right)$; hence by the definition of $s_1$,
\begin{equation}\label{eq:tau_1limit}
    \frac{\psi_1}{\psi_2}\tau_1\to \tilde{m}(\lambda)=s_1,
\end{equation}
which verifies the first statement in \eqref{eq:tau_i_new}. Now recall the value of interest $z=\frac{\lambda\psi_1}{\psi_2}$. Due to the relationship between $\tau_1$ and $m_1$, it is straightforward to deduce that $\lambda\frac{\psi_1}{\psi_2}m_1\to 1$ as $\psi_1\to \infty$. As for $\tau_2$, in terms of \eqref{eq:m2_def}, we have
\begin{align}
\frac{\psi_2}{\mu_1^2\psi_1}\tau_2=\theta_1^2\left(1-\lambda\frac{\psi_1}{\psi_2}m_2\right)
    = \theta_1^2\left(1-\frac{\lambda\frac{\psi_1}{\psi_2}m_1}{1+\psi_1\mu_1^2\lambda\frac{\psi_1}{\psi_2}m_1 \tau_1  }\right)\to \frac{\theta_1^2\mu_1^2\psi_2s_1}{1+\psi_2\mu_1^2s_1},\label{eq:limit_tau2}
\end{align}
as $\psi_1\to \infty$. 
Define $\tilde z=-(\mu_2^2+\lambda)/\mu_1^2< 0$ and consider the Stieltjes transform $m( \tilde z)>0$ of $\mu^{\MP}_{\psi_2}$, which satisfies the self-consistent equation
\begin{equation}\label{eq:self_consistent1}
    \psi_2\tilde{z}m^2(\tilde{z})-m(\tilde{z})+\psi_2m(\tilde{z})+\tilde{z}m(\tilde{z})+1=0,
\end{equation}
which can be rewritten as follows
\begin{equation}\label{eq:self_consistent2}
    \tilde{z}m(\tilde{z})+1=\frac{1}{1-\tilde{z}-\psi_2 \tilde{z}m(\tilde{z})}>0.
\end{equation}
Note that $s_1=m(\tilde{z})/\mu_1^2$. By \eqref{eq:self_consistent1}, we can simplify \eqref{eq:limit_tau2} to obtain $\frac{\psi_2}{\mu_1^2\psi_1}\tau_2\to \theta_2^2\psi_2(1-(\mu_2^2+\lambda)s_1)$ when $\psi_1\to\infty$. 
The calculation of $\tau_6, \tau_7$ and $\tau_9$ are essentially the same as $\tau_2$ based on properties of the Stieltjes transform \eqref{eq:m2_def} and \eqref{eq:self_consistent1}, the details of which we omit. 

Next we compute the limit of $m_1'$. Taking derivative with respect to $z$ at both sides of \eqref{tau_1m_tilde}, we arrive at $\frac{\psi_1^2}{\psi_2^2}\tau_1'\to -s_2$; here $\tau_1'$ represents the derivative $\tau_1'(z)$ at $z=\frac{\lambda\psi_1}{\psi_2}$. Combining this relation and \eqref{eq:tau_1limit},\eqref{eq:tau_1m1}, we can deduce that
\begin{align}
    &\frac{\mu_1^2\psi_1^2}{\psi_2^2}\tau_8= \frac{\frac{\psi_2}{\psi_1}\left(\tau_1+\frac{\psi_2}{\psi_1\lambda}\left(\frac{\psi_1}{\psi_2}-1\right)\right)+\lambda\left(\tau_1'+\left(1-\frac{\psi_1}{\psi_2}\right)\frac{\psi_2^2}{\psi_1^2\lambda^2}\right)}{\lambda^2\left(\frac{\psi_2}{\psi_1}\left(\tau_1+\frac{\psi_2}{\psi_1\lambda}\left(\frac{\psi_1}{\psi_2}-1\right)\right)\right)^2}
    = \frac{\frac{\psi_1}{\psi_2}\tau_1+\lambda\left(\frac{\psi_1}{\psi_2}\right)^2\tau_1'}{\left(\lambda\tau_1+1-\frac{\psi_2}{\psi_1}\right)^2}\to s_1-\lambda s_2.
\end{align}
Note that here we used $\tau_1\to 0$ when $\psi_1\to \infty$. Lastly, for $\tau_{11}$ and $\tau_{12}$, by \eqref{eq:m2_def},
\begin{align}
    \frac{m_2'}{m_1^2}=&\frac{m_1'-\psi_2\mu_1^2\lambda\frac{\psi_1^2}{\psi_2^2}\tau_1'm_1^2-\psi_1\mu_1^2m_1^2\tau_1}{m_1^2\left(1+\psi_2\mu_1^2\lambda\frac{\psi_1}{\psi_2}m_1\tau_1\right)^2}\\
    =& \frac{\frac{\psi_1}{\psi_2}\tau_1'+\left(1-\frac{\psi_1}{\psi_2}\right)\frac{\psi_2}{\psi_1\lambda^2}-\psi_2\mu_1^2\frac{\psi_1^2}{\psi_2^2}\lambda\tau_1'\left(\tau_1+\frac{1}{\lambda}-\frac{\psi_2}{\psi_1\lambda}\right)^2}{\left(\tau_1+\frac{1}{\lambda}-\frac{\psi_2}{\psi_1\lambda}\right)^2\left(1+\psi_2\mu_1^2\lambda\frac{\psi_1}{\psi_2}m_1\tau_1\right)^2}-\frac{\psi_1\mu_1^2\tau_1}{\left(1+\psi_2\mu_1^2\lambda\frac{\psi_1}{\psi_2}m_1\tau_1\right)^2}\\
    \to & \frac{1-\psi_2\lambda\mu_1^2s_2+\mu_1^2\psi_2 s_1}{\left(1+\psi_2\mu_1^2s_1\right)^2},\text{ \, as }\psi_1\to\infty,
\end{align}
where we applied the previously established convergence of $\tau_1$ and $\tau_1'$. Also recall that \eqref{eq:limit_tau2} implies that $m_2/m_1$ converges to $1/\left(1+\psi_2\mu_1^2s_1\right)$ as $\psi_1\to\infty$. Together with the convergence of $\tau_9$ in \eqref{eq:tau_i_new}, we get
\begin{align}
    \tau_{11}\to \theta_2^2 \psi_2(1-(\mu_2^2+\lambda)s_1)+\frac{-\psi_2\lambda\mu_1^2s_2}{\left(1+\psi_2\mu_1^2s_1\right)^2}.
\end{align}
Meanwhile, we also know that $s_2=m'(\tilde z)/\mu_1^4$, where $\tilde z=-(\mu_2^2+\lambda)/\mu_1^2$, and $m( \tilde z)$  is the Stieltjes transform $\mu^{\MP}_{\psi_2}$. Hence by \eqref{eq:self_consistent1}\eqref{eq:self_consistent2} and taking derivative in \eqref{eq:self_consistent1}, we have 
 \[-\frac{ \psi_2\lambda\mu_1^2s_2}{\left(1+\psi_2\mu_1^2s_1\right)^2}=\psi_2\lambda\left((\mu_2^2+\lambda)s_2-s_1\right),\]
which implies the convergence of $\tau_{11}$ in \eqref{eq:tau_i_new}. 

As a result, by replacing $\tau_i$'s in \eqref{delta_formula} with the corresponding reparameterized $\tau_i$'s in \eqref{eq:tau_i_new}, we arrive at the following expression of $\delta$:
\begin{equation}\label{eq:R0-R1_original}
    \lim_{\psi_1\to\infty}~~\lim_{n,d,N\to\infty}\left(\cR_0(\lambda)-\cR_1(\lambda)\right)= \mu_1^{*2}\left(\frac{\mu_1^2\theta_2^2s_1\alpha\beta}{\mu_1^2\theta_2^2s_1\alpha-1}-\frac{\mu_1^2\theta_2^2\alpha\gamma}{\left(\mu_1^2\theta_2^2s_1\alpha-1\right)^2}\right),
\end{equation}
in probability, where we defined
\begin{align}
    \alpha:=&\psi_2-1-\psi_2(\mu_2^2+\lambda)s_1,\label{alpha}\\
    \beta:= &1-\psi_2+\psi_2\mu_2^2s_1+\lambda\psi_2(\mu_2^2+\lambda)s_2,\label{beta}\\
    \gamma:=&s_1\beta+\alpha(s_1-\lambda s_2)=\lambda\left(2(\mu_2^2+\lambda)\psi_2 s_1 s_2-\psi_2 s_1^2+s_2(1-\psi_2)\right).\label{gamma}
\end{align}
By definitions of $A,B,C$ in \eqref{eq:Alimit}, \eqref{eq:Blimit} and \eqref{eq:Climit}, we can see that $A=-\mu_1^2\theta_2^2\alpha s_1$, $B=\beta$ and $C=-\mu_1^2\theta_2^2\alpha\gamma$; this leads to the equivalent expression
\begin{equation}
     \lim_{\psi_1\to\infty}~~\lim_{n,d,N\to\infty}\left(\cR_0(\lambda)-\cR_1(\lambda)\right)=\mu_1^{*2}\left(\frac{AB}{A+1}+\frac{C}{(A+1)^2}\right)=: \delta( \eta,\lambda,\infty,\psi_2). 
\end{equation}
Now we claim that $A,B,C$ are all non-negative, for any $\eta,\lambda,\psi_2\ge 0$. With a slight abuse of terminology, in the following we denote $z=-(\mu_2^2+\lambda)/\mu_1^2< 0$. 
Recall that $s_1=m(z)/\mu_1^2$ and $s_2=m'(z)/\mu_1^4$;  We can therefore simplify \eqref{alpha}, \eqref{beta} and \eqref{gamma} as follows
\begin{align}
    \alpha=&\psi_2-1+\psi_2zm(z)\overset{(i)}{=}-\frac{1+zm(z)}{m(z)}<0,\\
    \beta=& 1-\frac{\lambda}{\mu_1^2}\psi_2\left(m(z)+zm'(z)\right)-\psi_2\left(zm(z)+1\right),\\
    \gamma=&\frac{\lambda}{\mu_1^4}\left((1-\psi_2)m'(z)-2\psi_2zm(z)m'(z)-\psi_2 m^2(z)\right)\overset{(ii)}{=}\frac{\lambda}{\mu_1^4}\left(m(z)+zm'(z)\right), 
\end{align}
where $(i)$ is due to \eqref{eq:self_consistent2} and $(ii)$ is obtained by taking derivative with respect to $z$ in \eqref{eq:self_consistent1}. In addition, 
\begin{equation}
    m(z)+zm'(z)=\int \frac{\mu_1^4 x}{(\mu_1^2x+\mu_2^2+\lambda)^2}d\mu^{\MP}_{\psi_2}(x)>0,
\end{equation}
which implies that $\gamma>0$. We also denote the companion Stieltjes transform of $m(z)$ by $\bar m(z)$, which is the Stieltjes transform of the limiting eigenvalue distribution for $\vW_0\vW_0^\top$. Recall the following relation between $m(z)$ and $\bar m(z)$: $\bar m(z)+\frac{1}{z}=\psi_2\left(m(z)+\frac{1}{z}\right)$. 
Since $ m(z)+zm'(z)$ is positive, we can deduce that
\begin{align*}
    \beta \ge& 1-\frac{\lambda+\mu_2^2}{\mu_1^2}\psi_2\left(m(z)+zm'(z)\right)-\psi_2\left(zm(z)+1\right)\\
    =& 1+\psi_2\left(zm(z)+z^2m'(z)\right)-\psi_2\left(zm(z)+1\right) = 1+\psi_2\left(z^2m'(z)-1\right)
    =z^2\bar m'(z)>0,
\end{align*}
where the last equality is obtained by taking derivative of \eqref{eq:companinon_m(z)} on both sides with respect to $z$. In summary, we have shown that $\alpha<0$ and $\beta,\gamma>0$ when $\lambda,\mu_1>0$. Hence by definition, $A,B,C$ are all non-negative and so is $\delta( \eta,\lambda,\infty,\psi_2)$.

Finally, we verify that $\delta( \eta,\lambda,\infty,\psi_2)$ is an increasing function of $\eta\ge 0$. Observe that $\eta$ only appears in $\theta_2$ in the expression of $\delta( \eta,\lambda,\infty,\psi_2)$ in \eqref{eq:R0-R1_original}. Hence, it suffices to take the derivative of $\delta( \eta,\lambda,\infty,\psi_2)$ with respect to $\theta_2$ and verify that this partial derivative is positive. One can check that
\begin{align*}
    \frac{\partial}{\partial\eta}\delta( \eta,\lambda,\infty,\psi_2)=2\theta_2\mu_1\mu_1^*\cdot\frac{\mu_1^4\theta_2^2\alpha^2 s_1\left(\gamma-s_1\beta\right)+\mu_1^2\alpha\left(s_1\beta+\gamma\right))}{\left(\mu_1^2\theta_2^2 s_1\alpha-1\right)^3}.
\end{align*}
By the definition of $\gamma$ in \eqref{gamma}, we have $\left(\gamma-s_1\beta\right)=\alpha(s_1-\lambda s_2)$. Also from \eqref{eq:tau_1} we know that $\lambda s_2\le s_1$. Finally, recall that $\alpha< 0$ and $\beta,\gamma>0$; this implies $\delta( \eta,\lambda,\infty,\psi_2)$ is increasing with regard to $\eta\in[0,+\infty)$ and completes the proof.

\end{proof}

\subsubsection{Case II: Highly overparameterized regime}
Next we consider the large width limit $\psi_2\to\infty$ and establish Proposition \ref{prop:large-width}. 
\begin{proofof}[Proposition \ref{prop:large-width}]
We first highlight that the constant $\mu_2>0$ since the activation $\sigma$ is a nonlinear function. Similar to the proof of Theorem \ref{thm:case1}, we need to consider the limits of $\tau_i$ defined in Proposition \ref{prop:tau_i} as $\psi_2\to\infty$. 
We first study the asymptotics of $m_1,m_2,m_1'$ and $m_2'$ as $\psi_2\to\infty$.Recall that in Proposition \ref{prop:tau_i}, $m_1(z)$ is the Stieltjes transform of limiting spectrum of CK matrix at $-z$. In fact, the limiting eigenvalue distribution of CK is $\mu^{\MP}_{\psi_1/\psi_2}\boxtimes \left(\mu_2^2+\mu_1^2\cdot\mu^{\MP}_{\psi_1}\right)$, which has been analyzed in \cite[Theorem 3.4]{fan2020spectra}. Therefore due to \cite[Equation (6)]{fan2020spectra}, we know that for any $z\in\mathbb{C}^+\cup \R_{+}$, $m_1(z)$ satisfies the self-consistent equation 
\begin{equation}\label{eq:m1zequation}
    m_1(z)=\int \frac{d\mu^{\MP}_{\psi_1}(x) }{\left(\mu_1^2x+\mu_2^2\right)\left(1-\frac{\psi_1}{\psi_2}+\frac{\psi_1}{\psi_2}zm_1(z)\right)+z}.
\end{equation}
Note that $0\le zm_1(z)\le 1$, for all $z\ge 0$, and thus $0\le \frac{\psi_1}{\psi_2}zm_1(z)\le \frac{\psi_1}{\psi_2}$. On the other hand, $\mu^{\MP}_{\psi_1}$ is compactly supported. Therefore, by taking $z=\psi_1\lambda/\psi_2 $ and letting $\psi_2\to \infty$ at both sides of \eqref{eq:m1zequation}, we arrive at
\[\lim_{\psi_2\to\infty} m_1 = \int \frac{d\mu^{\MP}_{\psi_1}(x) }{ \mu_1^2x+\mu_2^2 }\in(0,1/\mu_2^2),\]
which is a finite positive value determined by $\psi_1,\mu_1,\mu_2$. With this in mind, we conclude that $\lim_{\psi_2\to\infty} m_2$ is also finite, since $m_2(z)$ is determined by \eqref{eq:m2_def} and we can take $z=\psi_1\lambda/\psi_2 $ with $\psi_2\to \infty$. In addition, since $0 \le -zm_1'(z)\le m_1(z)$ for any $z\ge 0$, we may take the derivative with respect to $z$ at both sides of \eqref{eq:m1zequation} to obtain $m_1'(z)$, and take $z=\psi_1\lambda/\psi_2 $ and $\psi_2\to \infty$ to conclude that the limit of $  m_1'$ is finite as well. Similarly, by taking derivative with respect to $z$ in \eqref{eq:m2_def}, one can also verify that as $\psi_2\to\infty$, the limit of $ m_2'$ remains finite. From these estimates we know that
\begin{align}
    \tau_2-\tau_3=&-\mu_1^2\theta_1^2\left(\frac{\psi_1}{\psi_2}\right)^2\lambda m_2,\\ \tau_7-\tau_5=&-\mu_1^2\mu_1^*\theta_2\left(\frac{\psi_1}{\psi_2}\right)^2\lambda m_2,\\ \left(\frac{\psi_1}{\psi_2}\right)^3\tau_8=&\frac{1}{m_1}\frac{1}{\mu_1^2\lambda^2}\left(\frac{\psi_1}{\psi_2}\right)+\frac{m_1'}{m_1^2}\left(\frac{\psi_1}{\psi_2}\right)^2\frac{1}{\mu_1^2\lambda},
\end{align}are vanishing as $\psi_2\to\infty$, whereas 
\begin{align}
    \tau_4 +\tau_{12}-2\tau_6=- \mu_1^*\theta_2\frac{m_2'}{m_1^2},\quad \tau_{11} +\tau_{10}-2\tau_9=- \theta_1^2\frac{m_2'}{m_1^2},\quad
    \frac{\psi_1}{\psi_2}\tau_1= \left(\frac{\psi_1}{\psi_2}\right)^2m_1+\left(1-\frac{\psi_1}{\psi_2}\right)\frac{1}{\lambda}
\end{align} 
will converge to some finite values. 
The proposition is established based on the definition of $\delta(\eta,\lambda,\psi_1,\psi_2)$ with the help of the above statements.
 
\end{proofof}
\bigskip
\section{Proof for Large Learning Rate ($\eta=\Theta(\sqrt{N})$)}

In this section we restrict ourselves to a single-index target function (generalized linear model): $f^*(\vx) = \sigma^*(\langle\vx,\vbeta_*\rangle)$, and study the impact of one gradient step with large learning rate $\eta=\Theta(\sqrt{N})$. For simplicity, we denote $\eta = \bar\eta \sqrt{N}$ where $\bar\eta>0$ is a fixed constant not depending on $N$. 

As the Gaussian equivalence property is no longer applicable, we instead establish an upper bound on the prediction risk of the CK ridge estimator. 
Our proof is divided into two parts: $(i)$ we show that there exists an ``oracle'' second-layer $\tilde{\va}$ that achieves small prediction risk $\tau^*$ when $n/d$ is large; $(ii)$ based on $\tau^*$, we provide an upper bound on the prediction risk when the second layer is estimated via ridge regression. 

Here we provide a short summary on the construction of $\tilde{\va}$ and upper bound on the prediction risk. 
\begin{itemize}[leftmargin=*,topsep=1mm, itemsep=0.4mm]
    \item We first introduce $f_r(\vx) := \frac{1}{\abs{\cA_r}}\sum_{i\in\cA_r} \sigma\left(\langle\vx,\vw^1_{i}\rangle\right)$, which is the average of a subset of neurons in $\cA_r\subset [N]$ defined in \eqref{eq:Ar}. Intuitively, this subset of neurons approximately matches the target direction $\vbeta_*$. This averaging corresponds to setting the second-layer $\tilde{\va}_i = \frac{\sqrt{N}}{\abs{\cA_r}}$ for all $i\in \cA_r$.  
    \item We show that $f_r$ can be approximated up to $\Theta(d/n)$-error by an ``expected'' single-index model $\bar{f}(\vx) := \E_{\vw\sim\cN(0,\,\vI/d)} \left[\sigma(\langle\vw+c\vbeta_*,\vx\rangle)\right]$, for some $c\in\R$ that depends on the learning rate and nonlinearities. To bound this substitution error, we establish a more refined control of gradient norm in Section~\ref{app:B-Frobenius}. 
    \item By choosing an ``optimal'' subset $\cA_r$, we simplify the prediction risk of $\bar{f}$ into the  one-dimensional expectation $\tau^*$ defined in \eqref{eq:tau*}. This provides a high-probability upper bound of the prediction risk of the constructed $\tilde{\va}$ up to $\Theta(d/n)$-error. 
\end{itemize}

After constructing some $\tilde{\va}$ that achieves reasonable test performance, we can then show that the prediction risk of CK ridge regression estimator with trained weight $\vW_1$ is also upper-bounded by $\tau^*$ when $n \gg d$. This result is established in Section~\ref{app:KRR} and follows from classical analysis of kernel ridge regression.

\subsection{Refined Properties of the First-step Gradient}
\label{app:B-Frobenius}

Recall that $\vW_1 = \vW_0 + \eta\sqrt{N}\vG_0$, where $\vG_0 = \vA_1 + \vA_2 + \vB + \vC$ is defined in Lemma~\ref{lemm:gradient-norm} and \ref{lemm:gradient-norm2}, and the full-rank term $\vB$ is given as
\[
\vB = \frac{1}{n}\cdot\frac{1}{\sqrt{N}}\vX^\top\left(\vy\va^\top\odot\sigma'_\perp(\vX\vW_0)\right).
\]
We first refine the estimate on the Frobenius norm of certain submatrix of $\vB$; the choice of such submatrices will be explained in Section \ref{app:oracle-estimator}.  
\begin{lemm}\label{lemm:frobenius}
Given Assumptions~\ref{assump:1} and \ref{assump:2}, take $\vB_r\in\R^{d\times N_r}$ which is a submatrix of $\vB$ via selecting any $N_r\in[N]$ columns in $\vB$, and let $\va_r\in\R^{N_r}$ be the corresponding 2nd layer coefficients. 
If entries of $\va_r$ are uniformly bounded by $\alpha/\sqrt{N}$, then for any $\eps\in(0,1/4)$, we have 
\begin{equation}\label{eq:frobenius_exp}
    \E\norm{\vB_r}_F^2 \le \frac{C_0\alpha^2 N_r}{N}\left(\frac{1}{Nd^{\frac{1}{2}-\eps}}+\frac{d}{nN}\right),
\end{equation}
and
\begin{equation}\label{eq:frobenius_prob}
    \Parg{\frac{N}{N_r}\norm{\vB_r}_F^2\le \frac{C_1}{Nd^{\frac{1}{4}-\eps}}+\frac{C_2d}{Nn} }\ge 1-\frac{\alpha^2}{d^{\frac{1}{4}}}-\frac{\alpha^4}{n},
\end{equation}
where constants $C_0,C_1,C_2>0$ only depend on $\lambda_\sigma$ and $\|f^*\|_{L^2(\R^d,\Gamma)}$.
\end{lemm}

\begin{proof}
    Let $\vX^\top=(\vx_1,\vx_2,\ldots,\vx_n)$ and $\vy^\top=(y_1,\ldots,y_n)$. Then, matrix $\vB_r$ can be written as
    \[\vB_r=\frac{1}{n\sqrt{N}}\sum_{i=1}^n y_i\vx_i\sigma'_\perp(\vx_i^\top \vW_0^r)\diag (\va_r),\]where $\vW^r_0\in\R^{d\times N_r}$ is a submatrix of $\vW_0$ by choosing any $N_r$ columns of $\vW_0$, and $\va_r\in\R^{N_r}$ is the corresponding second layer (note that by assumption $\|\va_r\|_\infty \le \alpha/\sqrt{N}$). Hence, 
    \begin{align*}
        \|\vB_r\|_F^2=&\Tr(\vB_r\vB_r^\top)=\frac{1}{n^2N}\sum_{i,j=1}^ny_iy_j\Tr\left[\vx_i\sigma'_\perp(\vx_i^\top \vW^r_0)\diag (\va_r)^2\sigma'_\perp(\vx_j^\top \vW^r_0)^\top\vx_j^\top\right]\nonumber\\
        =& \frac{1}{n^2N}\sum_{i,j=1}^ny_iy_j\left(\sigma'_\perp(\vx_i^\top \vW^r_0)\diag (\va_r)^2\sigma'_\perp(\vx_j^\top \vW^r_0)^\top\vx_j^\top \vx_i\right)\nonumber\\
        =& \frac{1}{n^2N}\sum_{i\neq j}^ny_iy_j\left(\sigma'_\perp(\vx_i^\top \vW^r_0)\diag (\va_r)^2\sigma'_\perp(\vx_j^\top \vW^r_0)^\top\vx_j^\top \vx_i\right)\nonumber\\
        &+\frac{1}{n^2N}\sum_{i=1}^ny_i^2\left(\sigma'_\perp(\vx_i^\top \vW^r_0)\diag (\va_r)^2\sigma'_\perp(\vW^{r\top}_0\vx_i)\|\vx_i\|^2\right)=:J_1+J_2.
    \end{align*}
    Here, $J_1$ represents the sum for distinct $i\neq j\in [n]$ and $J_2$ is the sum when $i=j\in [n]$. Therefore,
    \begin{align*}
        \E[ \|\vB_r\|_F^2]\le &\frac{\alpha^2N_r}{N} \frac{n(n-1)}{n^2N}\E\left[f^*(\vx_1)f^*(\vx_2)\sigma'_\perp(\vx_1^\top \vw)\sigma'_\perp(\vx_2^\top \vw)\vx_2^\top \vx_1\right]\\
        &+\frac{\alpha^2N_r}{nN^2}\E\left[(f^*(\vx_1)^2+\sigma_\eps^2)\sigma'_\perp(\vx_1^\top \vw)^2\|\vx_1\|^2\right],
    \end{align*}
    where $\vw\sim\cN(0,\vI)$ independent of $\va $ and $\vX$. We compute the aforementioned expectations as follows
    \begin{align*}
         \frac{N}{\alpha^2N_r}\E[ \|\vB\|_F^2] 
        \le &\frac{1}{ N}\left|\E\left[f^*(\vx_1)f^*(\vx_2)\sigma'_\perp(\vx_1^\top \vw)\sigma'_\perp(\vx_2^\top \vw)\vx_2^\top \vx_1\right]\right|\\
        &+\frac{1}{nN}\E\left[f^*(\vx_1)^2\sigma'_\perp(\vx_1^\top \vw)^2\|\vx_1\|^2\right]+\frac{\sigma_\eps^2}{nN}\E\left[\sigma'_\perp(\vx_1^\top \vw)^2\|\vx_1\|^2\right]=:I_1+I_2+I_3.
    \end{align*}
    
    To verify \eqref{eq:frobenius_exp}, we in turn control $I_1,I_2$ and $I_3$. Since the target function is a single-index model $f^*(\vx) = \sigma^*(\langle\vx,\vbeta_*\rangle)$ and $\sigma^*$ is Lipschitz, it is clear that $f^*$ belongs to $L^2(\R^d,\Gamma)$. Besides, $\vx_1,\vx_2$ are two independent standard Gaussian random vectors. Therefore, $\E_{\vx_1,\vx_2}\left[\left|f^*(\vx_1)f^*(\vx_2)\right|^2\right]=\|f^*\|_{L^2(\R^d,\Gamma)}^4$. Now given any $t\in(0,1)$, define event by
    \[\mathcal{A}_t:=\left\{\frac{|\vx_1^\top\vx_2|}{d}\le t,\left|\frac{\|\vx_i\|}{\sqrt{d}}-1\right|\le t\text{, for }i=1,2 \right\}.\]
    
Using the same rotational invariance argument as Step 1 in the proof of Lemma \ref{lem:Titilde}, WLOG, we can further consider $\vbeta_*\sim \text{Unif}(\mathbb{S}^{d-1})$ independent of $\vx_1,\vx_2$ and $\vw$, because $\vx_1,\vx_2$ and $\vw$ are rotationally invariant in distribution. Thus, for $I_1$, we have
\begin{align}
   I_1=&\frac{1}{N}\left|\E_{\vx_1,\vx_2}\left[\E_{\vbeta_*}[\sigma^*(\vbeta_*^\top\vx_1)\sigma^*(\vbeta_*^\top\vx_2)]\E_{\vw}[\sigma'_\perp(\vx_1^\top \vw)\sigma'_\perp(\vx_2^\top \vw)]\vx_2^\top \vx_1\right]\right|\\
   \le & \frac{1}{N}\E_{\vx_1,\vx_2}\left[\left|\E_{\vbeta_*}[\sigma^*(\vbeta_*^\top\vx_1)\sigma^*(\vbeta_*^\top\vx_2)]\right|\cdot\left|\E_{\vw}[\sigma'_\perp(\vx_1^\top \vw)\sigma'_\perp(\vx_2^\top \vw)]\right|\cdot\left|\vx_2^\top \vx_1\right|\right].
\end{align}
Since $\E_{\xi\sim\cN(0,1)}[\sigma^*(\xi)]=0$, we can adopt Lemma A.5 of \cite{montanari2020interpolation} (with a slight modification) to conclude that conditioned on event $\cA_t$, we have
\begin{equation}
    \left|\E_{\vbeta_*}[\sigma^*(\vbeta_*^\top\vx_1)\sigma^*(\vbeta_*^\top\vx_2)]\right|\le Ct,
\end{equation}for some constant $C>0$. In addition, based on Lemma D.3 and G.1 in \citep{fan2020spectra}, we can show the following inequality for any $t\in(0,1)$ under event $\cA_t$:
    \begin{equation}\label{eq:E_sigma'}
        \left|\E_{\vw}[\sigma'_\perp(\vx_1^\top \vw)\sigma'_\perp(\vx_2^\top \vw)]\right|\le C t.
    \end{equation}
Let $\zeta_1:=\vw^\top\vx_1$ and $\zeta_2:=\vw^\top\vx_2$. Conditioned on $\vx_1,\vx_2$, we know that \[(\zeta_1,\zeta_2 )\sim\cN\left(\boldsymbol{0},\begin{bmatrix}
\|\vx_1\|^2/d & \vx_1^\top\vx_2/d\\
 \vx_1^\top\vx_2/d & \|\vx_2\|^2/d 
\end{bmatrix}\right).\]
Now we make the reparameterization: $\zeta_1=\gamma_1\xi_1$,  $\zeta_2=\gamma_2\xi_2+\nu_2\xi_1$, where $\xi_1,\xi_2\iid \cN(0,1)$ are independent to $\vx_1$ and $\vx_2$, and
\[\gamma_1:=\frac{\|\vx_1\|}{\sqrt{d}},~\gamma_2:=\sqrt{\frac{\|\vx_2\|^2}{d}-\frac{(\vx_1^\top\vx_2)^2}{d\|\vx_1\|^2}}, ~\nu_2:=\frac{\vx_1^\top\vx_2}{\sqrt{d}\|\vx_1\|}.\]
For $i=1,2$, by Taylor expansion of $\sigma(\zeta_i)$ around $\xi_i$ (note that $\sigma$ is differentiable by assumption), there exists a random variable $\eta_i$ between $\zeta_i$ and $\xi_i$ such that
\begin{equation}
    \sigma'_\perp(\zeta_i)=\sigma'_\perp(\xi_i)+\sigma'' (\eta_i)(\zeta_i-\xi_i),
\end{equation}
where we use $\sigma''_\perp=\sigma''$. 
Because $\E[\sigma'_\perp(\xi_i)]=0$ and $|\sigma'_\perp(x)|,|\sigma''(x)|\le \lambda_\sigma$ almost surely for all $x$, we have
\begin{align*}
    \left|\E_{\vw}[\sigma'_\perp(\zeta_1)\sigma'_\perp(\zeta_2)]\right|\le C \left(|\nu_2|+|\gamma_1-1| \right)\le C t,
\end{align*}on the event $\cA_t$ with $t\in(0,1)$, where $C>0$ is a constant depending on $\lambda_\sigma$. This concludes \eqref{eq:E_sigma'}. 

Also, note the probability bound for Gaussian random vector $\vx_1$ and $\vx_2$ implies that
\begin{equation}\label{eq:gaussian_vector}
    \Parg{\cA_t^c}\le 4\Exp{-dt^2}.
\end{equation}

From the above arguments, we can bound the first term $I_1$ via the following steps:
\begin{align*}
    I_1\le& \frac{1}{N}\E_{\vx_1,\vx_2}\left[\left|\E_{\vbeta_*}[\sigma^*(\vbeta_*^\top\vx_1)\sigma^*(\vbeta_*^\top\vx_2)]\right|\cdot\left|\E_{\vw}[\sigma'_\perp(\vx_1^\top \vw)\sigma'_\perp(\vx_2^\top \vw)]\right|\cdot\left|\vx_2^\top \vx_1\right|\cdot\mathbf{1}_{\cA_t}\right] \\
    & +\frac{1}{N}\E_{\vx_1,\vx_2}\left[\left|\E_{\vbeta_*}[\sigma^*(\vbeta_*^\top\vx_1)\sigma^*(\vbeta_*^\top\vx_2)]\right|\cdot\left|\E_{\vw}[\sigma'_\perp(\vx_1^\top \vw)\sigma'_\perp(\vx_2^\top \vw)]\right|\cdot\left|\vx_2^\top \vx_1\right|\cdot\mathbf{1}_{\cA_t^c}\right] \\
    \le& \frac{Ct^3d}{N}+\frac{\lambda_\sigma^2}{N}\E[f^*(\vx_1)^2f^*(\vx_2)^2]^{\frac{1}{2}}\E\left[\left|\vx_2^\top \vx_1\right|^2 \mathbf{1}_{\cA_t^c}\right]^{\frac{1}{2}}\\
    \le& \frac{Ct^3d}{N}+\frac{\lambda_\sigma^2}{N}\E\left[\left|\vx_2^\top \vx_1\right|^4\right]^{\frac{1}{4}}\E\left[\mathbf{1}_{\cA_t^c}\right]^{\frac{1}{4}}\\
    \le & \frac{C t^3d}{N}+\frac{4\lambda_\sigma^2}{N}\E \left[\left\| \vx_1\right\|^4\right]^{\frac{1}{2}}e^{\frac{-dt^2}{4}}=\frac{C\lambda_\sigma^2t^3d}{N}+\frac{12\lambda_\sigma^2d}{N}e^{\frac{-dt^2}{4}},
\end{align*}
where we used $\|\sigma_{\perp} '\|_\infty\le \lambda_\sigma$ and the fact that  $\E[f^*(\vx_1)^2]^{1/2}=\|f^*\|_{L^{2}(\R,\Gamma)}$ is finite.
For any $\eps\in(0,1/2)$, if we choose $t=d^{\eps-1/2}$, then we can conclude $I_1\le C\frac{d}{N}d^{\eps-3/2}$, for all large $d$, sufficiently large constant $C>0$ and sufficient small $\eps$. Next we consider $I_2$ and $I_3$. Notice that
\begin{align*}
    I_2\le & \frac{\lambda_\sigma^2}{nN}\E[f^*(\vx_1)^4]^{\frac{1}{2}}\E[\|\vx_1\|^4]^{\frac{1}{2}}
    \le \frac{3C\lambda_\sigma^2d}{nN},
\end{align*} because $\sigma^*$ is Lipschitz and $f^*\in L^4(\R^d,\Gamma)$. Following the same computation, we also have $I_3\le \frac{\lambda_\sigma^2d}{nN}$. This establishes a bound for $\E[[\|\vB\|_F^2]$ in \eqref{eq:frobenius_exp}.

For the tail control \eqref{eq:frobenius_prob}, recall that $\|\vB_r\|_F^2=J_1+J_2$ where $\E[|J_1|]\le \frac{\alpha^2N_r}{N}I_1$ and $\E[J_2]\le \frac{\alpha^2N_r}{N}(I_2+I_3)$. Hence Markov's inequality and the upper bound for $I_1$ implies that
\begin{equation}
    \Parg{ \frac{N}{ N_r}|J_1|\ge t}\le  \frac{C\alpha^2}{td^{\frac{1}{2}-\eps}N}. 
\end{equation}
By choosing $t=C/Nd^{\frac{1}{4}-\eps}$, we conclude that $\frac{N}{ N_r}|J_1|$ cannot exceed $C/Nd^{\frac{1}{4}-\eps}$ with probability at least $1-\alpha^2/d^{\frac{1}{4}}$, for any $\eps\in (0,1/4)$. As for $J_2$, since $\sigma'_\perp$ is uniformly bounded by $\lambda_\sigma$ and all entries of $\va_r$ are bounded by $\alpha/\sqrt{N}$, we have
\[\frac{N}{N_r}|J_2|\le\frac{\lambda_\sigma^2\alpha^2}{Nn^2}\sum_{i=1}^ny_i^2\|\vx_i\|^2=:J'_2.\]

Similarly for $I_2$ and $I_3$, it is easy to check $|\E[J'_2]|\le \frac{3\alpha^2Cd}{Nn}$. Besides,
\begin{align*}
    \Var(J_2')=\frac{\lambda_\sigma^4\alpha^4}{N^2n^3}\Var(y_1^2\|\vx_1\|^2)\le \frac{\lambda_\sigma^4\alpha^4}{N^2n^3}\E[y_1^4\|\vx_1\|^4]\le \frac{c\alpha^4d^2}{N^2n^3},
\end{align*}
where constant $c>0$ only depends on $\lambda_\sigma$ and $\|f^*\|_{L^8(\R,\Gamma)}$. By Chebyshev's inequality, 
\[\Parg{|J_2'-\E[J_2']|>t}\le \frac{c\alpha^4d^2}{t^2N^2n^3}.\]
Letting $t=\sqrt{c} d/Nn$, we arrive at
\[\Parg{\frac{N}{N_r}J_2\le \frac{\sqrt{c} d}{Nn}+\frac{3C d}{Nn}}\ge \Parg{J_2'\le \frac{\sqrt{c}d}{Nn}+\frac{3Cd}{Nn}}\ge 1- \frac{\alpha^4}{n}.\]
We conclude \eqref{eq:frobenius_prob} by combining the above estimates of $J_1$ and $J_2$.

\end{proof}

\subsection{Constructing the ``Oracle'' Estimator}
\label{app:oracle-estimator}

In this subsection we prove the following lemma related to Lemma~\ref{lemm:optimal-estimator}. 
\begin{lemm}[Reformulation of Lemma~\ref{lemm:optimal-estimator}]
Suppose Assumptions \ref{assump:1} and \ref{assump:2} hold, $\eta=\Theta(\sqrt{N})$ and the activation $\sigma$ is bounded. Then given any $\varepsilon>0$, for $N$ sufficiently large, there exists some constant $C$ and second-layer $\tilde{\va}$ such that the model $\tilde{f}(\vx) = \frac{1}{\sqrt{N}}\tilde{\va}^\top\sigma(\vW_1^\top\vx)$ has prediction risk
\begin{align}
    \cR(\tilde{f}) \le \tau^* + C\left(\sqrt{\tau^*}\cdot\sqrt{\frac{d}{n}} + \frac{d}{n}\right) + \varepsilon + o_{d,\P}(1), 
    \label{eq:f1-risk}
\end{align}
where the scalar $\tau^*$ is defined in \eqref{eq:tau*}. 
\label{lemm:oracle-estimator}
\end{lemm}

We first introduce a constant $\alpha$ (independent to $N$). 
Recall that $[\va]_i=a_i\iid\cN\left(0,N^{-1}\right)$ for $i\in[N]$. For any $\alpha\in\R$, define the subset of initialized weights: 
\begin{align}
    \cA_r^\alpha = \left\{i\in [N] \,:\, \abs{\sqrt{N}\cdot a_i - \alpha} \le N^{-r}\right\}, \text{~~for any given~}  r>0.
\label{eq:Ar}
\end{align}
The size of the subset is given by $\abs{\cA_r^\alpha} = \sum_{i=1}^N \mathbf{1}_{ \abs{\sqrt{N}a_i-\alpha}\le N^{-r}}$, and hence its expectation is $\E|\cA_r^\alpha| = N\cdot\E_{z\sim\cN(0,1)}\left[\mathbf{1}_{ \abs{z-\alpha}\le N^{-r}}\right] = C(\alpha) N^{1-r}$ for some constant $C(\alpha)\propto\Exp{-\alpha^2}$. 
By Hoeffding's inequality, 
\begin{equation}\label{eq:A_r_alpha_prob}
    \Parg{\abs{\abs{\cA_r^\alpha} - \E\abs{\cA_r^\alpha}} \ge t} \le 2\Exp{-\frac{2t^2}{N}}.
\end{equation}
Hence we may conclude that for any $r\in (0, 1/2)$ and large enough $N$, $\abs{\cA_{r}^\alpha} = \Theta_{d,\P}(N^{1-r})$ with probability at least $1-2\Exp{-c\log^2 N}$. 
This is to say, for any constant $\alpha$, we know that with high probability, there exist a large number of initialized second-layer coefficients $a_i$'s that are close to $\alpha$.  
We specify our choice of $\alpha\in\R$ via \eqref{eq:tau*} in the subsequent analysis. 

\paragraph{Rank-1 Approximation of the Gradient.}  
Denote $N_r:=\abs{\cA_r^\alpha}$ for some constant $\alpha$, and $i_r\in [N]$ as the index such that $i_r\in\cA_r^\alpha$. We define $f_r$ as an average over neurons with indices $i_r\in\cA_r^\alpha$, and $f_{\vA}$ as an approximation of $f_r$ in which the first-step gradient matrix $\vG_0$ in \eqref{eq:decomposition_gradient} is replaced by the rank-1 matrix $\vA_1$ defined in \eqref{eq:A_1+A_2}: \begin{align}
    f_r(\vx) := \frac{1}{N_r}\sum_{i_r\in\cA_r^\alpha} \sigma\left(\langle\vx,\vw^1_{i_r}\rangle\right); \quad
f_{\vA}(\vx) := \frac{1}{N_r}\sum_{i_r\in\cA_r^\alpha} \sigma\left(\langle\vx,\vw^{\vA}_{i_r}\rangle\right),
\label{eq:f_r} 
\end{align}
where $\vw^1_{i_r}$ is the $i_r$-th neuron in $\vW_1$, $\vw^{\vA}_i = \vw^0_i + \eta\sqrt{N}[\vA_1]_i$, $[\vA_1]_i$ is the $i$-th column of $\vA_1$ and $\vw^0_i\in\R^d$ is the corresponding initial neuron in $\vW_0$. 
Applying the Lipschitz property of the activation function, one can control $\E_{\vx}[(f_{\vA}(\vx) - f_r(\vx))^2]$ as follows 
\begin{align}
    \left|f_{\vA}(\vx) - f_r(\vx)\right|
\lesssim
    \frac{1}{N_r} \sum_{i_r\in\cA_r^\alpha} \left|\langle\vw^1_{i_r} - \vw^{\vA}_{i_r}, \vx\rangle\right|
=
    \frac{\eta\sqrt{N}}{N_r} \sum_{i_r\in\cA_r^\alpha} \left|\langle \vdelta_{i_r}, \vx\rangle\right|.  
\end{align}
where the ``residual'' is entry-wisely defined as $[\vdelta_i]_j: = [\vA_2 + \vB + \vC]_{ji}$ for $i\in [N]$ and $j\in [d]$. Recall that $\vA_2,\vB$ and $\vC$ have been analyzed in Lemmas \ref{lemm:gradient-norm} and \ref{lemm:gradient-norm2}. Let us further denote $\vA_r\in\R^{d\times N_r}$ as a submatrix of $\vA_2$ by selecting all $i_r\in \cA_r^\alpha$ columns of $\vA_2$. Similarly, we choose $\vB_r,\vC_r\in\R^{d\times N_r}$ as submatrices of $\vB,\vC$ related to $\cA_r^\alpha$, respectively. Using Lemma \ref{lemm:gradient-norm2} applied to the submatrix, we have
\begin{equation}\label{eq:A_r}
    \Parg{\frac{N}{N_r}\|\vA_r\|_F^2\ge \frac{Cd}{nN}}\le C'\Big(ne^{-c\sqrt{n}}+\frac{1}{d}\Big).
\end{equation}
Moreover, by definition of $\cA_r^\alpha$, all $a_{i_r}$'s are close to $\frac{\alpha}{\sqrt{N}}$ for $i_r\in\cA_r^\alpha$; thus Lemma \ref{lemm:frobenius} (in particular \eqref{eq:frobenius_prob}) can be directly applied to $\vB_r$. As for $\vC_r$, since $\|\vC_r\|_F\le \|\vC\|_F$, we use part $(iii)$ in Lemma \ref{lemm:gradient-norm} to obtain 
\begin{equation}\label{eq:C_r}
    \Parg{\|\vC_r\|_F\ge \frac{C\log n\log N}{N}}\le C'\Big(ne^{-c\log^2 n}+Ne^{-c\log^2 N}\Big).
\end{equation}
With these concentration estimates, we know that when $n>d$,
\begin{align}
    &\E_{\vx}[(f_{\vA}(\vx) - f_r(\vx))^2]
    \lesssim  \E_{\vx} \left[\left(\frac{\eta\sqrt{N}}{N_r}\sum_{i_r\in\cA_r^\alpha} \left|\langle \vdelta_{i_r}, \vx\rangle\right|\right)^2\right]=  \frac{\eta^2 N}{N_r^2} \E_{\vx}\left[\sum_{i_r,j_r\in\cA_r^\alpha}\left|\vdelta_{i_r}^\top \vx \right|\left| \vdelta_{j_r}^\top \vx \right|\right]\\
    \le & \frac{\eta^2 N}{N_r^2} \sum_{i_r,j_r\in\cA_r^\alpha}\E_{\vx}\left[\left(\vdelta_{i_r}^\top \vx \right)^2\right]^{\frac{1}{2}}\E_{\vx}\left[\left( \vdelta_{j_r}^\top \vx \right)^2\right]^{\frac{1}{2}}
    =\frac{\eta^2 N}{N_r^2} \sum_{i_r,j_r\in\cA_r^\alpha} \|\vdelta_{i_r}\| \|\vdelta_{j_r}\| \\
    =& \frac{\eta^2 N}{N_r^2} \left(\sum_{i_r\in\cA_r^\alpha} \|\vdelta_{i_r}\|\right)^2  
    \le \frac{\eta^2 N}{N_r} \left(\|\vA_{r}\|_{F}^2+\|\vB_r\|_{F}^2+\|\vC_r\|_{F}^2\right)
    \lesssim \frac{d}{n}+\frac{1}{d^{\frac{1}{4}-\eps}}+\frac{\log^2 n \log^2 N}{N_r},  \label{eq:fA-f1}
\end{align}  
with probability at least $1-c\left(\frac{\alpha^2}{d^{\frac{1}{4}}}+\frac{\alpha^4}{n} + \frac{1}{\sqrt{N}} + ne^{-c\log^2 n} + Ne^{-\log^2 N}\right)$ for some constant $c>0$; this is due to the defined step size $\eta=\Theta(\sqrt{N})$, \eqref{eq:A_r}, \eqref{eq:C_r} in \eqref{eq:frobenius_prob} of Lemma~\ref{lemm:frobenius} outlined above. In \eqref{eq:fA-f1}, we ignore the constants in the upper bound since we are only interested in the rate with respect to $n,d,N$.

\paragraph{Simplification under ``Population'' Gradient.} 
 
Recall the definition of the single-index teacher: $f^*(\vx) = \sigma^*(\langle\vx,\vbeta_*\rangle)$, and the definition of rank-1 matrix $\vA_1 = \frac{1}{n}\cdot\frac{\mu_1\mu_1^*}{\sqrt{N}}\vX^\top\vX\vbeta_*\va^\top$. Define $\vv = \frac{\eta\mu_1\mu_1^*}{n\sqrt{N}}\vX^\top\vX\vbeta_*\in\R^d$, we can write
$$
f_{\vA}(\vx) = \frac{1}{N_r}\sum_{i_r\in\cA_r^\alpha}  \sigma\left(\langle\vw_{i_r} + \sqrt{N}a_{i_r}\vv,\vx\rangle\right), \quad 
\tilde{f}_{\vA}(\vx) := \frac{1}{N_r}\sum_{i_r\in\cA_r^\alpha} \sigma\left(\langle\vw_{i_r} + \alpha\vv,\vx\rangle\right),
$$
where we dropped the superscript in the initialized weights $\vw^0_{i_r}$ to simplify the notation.  Note that the difference between $f_{\vA}$ and $\tilde{f}_{\vA}$ is that the each second-layer coefficient $a_i$ is replaced by the same scalar $\alpha$. 

By the definition of $\cA_r^\alpha$ and the Lipschitz property of $\sigma$, one can obtain
\begin{align}
    \abs{f_{\vA}(\vx) - \tilde{f}_{\vA}(\vx)} 
\lesssim
    \frac{1}{N_r} \sum_{i_r\in\cA_r^\alpha} \frac{\eta N^{-r}}{\sqrt{N}}\cdot\abs{\left\langle\frac{1}{n}\vX^\top\vX\vbeta_*,\vx\right\rangle}
\lesssim
    N^{-r}\cdot\abs{\left\langle\frac{1}{n}\vX^\top\vX\vbeta_*,\vx\right\rangle}. \label{eq:f-tildef}
\end{align}  
Now define $\bar{\vv} := \frac{\eta\mu_1\mu_1^*}{\sqrt{N}}\vbeta_*=\bar{\eta}\mu_1\mu_1^*\vbeta_*$, which corresponds to the ``population'' version of $\vv$, and denote
\begin{equation}\label{eq:barf_A}
     \bar{f}_{\vA}(\vx) := \frac{1}{N_r}\sum_{i_r\in\cA_r^\alpha} \sigma(\langle\vw_{i_r}+\alpha\bar{\vv},\vx\rangle).
\end{equation}
Similar to \eqref{eq:f-tildef}, we have
\begin{align}
    \abs{\bar{f}_{\vA}(\vx) - \tilde{f}_{\vA}(\vx)} 
\lesssim
    \frac{1}{N_r}\sum_{i_r\in\cA_r^\alpha}\frac{\eta}{\sqrt{N}}\abs{\left\langle\left(\frac{1}{n}\vX^\top\vX-\vI\right)\vbeta_*,\vx\right\rangle}
\lesssim
    \abs{\left\langle\left(\frac{1}{n}\vX^\top\vX-\vI\right)\vbeta_*,\vx\right\rangle}. \label{eq:barf-tildef}
\end{align}
Combining the inequalities \eqref{eq:f-tildef} and \eqref{eq:barf-tildef}, we know that for some constant $C$, 
\begin{align}
    \E_{\vx}[(f_{\vA}(\vx) - \bar{f}_{\vA}(\vx))^2]
\lesssim&~ 
    N^{-2r}\cdot\E_{\vx}\left(\left\langle\frac{1}{n}\vX^\top\vX\vbeta_*,\vx\right\rangle\right)^2 + \E_{\vx}\left(\left\langle\left(\frac{1}{n}\vX^\top\vX-\vI\right)\vbeta_*,\vx\right\rangle\right)^2 \\
\le&~
    \left(N^{-2r}\norm{\frac{1}{n}\vX^\top\vX}^2 + \norm{\frac{1}{n}\vX^\top\vX-\vI}^2\right)\cdot\norm{\vbeta_*}^2 \\
\lesssim &~
    \left(\Big(1+\frac{d}{n}\Big)N^{-2r} + \frac{d}{n}\right),
\end{align}
where the last inequality holds with probability at least $1-\Exp{-cd}$ for some universal constant $c>0$, due to the operator norm bound and concentration of the sample covariance matrix $\frac{1}{n}\vX^\top\vX$ (for instance see \cite[Theorem 4.6.1]{vershynin2018high}). 

Now we take the expectation of $\bar{f}_{\vA}$ over initial weight $\vw_{i_r}$ in \eqref{eq:barf_A} to define
\begin{align}
    \bar{f}(\vx) := \E_{\vw\sim\cN(0,\,d^{-1}\vI)} \left[\sigma(\langle\vw+\alpha\bar{\vv},\vx\rangle)\right]. 
\end{align}
Note that for fixed $\vx$, $\langle\vw,\vx\rangle\sim\cN(0,\norm{\vx}^2/d)$. Since $\sigma$ is $\lambda_\sigma$-Lipschitz, by the Hoeffding bound on sub-Gaussian random variables, conditioned on $\vx$, we have
\begin{align}
    \Parg{\abs{\bar{f}_{\vA}(\vx) - \bar{f}(\vx)}>t\,\big|\,\vx}
\le
    2\Exp{-\frac{t^2 N_r}{2\lambda_\sigma^2 \cdot \norm{\vx}_2^2/d}},
     \label{eq:fA-fA} 
\end{align}
Also notice that  
\begin{align*}
    \E_{\vw}(\bar{f}(\vx) - \bar{f}_{\vA}(\vx))^2=& \int_0^\infty  \Parg{\abs{\bar{f}_{\vA}(\vx) - \bar{f}(\vx)}^2>t\,\big|\,\vx}\dt \\
    \le& \int_0^\infty 2\Exp{-\frac{t N_r}{2\lambda_\sigma^2 \cdot \norm{\vx}_2^2/d}} \dt=\frac{4\lambda_\sigma^2\|\vx\|^2}{N_r d}.
\end{align*}
Thus, by taking expectation over $\vx$ in the above bound, we know that $\E(\bar{f}(\vx) - \bar{f}_{\vA}(\vx))^2\le \frac{4\lambda_\sigma^2\E[\|\vx\|^2]}{N_r d}=\frac{4\lambda_\sigma^2}{N_r}$. By Markov's inequality, we have
\begin{equation}
    \P\left(\E_{\vx} (\bar{f}(\vx) - \bar{f}_{\vA}(\vx))^2\ge t\right)\le \frac{\E (\bar{f}(\vx) - \bar{f}_{\vA}(\vx))^2}{t}\le \frac{4\lambda_\sigma^2}{N_r t}. 
    \label{eq:f-fA}
\end{equation}
Hence we deduce that $\E_{\vx} (\bar{f}(\vx) - \bar{f}_{\vA}(\vx))^2\le \frac{4\lambda_\sigma^2}{\sqrt{N_r}}$ with probability $1-\frac{1}{\sqrt{N_r}}$.

Observe that $\bar{f}$ is given by an expectation over $\vw$ in a single-index model. To calculate its difference from the true model: $\E_{\vx} (\bar{f}(\vx) - f^*(\vx))^2$, first recall the assumption that $\norm{\vbeta_*}=1$, and $\vw\sim \cN(0,\vI/d)$, $\vx\sim \cN(0,\vI)$. 
Denote $\xi_1:=\langle\vx,\vbeta_*\rangle\sim\cN (0,1)$ and, condition on $\vx$, $\langle\vx,\vw\rangle\overset{d}{=}\xi_2\|\vx\|/\sqrt{d}$, where $\xi_2\sim \cN (0,1)$ independent of $\xi_1$. Since $\eta/\sqrt{N}=\bar\eta$, we can write $\kappa := \frac{\alpha\eta\mu_1\mu_1^*}{\sqrt{N}}=\alpha\bar\eta\mu_1\mu_1^*\in\R$. 
Following these definitions, we have  $\bar{f}(\vx)=\E_{\xi_2}[\sigma(\xi_2\|\vx\|/\sqrt{d}+\kappa\xi_1)],$ and 
\begin{align}
    \E_{\vx} (\bar{f}(\vx) - f^*(\vx))^2=\E_{\xi_1}\Big(\sigma^*(\xi_1) - \E_{\xi_2}[\sigma(\kappa\xi_1 + \xi_2\|\vx\|/\sqrt{d})] \Big)^2.\label{eq:barf-f*}
\end{align}
In addition, given $\kappa\in\R$, we introduce a scalar quantity
\begin{align}
    \tau:= 
    \E_{\xi_1}\Big(\sigma^*(\xi_1) - \E_{\xi_2}[\sigma(\kappa\xi_1 + \xi_2)] \Big)^2. \label{def:tau}
\end{align}
Note that $\sigma^*\in L^2(\R,\Gamma)$ and $\sigma$ is uniformly bounded by assumption; one can easily check that $\tau$ is uniformly bounded for all $\kappa\in\R$. Hence $\tau$ defined above is always finite. We now show that the difference between $\tau$ and $\E_{\vx} (\bar{f}(\vx) - f^*(\vx))^2$ is asymptotically negligible, again using the Lipschitz property of $\sigma$,  
\begin{align}
    \abs{\sigma(\kappa \xi_1 + \xi_2) - \sigma(\kappa \xi_1 + \xi_2\|\vx\|/\sqrt{d})}
\lesssim
    \abs{1-\frac{\|\vx\|}{\sqrt{d}}}\cdot\abs{\xi_2}.  
\end{align}
Since $\sigma^*\in L^2(\R,\Gamma)$, and $\sigma$ is uniformly bounded and Lipschitz, based on \eqref{eq:barf-f*} and \eqref{def:tau}, we can apply the Cauchy-Schwarz inequality to get
\begin{align}
    &\abs{\tau- \E_{\vx} (\bar{f}(\vx) - f^*(\vx))^2}\\
    \lesssim ~&\E\left[\Big(\sigma(\kappa \xi_1 + \xi_2) - \sigma(\kappa \xi_1 + \xi_2\|\vx\|/\sqrt{d})\Big)^2\right]^{\frac{1}{2}}
    \\
    \lesssim~ & \E[\xi_2^2]^{\frac{1}{2}}\E\left[\abs{1-\frac{\|\vx\|}{\sqrt{d}}}^2\right]^{\frac{1}{2}}\le \frac{C}{\sqrt{2d}}, \label{eq:tau-substitution}
\end{align}
where the last inequality is due to property of the sub-Gaussian norm $\|\|\vx\|/\sqrt{d}-1\|_{\psi_2}\le C/\sqrt{d}$ (see e.g.~\cite[Theorem 3.1.1]{vershynin2018high}) for some universal constant $C>0$.

\begin{proofof}[Lemma~\ref{lemm:oracle-estimator}]
Based on above calculations, we now control the prediction risk of $\tilde{f}$ by combining the substitution errors, where $\tilde{f}=f_r$ is constructed as the average over subset $\cA_r^\alpha$ defined in \eqref{eq:f_r}. 

Given any $\alpha\in\R$ and $r\in(0,1/2)$, we define the subset $\cA_r^\alpha$ and the corresponding $\tilde{f}(\vx)=f_r(\vx) = \frac{1}{\sqrt{N}}\tilde{\va}^\top\sigma(\vW_1^\top\vx)$, where the second-layer $\tilde{\va}$ is given as $[\tilde{\va}]_i=\sqrt{N}/N_r$ if $i\in\cA_r^\alpha$, otherwise $[\tilde{\va}]_i=0$. 
Moreover, \eqref{eq:A_r_alpha_prob} implies that $N_r = \Theta_{d,\P}(N^{1-r})$ with probability at least $1-\Exp{-\log^2 N}$. Therefore, together with \eqref{eq:fA-f1}, \eqref{eq:fA-fA}, \eqref{eq:f-fA}, and \eqref{eq:tau-substitution}, we know that 
\begin{align} 
    \E_{\vx}(f_r(\vx) - \bar{f}(\vx))^2 \le \frac{Cd}{n} + o_{d,\P}(1); \quad
    \E_{\vx}(f^*(\vx) - \bar{f}(\vx))^2 = \tau + o_{d}(1),
\end{align} 
as $n,d,N\to\infty$, for some constant $C>0$. By the Cauchy-Schwarz inequality,  
\begin{align}
    \E_{\vx}(f^*(\vx) - f_r(\vx))^2 \le \tau + C\left(\sqrt{\tau}\cdot\sqrt{\frac{d}{n}} + \frac{d}{n}\right) + o_{d,\P}(1),
\end{align} 
where the failure probability only relates to $r,\alpha,N,d,n$ and is vanishing as $N,d,n\to\infty$. For simplicity, we only keep the leading orders and ignore the subordinate terms in the exact probability bounds.

Note that the above characterization holds for any finite $\alpha$; since our goal is to construct an estimator $f_r$ that achieves as small prediction risk as possible, we optimize over $\alpha\in\R$ by defining
\begin{align}
    \tau^* 
:=&~
    \inf_{\alpha\in\R}\E_{\xi_1}\Big(\sigma^*(\xi_1) - \E_{\xi_2}\left(\sigma\big( \alpha\bar\eta\mu_1\mu_1^*\cdot\xi_1 + \xi_2\big)\right) \Big)^2,  
    \quad
    \tau_{\varepsilon}^* := \tau^* + \varepsilon, 
\end{align}
where $\varepsilon\ge0$ is a small constant. This definition of $\tau^*$ is identical to \eqref{eq:tau*} and is always finite because $\tau$ defined in \eqref{def:tau} is uniformly bounded and non-negative (observe that optimizing over $\kappa$ or $\alpha\in\R$ are equivalent, since we can reparameterize $\kappa=\alpha\bar\eta\mu_1\mu_1^*$ where $\mu_1,\mu_1^*\neq 0$).
When $\tau^*$ is attained at some finite $\alpha$, then we may simply set $\varepsilon=0$ and define
\begin{align}
    \alpha^* 
:=&~
    \argmin_{\alpha\in\R}\E_{\xi_1}\Big(\sigma^*(\xi_1) - \E_{\xi_2}\left(\sigma\big( \alpha\bar\eta\mu_1\mu_1^*\cdot\xi_1 + \xi_2\big)\right) \Big)^2. 
\end{align}
Otherwise, observe that as a bounded and continuous function of $\alpha$ on the real line, $\tau(\alpha) := \E_{\xi_1}\Big(\sigma^*(\xi_1) - \E_{\xi_2}\left(\sigma\big(\alpha\bar\eta\mu_1\mu_1^*\cdot\xi_1 + \xi_2\big)\right) \Big)^2$ will approach its minimum at infinity. Therefore, in this case, for any $\varepsilon>0$, we can find some finite $\alpha_{\eps}^*$ such that $\tau(\alpha_{\eps}^*)\le\tau_{\eps}^*=\tau^*+\eps$; hence, we may set $\alpha=\alpha_{\eps}^*$ and conclude the proof. 
Finally, note that given nonlinearities $\sigma$ and $\sigma^*$ (which determine the relation between $\varepsilon$ and $\alpha_\varepsilon$), we can take $\varepsilon\to 0$ at a slow enough rate as $n,d,N\to\infty$, as long as $C(\alpha)\cdot N^{1-r}\to\infty $. 
Thus we also obtain an asymptotic version of Lemma \ref{lemm:oracle-estimator}: with probability one, there exists some second-layer $\tilde{\va}$ such that the prediction risk of the corresponding student model $\tilde{f}(\vx) = \frac{1}{\sqrt{N}}\tilde{\va}^\top\sigma(\vW_1^\top\vx)$ satisfies
\begin{align}\label{eq:f1-risk-asym}
    \cR(\tilde{f}) \le \tau^* + C\left(\sqrt{\tau^*}\cdot\sqrt{\frac{d}{n}} + \frac{d}{n}\right),
\end{align}
for some constant $C>0$, as $n,d,N\to\infty$ proportionally.

\end{proofof}

The above analysis illustrates that because of the Gaussian initialization of $a_i$, for any $\eta=\Theta(\sqrt{N})$, we can find a subset of neurons $\cA_r^\alpha$ that receive a ``good'' learning rate, in the sense that the corresponding (sub-) network defined by $f_r$ can achieve the prediction risk close to $\tau_\eps^*$ when $n\gg d$.

\paragraph{Some Examples.}
Equation~\eqref{eq:f1-risk} reduces the prediction risk of our constructed $f_r$ to a one-dimensional Gaussian integral, which can be numerically evaluated for pairs of $(\sigma,\sigma^*)$. Denote $\kappa^* = \alpha^*\bar\eta\mu_1\mu_1^*$, we give a few examples in which we set $\eps=0$ and the corresponding $\tau^*$ is small. Note that due to Assumptions \ref{assump:1} and \ref{assump:2}, choices of $\sigma$ and $\sigma^*$ considered below are centered with respect to standard Gaussian measure $\Gamma$. 
\begin{itemize}[leftmargin=*,topsep=1.2mm, itemsep=0.5mm]
    \item \underline{$\sigma=\sigma^*=\text{erf}$}. Note that for $c_1,c_2\in\R$, $\E_{z\sim\cN(0,1)} [\text{erf}(c_1 z+c_2)] = \text{erf}\Big(\frac{c_2}{\sqrt{1+2c_1^2}}\Big)$.  
    Hence we can choose $\kappa^* = \sqrt{3}$, and the corresponding minimum value $\tau^*=0$. 
    \item \underline{$\sigma=\sigma^*=\text{tanh}$}. Numerical integration yields $\tau^*\approx 3\times 10^{-4}$, $\kappa^*\approx 1.6$.   
    \item \underline{$\sigma=\sigma^*=\text{SoftPlus}$}.  Numerical integration yields $\tau^*\approx 0.03$, $\kappa^* \approx 0.96$.  
    \item \underline{$\sigma=\text{ReLU},\sigma^*=\text{SoftPlus}$}. Numerical integration yields $\tau^*\approx 0.09$, $\kappa^* \approx 0.94$.    
\end{itemize}
Observe that in all the above examples, $\tau^*$ can be obtained by some finite $\alpha^*$ (or equivalently $\kappa^*$). In the following analysis of kernel ridge regression, we drop the small constant $\eps$ in Lemma~\ref{lemm:oracle-estimator} and directly apply the asymptotic statement given in \eqref{eq:f1-risk-asym}. 
\begin{remark}
We make the following remarks on the calculation of $\tau^*$ in \eqref{eq:tau*}. 
\begin{itemize}[leftmargin=*,topsep=1mm, itemsep=0.4mm]
    \item When $\sigma=\sigma^*$, we intuitively expect $\tau^*$ to be small when the nonlinearity is smooth such that it is to some extent unchanged under Gaussian convolution (when $\kappa$ is chosen appropriately). 
    \item Adding weight decay with strength $\lambda<1$ to the first-layer parameters $\vW_0$ simply corresponds to multiplying $\xi_2$ in the definition of $\tau$ \eqref{def:tau} by a factor of $(1-\lambda)$.  
\end{itemize}
\end{remark}

\subsection{Prediction Risk of Ridge Regression}
\label{app:KRR}

In this section we prove Theorem~\ref{thm:risk-large-lr}. Recall that we aim to upper-bound the prediction risk of the CK ridge regression estimator defined as
\begin{equation}\label{eq:hat_f_large}
    \hat{f}(\vx) = \Big\langle\frac{1}{\sqrt{N}}\sigma(\vW_1^\top\vx), \hat{\va}\Big\rangle, \quad \text{where } \hat{\va} := \left(\vPhi^\top\vPhi + \lambda n\vI\right)^{-1} \vPhi^\top\tilde{\vy}, ~~
\vPhi := \frac{1}{\sqrt{N}}\sigma(\tilde{\vX}\vW_1),
\end{equation}
where $\{\tilde{\vX},\tilde{\vy}\}$ is a new set of training data independent of $\vW$. For concise notation, in this section we rescale the ridge parameter in \eqref{eq:ridge-risk} by replacing $\frac{\lambda}{N}$ with $\lambda$.  

Given feature map $\vx\to\frac{1}{\sqrt{N}}\sigma(\vW_1^\top\vx)$ conditioned on first layer weights $\vW_1$, we denote the associated Hilbert space as $\cH$. Note that $\cH$ is a finite-dimensional reproducing kernel Hilbert space and is hence closed; we define the optimal predictor in the RKHS as $\check{f} := \argmin_{f\in\cH} \E_{\vx} (f(\vx) - f^*(\vx))^2$, which takes the form of $\check{f}(\vx) = \langle\frac{1}{\sqrt{N}}\sigma(\vW_1^\top\vx), \check{\va}\rangle$ for some $\check{\va}\in\R^N$. 
In addition, we may write the orthogonal decomposition in $L^2(\R^d,\Gamma)$: $f^*(\vx) = \check{f}(\vx) + f_\perp(\vx)$.  
By definition of $f_\perp$, we have $\norm{f_\perp}^2_{L^2 }=\E_{\vx} [f_\perp(\vx)^2] \le \cR(h)= \norm{f^* - h}_{L^2}^2,$ for any $h\in\cH$ and $\vx\sim\cN (0,1)$. 
Finally, from Assumption \ref{assump:2} we know that $\|f^*\|_{L^2}$ is bounded by some constant, and thus $\|f_\perp\|_{L^2}$ is also bounded.
  
We are interested in the prediction risk of the CK ridge regression estimator denoted as $\cR_1(\lambda)$. 
We first define the following quantities which $\cR_1(\lambda)$ can be decomposed into (see Lemma~\ref{lemm:R1-decomposition}): 
\begin{equation}\label{eq:R1-decomposition}
\begin{aligned}
    B_1 :=&~ \E_{\vx}\left(f^*(\vx) - \check{f}(\vx)\right)^2, \\
    B_2 :=&~ \E_{\vx}\left(\check{f}(\vx) - \frac{1}{n}\vphi_x\left(\widehat{\vSigma}_\Phi + \lambda\vI\right)^{-1} \vPhi^\top \check{\vf}\right)^2, \\
    V_1 :=&~ \frac{1}{n^2}\tilde\vvarepsilon^\top\vPhi\left(\widehat{\vSigma}_\Phi + \lambda \vI\right)^{-1} \vSigma_\Phi\left(\widehat{\vSigma}_\Phi + \lambda \vI\right)^{-1} \vPhi^\top\tilde\vvarepsilon, \\
    V_2 :=&~ \frac{1}{n^2} \vf_\perp^\top\vPhi\left(\widehat{\vSigma}_\Phi + \lambda \vI\right)^{-1} \vSigma_{\Phi}\left(\widehat{\vSigma}_\Phi + \lambda \vI\right)^{-1}\vPhi^\top \vf_\perp, 
\end{aligned} 
\end{equation}
where the $i$-th entry of vector $\check{\vf}$ and $\vf_{\!\perp}$ are given by $[\check{\vf}]_i = \check{f}(\tilde{\vx}_i),$  $[\vf_{\!\perp }]_i = f_{\!\perp}(\tilde{\vx}_i)$, respectively, and $\widehat{\vSigma}_\Phi : = \frac{1}{n}\vPhi^\top\vPhi$, $\vSigma_\Phi := \frac{1}{N}\E_{\vx} \left[\sigma(\vW_1^\top\vx)\sigma(\vW_1^\top\vx)^\top\right]$. Also, $\vphi_{\vx} := \frac{1}{\sqrt{N}}\sigma(\vx^\top\vW_1)$ for $\vx\in\R^d$, which gives $\vPhi^\top = [\vphi_{\tilde{\vx}_1}^\top,\ldots,\vphi_{\tilde{\vx}_i}^\top,\ldots,\vphi_{\tilde{\vx}_n}^\top]$, where $\tilde{\vx}_i^\top$ is the $i$-th row of $\tilde{\vX}$. To simplify the notation, we omit the accent in $\tilde{\vx},\tilde{\varepsilon}$ when the context is clear.  
In the following subsections, to control $\cR_1(\lambda)$, we provide high-probability upper-bounds for $B_1,B_2,V_1$ and $V_2$ separately.

\paragraph{Concentration of Feature Covariance.} 
We begin by defining a concentration event $\cA$ on the empirical feature matrix $\hSigmaPhi$, under which the prediction risk can be controlled. We modify the proof of \cite[Theorem 4.7.1]{vershynin2018high} to obtain a normalized version of the concentration for CK matrix as follows.
 
\begin{lemm}\label{lemm:concentration-largelr}
Under Assumptions \ref{assump:1}, \ref{assump:2} and using the above notations, there exists some constant $c>0$ such that the following holds\footnote{Note that for $\lambda=0$, the LHS of the inequality may be interpreted as a pseudo-inverse. }
\begin{align*}
 \Parg{\left\|\left(\vSigma_\Phi + \lambda \vI\right)^{-1/2} (\vSigma_{\Phi}-\widehat{\vSigma}_{\Phi})\left(\vSigma_\Phi + \lambda \vI\right)^{-1/2}\right\|\ge 2K^2\cdot\sqrt{\frac{N}{n}}}\le 2\Exp{-c\sqrt{N}},
\end{align*}
for all large $n>N$, where $K:=\frac{\lambda_\sigma}{\sqrt{N}}\|\vW_1\|_F$. 
\end{lemm}
\begin{proof} 
 First observe that the null space of $\vSigma_{\Phi}$ contains the null space of $\widehat{\vSigma}_{\Phi}$. Also, notice that $\widehat{\vSigma}_\Phi$ is a sample covariance matrix taking the form of
 \[\widehat{\vSigma}=\frac{1}{n}\sum_{i=1}^n\vphi_{\tilde{\vx}_i}\vphi_{\tilde{\vx}_i}^\top,\]
 where the covariance of $\vphi_{\tilde{\vx}_i}$ is $\vSigma_\Phi$. This entails that there exists some independent isotropic random vector $\vz_i$ such that $\vphi_{\tilde{\vx}_i}=\vSigma_\Phi^{1/2}\vz_i$. We first show that $\vphi_{\tilde{\vx}_i}$ is a sub-Gaussian random vector in $\R^N$. Consider any unit vector $\vr\in\R^N$. Let $f(\vx):=\langle \vphi_{\vx},\vr\rangle$, we can easily validate that $f(\vx)$ is a $\frac{\lambda_\sigma}{\sqrt{N}}\|\vW_1\|_F$-Lipschitz function. Therefore, by Gaussian Lipschitz concentration we know that the sub-Gaussian norm of $f(\tilde{\vx}_i)$ is at least $K=\frac{\lambda_\sigma}{\sqrt{N}}\|\vW_1\|_F$, which also implies the sub-Gaussian norms of $\vphi_{\tilde{\vx}_i}$ and $\vz_i$. Denote $\vR_n:=\frac{1}{n}\sum_{i=1}^n\vz_i\vz_i^\top-\vI_N$. From \cite[Theorem 4.6.1]{vershynin2018high} we know that with probability at least $1-2e^{-ct}$,
 \[\|\vR_n\|\le K^2 \left(\sqrt{\frac{N}{n}}+\frac{t}{\sqrt{n}}\right),\]
 for all large $n,N$. This proposition is proved by setting $t=\sqrt{N}$ and noting that
 \[\left\|\left(\vSigma_\Phi + \lambda \vI\right)^{-1/2} (\vSigma_{\Phi}-\widehat{\vSigma}_{\Phi})\left(\vSigma_\Phi + \lambda \vI\right)^{-1/2}\right\|\le \|\vR_n\|.\]
 
\end{proof}

From Lemma~\ref{lemm:gradient-norm} we know that when $n,d,N$ are proportional and $\eta=\Theta(\sqrt{N})$, there exist some constants $c,C$ such that $\Parg{\norm{\vW_1}_F\ge C\sqrt{N}}\le \Exp{-cN}$.
We denote $t = 2C^2N/n$ and consider sufficiently large $n$ (but still proportional to $d$) such that $t<1$. Now given fixed $\lambda>0$, we define the concentration event 
\begin{align}
    \cA_\lambda = \left\{-t\vI\preccurlyeq\left(\vSigma_\Phi + \lambda \vI\right)^{-1/2} (\vSigma_{\Phi}-\widehat{\vSigma}_{\Phi})\left(\vSigma_\Phi + \lambda \vI\right)^{-1/2}\preccurlyeq t\vI\right\}. 
    \label{eq:event-A}
\end{align} 
Similarly, for the ``ridgeless'' case $\lambda=0$, we define
\begin{align}
    \cA_0 = \left\{-t\vI\preccurlyeq\vSigma_\Phi^{-1/2} (\vSigma_{\Phi}-\widehat{\vSigma}_{\Phi})\vSigma_\Phi^{-1/2}\preccurlyeq t\vI\right\}. 
    \label{eq:event-A0}
\end{align} 
Lemma~\ref{lemm:concentration-largelr} entails that both $\cA_\lambda$ and $\cA_0$ hold with probability at least $1-2e^{-c\sqrt{N}}$. Following the remark on \cite[Lemma 7.1]{bach2021learning}, under events $\cA_\lambda$ and $\cA_0$, we can obtain that
\begin{align}
    \norm{\SigmaPhi^{-1/2}\left(\widehat{\vSigma}_\Phi - \vSigma_\Phi\right)\SigmaPhi^{-1/2}}\le t,\label{eq:t_bound}
\end{align}
and $(1-t)(\vSigma_{\Phi}-\widehat{\vSigma}_{\Phi})\preccurlyeq t( \widehat{\vSigma}_{\Phi}+\lambda\vI)$, which implies that
\begin{align}
    \vSigma_{\Phi}\left(\widehat{\vSigma}_{\Phi}+ \lambda \vI\right)^{-1} \preccurlyeq \frac{t}{1-t}\vI+\widehat{\vSigma}_{\Phi}\left(\widehat{\vSigma}_{\Phi}+ \lambda \vI\right)^{-1}\preccurlyeq \frac{1}{1-t}\vI,
\end{align}
since $\norm{\widehat{\vSigma}_{\Phi}\left(\widehat{\vSigma}_{\Phi}+ \lambda \vI\right)^{-1}}\le 1$. Analogously, we claim that $\left(\widehat{\vSigma}_{\Phi}+ \lambda \vI\right)^{-1/2}\vSigma_{\Phi}\left(\widehat{\vSigma}_{\Phi}+ \lambda \vI\right)^{-1/2} \preccurlyeq \frac{1}{1-t}\vI$. Thus, under events $\cA_\lambda$ and $\cA_0$, we know that
\begin{align}
    \norm{\vSigma_{\Phi}^{1/2}\left(\widehat{\vSigma}_{\Phi}+ \lambda \vI\right)^{-1}\vSigma_{\Phi}^{1/2}}, ~\norm{\vSigma_{\Phi}\left(\widehat{\vSigma}_{\Phi}+ \lambda \vI\right)^{-1}}\le \frac{1}{1-t}.\label{eq:1/1-t}
\end{align}
We now control $B_1,B_2,V_1,V_2$ under the high probability events $\cA_\lambda$ and $\cA_0$. 
 
\paragraph{Controlling $B_1, B_2$.} By the definition of $\check{f}$, we have
\begin{align}
    B_1 = \inf_{f\in\cH}\E_{\vx}\left(f^*(\vx) - f(\vx)\right)^2
\le
    \E_{\vx}\left(f^*(\vx) - f_r(\vx)\right)^2=\cR(\tilde{f}), \label{eq:B_1_bound}
\end{align}
where $f_r=\tilde{f}\in\cH$ is the estimator we constructed in Lemma~\ref{lemm:oracle-estimator}. Note that the upper bound $\cR(\tilde{f})$ has already been characterized in \eqref{eq:f1-risk-asym} in the previous subsection.

As for $B_2$, since $\check{\vf} = \frac{1}{\sqrt{N}}\sigma(\tilde\vX\vW_1)\check{\va}$, simple calculation yields,
\begin{align}
    B_2 =&~ \Trarg{\left(\vI - \left(\widehat{\vSigma}_\Phi + \lambda\vI\right)^{-1} \widehat{\vSigma}_\Phi\right)^\top\vSigma_\Phi \left(\vI - \left(\widehat{\vSigma}_\Phi + \lambda\vI\right)^{-1} \widehat{\vSigma}_\Phi\right)\check{\va}\check{\va}^\top} \\
=&~
    \lambda^2 \left\langle\check{\va},  \left(\widehat{\vSigma}_\Phi + \lambda\vI\right)^{-1}\vSigma_\Phi\left(\widehat{\vSigma}_\Phi + \lambda\vI\right)^{-1}\check{\va}\right\rangle. 
\end{align}
Following \cite[Proposition 7.2]{bach2021learning}, we define $\va_\lambda = \vSigma_\Phi\left(\vSigma_\Phi + \lambda\vI\right)^{-1}\check{\va}$ and obtain
\begin{align}
    B_2 \le 2\lambda^2\norm{\vSigma_\Phi^{1/2}\left(\vSigma_\Phi + \lambda\vI\right)^{-1}\check{\va}}^2
    + 2\norm{\vSigma_\Phi^{1/2}\left(\left(\vSigma_\Phi + \lambda\vI\right)^{-1}\vSigma_\Phi - \left(\widehat{\vSigma}_\Phi + \lambda\vI\right)^{-1}\widehat{\vSigma}_\Phi \right)\check{\va}}^2. 
    \label{eq:B_2_bound0} 
\end{align}
In addition, note that
\begin{align}
    &\left(\widehat{\vSigma}_\Phi + \lambda\vI\right)^{-1}\widehat{\vSigma}_\Phi - \left(\vSigma_\Phi + \lambda\vI\right)^{-1}\vSigma_\Phi \\
=&
    \left(\widehat{\vSigma}_\Phi + \lambda\vI\right)^{-1}\left(\widehat{\vSigma}_\Phi - \vSigma_\Phi\right) + \left[\left(\widehat{\vSigma}_\Phi + \lambda\vI\right)^{-1} - \left(\vSigma_\Phi + \lambda\vI\right)^{-1}\right]\vSigma_\Phi \\
=&
    \lambda\left(\widehat{\vSigma}_\Phi + \lambda\vI\right)^{-1}\left(\widehat{\vSigma}_\Phi - \vSigma_\Phi\right)\left(\vSigma_\Phi + \lambda\vI\right)^{-1}. 
\end{align}
Therefore, we know that under events $\cA_\lambda$ and $\cA_0$, 
\begin{align}
    &~\norm{\vSigma_\Phi^{1/2}\left(\left(\vSigma_\Phi + \lambda\vI\right)^{-1}\vSigma_\Phi - \left(\widehat{\vSigma}_\Phi + \lambda\vI\right)^{-1}\widehat{\vSigma}_\Phi \right)\check{\va}}^2 \\
=&~
    \lambda^2\norm{\vSigma_\Phi^{1/2}\left(\widehat{\vSigma}_\Phi + \lambda\vI\right)^{-1}\left(\widehat{\vSigma}_\Phi - \vSigma_\Phi\right)\left(\vSigma_\Phi + \lambda\vI\right)^{-1}\check{\va}}^2 \\
\le&~
    \norm{\vSigma_\Phi^{1/2}\left(\widehat{\vSigma}_\Phi + \lambda\vI\right)^{-1}\SigmaPhi^{1/2}}^2\cdot\norm{\SigmaPhi^{-1/2}\left(\widehat{\vSigma}_\Phi - \vSigma_\Phi\right)\SigmaPhi^{-1/2}}^2\cdot\lambda^2\norm{\vSigma_\Phi^{1/2}\left(\vSigma_\Phi + \lambda\vI\right)^{-1}\check{\va}}^2 \\
\overset{(i)}{\le}&~ 
    \frac{t^2}{(1-t)^2}\cdot\lambda^2\norm{\vSigma_\Phi^{1/2}\left(\vSigma_\Phi + \lambda\vI\right)^{-1}\check{\va}}^2, \label{eq:B_2_bound1}
\end{align}
where $(i)$ follows from the definition of the concentration events $\cA$, \eqref{eq:t_bound} and \eqref{eq:1/1-t}. 

Finally, from \cite[Lemma 7.2]{bach2021learning}, we have 
\begin{align}
    &~\lambda^2\norm{\vSigma_\Phi^{1/2}\left(\vSigma_\Phi + \lambda\vI\right)^{-1}\check{\va}}^2 
\le
    \lambda\left\langle\check{\va},(\vSigma_\Phi+\lambda\vI)^{-1}\vSigma_\Phi\check{\va}\right\rangle \\
=&~
    \inf_{f\in\cH}\left\{\|f - \check{f}\|^2_{L^2} + \lambda\norm{f}^2_{\cH}\right\} 
\le 
    2\|f^* - f_r\|^2_{L^2} + \lambda\|f_r\|^2_{\cH}, \label{eq:B_2_bound2}
\end{align}
where the last step is a triangle inequality due to $\|f^*- \check{f}\|^2_{L^2}\le \|f^*- f_r\|^2_{L^2}$.

\paragraph{Controlling $V_1, V_2$.} 

For $V_1$, note that under event $\cA_\lambda$, 
\begin{align}
    V_1 =&~ \frac{1}{n^2}\tilde\vvarepsilon^\top\vPhi\left(\widehat{\vSigma}_\Phi + \lambda \vI\right)^{-1} \vSigma_\Phi\left(\widehat{\vSigma}_\Phi + \lambda \vI\right)^{-1} \vPhi^\top\tilde\vvarepsilon \\
\le&~ 
    \norm{\frac{1}{n}\vvarepsilon^\top\vPhi}^2 \cdot\norm{\left(\widehat{\vSigma}_\Phi + \lambda \vI\right)^{-1} \vSigma_{\Phi}}\cdot\norm{\left(\widehat{\vSigma}_\Phi + \lambda \vI\right)^{-1}} 
\overset{(ii)}{\lesssim}
    \frac{1}{\lambda(1-t)}\cdot\norm{\frac{1}{n}\vPhi^\top\tilde\vvarepsilon}^2,   
\end{align} 
where $\vPhi$ is defined in \eqref{eq:hat_f_large}, and $(ii)$ is based on the concentration property for $\cA_\lambda$ given in \eqref{eq:1/1-t}. Denote $\chi^i_k := [\vphi_k]_i\cdot\tilde\varepsilon_k$ whence $\left[\vPhi^\top\tilde\vvarepsilon\right]_i = \sum_{k=1}^n \chi_k^i$. 
Note that $\E[\chi^i_k\chi^j_k] = 0$, $ \E[(\chi_k^i)^2] \lesssim \frac{\sigma_\eps^2}{N}$ for any $k\in [n]$ and $i\neq j \in [N]$, due to the assumptions on label noise and bounded activation $\sigma$. 
Therefore, by Markov's inequality, for any $x>0$, we have
\begin{align}
    \Parg{\frac{1}{n^2}\|\vPhi^\top\tilde\vvarepsilon\|^2\ge x} 
\le
    \frac{\E\|\vPhi^\top\tilde\vvarepsilon\|^2}{n^2x}
\lesssim
    \frac{\sigma_\eps^2}{nx }.  
\label{eq:variance-Markov}
\end{align}
Similarly for $V_2$, under event $\cA_\lambda$, we have 
\begin{align}
    V_2 = &~\frac{1}{n^2}\vf_\perp^\top\vPhi\left(\widehat{\vSigma}_\Phi + \lambda \vI\right)^{-1} \vSigma_{\Phi}\left(\widehat{\vSigma}_\Phi + \lambda \vI\right)^{-1}\vPhi^\top \vf_\perp  \\
\le&~
    \norm{\frac{1}{n}\vf_\perp^\top\vPhi}^2 \cdot\norm{\left(\widehat{\vSigma}_\Phi + \lambda \vI\right)^{-1} \vSigma_{\Phi}}\cdot\norm{\left(\widehat{\vSigma}_\Phi + \lambda \vI\right)^{-1}} 
\lesssim
    \frac{1}{\lambda(1-t)}\cdot\norm{\frac{1}{n}\vPhi^\top\vf_\perp}^2.  
\end{align} 
Recall that $\E[\vphi_{\vx} f_\perp(\vx)] = \mathbf{0}$ due to the orthogonality condition. Hence we may apply the exact same argument as $V_1$ to obtain an upper bound similar to \eqref{eq:variance-Markov}; by Markov's inequality,
\begin{align}
    \Parg{\frac{1}{n^2}\|\vPhi^\top\vf_\perp\|^2\ge x} 
\le
    \frac{\E\|\vPhi^\top\vf_\perp\|^2}{n^2x}
\overset{(iii)}{\lesssim}
    \frac{\|f_{\perp}\|^2_{L^2}}{nx}, \label{eq:variance-Markov2}
\end{align}
where $(iii)$ is due to the boundedness of $\sigma$ and $\|f_{\perp}\|_{L^2}$.
Combining $V_1$ and $V_2$, and taking $x=C n^{\eps-1}$ in \eqref{eq:variance-Markov} and \eqref{eq:variance-Markov2}, for some $C>0$ and any small $\eps>0$, we arrive at
\begin{align}
    V_1 + V_2 \lesssim \frac{\sigma_\eps^2 + \norm{f_\perp}_{L^2}^2}{n^{1-\eps}\lambda(1-t)}, \label{eq:V_1+V_2}
\end{align}
with probability at least $1-n^{-\eps}$.

\paragraph{Putting Things Together.}
The following lemma provides a decomposition of the prediction risk $\cR_1(\lambda)$ in terms of $B_1,B_2,V_1,V_2$ analyzed above. 
\begin{lemm}
Under the same assumptions as Lemma~\ref{lemm:optimal-estimator}, if we choose $\lambda=\Omega(n^{\eps-1})$ for small $\eps>0$, then the prediction risk of the CK ridge estimator  admits the following upper bound
$$
\cR_1(\lambda) \le B_1 + B_2 + 2\sqrt{B_1 B_2} + o_{d,\P}(1), 
$$
where $B_1, B_2$ are defined in \eqref{eq:R1-decomposition}. 
\label{lemm:R1-decomposition}
\end{lemm}
\begin{proof}
Based on the definition of prediction risk, we have
\begin{align}
    \cR_1(\lambda) =&~ \E_{\vx}\left(\left(f^*(\vx) - \check{f}(\vx)\right) +  \left(\check{f}(\vx) - \frac{1}{n}\vphi_x\left(\widehat{\vSigma}_\Phi + \lambda\vI\right)^{-1} \vPhi^\top \tilde{\vy}\right)\right)^2 \\
\le&~
    \underbrace{\E_{\vx}\left(f^*(\vx) - \check{f}(\vx)\right)^2}_{B_1} + 2\sqrt{B_1 S_1} + S_1,  
\end{align}
where we defined
\begin{align}
    S_1
:=&~
    \E_{\vx}\left(\check{f}(\vx) - \frac{1}{n}\vphi_x\left(\widehat{\vSigma}_\Phi + \lambda\vI\right)^{-1} \vPhi^\top (\check{\vf} + \vf_\perp + \tilde{\vvarepsilon})\right)^2 \\
\le&~
    \underbrace{\E_{\vx}\left(\check{f}(\vx) - \frac{1}{n}\vphi_x\left(\widehat{\vSigma}_\Phi + \lambda\vI\right)^{-1} \vPhi^\top \check{\vf}\right)^2}_{B_2} + 2\sqrt{B_2 S_2} + S_2,
\end{align}
in which 
\begin{align}
    S_2 :=& \underbrace{\frac{1}{n^2}\tilde\vvarepsilon^\top\vPhi\left(\widehat{\vSigma}_\Phi + \lambda \vI\right)^{-1} \vSigma_\Phi\left(\widehat{\vSigma}_\Phi + \lambda \vI\right)^{-1} \vPhi^\top\tilde\vvarepsilon}_{V_1} + \underbrace{\frac{1}{n^2} \vf_\perp^\top\vPhi\left(\widehat{\vSigma}_\Phi + \lambda \vI\right)^{-1} \vSigma_{\Phi}\left(\widehat{\vSigma}_\Phi + \lambda \vI\right)^{-1}\vPhi^\top \vf_\perp}_{V_2} \\
    &+
    \frac{2}{n^2} \vf_\perp^\top\vPhi\left(\widehat{\vSigma}_\Phi + \lambda \vI\right)^{-1} \vSigma_{\Phi}\left(\widehat{\vSigma}_\Phi + \lambda \vI\right)^{-1}\vPhi^\top \tilde{\vvarepsilon} \\
\le&~
    2 (V_1 + V_2). 
\end{align}

Recall that Lemma \ref{lemm:concentration-largelr} entails that events $\cA_\lambda$ and $\cA_0$ occur with high probability, for constant $t\in(0,1)$. Hence from \eqref{eq:V_1+V_2} we know that for $\lambda=\Omega(n^{\eps-1})$ with small $\eps>0$, $V_1 + V_2 = o_{d,\P}(1)$, and thus $S_2$ is vanishing when $n,d,N\to\infty$ proportionally. On the other hand, \eqref{eq:B_1_bound} and \eqref{eq:B_2_bound2} entail that $B_1$ and $B_2$ are both finite. 
The claim is established by combining the calculations. 

\end{proof}

\begin{proofof}[Theorem~\ref{thm:risk-large-lr}]
Since Lemma~\ref{lemm:concentration-largelr} ensures that events $\cA_\lambda$ and $\cA_0$ happens with high probability for fixed $t\in(0,1)$, if we set $\lambda=\Omega(n^{\eps-1})$ for some small $\eps>0$, then Lemma \ref{lemm:R1-decomposition} entails
\begin{align}
    &\cR_1(\lambda) \le B_1 + B_2 + 2\sqrt{B_1 B_2} + o_{d,\P}(1), \\
    &\text{where~} 
    B_1 \le \|f^* - f_r\|^2_{L^2}, \quad
    B_2 \le 2\left(1 + \frac{t^2}{(1-t)^2}\right)\cdot\left(2\|f^* - f_r\|^2_{L^2} + \lambda\norm{f_r}_{\cH}^2\right)+ o_{d,\P}(1), 
\end{align}
in which $f_r$ is defined by \eqref{eq:f_r} in the proof of Lemma~\ref{lemm:oracle-estimator}. Here, we applied the upper bounds on $B_1$ in \eqref{eq:B_1_bound} and $B_2$ in \eqref{eq:B_2_bound2}. Since $\|f^* - f_r\|^2_{L^2}=\cR(\tilde f)$, by \eqref{eq:f1-risk-asym} we know that as $n,d,N\to\infty$, with probability one, 
\begin{align}
    \|f^* - f_r\|^2_{L^2}\le\tau^*+ C\left(\sqrt{\tau^*}\cdot\sqrt{\tfrac{d}{n}} + \tfrac{d}{n}\right), 
    \label{eq:fr-risk}
\end{align}
for some constant $C>0$. 
Finally, recall that in the proof of Lemma~\ref{lemm:oracle-estimator}, we constructed an estimator $f_r\in\cH$ with  $\norm{f_r}_{\cH}^2=\|\tilde{\va}\|^2=N/|\cA_r^\alpha| = \Theta_{d,\P}(N^r)$, for $0<r<1/2$. In other words, $\lambda\norm{f_r}_{\cH}^2 = o_{d,\P}(1)$ as long as $N^{r}\lambda\to 0$ as $n,d,N\to\infty$; this provides a way to choose $r\in(0,1/2)$ given $\lambda$. 
Now from Lemma~\ref{lemm:concentration-largelr} we know that there exists some constant $\psi_1^*$ such that both $\cA_\lambda$ and $\cA_0$ hold with high probability for $t<0.1$ when $n/d>\psi_1^*$. 
In this case, given any $\lambda=n^{-\rho}$ for some $\rho\in (0,1)$, we know that  
\begin{align}
      \cR_1(\lambda) \le&~ 
      B_1 + B_2 + 2\sqrt{B_1 B_2} + o_{d,\P}(1) \\
      \le&~ \|f^* - f_r\|^2_{L^2} + 4\left(1 + \frac{t^2}{(1-t)^2}\right)\cdot\|f^* - f_r\|^2_{L^2} + 4\sqrt{1 + \frac{t^2}{(1-t)^2}}\cdot\|f^* - f_r\|^2_{L^2} + o_{d,\P}(1).
\end{align}
Finally, due to the upper-bound \eqref{eq:fr-risk}, we conclude that
\begin{align}
      \cR_1(\lambda) \le 10\tau^* + C'\Big(\sqrt{\tau^*}\cdot\sqrt{\tfrac{d}{n}} + \tfrac{d}{n}\Big),  
\end{align}
with probability one as $n,d,N\to\infty$ proportionally and $n/d>\psi_1^*$, where $\tau^*$ is defined in \eqref{eq:tau*}.

\end{proofof}

\end{document}